\documentclass{article}
\usepackage{amsmath,amsfonts,bm}
\usepackage{bbm}
\usepackage[mathscr]{eucal}

\newcommand{\dt}{\textnormal{d}t}

\DeclareMathOperator{\TV}{TV}
\DeclareMathOperator{\BV}{BV}

\DeclareMathOperator{\ReLU}{\mathrm{ReLU}}
\DeclareMathOperator{\LeakyReLU}{\mathrm{LeakyReLU}}

\DeclareMathOperator{\GELU}{\mathrm{GELU}}

\DeclareMathOperator{\SELU}{\mathrm{SELU}}

\DeclareMathOperator{\sgn}{sgn}

\DeclareMathOperator{\Lip}{\mathrm{Lip}}

\DeclareMathOperator{\aconv}{\mathrm{aconv}}

\DeclareMathOperator{\poly}{\mathrm{poly}}
\DeclareMathOperator{\fat}{\mathrm{fat}}

\def\1{\bm{1}}

\def\<{\langle}
\def\>{\rangle}

\def\vzero{{\bm{0}}}

\def\vtheta{{\bm{\theta}}}

\def\vomega{{\bm{\omega}}}
\def\va{{\bm{a}}}
\def\vb{{\bm{b}}}
\def\vc{{\bm{c}}}

\def\ve{{\bm{e}}}

\def\vs{{\bm{s}}}
\def\vt{{\bm{t}}}
\def\vu{{\bm{u}}}
\def\vv{{\bm{v}}}
\def\vw{{\bm{w}}}
\def\vx{{\bm{x}}}

\def\mA{{\bm{A}}}
\def\mB{{\bm{B}}}

\def\mM{{\bm{M}}}

\def\mU{{\bm{U}}}
\def\mV{{\bm{V}}}
\def\mW{{\bm{W}}}

\DeclareMathAlphabet{\mathsfit}{\encodingdefault}{\sfdefault}{m}{sl}
\SetMathAlphabet{\mathsfit}{bold}{\encodingdefault}{\sfdefault}{bx}{n}

\def\cA{{\mathcal{A}}}
\def\cB{{\mathcal{B}}}
\def\cC{{\mathcal{C}}}
\def\cD{{\mathcal{D}}}
\def\cE{{\mathcal{E}}}
\def\cF{{\mathcal{F}}}
\def\cG{{\mathcal{G}}}
\def\cH{{\mathcal{H}}}

\def\cL{{\mathcal{L}}}
\def\cM{{\mathcal{M}}}
\def\cN{{\mathcal{N}}}
\def\cO{{\mathcal{O}}}
\def\cP{{\mathcal{P}}}

\def\cR{{\mathcal{R}}}
\def\cS{{\mathcal{S}}}

\def\cU{{\mathcal{U}}}
\def\cV{{\mathcal{V}}}
\def\cW{{\mathcal{W}}}

\def\bB{{\mathbb{B}}}

\def\bE{{\mathbb{E}}}

\def\bN{{\mathbb{N}}}

\def\bR{{\mathbb{R}}}
\def\bS{{\mathbb{S}}}

\newcommand{\R}{\mathbb{R}}

\usepackage{amsmath,amssymb,amsthm,amsfonts,bm,amsthm,color}
\usepackage{mathtools}
\usepackage{bbm}
\usepackage{fullpage}
\usepackage{enumitem}% http://ctan.org/pkg/enumitem
\usepackage[normalem]{ulem}
\usepackage{pifont}
\usepackage{wrapfig}
\usepackage{pgfplots}
\pgfplotsset{compat=1.18}

\usepackage{caption}
\usepackage{subcaption}
\usepackage{appendix}

\usepackage{natbib}
\setcitestyle{authoryear,round}

\usepackage{tcolorbox}
\usepackage[utf8]{inputenc} % allow utf-8 input
\usepackage[T1]{fontenc}    % use 8-bit T1 fonts
\usepackage{hyperref}       % hyperlinks
\usepackage{url}            % simple URL typesetting
\usepackage{booktabs}       % professional-quality tables
\usepackage{amsfonts}       % blackboard math symbols
\usepackage{nicefrac}       % compact symbols for 1/2, etc.
\usepackage{microtype}      % microtypography
\usepackage{xcolor}         % colors
\usepackage[nameinlink, capitalize]{cleveref}
\AtBeginEnvironment{appendices}{\crefalias{section}{appendix}}
\AtBeginEnvironment{appendices}{\crefalias{subsection}{appendix}}
\crefname{assumption}{Assumption}{Assumptions}
\Crefname{assumption}{Assumption}{Assumptions}
\newtheorem{theorem}{Theorem}[section]
\newtheorem{corollary}[theorem]{Corollary}
\newtheorem{proposition}[theorem]{Proposition}
\newtheorem{lemma}[theorem]{Lemma}

\newtheorem{remark}[theorem]{Remark}

\title{\bf Deep Neural Variation Spaces: A Unifying Perspective on Depth and Complexity}
\author{%
  Julia Nakhleh \\
  Department of Computer Science\\
  University of Wisconsin-Madison\\
  \texttt{jnakhleh@wisc.edu} \\
  \and
  Robert D. Nowak \\
  Department of Electrical \& Computer Engineering \\
  University of Wisconsin-Madison\\
  \texttt{rdnowak@wisc.edu}
}
\date{}
\begin{document}
\maketitle
\begin{abstract}
    We develop a unified function space theory of deep fully connected neural networks. Functions in our spaces are defined recursively as $\ell^1$-bounded linear combinations of activated functions from preceding layers, with a dictionary of affine functions at the first layer. Unlike existing theories that are largely specialized to homogeneous activations such as the ReLU, our framework provides a meaningful notion of functional complexity for deep networks with a broad range of homogeneous and non-homogeneous activation functions commonly used in practice. This simple construction unites several seemingly disparate ideas from the literature, including norm-based complexity bounds and variational characterizations of depth, and facilitates novel analyses of what kinds of functions deep norm-constrained networks can represent. To this end, we prove a novel representer theorem for our spaces and establish novel function-space complexity bounds showing that the associated function classes remain qualitatively small at arbitrary depth. In the univariate ReLU case, we prove a “depth saturation” result: depth in this setting yields only a small constant rescaling of the function class, with no added functional diversity. As a consequence, we show that deep norm-controlled ReLU functions in any dimension cannot exhibit high frequencies along any direction. This finding reveals that some commonly cited expressivity benefits of depth disappear once network complexity is controlled by an appropriate function space norm, rather than parameter count or other representational costs that permit compounded rescaling across layers. Overall, our results illustrate how a function space perspective yields new structural insights into the relationship between depth and complexity.
\end{abstract}

\section{Introduction}

Deep neural networks are the backbone of nearly all modern AI systems. \textit{Depth} is the defining architectural feature of these models, and increasing depth has enabled striking empirical advances across a wide range of applications. Nevertheless, the precise role of depth remains poorly understood from a theoretical perspective. It is widely postulated that the main advantage of depth is increased representational capacity: repeated composition with nonlinear feature transformations may allow deep networks to represent complex functions much more efficiently than shallow ones. Any such statement about representational efficiency, however, is necessarily relative to a chosen notion of complexity. Most of the existing literature on depth and expressivity in neural networks measures the cost of representing a function through the width, number of neurons, or total number of parameters in a network. These measures, however, provide only a partial description of the effective complexity of a learned predictor. Modern neural networks are frequently highly overparameterized, with many more parameters than training data points: in such a regime, the size of a particular network may say relatively little about which functions are actually favored by learning. Indeed, it is well-known that the complexities of neural network function classes can be controlled by the magnitudes of their weights, in some cases independently of the number of hidden units or parameters (\cite{bartlett1996valid,neyshabur2015norm,golowich2018size,barron2019complexity}).

These observations motivate the study of neural networks from a \textit{function space} perspective. This perspective associates a neural network primarily with the \textit{function} it represents, rather than any particular parametric representation of that function. In these neural networks of possibly infinite-width, but wherein the overall magnitude of the parameters is controlled, are mathematically modeled as elements of an abstract function space. By characterizing the complexities of function classes (i.e., norm-bounded sets) in these function spaces, one can address the question: give some bound on the parameter magnitude of my network (i.e., a bound on the norm of that network), but no bounds on network width, what kinds of functions are representable? Additionally, these theories shed light on the inductive biases associated with various types of explicit neural network regularization: by studying the geometric properties of the function space norm induced by a parametric regularizer, we gain insight into what kinds of functional structure are imposed by this type of regularization. A rich function space theory for shallow networks has been developed based on these ideas, incorporating diverse concepts from harmonic, convex, and functional analysis and approximation theory (\cite{bach2017breaking,ongie2019function,parhi2021banach,parhi2022near,ma2022barron,wojtowytsch2022representation,bartolucci2023understanding,siegel2023characterization,siegel2024sharp}). A number of works have extended different aspects of these shallow analyses to deep networks (\cite{wojtowytsch2020banach,parhi2022kinds,shenouda2024variation, bartolucci2024neural,heeringa2025deep,ongie2026representation}). 

Despite this progress, the existing function space theory of deep networks is highly fragmented, because the various deep function space norms studied in these papers often correspond to very different types of parametric representation costs. These differences may lead to important differences in the resulting function space geometry and complexity bounds. Consequently, it is often unclear how these notions of complexity relate to one another, or what they collectively imply about the functional effect of depth. Moreover, many fundamental questions about the representational advantages conferred by depth---such as structural differences between shallow and deep function spaces, and what types of functions may be in the deep spaces but not the shallow ones---remain entirely open. These are the main issues which we seek to address in the present paper. To this end, we introduce a novel family of deep neural variation spaces. Our spaces recover existing constructions in special cases, and extend many of the conceptual advantages of these constructions to broader families of commonly-used activation functions. Our framework allows us to unify a number of previously-distinct ideas for studying the functional complexity associated with depth. The central question we ask is: as depth increases, how do the norm-constrained function classes in our deep variation spaces change? By clarifying the relationships among existing notions of deep-network complexity and developing new tools for analyzing the corresponding function classes, we obtain several novel results concerning the structural advantages and limitations of depth. Our specific contributions are as follows.

\begin{enumerate}

    \item \textbf{Deep neural variation spaces for general activations.}
    We develop a unified function-space framework for deep fully connected networks, of potentially infinite width at each hidden layer, with general continuous activation functions. Our depth-$L$ norm is an atomic/gauge norm, which views the functions represented by depth-$L$ networks as absolutely convex combinations of depth-$L-1$ network functions composed with nonlinear activation functions. For homogeneous activations such as the ReLU, our function space norm coincides with the \textit{path norm} studied in \cite{neyshabur2015norm,barron2019complexity,wojtowytsch2020banach}. By introducing an activation-dependent family of normalized nonlinearities, we extend the conceptual advantages of this construction to a broad family of non-homogeneous activation functions. We show that the function spaces induced by many of the commonly-used ``ReLU-like'' activation functions (such as Leaky ReLU, GELU, and SiLU/Swish) are norm-equivalent to the ReLU function space, with explicit depth-dependent bounds on the equivalence constants.
    
    \item \textbf{Comparison with existing deep network representational costs.} We compare our norm with several existing parameter-space representation costs. For ReLU networks, suitably normalized classes controlled by these costs are contained in our unit balls, so our complexity bounds apply to them as well. Unlike these other representation costs, our norm is a true Banach norm on the space of functions represented by deep networks. As a result, our norm allows for a useful distinction between two functional effects of depth: those arising from compounded linear rescaling across hidden layers, and those arising from repeated composition with the nonlinear activation function itself. 
    
    \item \textbf{A novel representer theorem.} We prove that functions in our spaces admit equivalent representations as integrals with respect to finite measures. This shows that, for the ReLU activation, our spaces are equivalent to the generalized Barron/neural tree spaces of \cite{wojtowytsch2020banach}. Using these integral representations, we prove a novel representer theorem which shows solutions to norm-penalized data fitting problems over our spaces are realized by width-bounded neural networks.
    
    \item \textbf{Function-space complexity bounds.}
    We derive bounds on the worst-case empirical Rademacher complexities and on the $L^p(\mu)$ metric entropies of the unit balls $\cB_L$ of our depth-$L$ function spaces, for all $1 \leq p \leq \infty$ and arbitrary finite measures $\mu$. In many cases of interest, these upper bounds grow only mildly with depth. For the $\ReLU^m$ ($m \geq 1$) activation, the classes $\cB_L$ inherit known lower bounds on the metric entropy of the corresponding shallow class $\cB_2$. The gap in the entropy exponent between these upper and lower bounds is small when the input dimension $d$ is large.

    \item \textbf{Depth saturation and frequency control.}
    In the univariate ReLU case, we prove that depth has an extremely limited effect on functional complexity: it provides at most a fixed, depth-independent enlargement of the shallow unit ball. As a result, the deep univariate ReLU classes $\cB_L$ inherit the same complexity bounds as the shallow class $\cB_2$, with no depth-dependence whatsoever. This ``depth saturation'' result also implies that the deep spaces corresponding to many activation functions are norm-equivalent to the corresponding shallow spaces: however, for activations other than ReLU, the constants in this equivalence relation may grow with depth. Finally, this univariate depth saturation result carries an interesting implication for functions in the \textit{multivariate} deep classes $\cB_L$: they must exhibit strong frequency control along any individual direction of the input domain. This shows that many ``depth separation'' type-results, which concern the ability of deep networks to represent or approximate high-frequency functions more efficiently than their shallow counterparts, are fundamentally reliant on the effect of compounded layerwise rescaling rather than repeated application of the nonlinearity itself. These results provide insight into the structural effects of depth, and the nature of the different inductive biases corresponding to different measures of representational cost.
\end{enumerate}

\section{Deep neural variation spaces on compact domains}
\subsection{Normalized activation functions}
Let $\sigma: \R \to \R$ be a continuous activation function.\footnote{Our theory in this paper is applicable to a broad family of activation functions used in modern practice, including the ReLU, Leaky ReLU, GELU, SiLU/Swish, and many sigmoidal functions. The comprehensive list of activations that we consider is in \cref{tab:activation_summary} in \cref{appendix:activation_summary}.} We will construct our deep neural function spaces using the \textit{normalized activation functions}
\begin{align} \label{eq:normalized_activations}
    \sigma_s(t) := \frac{\sigma(st)}{s}, \qquad s >0, \ t \in \R.
\end{align}
The utility of these normalized activations $\sigma_s$ is that they allow for explicit separation of the neural feature \textit{shapes} from their \textit{sizes}. For \textit{homogeneous}\footnote{A function $\sigma: \R \to \R$ is \textit{(positively) homogeneous} if $\sigma(st) = s\sigma(t)$ for all $t \in \R$ and $s > 0$. More generally, $\sigma$ is homogeneous of degree $m$ if $\sigma(st) = s^m \sigma(t)$ for all $t \in \R$ and $s > 0$. \label{footnote:homogeneous}} activations such as ReLU and Leaky ReLU, this distinction is immaterial: in these cases, any rescaling of the input to $\sigma$ produces an identical scaling of the output of $\sigma$, so the normalized activations $\sigma_s$ are equivalent to the original activation $\sigma$. However, many of the commonly-used activations that we consider in this paper (such as GELU, SiLU/Swish, and various sigmoidal functions) are non-homogeneous. In these cases, rescaling the input to $\sigma$ does not merely rescale the function $\sigma$; it produces a new function with a distinct shape. In particular, almost all of the activations $\sigma$ that we consider are approximately 1-Lipschitz, so the scaled feature shapes $t \mapsto \sigma(st)$ are approximately $s$-Lipschitz (see \cref{tab:activation_summary}). To address this situation, we construct our deep neural spaces so that the normalized activations $\sigma_s$ (which have Lipschitz constant of approximately one) have unit norm in the corresponding neural function space. As a consequence, our spaces include all distinct feature shapes $t \mapsto \sigma(st)$ produced at all input scales $s$; each shape incurs a norm cost of $s$, approximately equal to its Lipschitz constant. (See \cref{fig:normalized-activations}.) Importantly, this does not impose any parametric or architectural limitation on the actual neural networks that our spaces describe. Indeed, our spaces include \textit{all} deep, fully-connected neural networks with activation $\sigma$, with the network's function-space norm controlled by a norm on all parameters in the network (\cref{sec:relationship_neural_networks}).

For the $\ReLU^m$ activation, we will instead define the normalized activation functions as
\begin{align} \label{eq:normalized_activations_relu_m}
    \sigma_s(t) := \frac{\sigma(st)}{s^m}, \qquad s >0, \ t \in \R.
\end{align}
This is the natural choice in this case because $\ReLU^m$ is homogeneous of degree $m$. (See footnote \footref{footnote:homogeneous}.) Therefore, with this choice, the normalized activations for $\ReLU^m$ are all equal to the $\ReLU^m$ itself.

Additionally, in several parts of our analysis, it will be of interest to consider the behavior of the normalized activations $\sigma_s$ as $s \downarrow 0$ and $s \uparrow \infty$. For all activations $\sigma$ in \cref{tab:activation_summary}, the normalized activations $\sigma_s$ converge uniformly on compact intervals as $s \uparrow \infty$ and $s \downarrow 0$. We will denote these uniform limiting functions as $\sigma_\infty$ and $\sigma_0$, respectively. The values of these limiting functions $\sigma_\infty$ and $\sigma_0$ are summarized in \cref{tab:activation_summary}. For many ``ReLU-like'' activations, $\sigma_\infty$ is equal to the ReLU itself and $\sigma_0$ is linear, while for many sigmoidal activations, $\sigma_\infty$ is zero and $\sigma_0$ is linear.

\begin{figure}
    \centering

    \begin{subfigure}{\textwidth}
        \centering
        \includegraphics[width=0.48\textwidth]{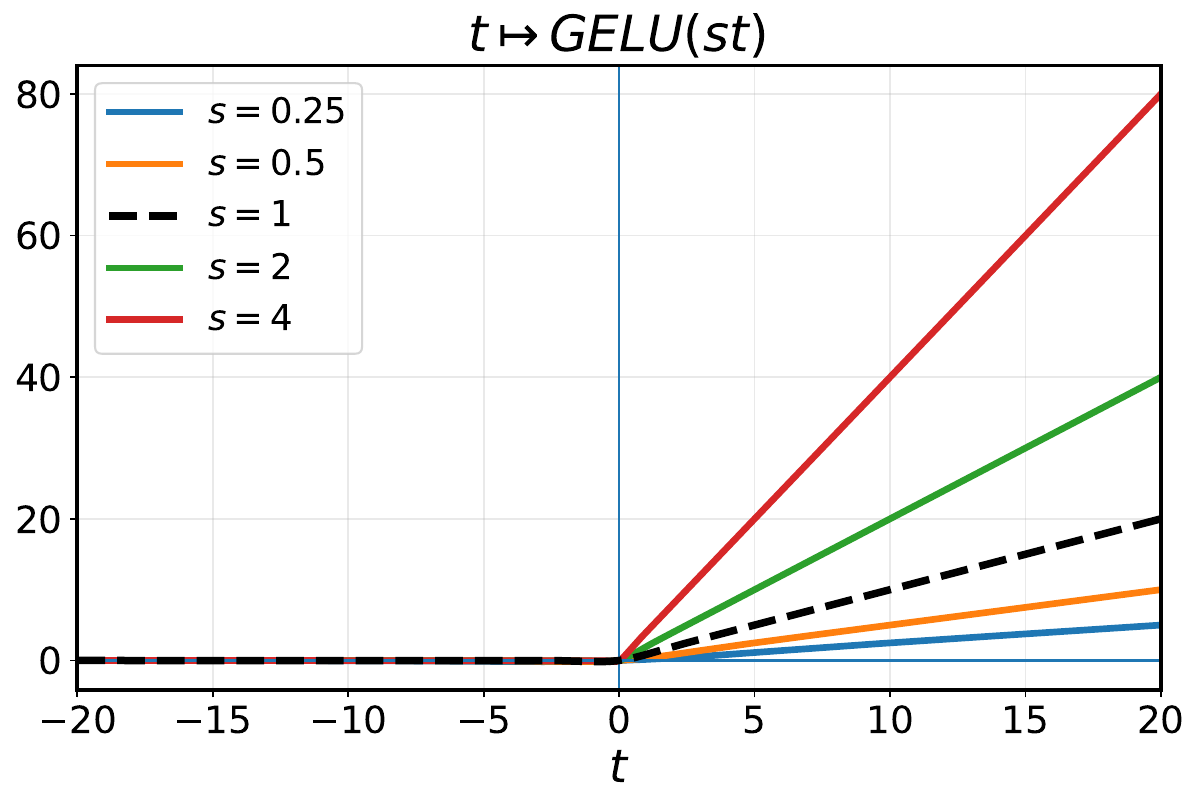}
        \hfill
        \includegraphics[width=0.48\textwidth]{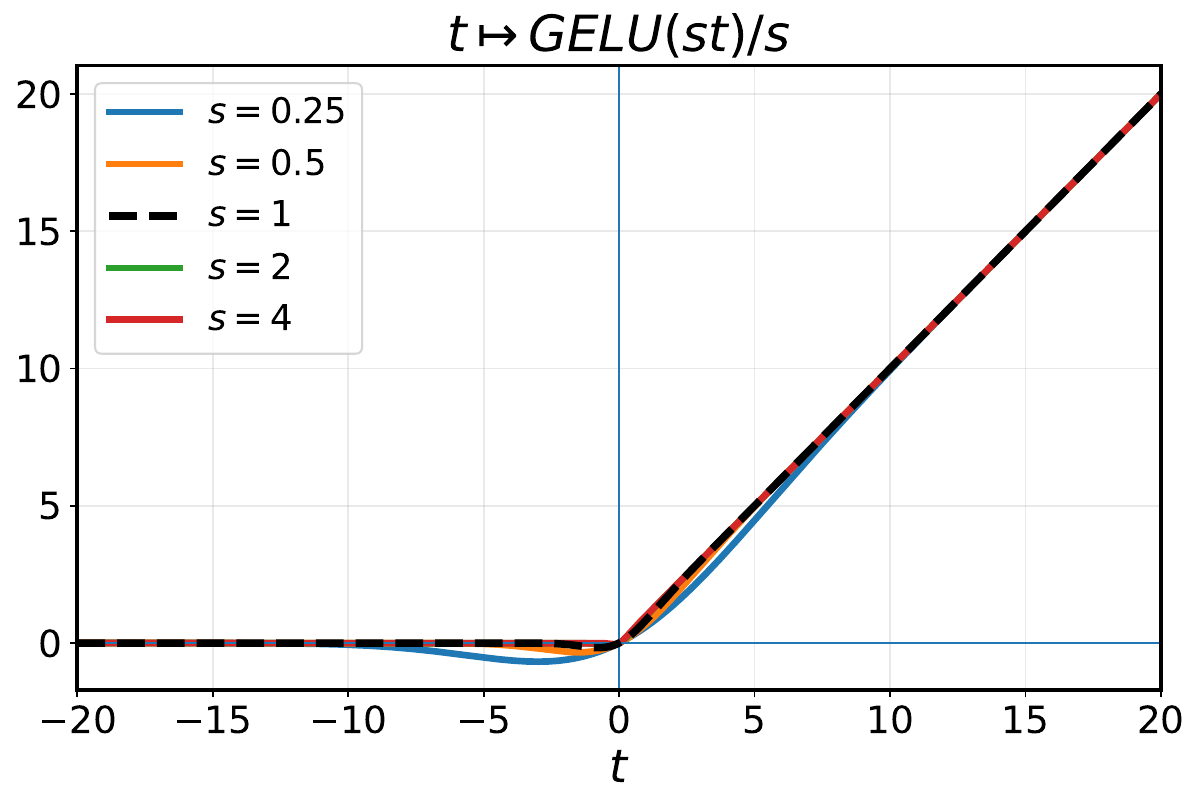}
        \caption{Left: the rescaled GELU feature shapes $t \mapsto \GELU(st)$ have slightly different shapes around the origin, and have limiting slopes of $\approx s$ as $t \uparrow \infty$. Right: the normalized features $t \mapsto \GELU(st)/s$ exhibit different shapes around the origin, but all have Lipschitz constant equal to that of the original GELU (which is approximately one).}
        \label{fig:gelu_normalized}
    \end{subfigure}

    \vspace{0.5cm}

    \begin{subfigure}{\textwidth}
        \centering
        \includegraphics[width=0.48\textwidth]{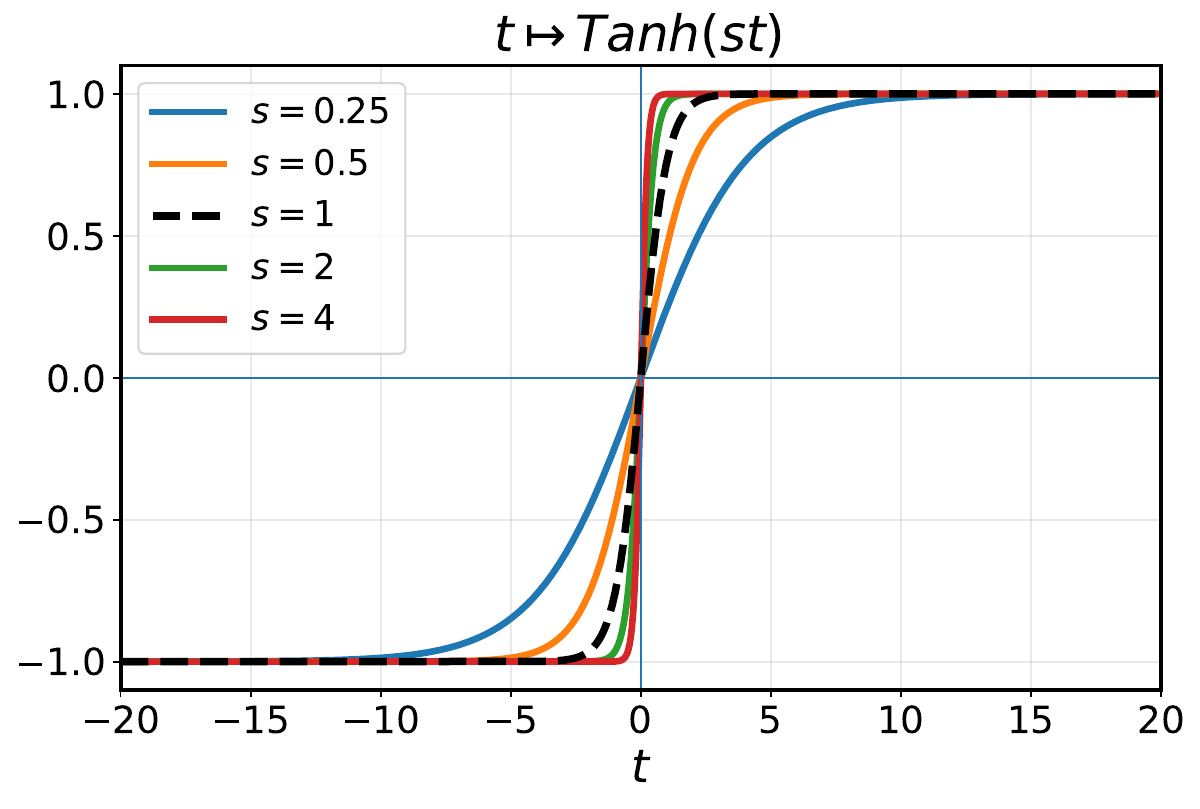}
        \hfill
        \includegraphics[width=0.48\textwidth]{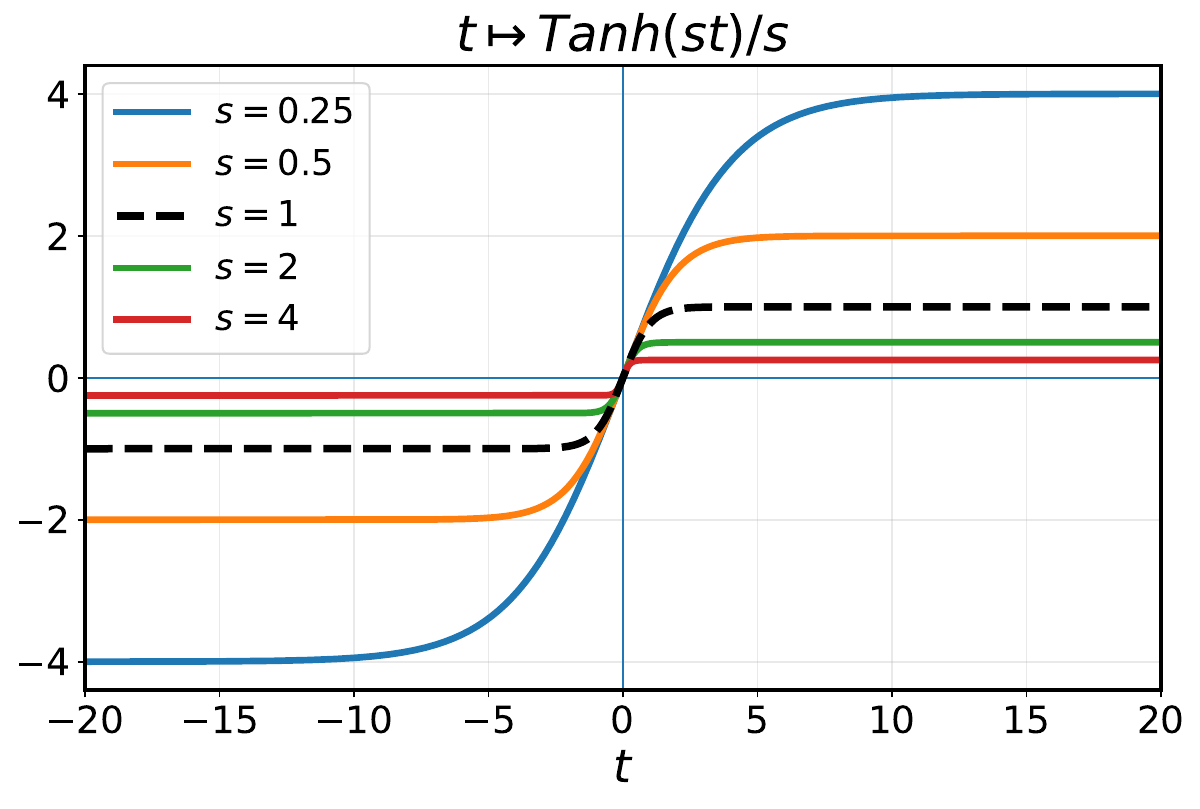}
        \caption{Left: the rescaled tanh feature shapes $t \mapsto \tanh(st)$ have slopes of $s$ at the origin, with the same limits as $t \uparrow \infty$. Right: the normalized features $t \mapsto \tanh(st)/s$ have different limiting behavior at $t = \pm \infty$, but all have Lipschitz constant equal to one.}
        \label{fig:tanh_normalized}
    \end{subfigure}

    \caption{Scaled feature shapes $t \mapsto \sigma(st)$ and normalized features $\sigma_s(t) := \sigma(st)/s$ for the GELU (\cref{fig:gelu_normalized}) and tanh (\cref{fig:tanh_normalized}) activation functions.}
    \label{fig:normalized-activations}
\end{figure}

\subsection{Definition of the deep variation spaces $\cV_L$}
We now proceed to construct our deep neural function spaces using the normalized activations $\sigma_s$ as defined in \eqref{eq:normalized_activations} (or \eqref{eq:normalized_activations_relu_m} for $\ReLU^m$). Let $\Omega \subset \mathbb{R}^d$ be a compact set and define the base dictionary of linear (affine) functions on $\Omega$ as
\begin{equation} \label{eq:B1}
    \cB_1 := \{ \vx \mapsto \vw^\top \vx + b: \vx \in \Omega, \vw \in \cW, b \in \cB \}
\end{equation}
for some compact \textit{weight} and \textit{bias sets} $\cW \subset \R^d$ and $\cB \subset \R$. Next, recursively define the closed \textit{absolutely convex hulls}
\begin{align} \label{eq:BL}
    \cB_L &:= \overline{\aconv}\left( \{ \sigma_s \circ f: s > 0, f \in \cB_{L-1} \} \right) \\ 
    &= \overline{\left\{ \sum_{k=1}^K v_k (\sigma_{s_k} \circ f_k): K \in \bN, f_k \in \cB_{L-1}, s_k > 0, v_k \in \R, \sum_{k=1}^K |v_k| \leq 1 \right\}} 
\end{align}
for $L \geq 2$. The closure in the definition of $\cB_L$ is taken with respect to any ambient Banach space $\cF$ containing the sets $\{ \sigma_s \circ f: s > 0, f \in \cB_{L-1} \}$.\footnote{If $\cF$ is a space of equivalence classes of functions rather than bona fide functions (such as an $L^p$ space), the set $\{ \sigma_s \circ f: s > 0, f \in \cB_{L-1} \}$ is  understood to refer to the set of equivalence classes of the functions $\sigma_s \circ f$, for all $s > 0$ and all functions in all equivalence classes in $\cB_{L-1}$.} If the set $\cB_L$ is \textit{bounded} in $\cF$ (meaning that $\sup_{f \in \cB_L} \| f \|_\cF < \infty$), the \textit{variation norm} (also called the \textit{atomic} or \textit{gauge norm})
\begin{equation}
    \| f \|_{\cV_L} := \inf \{ c > 0: f \in c \cB_L \}
\end{equation}
is a Banach norm on the \textit{variation space}
\begin{eqnarray*}
    \cV_L := \{ f \in \cF: \| f \|_{\cV_L} < \infty \}
\end{eqnarray*}
and $\cB_L$ is the unit ball of $\cV_L$ (\cite{siegel2023characterization}, Lemma 1; see also \cite{kurkova2002bounds,kuurkova1997dimension}). Intuitively, the $\| \cdot \|_{\cV_L}$ norm measures how much the ball $\cB_L$ must be dilated in order for a given function $f$ to be included in that dilation. $\cV_L$ is the space of all functions which are included in some finite (although possibly very large) dilation of $\cB_L$.

The $\cF$-norm closure in the definition of $\cB_L$ expands our spaces to include not only finite-width deep networks (which may be highly limited in terms of their functional expressivity) but also all functions which they approximate with appropriate parameter control. The choice of $\cF$ determines the precise sense in which these infinite-width limits are taken, with different choices imposing different levels of regularity on the constituent functions. Two useful choices of $\cF$ that we will consider in the remainder of the paper are $L^p(\mu)$ for some finite measure $\mu$ on $\Omega$, and $C(\Omega)$, the Banach space of continuous functions on $\Omega$, equipped with the uniform norm $\| f \|_\infty := \sup_{\vx \in \Omega} |f(\vx)|$. With infinite-width limits taken $\cF = C(\Omega)$, all functions in $\cV_L$ are continuous; on the other hand, in $\cF = L^p(\mu)$, there is no such guarantee a priori as pointwise behavior of the limits is not even well-defined. This distinction is important in several respects: while the $L^p(\mu)$ closures are most directly relevant to certain aspects of our complexity analysis in \cref{sec:statistical_complexities}, the additional regularity imposed by the $C(\Omega)$ norm is necessary to analyze their behavior on finite datasets (as in the representer theorem \cref{th:rep_theorem}) and to prove the univariate depth-saturation result in \cref{sec:uni_depth_saturation}. Nonetheless, as the following proposition establishes, these two choices are essentially equivalent up to sets of $\mu$-measure zero:
 \begin{proposition} \label{prop:Lp_C_representative}
    Let $\mu$ be a finite measure on $\Omega$. Let $\cB_L^{L^p(\mu)}$ and $\cB_L^{\infty}$ denote the sets \eqref{eq:BL} with closures taken in $\cF = L^p(\mu)$ and $C(\Omega)$, respectively. Then $\cB_L^{\infty} \subset \cB_L^{L^p(\mu)}$, and every function in $\cB_L^{L^p(\mu)}$ agrees $\mu$-almost everywhere (a.e.) with a function in $\cB_L^{\infty}$. 
\end{proposition}
The proof of \cref{prop:Lp_C_representative} is in \cref{appendix:Lp_C_representative_proof}. It will also be convenient to establish notation for the sets
\begin{align} \label{eq:B_tilde}
    \widetilde{\cB}_L := \aconv \left( \{ \sigma_s \circ f: s > 0, f \in \widetilde{\cB}_{L-1} \} \right)
\end{align}
which are identical to the sets $\cB_L$, but with no norm closures taken. The following technical lemma establishes that the closure in the definition of the sets $\cB_L$ need only be taken once after the final $\aconv$. (In other words, it is sufficient to take infinite-width limits at only the final hidden layer, rather than at each intermediate layer.)
\begin{proposition} \label{prop:limit_final_layer}
    Take either $\cF = C(\Omega)$ or $\cF = L^p(\mu)$. Then for any $L \geq 2$, the sets $\widetilde{\cB}_L$ defined in \eqref{eq:B_tilde} satisfy $\overline{\widetilde{\cB}_L} = \cB_L$.
\end{proposition}
The proof is in \cref{appendix:proof_limit_final_layer}. Both \cref{prop:Lp_C_representative} and \cref{prop:limit_final_layer} simplify various aspects of our analysis and will be used throughout the paper.

\subsection{Relationship of the spaces $\cV_L$ to neural networks} \label{sec:relationship_neural_networks}
In order to concretize the relationship between our spaces $\cV_L$ and parametric neural network architectures used in practice, it is instructive to first consider the case of homogeneous activation functions $\sigma$ such as the ReLU. In this case, the sets $\cB_L$ for $L \geq 2$ can be equivalently defined as
\begin{align} \label{eq:B_L_homogeneous}
    \cB_L := \overline{\aconv} \left( \{ \sigma \circ f: f \in \cB_{L-1} \} \right).
\end{align}
From \eqref{eq:B_L_homogeneous}, it is easy to see that, when $\sigma$ is homogeneous, the space $\cV_L$ contains all neural networks of the form
\begin{align} \label{eq:deep_nn_sum}
    f(\vx) &:= \sum_{k_{L-1}=1}^{K_{L-1}} w_{k_{L-1}}^{(L)} \sigma \left( \sum_{k_{L-2}=1}^{K_{L-2}} W_{k_{L-1},k_{L-2}}^{(L-1)} \sigma \left(  \dots \sigma \left( \sum_{k_1=1}^{K_1} W_{k_2, k_1}^{(2)} \sigma \left( ( \vw_{k_1}^{(1)} )^\top \vx + b_{k_1}^{(1)}  \right) \dots  \right)\right) \right),
\end{align}
or equivalently
\begin{align}  \label{eq:deep_nn_matrix}
        f(\vx) := (\vw^{(L)})^\top \sigma \left( \mW^{(L-1)} \sigma \left( \dots \sigma \left( \mW^{(2)} \sigma \left( \mW^{(1)} \vx + \vb^{(1)} \right) \right) \dots \right) \right),
    \end{align}
where the $\vw_{k_1}^{(1)} \in \cW$ and the $b_{k_1}^{(1)} \in \cB$ are the rows of $\mW^{(1)}$ and $\vb^{(1)}$, and $W_{k_\ell, k_{\ell-1}}^{(\ell)}$ is the entry of $\mW^{(\ell)} \in \R^{K_\ell \times K_{\ell-1}}$ at row $k_\ell$ and column $k_{\ell-1}$. The recursive absolutely convex hull construction of $\cB_L$ induces the hidden-layer weight constraints
\begin{align} \label{eq:hidden_l1_constraint_sum}
    \sum_{k_{\ell-1}}^{K_{\ell-1}} \left| W_{k_\ell, k_{\ell-1}}^{(\ell)} \right| \leq 1, \qquad \ell=2, \dots, L-1, \quad  k_\ell = 1, \dots, K_\ell,
\end{align}
or equivalently
\begin{align} \label{eq:hidden_l1_constraint_matrix}
    \| \mW^{(\ell)} \|_{1,\infty} \leq 1, \quad \ell=2, \dots, L-1.
\end{align}
Here, $\| \mM \|_{1,\infty} := \max_i \| \mM_{i,:} \|_1$ denotes the maximum rowwise $\ell^1$ norm of a matrix $\mM$. Any network $f$ of the form \eqref{eq:deep_nn_sum}/\eqref{eq:deep_nn_matrix} satisfying these constraints \eqref{eq:hidden_l1_constraint_sum}/\eqref{eq:hidden_l1_constraint_matrix} has
\begin{align}
    \| f \|_{\cV_L} \leq \sum_{k_{L-1}}^{K_{L-1}} 
    \left| w_{k_{L-1}}^{(L)} \right|.
\end{align}
The fact that this is an upper bound, and not an equality, comes from the infimum in the definition of $\| \cdot \|_{\cV_L}$ (since  there may be some way of representing the same function as $f$, but with smaller norm). 

For homogeneous activations, the hidden $\ell^1$ constraints \eqref{eq:hidden_l1_constraint_sum}/\eqref{eq:hidden_l1_constraint_matrix} do not place any meaningful limitation on the set of neural networks described by $\cV_L$. This is because, by successively $\ell^1$-normalizing the inputs to each hidden neuron and then rescaling output of that neuron accordingly, we can always obtain a neural network that represents the same function as the original, while meeting the hidden layer constraints \eqref{eq:hidden_l1_constraint_sum}/\eqref{eq:hidden_l1_constraint_matrix}. The output weights obtained after this rescaling will depend on the weight magnitudes from all hidden layers in the original network, providing an upper bound on the norm of that network function in terms of all of the original network parameters at all layers. On the other hand, for non-homogeneous activations $\sigma$, attempting to define $\cV_L$ directly as in \eqref{eq:B_L_homogeneous} would exclude many deep networks with that activation (any networks not satisfying the hidden $\ell^1$ constraints \eqref{eq:hidden_l1_constraint_sum}/\eqref{eq:hidden_l1_constraint_matrix}) from the space $\cV_L$ entirely. By constructing the spaces $\cV_L$ in terms of the normalized activations $\sigma_s$, rather than the plain activations $\sigma$, we circumvent this difficulty: the feature shapes $t \mapsto \sigma(st)$ are allowed in our spaces at all scales $s$, but features induced by a large $s$ incur a proportionally large $\cV_L$ norm penalty. As a consequence, under reasonable assumptions on $\cW$ and $\cB$, our spaces $\cV_L$ include all fully-connected depth-$L$ networks with activation $\sigma$, with any parameter values at any of the hidden layers, with the norm of the network controlled by the magnitudes of these weights at all hidden layers. This is summarized in the following proposition, whose proof is in \cref{appendix:proof_path_norm_equivalence}. 
\begin{proposition} \label{prop:path_norm_equivalence}
      Let $\sigma$ be any of the activations in \cref{tab:activation_summary} except $\ReLU^m$, so that $\sigma_s(t) = \sigma(st)/s$. Let $f$ be a neural network of the form \eqref{eq:deep_nn_sum}/\eqref{eq:deep_nn_matrix} with activation function $\sigma$. Suppose that for each $k_1 = 1, \dots, K_1$, there exists an $s_{k_1} > 0$ such that 
    \begin{align} \label{eq:s_k_assumption}
         \vw_{k_1}^{(1)} \in s_{k_1} \cW, \qquad b_{k_1}^{(1)} \in s_{k_1} \cB.
    \end{align}
    Then
    \begin{align} 
        \| f \|_{\cV_L} \leq \Phi(\vtheta) &:= \sum_{k_{L-1}=1}^{K_{L-1}}     \sum_{k_{L-2}=1}^{K_{L-2}} \dots     \sum_{k_{2}=1}^{K_{2}}     \sum_{k_{1}=1}^{K_{1}}     \left| w_{k_{L-1}}^{(L)} W_{k_{L-1},k_{L-2}}^{(L-1)}  \dots  W_{k_2,k_1}^{(2)} s_{k_1} \right| \label{eq:path_norm_sum} \\
        &= \big| \vw^{(L)} \big|^\top \big| \mW^{(L-1)} \big| \dots \big| \mW^{(2)} \big| |\vs| \label{eq:path_norm_prod}
    \end{align}
    with the absolute values in \eqref{eq:path_norm_prod} taken element-wise. Moreover,
    \begin{align} \label{eq:path_norm_upper_bound_psi}
        \Phi(\vtheta) \leq \Psi(\vtheta) := \| \vw^{(L)} \|_1 \left( \prod_{\ell=2}^{L-1} \| \mW^{(\ell)} \|_{1,\infty} \right) \| \vs \|_\infty,
    \end{align}
    where $\vs := [s_1, \dots, s_{K_1}]^\top$. If $\sigma$ is homogeneous, there is another network with parameters $\widetilde{\vtheta}$ which represents the same function as $f$ and satisfies $\Phi(\widetilde{\vtheta}) = \Psi(\widetilde{\vtheta})$. An analogous statement for $\ReLU^m$ is given in \cref{appendix:relu_m_path_norm_equivalence}.
\end{proposition}
\begin{remark} \label{remark:admissible_weight_bias_sufficient_condition}
    If $\cW$ and $\cB$ are defined as
    \begin{align} \label{eq:admissible_weight_bias_sufficient_condition}
        \cW := \{ \vw \in \R^d: \| \vw \|_p \leq C_{\cW} \}, \qquad \cB := [-C_{\cB}, C_{\cB}], \qquad 1 \leq p \leq \infty, \ C_{\cW}, C_{\cB} > 0,
    \end{align}
    then the condition \eqref{eq:s_k_assumption} holds for \textit{all} $\vw_{k_1}^{(1)} \in \R^d$ and $b_{k_1}^{(1)} \in \R$ by taking
    \begin{align}
        s_{k_1} \geq \max \left\{ \frac{\| \vw_{k_1}^{(1)} \|_p}{C_{\cW}}, \frac{|b_{k_1}^{(1)}|}{C_{\cB}} \right\}.
    \end{align}
    Therefore, with this choice of $\cW$, $\cB$, \cref{prop:path_norm_equivalence} (and its $\ReLU^m$ analogue in \cref{appendix:relu_m_path_norm_equivalence}) show that the spaces $\cV_L$ contain \textit{all} neural networks of the form \eqref{eq:deep_nn_sum}/\eqref{eq:deep_nn_matrix}, with arbitrary parameter values at all layers, for any of the activations $\sigma$ in \cref{tab:activation_summary}. 
\end{remark}
\begin{remark} \label{remark:biased_net_representation}
    With $\cW$ and $\cB$ defined as in \eqref{eq:admissible_weight_bias_sufficient_condition}, our spaces $\cV_L$ also incorporate networks with hidden-layer biases. This holds because \eqref{eq:admissible_weight_bias_sufficient_condition} guarantees that the linear class $\cB_1$ contains at least one positive constant function of the form $f(\vx) = c$, $0 < c < C_{\cB}$. All activations $\sigma(t)$ in \cref{tab:activation_summary} are positive on $t > 0$, and thus for any $\ell \geq 1$, the $\ell$-fold composition $\sigma^{\circ \ell}$ satisfies $\sigma^{\circ \ell}(c) \neq 0$, and this positive constant function $ f(\vx) = \sigma^{\circ \ell} (c)$ is in $\cB_{\ell+1}$. Therefore, each bias term $b_{k_\ell}^{(\ell)}$ at depth $\ell \geq 2$ can be represented in $\cV_\ell$ as the constant function
    \begin{align}
        f(\vx) = b_{k_\ell}^{(\ell)} =  \underbrace{\left( \frac{b_{k_\ell}^{(\ell)}}{\sigma^{\circ (\ell-1)} (c)}  \right)}_{\displaystyle  W_{k_\ell,K_{\ell-1} + 1}} \sigma \left( \underbrace{\sigma^{\circ (\ell-2)} (c)}_{\displaystyle \in \cB_{\ell-1}} \right) \in W_{k_\ell,K_{\ell-1} + 1} \cB_\ell.
    \end{align}
    These additional coefficients $W_{k_\ell, K_{\ell-1}+1}$ will appear in the corresponding bound \eqref{eq:path_norm_sum}.
\end{remark}

We will refer to the parametric upper bound $\Phi$ in \eqref{eq:path_norm_sum}/\eqref{eq:path_norm_prod} as the \textit{path norm}. For the ReLU activation, the properties of this path norm have been studied in several previous works (\cite{neyshabur2015norm,barron2019complexity,wojtowytsch2020banach}, among others). The absolute values in \eqref{eq:path_norm_prod} denote elementwise absolute values of vectors/matrices. This is exactly (up to possible differences in first-layer normalization) the $p=1$ \textit{path regularizer} of \cite{neyshabur2015norm} (Equation 7), which computes the sum of all possible start-to-finish ``absolute path products'' in the network. Theorem 5 in that paper shows that $\Phi(\vtheta) \leq \Psi(\vtheta)$, with equality achieved by some rescaled representation of the same network function. \cite{barron2019complexity} refer to $\Phi$ as the \textit{total path variation}, and bound the Rademacher complexities and $L^2$ metric entropies of ReLU neural networks in terms of $\Phi$. In \cref{sec:statistical_complexities}, we show that similar bounds hold for many other activation functions (including non-homogeneous ones) and all values of $p \in [1,\infty]$. The quantity $\Phi$ also appears in \cite{wojtowytsch2020banach} as a parametric upper bound on a function space norm associated with deep ReLU neural networks. As we will show in \cref{sec:int_rep_thm}, the function spaces studied in that paper are equivalent to ours in the $\ReLU$ case.

\subsection{Integral representations and a representer theorem} \label{sec:int_rep_thm}

Within the broader literature on function space properties of neural networks, it is common to define infinite-width networks formally as integrals taken with respect to some finite measure (\cite{bach2017breaking,savarese2019infinite,ongie2019function,shenouda2024variation,bartolucci2023understanding}). Intuitively, these integrals represent sums of uncountably many neurons, and the associated measure determines the parametric weight assigned to each of these neurons. In order to situate our work within this context, we show that our spaces $\cV_L$ admit an analogous description in terms of integral representations:

\begin{lemma} \label{lemma:integral_rep}
    Take $\cF := C(\Omega)$. For every $L \geq 2$, the functions $f \in \cV_L$ are exactly those which admit a pointwise integral representation
    \begin{align} \label{eq:pointwise_integral_rep}
        f(\vx) = \int_{\cB_{L-1} \times [0,\infty]} \sigma_s(g(\vx)) \ d \mu(s,g), \qquad \forall \vx \in \Omega
    \end{align}
    for some $\mu \in \cM(\cB_{L-1} \times [0,\infty])$. Any such $f$ has
    \begin{align} 
        \| f \|_{\cV_L} = \inf \left\{ \| \mu \|_{\TV}: \mu \in \cM(\cB_{L-1} \times [0,\infty]), \  f(\vx) = \int_{\cB_{L-1} \times [0,\infty]} \sigma_s (g(\vx)) \ d \mu(s,g), \  \vx \in \Omega \right\}
    \end{align}
    In particular, if $\sigma$ is homogeneous, the functions $f \in \cV_L$ are those which admit pointwise integral representations
    \begin{align} \label{eq:pointwise_integral_rep_homogeneous}
        f(\vx) = \int_{\cB_{L-1}} \sigma(g(\vx)) \ d \mu(g), \qquad \forall \vx \in \Omega
    \end{align}
    or, in the $L=2$ case:
    \begin{align} \label{eq:pointwise_integral_rep_homogeneous_shallow}
        f(\vx) = \int_{\cW \times \cB} \sigma(\vw^\top \vx + b) \ d \mu(\vw,b), \qquad \forall \vx \in \Omega
    \end{align}
    for some $\mu \in \cM(\cB_{L-1})$ (resp. $\cM(\cW \times \cB)$), with $\| f \|_{\cV_L}$ (resp. $\| f \|_{\cV_2}$) given by the infimal value of $\| \mu \|_{\TV}$ over all such representing measures.
\end{lemma}

The proof is in \cref{appendix:proof_integral_rep}. \cref{lemma:integral_rep} shows that, in the shallow case $L = 2$, our spaces belong to the same general family as the aforementioned works which model infinite-width shallow networks as integrals. Moreover, \cref{lemma:integral_rep} reveals that, in the ReLU case, our deep spaces $\cV_L$ are equivalent (again, up to possible differences in first-layer normalization) to the \textit{generalized Barron space} (or \textit{neural tree space}) of \cite{wojtowytsch2020banach}. In terms of the function space construction itself, the most notable improvement of our spaces over those of \cite{wojtowytsch2020banach} is that ours are compatible with large families of activations, and extend many of the advantages of the ReLU construction to those cases. We also comprehensively evaluate the relationships between our construction and various other function space norms and related capacity controls for deep networks; as a result of this comparison, we are able derive tighter complexity bounds on the associated function classes in \cref{sec:statistical_complexities} than those in \cite{wojtowytsch2020banach}. Finally, our construction makes the convex hull structure of our spaces explicit. This structural aspect of the spaces is conceptually useful in proving our depth saturation result in \cref{sec:uni_depth_saturation} (which is, to our knowledge, entirely novel).

It is important to note that the representation \eqref{eq:pointwise_integral_rep} allows $s$ to assume the values 0 and $\infty$. This is a consequence of topological closedness of the balls $\cB_L$: because they must contain their infinite-width uniform limits, they must contain the limiting atoms $\sigma_0$ and $\sigma_\infty$ described in \cref{tab:activation_summary}. For homogeneous activations, the scaling parameter $s$ is redundant---it has no effect on the underlying feature dictionaries in these cases, so the corresponding spaces can be equivalently expressed with no need for the scale parameter $s$ at all. For non-homogeneous activations, however, the parameter $s$ genuinely changes the shape of the function $\sigma_s$. As a result, the limiting atoms $\sigma_0$ and $\sigma_\infty$ may correspond to entirely new functions which are not in the finite feature dictionary $\sigma_s$ at all, but are approximated by features $\sigma_s$ asymptotically. For example, for many of the ``ReLU-like'' activations (GELU, SiLU/Swish, ELU, etc.), the limiting atom $\sigma_\infty$ is simply a ReLU neuron (see \cref{tab:activation_summary}). In fact, \cref{lemma:integral_rep} allows us to show that the depth-$L$ spaces associated with many activation functions are \textit{equivalent} to the corresponding depth-$L$ ReLU space. 
\begin{lemma} \label{lemma:relu_nonrelu_space_equivalence}
    Let $L \geq 2$, and let $\sigma$ be any of the following activations: Leaky ReLU,  GELU, SiLU/Swish, Mish, ELU, SELU, CELU, centered softplus, absolute value, or bent identity. (See \cref{tab:activation_summary}.) Consider the spaces $\cV_L^\sigma$ with $\cW$, $\cB$ defined as in \eqref{eq:admissible_weight_bias_sufficient_condition}. Then there are constants $A_{L,\sigma}$ and $B_{L,\sigma}$, depending on $L$ and $\sigma$ such that
    \begin{align}
        A_{L,\sigma} \| f \|_{\cV_L^{\sigma}} \leq \| f \|_{\cV_L^{\ReLU}} \leq B_{L,\sigma} \| f \|_{\cV_L^{\sigma}} 
    \end{align}
    for all $f$. The depth-dependence of admissible choices of $A_{L,\sigma}$ and $B_{L,\sigma}$ for different activations $\sigma$ is summarized in \cref{tab:activation_equivalence_constants}.
\end{lemma}

\begin{table}
\centering
\begingroup
\setlength{\tabcolsep}{6pt}
\begin{tabular}{c|c|c}
\hline
\textbf{Activation}
& \textbf{ $A_{L,\sigma}$}
& \textbf{$B_{L,\sigma}$}
\tabularnewline[1pt]
\hline
\rule{0pt}{2.6ex}
$\mathrm{LeakyReLU}_\alpha$ $(0<\alpha<1)$
& $(1-\alpha)^{L-1}$
& $1-\alpha$ 
\tabularnewline[1pt]

GELU
& $1$
& $\cO(1.516^L)$
\tabularnewline[1pt]

SiLU/Swish
& $1$
& $\cO(1.400^L)$
\tabularnewline[1pt]

Mish
& $1$
& $\cO(1.403^L)$
\tabularnewline[1pt]

$\mathrm{ELU}_\alpha$ $(0<\alpha\leq1)$
& $1$
& $\cO(L)$
\tabularnewline[1pt]

$\mathrm{SELU}_{\alpha,\lambda}$
$(\alpha\approx1.67,\ \lambda\approx1.05)$
& $\lambda^{L-1}$
& $\cO(2.466^L)$
\tabularnewline[1pt]

$\mathrm{CELU}_\alpha$ $(\alpha>0)$
& $1$
& $\cO(L)$
\tabularnewline[1pt]

Centered softplus
& $1$
& $\cO(L)$
\tabularnewline[1pt]

Absolute value
& $\cO(1)$
& $\cO(2^L)$
\tabularnewline[1pt]

Bent identity
& $\left(\frac{2}{3}\right)^{L-1}$
& $\cO \left(L\left(\frac{3}{2}\right)^L\right)$
\tabularnewline[2pt]
\hline
\end{tabular}
\endgroup
\caption{Depth dependence of admissible norm-equivalence constants in
\cref{lemma:relu_nonrelu_space_equivalence}. (See proof in \cref{appendix:proof_relu_nonrelu_space_equivalence} for exact values.)}
\label{tab:activation_equivalence_constants}
\end{table}

The proof of \cref{lemma:relu_nonrelu_space_equivalence} is in \cref{appendix:proof_relu_nonrelu_space_equivalence}. This lemma confirms what one might reasonably expect: from a function space perspective, there is no fundamental distinction between ReLU and many of its common smooth or modified variants. \cite{heeringa2024embeddings} demonstrate a similar type of equivalence between shallow networks of various activations; to our knowledge, we are the first to do so for the deep case.

Another advantage of \cref{lemma:integral_rep} is that it allows us to prove a \textit{representer theorem} for our spaces $\cV_L$. This shows that minimum-norm data-fitting problems over $\cV_L$ are always solved by width-bounded neural networks.

\begin{theorem} \label{th:rep_theorem}
    Take $\cF := C(\Omega)$ and $L \geq 2$. Let $\sigma$ be any of the activations in \cref{tab:activation_summary}, and consider the corresponding space $\cV_L$ with $\cW$, $\cB$ as in \eqref{eq:admissible_weight_bias_sufficient_condition}. Then for any $\lambda > 0$, any dataset $(\vx_1, y_1), \dots, (\vx_N, y_N) \in \Omega \times \R$, and any loss function $\cL$ which is nonnegative and lower semicontinuous in its second argument, there exists a solution $f^*$ to the problem
    \begin{equation} \label{opt:reg_loss}
        \min_{f \in \cV_L} \sum_{i=1}^N \cL(y_i, f(\vx_i)) + \lambda \| f \|_{\cV_L}.
    \end{equation}
    If $\sigma$ is homogeneous, $f^*$ is exactly a deep neural network of the form \eqref{eq:deep_nn_sum}/\eqref{eq:deep_nn_matrix}. If $\sigma$ is non-homogeneous, $f^*$ is a deep network of the form \eqref{eq:deep_nn_sum}, except that some of the activations $\sigma$ may be replaced by the limiting normalized atoms $\sigma_0$, $\sigma_\infty$ of $\sigma$ (see \cref{tab:activation_summary}). In either case, the hidden layer widths satisfy $K_{\ell} \leq N^{L-\ell}$ for $\ell = 1, \dots, L-1$. 
    
    For all but the linear activation $\sigma(t) = t$, the above statement holds for the interpolation problem
    \begin{equation} \label{opt:interp}
     \min_{f \in \cV_L} \| f \|_{\cV_L} \ , \ \mbox{subject to } f(\vx_i)=y_i, \,i=1,\dots,N
    \end{equation}
    as long as $y_i \neq y_j$ whenever $\vx_i = \vx_j$.
\end{theorem}

The proof is in \cref{appendix:rep_theorem_proof}. We again note that the finite representer produced by the theorem is synthesized from the
atomic dictionary associated with the compactified scale parameter $s \in [0,\infty]$. As a result, if $\sigma$ is non-homogeneous, the minimum-norm representing network may have neurons with activations corresponding to the limiting atoms $\sigma_0$ or $\sigma_\infty$, and not the original activation $\sigma$. This is not a defect of the theory, but a natural consequence of the dictionary-level
asymmetry in our construction. For non-homogeneous activations, the compactified dictionary includes
infinite-scale limiting atoms; for homogeneous activations, the scale parameter is
redundant, and introduces no new limiting atoms.

Several other works prove representer theorems for function spaces associated with deep ReLU neural networks, including \cite{parhi2022kinds,shenouda2024variation,bartolucci2024neural,heeringa2025deep}. The width bounds obtained by \cite{parhi2022kinds,shenouda2024variation,bartolucci2024neural} are smaller than ours. This is a consequence of the fact that, as we discuss in more detail in \cref{sec:relationship_other_spaces}, those representational costs penalize group $1,1$- and $2,1$-type norms on the hidden weight matrices. This type of control (particularly the $2,1$ group norm) has a strong per-neuron sparsifying effect, encouraging many neurons at each hidden layer to be zeroed out. In contrast, our function space norm imposes control on the $1,\infty$ norms of the hidden matrices, which does not have the same neuron-sparsifying effect: it only encourages the weights going into any hidden neuron to be sparse, but does not necessarily encourage entire neurons to be zero. Whether or not the width bound in our representer theorem can be improved is an open question. The representer theorem in \cite{heeringa2024embeddings} also has a tighter width bound of $N$ at each layer because it is derived for a finite $N$-dimensional space; this is incomparable with our representer theorem (and most others), which consider optimization problems over \textit{infinite}-dimensional spaces.

\subsection{Comparison to other deep ReLU spaces and representation costs} \label{sec:relationship_other_spaces}

Our construction is modeled on the shallow variation spaces studied in \cite{siegel2023characterization}, which correspond to ours for homogeneous activations and depth $L=2$. In particular, when $\cW = \bS^{d-1}$, $\cB = [c_1,c_2]$, and $\sigma = \ReLU^m$, our space $\cV_2$ coincides with the Radon BV spaces of \cite{parhi2021banach,ongie2019function} (see also \cite{bartolucci2023understanding}). These spaces characterize shallow $\ReLU^m$ networks in terms of the total variation of the $m^\textrm{th}$ ``directionalized'' derivative (defined via the Radon transform) of the represented function. In the univariate case, this characterization admits a simpler description in terms of standard distributional derivatives; see \cite{savarese2019infinite} and our discussion in \cref{sec:uni_depth_saturation}. For the ReLU activation itself, this shallow variation norm is also equivalent to the infimal \textit{sum of squared weights} among finite-width shallow ReLU networks which uniformly approximate the target function on compact subsets \citep[Equation~7]{ongie2019function}. 

Prior works employ a number of strategies to extend these shallow function space constructions to deep networks. As described in \cref{sec:int_rep_thm} above, \cite{wojtowytsch2020banach} define deep function spaces for ReLU neural networks as nested integral representations.  \cref{lemma:integral_rep} shows these spaces coincide with ours for the ReLU activation. We proceed to relate this (our) norm to several other norms and representational costs that have been studied in the context of deep ReLU neural networks. For a deep network of the form \eqref{eq:deep_nn_matrix} with parameters $\vtheta$, consider the \textit{sum of squared weights} (SOSW) representation cost
\begin{align} \label{eq:sosw}
    R_{\mathrm{SOSW}}(\vtheta)
    := \| \vb^{(1)} \|_2^2 + \sum_{\ell=1}^{L-1} \| \mW^{(\ell)} \|_F^2
    +
    \| \vw^{(L)} \|_2^2 .
\end{align}
This cost corresponds to the \textit{weight decay} regularization commonly used in practice. A comprehensive function space theory associated with the SOSW cost is developed in \cite{ongie2026representation}. SOSW representation costs are also studied in \cite{parkinson2024depth} in terms of depth separation, and in \cite{neyshabur2015norm} (Equation 3, case $p = q = 2$) in terms of capacity control. Yet another approach is to build deep function spaces by composing shallow function spaces, as in \cite{bartolucci2024neural,parhi2022kinds,shenouda2024variation}. The norms in those works correspond to the following parametric representation costs:
\begin{align}
    R_{\mathrm{Ba}} (\vtheta) &:= \| \vb^{(1)} \|_2 + \sum_{\ell=1}^{L-1} \| \mW^{(\ell)} \|_{2,1} +  \| \vw^{(L)} \|_1, &\textrm{(\cite{bartolucci2024neural})} \label{eq:R_Ba} \\
    R_{\mathrm{Pa}} (\vtheta) &:=  \sum_{\ell=1}^{L-1} \sum_{k=1}^{K_\ell} 
    \| \mV^{(\ell)}_{:,k} \|_1 \| \mU^{(\ell)}_{k,:} \|_2, &\textrm{(\cite{parhi2022kinds})} \label{eq:R_Pa} \\
    R_{\mathrm{Sh}} (\vtheta) &:=  \sum_{\ell=1}^{L-1} \sum_{k=1}^{K_\ell} 
    \| \mV^{(\ell)}_{:,k} \|_2 \| \mU^{(\ell)}_{k,:} \|_2, &\textrm{(\cite{shenouda2024variation})}. \label{eq:R_Sh}
\end{align}
In \eqref{eq:R_Pa} and \eqref{eq:R_Sh}, the matrix $\mU^{(1)} := [\mW^{(1)}, \vb^{(1)}]$, each hidden $\mW^{(\ell)}$ is factorized as $\mW^{(\ell)} = \mU^{(\ell)} \mV^{(\ell-1)}$ for $\ell=2, \dots, L-1$, and the final $\mV^{(L-1)} := (\vw^{(L)})^\top$ is a single-row matrix.

The following proposition gives a unified comparison between these representation costs and our $\cV_L$ norm (and its parametric path norm analogue $\Phi$ in \eqref{eq:path_norm_sum}). The proof is in \cref{appendix:proof_rep_cost_comparison}.
\begin{proposition} \label{prop:rep_cost_comparison}
    Let $f$ be a depth-$L$ ReLU neural network of the form \eqref{eq:deep_nn_matrix} with parameters $\vtheta$. Consider our spaces $\cV_L$ defined with $\cW = \bB_2^d := \{ \vx \in \R^d: \| \vx \|_2 \leq 1 \}$ and $\cB = [-1,1]$, and suppose the rows of $\mW^{(1)}, \vb^{(1)}$ are contained in $\cW,\cB$. Then:
    \begin{align}
        \| f \|_{\cV_L} &\leq \Phi (\vtheta) \leq \left( \frac{R_\mathrm{SOSW} (\vtheta)}{L} \right)^{L/2}, \quad \textrm{and} \quad \| f \|_{\cV_L} \leq \Phi (\vtheta) \leq \left( \frac{R_\mathrm{Ba,Pa,Sh} (\vtheta)}{L-1} \right)^{L-1}.
    \end{align}
    Consequently, if $R_{\mathrm{SOSW}} (\vtheta) \leq L$ or any of $R_{\mathrm{Ba},\mathrm{Pa},\mathrm{Sh}} (\vtheta) \leq L-1$, then $f \in \cB_L$.
\end{proposition}

The powers appearing in \cref{prop:rep_cost_comparison} are not an artifact of the proof; they reflect a basic distinction between genuine function norms and additive layerwise complexity controls. The norm $\|\cdot\|_{\mathcal V_L}$ is a norm on the scalar function represented by the network. By contrast, the deep norms constructed in \cite{bartolucci2024neural,parhi2022kinds,shenouda2024variation} are most naturally viewed as norms on Cartesian product spaces whose elements encode tuples of shallow-space functions across the hidden layers. For finite networks, they correspond to additive penalties across hidden-layer weights. Even if one takes the infimum of such quantities over all parametric representations of a fixed scalar output function, the resulting quantity cannot behave like a norm on the output function itself. Indeed, in a depth-$L$ ReLU network, multiplying every layer by a scalar $a>0$ multiplies the represented function by $a^L$, while the additive costs in \eqref{eq:R_Ba}, \eqref{eq:R_Pa}, \eqref{eq:R_Sh} scale only linearly in $a$, and the sum-of-squared-weights cost in \eqref{eq:sosw} scales only quadratically in $a$. Thus an additive hidden-layer cost can grow only linearly, or quadratically in the sum-of-squared-weights case, under a rescaling which changes the function output by a factor exponential in depth. This mismatch is precisely why comparisons between such costs and a true function norm naturally involve powers of order $L$. In the SOSW case, this phenomenon is explored in detail by \cite{ongie2026representation}, who prove that this representation cost does not induce a true Banach norm for depths $L > 2$, but instead a Banach quasi-norm.

This distinction between representational costs which do or do not correspond to true function space norms is not a mere technicality; it is in fact one of the core reasons that our construction is useful for studying the effect of depth. Since $\|\cdot\|_{\mathcal V_L}$ is a genuine norm on scalar functions, it accounts for the compounded effect of repeated positive-homogeneous rescaling across layers. This allows us to separate growth due merely to rescaling from growth due to the repeated application of nonlinear feature maps. As such, any growth (with depth) of our unit norm-constrained classes $\cB_L$ is due only to additional functional complexity produced by repeated application of the nonlinearity, rather than constant rescaling across layers. \cref{prop:rep_cost_comparison} shows that, with this layerwise rescaling effect properly controlled for, our ReLU classes $\mathcal B_L$ are still large enough to contain the function classes
\begin{align}
    \cB_L^{R(\vtheta)} := \{ f_{\vtheta}: R(\vtheta) \leq L-1 \}
\end{align}
induced by these representation costs $R(\vtheta)$. Consequently, the metric entropy and Rademacher complexity bounds we prove for our function classes $\cB_L$ apply automatically to any of these classes $\cB_L^{R(\vtheta)}$.

\section{Function-space complexities of deep neural function classes} \label{sec:statistical_complexities}
In this section, we will derive bounds on two important complexity measures---\textit{Rademacher complexity} and \textit{metric entropy}---for our function classes $\cB_L$. These bounds will show that, no matter how large the depth $L$ is, the complexity of $\cB_L$ is limited. In all cases, the growth rate bounds we obtain are polynomial (up to polylogarithmic factors), and the implicit constants in many cases grow no more than polynomially (and often linearly) in $L$.

\subsection{Rademacher complexities of the classes $\cB_L$} \label{sec:rad_complexity}

The first quantity we consider is the \textit{worst-case empirical Rademacher complexity}
\begin{equation} \label{eq:worst_case_empirical_rad_def}
    \cR_N(\cB_L) := \sup_{\vx_1, \dots, \vx_N \in \Omega} \bE_{r_1, \dots, r_N \, \overset{\mathrm{iid}}{\sim} \, \mathrm{Rad}}  \left[ \sup_{f \in \widetilde{\cB}_L} \frac{1}{N}  \sum_{i=1}^N r_i f(\vx_i) \right].
\end{equation}
Here, the closure in the definition $\cB_L$ is taken in $C(\Omega)$---this is because the pointwise behavior of functions in the $L^p(\mu)$ closure may not be well-defined, in which case the notion of Rademacher complexity is inapplicable. Conceptually, the Rademacher complexity of a function class measures the ability of functions in that class to fit random noise. This makes Rademacher complexity a useful tool in statistical learning theory and machine learning, where bounds on this quantity can be converted into quantitative guarantees on generalization performance (see \cref{sec:metric_entropy} for further discussion).

To state our Rademacher bound as well as our subsequent metric entropy bound, let us establish overarching assumptions and notation that we will use here in \cref{sec:rad_complexity} and in \cref{sec:metric_entropy}. Assume that $\sigma$ is locally Lipschitz\footnote{An $\R \to \R$ function is \textit{locally Lipschitz continuous} if it is Lipschitz on compact intervals. All normalized activations $\sigma_s$ for all activations $\sigma$ considered in \cref{tab:activation_summary} are locally Lipschitz (all except $\ReLU^m$ are in fact globally Lipschitz).} and satisfies $\sigma(0) = 0$. For each $L$, let $C_L := \sup_{f \in \cB_L} \| f \|_{\infty}$, and let $\rho_L := \sup_{s > 0} \Lip_{[-C_L, C_L]}(\sigma_s)$, where $\Lip_{[-C_L, C_L]}(\sigma_s)$ denotes the local Lipschitz constant of $\sigma_s$ on $[-C_L, C_L]$. Let $\Pi_L := \prod_{\ell=1}^{L-1} \rho_\ell$ and $\pi_\ell := \prod_{j=\ell+1}^{L-1} \rho_j$ be the full and partial products, respectively, of these local Lipschitz constants. We note that many of the activations considered in \cref{tab:activation_summary}---including ReLU, Leaky ReLU, and various sigmoidal functions---are globally 1-Lipschitz, in which case $\Pi_L$ and $\pi_\ell$ are simply equal to 1. Additionally, let $A_L$ denote a depth-dependent constant whose role is discussed further in the proof of \cref{lemma:rad_bound} (see \cref{appendix:proof_rad_bound}) and in \cref{remark:rad_bound_A_rho} below. Finally, let $C_{\cW, \cB, \Omega}$ denote a constant depending on $\cW$, $\cB$, and $\Omega$, and recall that $d$ is the input dimension. 

\begin{lemma} \label{lemma:rad_bound}
    With the assumptions/notation as stated above: for any $\delta_1, \dots, \delta_{L-1} > 0$, we have
    \begin{align} \label{eq:rad_bound_nonhomogeneous_delta}
        \cR_N(\cB_L) \leq 2 C_{\cW, \cB,\Omega} \Pi_L \sqrt{\frac{L \log 2 + \log(d+1) + \sum_{\ell=1}^{L-1} \log (1+A_\ell/\delta_\ell)}{N}} + \sum_{\ell=1}^{L-1} \delta_{\ell} \pi_\ell.
    \end{align}
    For example, choosing $\delta_\ell := 1/(\sqrt{N(L-1)} \max\{1,\pi_\ell\})$ yields 
    \begin{align} \label{eq:rad_bound_nonhomogeneous}
        \cR_N(\cB_L) \leq 2 C_{\cW, \cB,\Omega} \Pi_L \sqrt{\frac{L \log 2 + \log(d+1) + \sum_{\ell=1}^{L-1} \log (1+ A_\ell \sqrt{N(L-1)} \max\{1, \pi_\ell\} )}{N}} + \sqrt{\frac{L-1}{N}}.
    \end{align}
    Furthermore, if $\sigma$ is homogeneous of any degree, we have the improved bound
    \begin{equation} \label{eq:rad_bound_homogeneous}
        \cR_N(\cB_L) \leq 2 C_{\cW, \cB, \Omega} \Pi_L \sqrt{\frac{L \log 2 + \log(d+1)}{N}}.
    \end{equation}
\end{lemma}
\begin{remark} \label{remark:rad_bound_A_rho}
    All of the non-homogeneous activations $\sigma$ in \cref{tab:activation_summary}, except SELU and bent identity, have $A_\ell \leq A_1$ and $\rho_\ell \leq \rho_1$ for all $\ell$ (see proof of \cref{lemma:rad_bound} in \cref{appendix:proof_rad_bound}). Therefore, in all but these two cases, the term
    \begin{align}
        \sum_{\ell=1}^{L-1} \log (1+A_\ell \sqrt{N(L-1)} \max\{ 1, \pi_\ell \})
    \end{align}
    in the numerator of \eqref{eq:rad_bound_nonhomogeneous} is further upper bounded by
    \begin{align} 
        (L-1) \log (1 + A_1 \sqrt{N(L-1)})
    \end{align}
    if $\rho_1 \leq 1$, and by
    \begin{align} 
        (L-1) \log (1 + A_1 \sqrt{N(L-1)} \rho_1^L ) &\leq (L-1) \log \left( \left(1 + A_1 \sqrt{N(L-1)} \right) \rho_1^L \right) \label{eq:rho_1_L_geq_1_rad} \\
        &\leq (L-1) \left( \log \left(1 + A_1 \sqrt{N(L-1)} \right) + L \log \rho_1 \right)
    \end{align}
    if $\rho_1 > 1$. (The inequality in \eqref{eq:rho_1_L_geq_1_rad} uses $\rho_1^L \geq 1$.) Therefore, in almost all cases, the explicit dependence of the Rademacher bound \eqref{eq:rad_bound_nonhomogeneous} on depth $L$ is---up to constant and polylogarithmic factors---no more than $\sqrt{L}$ if $\rho_1 \leq 1$,  or $L$ if $\rho_1 > 1$.
\end{remark}

The proof of \cref{lemma:rad_bound}, presented in \cref{appendix:proof_rad_bound}, is strongly inspired by \cite{golowich2018size}, but requires an important modification to adapt to our normalized-activation setup. In particular, at each layer, the infinite family of normalized activations $\sigma_s$, $s > 0$ must be discretized and each possible $\sigma_s$ in this family must be approximated by some element from this discretization. The constants $A_\ell$ arise from this discretization step. For homogeneous activations, we have $\sigma_s = \sigma$ for all $s$, so this discretization procedure is not necessary; as a result, the final homogeneous bound \eqref{eq:rad_bound_homogeneous}---which recovers that of \cite{golowich2018size}---improves upon the non-homogeneous bound \eqref{eq:rad_bound_nonhomogeneous} by completely removing the summation term inside the square root. Nonetheless, both bounds are favorable in many key respects. In either case, the $N$-dependence in both cases is no more than $N^{-1/2} \, \mathrm{polylog}(N)$, and the \textit{explicit} depth-dependence is no more than $L \, \mathrm{polylog}(L)$, and in many cases no more than $\sqrt{L} \, \mathrm{polylog}(L)$ or $\sqrt{L}$. The Lipschitz-product factor $\Pi_L$ is a possible additional source of depth-dependence. However, many of the activation functions we consider in \cref{tab:activation_summary}---including ReLU, LeakyReLU (for canonical parameter choice $0 < \alpha \leq 1$), logistic sigmoid, tanh, and various others---are globally 1-Lipschitz, in which case both the $\Pi_L$ and $\pi_\ell$ terms can be taken as 1 and thus contribute no depth-dependence whatsoever. The GELU, SiLU/Swish, and Mish activations are globally Lipschitz with constants slightly larger than one: in these cases, the $\Pi_L$ prefactor may be exponential in $L$, but with a very small base.

\subsection{Metric entropies of the classes $\cB_L$} \label{sec:metric_entropy}
The next complexity measure that we consider is \textit{metric entropy}, which reflects how many small balls (in some ambient normed space $\cF$) are needed to cover a set $\cS \subset \cF$. This provides a natural indicator of the complexity of $\cS$, as measured with respect to the ambient geometry of $\cF$. More precisely, define the \textit{$\epsilon$-covering number} of $\cS$ as
\begin{equation} \label{eq:covering_number}
    \cN(\cS, \epsilon, \cF) := \inf \left\{ N \in \bN: \exists f_1, \dots, f_N \in \cF \ \textrm{s.t.} \ \cS \subset \bigcup_{n=1}^N \left\{ g \in \cF: \| f_n - g \|_{\cF} \leq \epsilon \right\} \right\}.
\end{equation}
In words, $\cN(\cS, \epsilon, \cF)$ is the infimal number of $\cF$-norm balls of radius $\epsilon$ needed to completely cover $\cF$. The metric entropy of $\cS$ is the logarithm of its $\epsilon$-covering number.

Before presenting the entropy bounds on our classes $\cB_L$, we discuss some important conceptual differences between metric entropy and Rademacher complexity. The motivation for bounding the metric entropy of the classes $\cB_L$, rather than stopping with the Rademacher bound in \cref{lemma:rad_bound}, is that metric entropy is a purely geometric notion of complexity which reflects the structure imposed by the ambient $\cF$ norm. In contrast, Rademacher complexity is a finite-sample metric which can fail to capture certain infinite-dimensional function space characteristics, such as highly localized or oscillatory structure. For instance, it is straightforward to construct classes of steep localized bump functions whose Rademacher complexity is small, while their metric entropy in the ambient function-space norm is arbitrarily large; see \cref{appendix:rad_entropy_example} for one such example and further discussion. 

This distinction between these two complexity measures manifests itself in their respective statistical uses. The metric entropy of a function class $\cS$ can often be used to bound the \textit{minimax risk} of $\cS$, defined as
\begin{align} \label{eq:minimax_risk}
    \inf_{\hat{f}} \sup_{f \in \cS} \bE_{\{ X_i, \xi_i \}_{i=1}^N} \left\| \hat{f} \left(\{ X_i, Y_i \}_{i=1}^N \right) - f \right\|_{\cF}.
\end{align}
Here, $\{ X_i, \xi_i \}_{i=1}^N$ are random data/noise variables, and the infimum is over all possible estimation rules $\hat{f}$ which construct an estimator $\hat{f} \left(\{ X_i, Y_i \}_{i=1}^N\right)$ based on $N$ observations $\{ Y_i := f(X_i) + \xi_i \}_{i=1}^N$ of a function $f$. On the other hand, a Rademacher complexity bound on $\cS$ is more directly relevant for deriving high-probability bounds on the \textit{generalization error} 
\begin{align} \label{eq:gen_error}
    \sup_{f \in \cS} \left| R(f,\cL, P) - R_N(f, \cL, P) \right|,
\end{align}
associated with the class $\cS$, together with a loss function $\cL$ and a data distribution $P$. Here,
\begin{align}
    R(f, \cL, P) := \bE_{X,Y \sim P} \left[ \cL \left(Y, f(X)  \right) \right] 
\end{align}
is the \textit{population risk}, and 
\begin{align}
    R_N(f, \cL, P) := \frac{1}{N} \sum_{i=1}^N \cL \left( Y_i, f(X_i) \right), \qquad \{ X_i, Y_i \}_{i=1}^N \sim P \ \mathrm{i.i.d.}
\end{align}
is the \textit{empirical risk}. The minimax risk \eqref{eq:minimax_risk} and generalization error \eqref{eq:gen_error} measure fundamentally different qualities of the class $\cS$: \eqref{eq:minimax_risk} measures how difficult it is to reconstruct a function in $\cS$ given noisy observations, whereas \eqref{eq:gen_error} measures consistency of data-fitting ability of $\cS$ across samples. As such, Rademacher complexity and metric entropy are distinct but complementary complexity measures: broadly speaking, Rademacher complexity captures finite sample learnability of a function class, while metric entropy is a more direct indicator of the function-space complexity of that class, with implications for function estimation/reconstruction. 

Despite these important conceptual differences, however, Rademacher complexity and metric entropy are related notions, and bounds on one can be converted into bounds on the other in certain situations, depending both on the structural characteristics of the $\cF$-norm and the regularity of the underlying function class $\cS$. As we will now show, this is possible for our classes $\cB_L$. Using the Rademacher bound \cref{lemma:rad_bound} along with several additional estimation steps, we derive the following bounds on the $L^p$ metric entropies of our classes $\cB_L$ with respect to arbitrary finite measures $\mu$ on the domain $\Omega$.
 
\begin{theorem} \label{th:entropy_bound}
    Let $\mu$ be a finite measure on $\Omega$, and suppose that the assumptions of \cref{lemma:rad_bound} hold. Then for all $0 < \epsilon \leq 1/2$, we have
    \begin{align} \label{eq:entropy_bounds_th}
        \log \cN(\cB_L, \epsilon, L^p(\mu)) \lesssim \begin{cases}
            \Pi_L^2 T_L(\epsilon) \epsilon^{-2}, &1 \leq p \leq 2 \\ 
           (\Pi_L+1)^2 T_L(\epsilon) \epsilon^{-2} d^2 \log^2 \left( (\Pi_L + 1) \epsilon^{-1} \right)  , &2 < p \leq \infty, 
        \end{cases}
    \end{align}
    where
    \begin{align} \label{eq:T_eps}
        T_L(\epsilon) := \begin{dcases}
            L + \log d + \sum_{\ell=1}^{L-1} \log \left(1 + \frac{L A_\ell \max\{ 1, \pi_\ell \}}{\epsilon} \right), &\textrm{$\sigma$ is non-homogeneous,} \\
            L + \log d, &\textrm{$\sigma$ is homogeneous (of any degree).}
        \end{dcases}
    \end{align}
    The $\lesssim$ in \eqref{eq:entropy_bounds_th} above hides multiplicative constants which are independent of $\epsilon$, $d$, and the depth-dependent quantities $L$, $C_L$, $A_L$, $\Pi_L$, $\rho_L$, and $\pi_1, \dots, \pi_L$. (See proof for exact bounds with explicit constants.)
\end{theorem}
\begin{remark}
    As in \cref{remark:rad_bound_A_rho}, all non-homogeneous activations $\sigma$ in \cref{tab:activation_summary}, except for SELU and bent identity, satisfy $A_\ell \leq A_1$ and $\rho_\ell \leq \rho_1$ for all $\ell$. Therefore, in these cases, $T_L(\epsilon)$ is upper bounded as
    \begin{align} 
        T_L(\epsilon) \leq L + \log d + (L-1) \log \left(1 + \frac{L A_1}{\epsilon} \right)
    \end{align}
    if $\rho_1 \leq 1$, and as
    \begin{align} 
        T_L(\epsilon) &\leq 
        L + \log d + (L-1) \log \left(1 + \frac{L A_1 \rho_1^L}{\epsilon} \right) \leq L + \log d + (L-1) \log \left( \left( 1 + \frac{L A_1 }{\epsilon} \right) \rho_1^L \right) \label{eq:T_L_eps_ineq} \\
        &= L +  \log d + (L-1) \left( \log \left( 1 + \frac{L A_1 }{\epsilon} \right) + L \log \rho_1 \right) 
    \end{align}
    if $\rho_1 > 1$. (The second inequality in \eqref{eq:T_L_eps_ineq} uses the fact that $\rho_1^L > 1$.)
\end{remark}
The proof of \cref{th:entropy_bound} is in \cref{appendix:proof_entropy_bound}. Note that, in the context of \cref{th:entropy_bound}, the classes $\cB_L$ can be understood as closed with respect to either $C(\Omega)$ or $L^p(\mu)$. This follows from \cref{prop:Lp_C_representative}, which shows that any $L^p(\mu)$ cover of the $C(\Omega)$-closed class $\cB_L^\infty$ is also an $L^p(\mu)$ cover of the $L^p(\mu)$-closed class $\cB_L^{L^p(\mu)}$, and vice versa. We also note that, in the case of the $\ReLU$ activation for $p=2$, \cref{th:entropy_bound} recovers the bound of \cite{barron2019complexity}.

Like the Rademacher bounds in \cref{lemma:rad_bound}, the metric entropy bounds in \cref{th:entropy_bound} are favorable in many respects: up to constant and polylogarithmic factors, the $\epsilon$-growth is at most $\epsilon^{-2}$, and the \textit{explicit} dependence on depth $L$ is at most linear. As discussed after \cref{lemma:rad_bound} above, the prefactor $\Pi_L$ is no more than 1 for the many globally 1-Lipschitz activations (ReLU, Leaky ReLU, and many sigmoidal activations);  for GELU, SiLU/Swish, and Mish, this prefactor is exponential in $L$, but with base only slightly larger than 1. Therefore, in many cases, the entropies of the classes $\cB_L$ are small, and grow very mildly as depth $L$ increases.

\subsubsection{Tightness of the upper bounds} \label{sec:tightness_upper_bounds}
The best strategy that we are aware of for lower bounding the metric entropies of our deep classes $\cB_L$ is to relate them to known entropy lower bounds for the shallow classes $\cB_2$. Here, we specifically consider the $d$-dimensional $\ReLU^m$ classes $\cB_2^{\ReLU,m}$, defined with $\cW := \bS^{d-1}$ and $\cB := [b_1, b_2]$, where
\begin{align} \label{eq:bias_set_assumption}
    b_1 < \inf_{\vw \in \cW, \vx \in \Omega} \vw^\top \vx < \sup_{\vw \in \cW, \vx \in \Omega} \vw^\top \vx < b_2.
\end{align}
\cite{siegel2024sharp} show that, in this scenario, the entropy of $\cB_2^{\ReLU,m}$ is lower bounded as
\begin{align} \label{eq:siegel_xu_L2_Lebesgue_entropy_LB}
    \log \cN(\cB_2^{\ReLU,m}, \epsilon, L^2(d \vx)) \gtrsim \epsilon^{-\frac{2d}{d+2m+1}}.
\end{align}
Here $d \vx$ denotes the Lebesgue measure on $\R^d$, and $\gtrsim$ hides implicit multiplicative constants (which in this case depend on both $d$ and $m$). This bound can be extended to other values of $p$, and certain other measures $\mu$, as follows:
\begin{proposition} \label{prop:B2_entropy_lower_bounds}
    Let $\mu$ be a finite measure on $\Omega$ such that $\mu(\cA) \geq C_\mu \,d \vx( \cA)$ for all Lebesgue measurable sets $\cA \subset \Omega$, where $C_\mu > 0$ is some constant. Then:
    \begin{align}
        \log \cN(\cB_2^{\ReLU,m}, \epsilon, L^p(\mu)) \gtrsim \begin{cases}
            \epsilon^{-\frac{pd}{d+2m+1}}, &1 \leq p < 2 \\
            \epsilon^{-\frac{2d}{d+2m+1}}, &2 \leq p \leq \infty.
        \end{cases}
    \end{align}
    Here $\gtrsim$ hides implicit multiplicative constants depending on $d$, $m$, $\mu$, $\Omega$, and $p$.
    \end{proposition} 
    \begin{remark}
        The assumption of \cref{prop:B2_entropy_lower_bounds} is satisfied by any measure $\mu$ on $\Omega$ of the form $\mu(\cA) = \int_{\cA} w(\vx) \, d \vx$, where the density function $w \in L^1(d \vx)$ satisfies $w(\vx) \geq C_\mu > 0$ for Lebesgue-a.e. $\vx \in \Omega$. For example, if $\mu$ is any measure on $\R^n$ with a continuous and strictly positive density function $w$ (such as a Gaussian, Laplacian, Cauchy, logistic, etc.), then $w(\vx) \geq C_\mu > 0$ for all $\vx$ in the compact set $\Omega$: as a result, the restricted measure $\mu \rvert_\Omega( \cA) := \mu(\cA \cap \Omega)$ fulfills the stated requirement.
    \end{remark}
    The proof of \cref{prop:B2_entropy_lower_bounds} is in \cref{appendix:B2_entropy_lower_bounds}. The next proposition shows that, for the ReLU activation, these exact shallow bounds extend to the deep classes $\cB_L^{\ReLU}$.
    \begin{proposition} \label{prop:relu_nested_containment_entropy_lb}
        The $\ReLU$ classes $\cB_L^{\ReLU}$ satisfy the nested containment relationship
        \begin{align} \label{eq:relu_nested_containment}
            \cB_2^{\ReLU} \subset \cB_3^{\ReLU} \subset \cB_4^{\ReLU} \subset \dots .
        \end{align}
        In conjunction with \cref{prop:B2_entropy_lower_bounds}, this implies that
        \begin{align} \label{eq:relu_deep_entropy_lower_bound}
            \log \cN(\cB_L^{\ReLU}, \epsilon, L^p(\mu)) \gtrsim \begin{cases}
                \epsilon^{-\frac{pd}{d+3}}, &1 \leq p < 2 \\
                \epsilon^{-\frac{2d}{d+3}}, &2 \leq p \leq \infty
            \end{cases}
        \end{align}
        for any $L \geq 2$, any $1 \leq p \leq \infty$, and any $\mu$ satisfying the assumption of \cref{prop:B2_entropy_lower_bounds}. Here, $\gtrsim$ hides implicit constants which are dependent on $d$, $\mu$, $\Omega$, and $p$, but are independent of depth $L$.
    \end{proposition}
    \begin{remark}
        By \cref{lemma:relu_nonrelu_space_equivalence}, the entropy lower bound \eqref{eq:relu_deep_entropy_lower_bound} also applies to the deep classes $\cB_L$ associated with the Leaky ReLU, GELU, SiLU/Swish, Mish, ELU, SELU, CELU, centered softplus, absolute value, and bent identity activation functions. However, in these cases, the implicit multiplicative constant may depend on depth $L$ (see \cref{tab:activation_equivalence_constants}).
    \end{remark}
    The proof of \cref{prop:relu_nested_containment_entropy_lb} is in \cref{appendix:proof_relu_nested_containment_entropy_lb}. The metric entropies of the deep $\ReLU^m$ classes can be lower bounded in a similar manner:

\begin{proposition} \label{prop:relu_m_nested_containment_entropy_lb}
Under assumption \eqref{eq:bias_set_assumption}, the deep $\ReLU^m$ classes obey the scaled containment relationship
\begin{align} \label{eq:relu_m_nested_containment}
    \cB_2^{\ReLU,m} \subset C^m \cB_3^{\ReLU,m} \subset C^{m^2+m} \cB_4^{\ReLU,m} \subset \dots \subset C^{\sum_{\ell=1}^{L-2} m^\ell} \cB_L^{\ReLU,m}
\end{align}
for some constant $C > 0$. In conjunction with \cref{prop:B2_entropy_lower_bounds}, this implies that
\begin{align}
    \log \cN(\cB_L^{\ReLU,m}, \epsilon, L^p(\mu))
    \gtrsim
    \begin{cases}
        \displaystyle
        \left(
            C^{-\sum_{\ell=1}^{L-2}m^\ell}\epsilon^{-1}
        \right)^{\frac{pd}{d+2m+1}},
        & 1 \leq p < 2,
        \\[8pt]
        \displaystyle
        \left(
            C^{-\sum_{\ell=1}^{L-2}m^\ell}\epsilon^{-1}
        \right)^{\frac{2d}{d+2m+1}},
        & 2 \leq p \leq \infty.
    \end{cases}
\end{align}
for any $L \geq 2$, any $1 \leq p \leq \infty$, and any $\mu$ satisfying the assumption of \cref{prop:B2_entropy_lower_bounds}. Here, $\gtrsim$ hides implicit constants which are dependent on $d$, $m$, $\mu$, $\Omega$, and $p$.
\end{proposition}
\begin{remark}
    It is always possible to force the constant in \cref{prop:relu_m_nested_containment_entropy_lb} to satisfy $C > 1$ by taking the bias set $\cB$ to be large enough. See the proof of \cref{prop:relu_m_nested_containment_entropy_lb} in \cref{appendix:proof_relu_m_nested_containment_entropy_lb} for more detail.
\end{remark}
\begin{remark}
    Another viable strategy for extending the shallow entropy lower bounds to the deep classes $\cB_L$ is to consider a slightly modified definition of $\cB_L$, which corresponds to a ResNet-style architecture with penalized residual (skip) connections between hidden layers. These alternative ResNet classes exhibit the useful property that $\cB_2 \subset \cB_3 \subset \cB_4 \subset \dots$ for any activation $\sigma$ (meaning that $\cB_L$ automatically inherits the best known entropy lower bound for any of the classes $\cB_{L'}, L' \leq L$) and they enjoy nearly-identical Rademacher complexity and metric entropy upper bounds to those in \cref{lemma:rad_bound,th:entropy_bound}. More details on this setup, which may be of independent interest, are discussed in \cref{appendix:resnet_classes}.
\end{remark}

Clearly, there is a gap between the main $\epsilon^{-2}$ rate in the upper bound in \cref{th:entropy_bound} and the $\epsilon^{-\frac{2d}{d+2m+1}}$ rate in the lower bound of \cref{prop:relu_nested_containment_entropy_lb,prop:relu_m_nested_containment_entropy_lb}, although if the input dimension $d$ is large, this gap is relatively small. Obtaining tight (or at least tighter) bounds on the metric entropies of the classes $\cB_L$ is an interesting open problem which appears to be highly nontrivial. The major difficulty is that our understanding of what kinds of functions are in the classes $\cB_L$, beyond those in the shallow class $\cB_2$, remains extremely limited. For the $\ReLU$ activation, there are several special examples of functions known to be in $\cB_3$ but not in $\cB_2$. For example, $\ell^1$ pyramid functions of the form $\vx \mapsto (a - \| \vx - c \|_1)_+$ are in $\cB_3$ for appropriate choices of $a$ and $c$ (see \cref{appendix:pyramid_relu_rep}), but are known to have infinite two-layer variation norm $\cV_2$ (\cite{ongie2019function}) whenever the input dimension $d > 1$. It is still unknown what other kinds of functions (if any) are in $\cB_3 \setminus \cB_2$, or more generally in $\cB_L \setminus \cB_{L'}$ for any $3 \leq L' < L$. Therefore, in order to tighten the entropy gap between the upper bound in \cref{th:entropy_bound} and the inherited lower bounds in \cref{prop:relu_nested_containment_entropy_lb,prop:relu_m_nested_containment_entropy_lb}, it may be necessary to first acquire a firmer understanding of exactly what types of functions are contained in the classes $\cB_L$ for depths $L > 2$, and how these differ from those in $\cB_2$. In the next section, we will provide a partial answer to this question in the univariate case $d=1$ for the ReLU activation function. We will show that, in this scenario, depth provides no additional benefit to functional expressivity besides dilation by a small, depth-independent constant. As a result, the classes at every depth $\cB_L$ exhibit completely depth-independent bounds on their function-space complexities. Furthermore, because we have already shown that the spaces corresponding to the ReLU activation are equivalent to those corresponding to many other activation functions (\cref{lemma:relu_nonrelu_space_equivalence}), this ``depth-saturation'' result also applies to any of these equivalent activation functions, although in those cases the constant dilation factor may grow with depth. As we will see, these univariate results also imply that the functions in the multivariate classes $\cB_L$ must exhibit low frequencies along all one-dimensional subspaces.

\section{Depth saturation of univariate neural function classes} \label{sec:uni_depth_saturation}

Throughout this section, we consider the variation spaces $\cV_L$ associated with the $\ReLU$ activation $\sigma(\cdot):= (\cdot)_+$ on the domain $\Omega := [-1,1]$, with the base linear dictionary $\cB_1$ constructed by taking $\cW = \cB := [-1,1]$. The infinite-width limits are taken in the ambient Banach space $\cF := C[-1,1]$. Let us motivate and provide some intuition for this setup.\footnote{The main results of this section can be adapted to many other choices of univariate $\Omega$, $\cW$, and $\cB$, with different constants depending on these quantities appearing in the resulting bounds in \cref{lemma:uni_f_V2_bound,lemma:I_f_plus_bound,th:univariate_containment}.} With these definitions of $\cW$ and $\cB$, the atoms of the dictionary $\{ (f)_+: f \in \cB_1 \} $ from which $\cB_2$ is synthesized are right- or left-facing ReLU ``hinge'' functions (i.e., ReLU neurons) of the form $x \mapsto (wx+b)_+$. Any such neuron has slope magnitude $|w| \leq 1$ and, when $w \neq 0$, activation threshold $-b/w$. Finite linear combinations $x \mapsto \sum_k v_k (w_k^\top x + b_k)_+$ of such ReLU neurons---i.e., two-layer finite-width ReLU neural networks in $\widetilde{\cB}_2$---are continuous piecewise linear (CPWL) functions with $\cV_2$ norm no more than $\sum_k |v_k|$. Any individual neuron in such a network whose activation threshold $-b_k/w_k$ lies in the interior of the domain $\Omega$ creates a kink in the function, where the slope changes by the corresponding $v_k |w_k|$. On the other hand, any neuron whose activation threshold $-b_k/w_k$ lies at or outside the boundary of $\Omega$ does not create a kink in the function on $\Omega$, but instead acts as a purely affine function on $\Omega$. Therefore, to represent any CPWL function $f$ on $\Omega$ as a shallow ReLU network of this form, some amount of variation norm budget must be allocated to represent one of the affine pieces of $f$ with ReLU neurons whose activations are at or outside the boundary of $\Omega$, and the remaining neurons are allocated to represent the kinks in the function. This shows that, for finite-width shallow ReLU networks, the $\cV_2$ variation norm is controlled by the sum of absolute slope changes (equivalently, the total variation of the second derivative) and by the slope and intercepts of $f$ on its individual affine pieces. 

The above characterization can be extended rigorously to the complete classes $\cB_2$ (which includes the finite-width networks in $\widetilde{\cB}_2$ and their infinite-width limits) by relating $\cV_2$ to the classical \textit{bounded variation} space
\begin{align}
    \BV(-1,1) := \{ f \in L^1(-1,1): D f \in \cM(-1,1) \}.
\end{align}
Here $Df$ is the \textit{distributional derivative}\footnote{See the beginning of \cref{appendix:proof_uni_f_V2_bound} for a brief review of distributions and distributional/weak derivatives.} of $f$ and $\cM(-1,1)$ is the space of signed Radon measures on $(-1,1)$ with finite total variation norm $\| \cdot \|_{\TV}$. Roughly speaking, the distributional derivative is an extension of the standard (classical) derivative which is applicable to certain functions that are not smooth enough to be globally differentiable in the classical sense, such as CPWL functions. The total variation norm $\| \cdot \|_{\TV}$ of the distributional derivative $Df$ measures how much this derivative varies across the domain $[-1,1]$. For a CPWL function, this is exactly equivalent to the sum of absolute slope changes of the function. This and similar characterizations have been explored in detail a number of previous works (\cite{savarese2019infinite,siegel2023characterization,wojtowytsch2022representation}). For completeness, we give a self-contained proof of the following fact, which we will use for our subsequent depth-saturation result, in \cref{appendix:proof_uni_f_V2_bound}.

\begin{lemma} \label{lemma:uni_f_V2_bound}
    Suppose that $f \in C[-1,1]$ has $D^2 f \in \cM(-1,1)$. Then $f$ has a weak derivative (see \cref{appendix:proof_uni_f_V2_bound}) $f' \in \BV(-1,1)$ which admits one-side limits $f'(-1^{+}) := \lim_{x \downarrow -1} f'(x)$ and $f'(1^{-}) := \lim_{x \uparrow 1} f'(x)$. Any such $f$ obeys
    \begin{align} \label{eq:uni_f_V2_bound}
        \| D^2 f \|_{\TV} \leq \| f \|_{\cV_2} \leq \max\{ |f(-1)+f'(-1^{+})|, |f'(-1^{+})|  \} + \| D^2 f \|_{\TV}
    \end{align}
    where $\cV_2$ is defined with $\sigma(\cdot) := (\cdot)_+$ and $\Omega := \cW := \cB := [-1,1]$.
\end{lemma}
To interpret \cref{lemma:uni_f_V2_bound}, it is again instructive to consider the the case where $f$ is a CPWL function on $[-1,1]$. Any kink in $f$ can only be represented by placing one or more ReLU neurons of the form $x \mapsto v(wx+b)_+$ at that point. The slope change of the represented function at that point is no more than the sum of the values $v_k |w_k|$ across all neurons $k$ which activate at that point. Therefore, the sum of absolute slope changes of the represented function is upper bounded by $\sum_k |v_k w_k| \leq \sum_k |v_k|$, which yields the first inequality in \eqref{eq:uni_f_V2_bound}. The second inequality holds because it is always possible to represent any CPWL $f$ by representing the first affine piece $x \mapsto f(-1) + f'(-1^+)(x+1)$ with two ReLU neurons, incurring $\cV_2$ norm penalty of $\max\{ |f(-1)+f'(-1^+)|, |f'(-1^+)| \}$ (see proof in \cref{appendix:proof_uni_f_V2_bound}), and then placing one right-facing unit-slope neuron at each kink, each incurring norm penalty equal to the slope change at that kink. The proof of \cref{lemma:uni_f_V2_bound} extends this reasoning to the more general case of continuous functions with finite variation of the second derivative, which may require uncountably many neurons to represent in this fashion.

Importantly, \cref{lemma:uni_f_V2_bound} also highlights the fact that composing a function with a ReLU can increase its $\cV_2$ norm. This is because the composition $(f)_+$ (which we will simply denote as $f_+$) has the effect of flattening $f$ on its negative regions. This flattening can introduce additional kinks at points where $f$ goes from positive to negative or vice versa, and these kinks may incur additional total variation in the second derivative; as a result, there are functions $f \in \cB_2$ whose positive parts $f_+$ are not in $\cB_2$. (See \cref{remark:f_alpha,fig:f_alpha_positive_parts} for one such example.) However, it is also apparent in this example that the increase in the variation of the second derivative which can be incurred by composition with ReLU is inherently limited by the slopes of the function around where the kink in $f_+$ is introduced. In particular, as we will now show in \cref{lemma:I_f_plus_bound}, we can design a functional $I$ which is ``similar'' to $\| \cdot \|_{\cV_2}$ (in that it depends on the values of $f$ and $f'$ near the boundary of $[-1,1]$ and on the variation of $D^2 f$), and in particular is an upper bound on $\| \cdot \|_{\cV_2}$; unlike $\| \cdot \|_{\cV_2}$, however, the value of $I$ is not increased by composition with ReLU. This fact will subsequently be of key importance for our main result of this section.
\begin{lemma} \label{lemma:I_f_plus_bound}
    Suppose that $f \in C[-1,1]$ has $D^2 f \in \cM(-1,1)$. Then $D^2 f_+ \in \cM(-1,1)$, and the functional
    \begin{align}
        I(f) := \max\{ |f(-1)+f'(-1^{+})|, |f'(-1^{+})| \} + \max\{|f(1)-f'(1^{-})|, |f'(1^{-})|\} + \| D^2 f \|_{\TV}
    \end{align}
    satisfies $I(f_+) \leq I(f)$.
\end{lemma}
The proof is in \cref{appendix:proof_I_f_plus_bound}. Using \cref{lemma:uni_f_V2_bound,lemma:I_f_plus_bound}, it is straightforward to prove the main ``depth saturation'' result of this section:

\begin{theorem} \label{th:univariate_containment}
    The function classes $\cB_L$ associated with the ReLU activation $\sigma(\cdot) := (\cdot)_+$ in $d = 1$, with $\Omega := \cW := \cB := [-1,1]$, satisfy
    \begin{align} \label{eq:relu_ball_uni_nesting}
        \cB_2 \subset \cB_L \subset 2 \cB_2
    \end{align}
    for all $L \geq 2$. As a consequence:
    \begin{align}
        \| f \|_{\cV_L} \leq \| f \|_{\cV_2} \leq 2 \| f \|_{\cV_L}
    \end{align}
    for any $f$ and any $L \geq 2$.
\end{theorem}
\begin{remark} \label{remark:f_alpha}
    The constant $2$ in \cref{th:univariate_containment} cannot be improved. To see this, consider the functions 
    \begin{align} \label{eq:f_alpha}
        f_\alpha(x) := \alpha (x+1/2)_+ - (1-\alpha) (x - 1/2)_+ 
    \end{align}
    for $0 < \alpha < 1$. These functions $f_\alpha$ are in $\cB_2$, but their positive parts have
    \begin{align}
        \| (f_\alpha)_+ \|_{\cV_2} \geq \| D^2 (f_\alpha)_+ \|_{\TV} = |\alpha| + |1| + |1-\alpha| \to 2
    \end{align}
    as $\alpha \to 0$. (See \cref{fig:f_alpha_positive_parts}.) Therefore, there is no constant $C < 2$ such that $\{ (f_\alpha)_+ \}_{\alpha \in (0,1)} \subset \cB_3 \subset C \cB_2$.
\end{remark}
\begin{figure}
    \centering
    \includegraphics[width=0.32\textwidth]{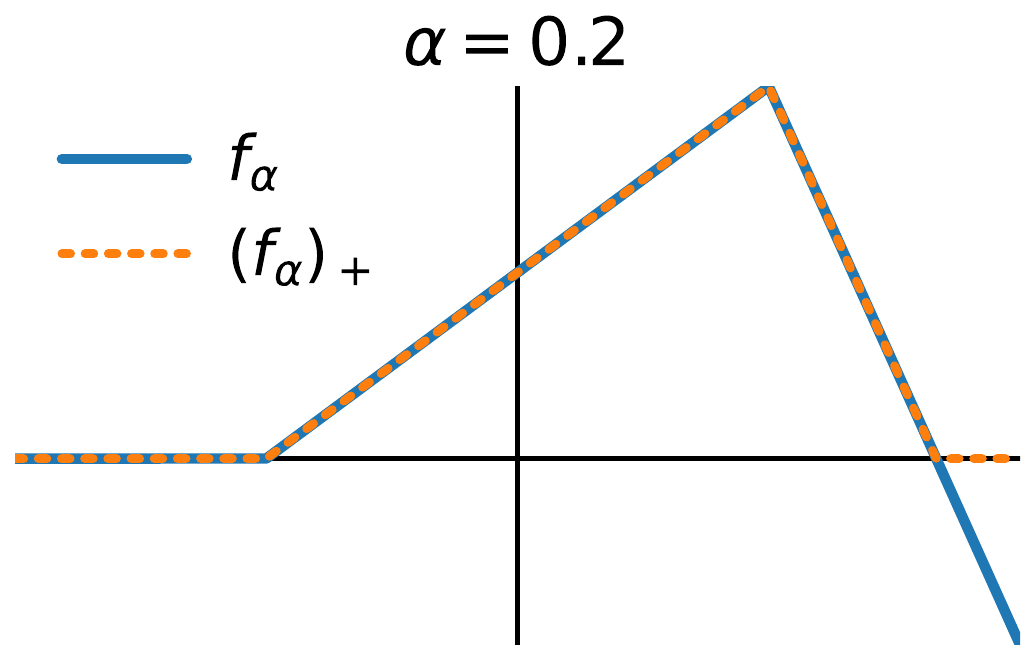}
    \hfill
    \includegraphics[width=0.32\textwidth]{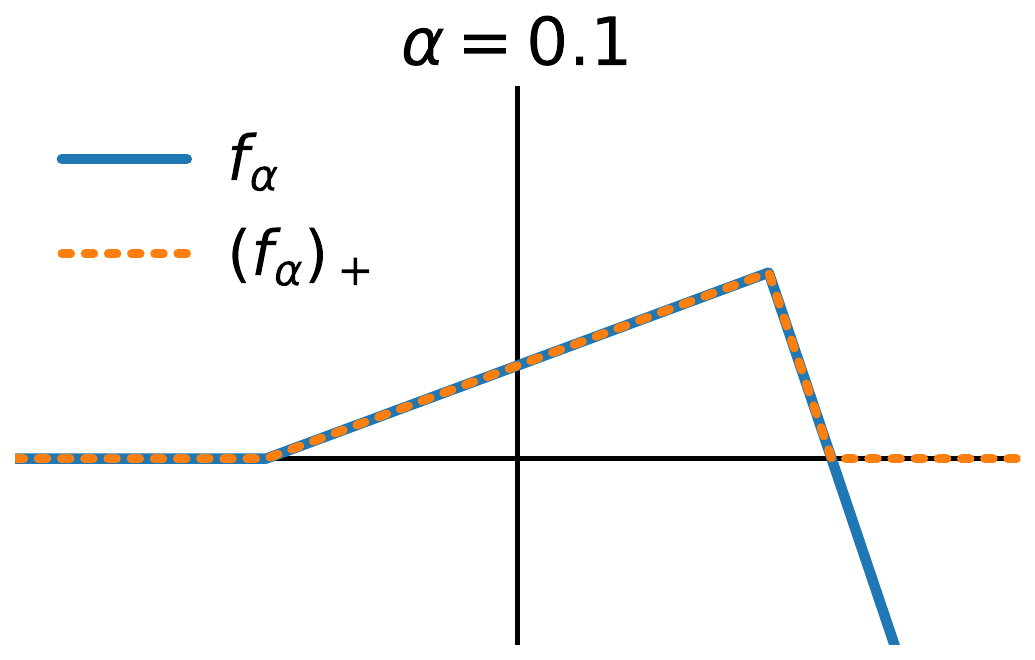}
    \hfill
    \includegraphics[width=0.32\textwidth]{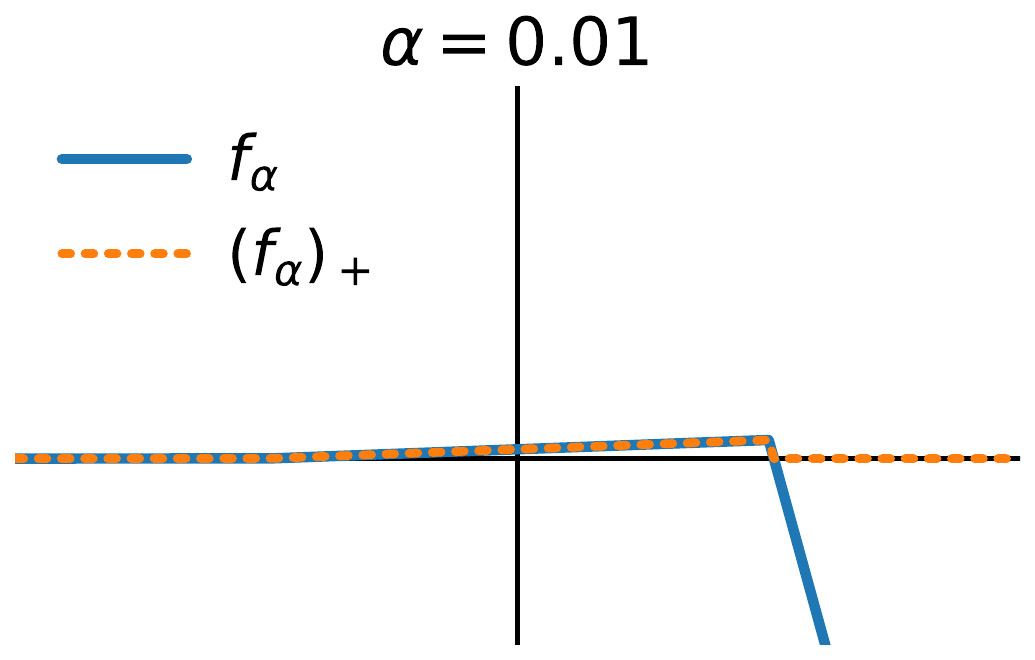}

    \caption{The functions $f_\alpha$ (as defined in \eqref{eq:f_alpha}) and $(f_\alpha)_+$ for
    $\alpha=0.2$, $\alpha=0.1$, and $\alpha=0.01$, respectively. Each $f_\alpha$ is in $\cB_2$, but as $\alpha \downarrow 0$, the total variation of $D^2 (f_\alpha)_+$ approaches 2. Therefore, by \cref{lemma:uni_f_V2_bound},  $(f_\alpha)_+ \notin \cB_2$ for all sufficiently small $\alpha$. }
    \label{fig:f_alpha_positive_parts}
\end{figure}
\begin{remark}
    By \cref{lemma:relu_nonrelu_space_equivalence}, a weaker version of \cref{th:univariate_containment} also applies to the Leaky ReLU, GELU, SiLU/Swish, Mish, ELU, SELU, CELU, centered softplus, absolute value, and bent identity activation functions. In particular, with $\cB_L^\sigma$ and $\cV_L^\sigma$ defined on $\Omega = \cW = \cB = [-1,1]$ for any of the aforementioned activations $\sigma$, \cref{lemma:relu_nonrelu_space_equivalence} and \cref{th:univariate_containment} imply that there are constants $\widetilde{A}_{L, \sigma}$, $\widetilde{B}_{L, \sigma}$ such that 
    \begin{align}
        \widetilde{A}_{L, \sigma} \cB_2^\sigma \subset \cB_L^\sigma \subset \widetilde{B}_{L, \sigma} \cB_2^\sigma
    \end{align}
    and therefore
    \begin{align}
        \widetilde{A}_{L, \sigma} \| f \|_{\cV_L^\sigma} \leq \| f \|_{\cV_2^\sigma} \leq \widetilde{B}_{L, \sigma} \| f \|_{\cV_L^\sigma}
    \end{align}
    for any $L \geq 2$. In other words, the deep univariate spaces associated with all of these activations are equivalent to the corresponding shallow spaces. However, in these cases, the constants $\widetilde{A}_{L, \sigma}$, $\widetilde{B}_{L, \sigma}$ in this equivalence relation may depend on depth $L$ (see \cref{tab:activation_equivalence_constants}).
\end{remark}

The proof, presented in \cref{appendix:proof_univariate_containment}, proceeds by showing that the base linear class $\cB_1$ has $\sup_{f \in \cB_1} I(f) \leq 2$, and then using \cref{lemma:I_f_plus_bound} to argue that this bound is inherited by the subsequent classes $\cB_L, L \geq 1$. In combination with \cref{lemma:uni_f_V2_bound}, this implies that $\| f \|_{\cV_2} \leq I(f) \leq 2$ for all $f \in \cB_L, L \geq 1$. \cref{th:univariate_containment} shows that, in the univariate ReLU case, the deep ReLU classes $\cB_L$ contain no ``fundamentally'' different functions outside of those in $\cB_2$: any function in the deep class $\cB_L$ is either in the shallow class $\cB_2$, or can be obtained by taking a function in $\cB_2$ and scaling it by at most a factor of 2. Therefore, in this case, the effect of depth on functional expressivity is extremely limited. 

As an interesting point of comparison, consider the sawtooth functions
\begin{align} \label{eq:telgarsky_sawtooth}
    T(x) := 2(x)_+ - 4(x-1/2)_+ + 2(x-1)_+, \qquad
    T^{\circ L} := \underbrace{T \circ T \circ \dots \circ T}_{L \textrm{ times}}
\end{align}
studied in \cite{telgarsky2016benefits}. The function $T^{\circ L}$ is a sawtooth function with $2^{L/2}$ teeth on $x \in [0,1]$, each of equal width and height 1 (see \cref{fig:telgarsky_sawtooth}). Although these sawtooth functions are often cited as an example of the representational benefits of depth, our \cref{th:univariate_containment} shows that this benefit disappears if representational efficiency is measured in terms of \textit{norm} rather than neuron count. The function $T^{\circ L}$ has $2^L-1$ total knots on the interior of $(0,1)$: therefore, it would require a width of at least $2^L-1$ to represent with a shallow ReLU network, but it can instead be represented compositionally with a depth $L$ ReLU network of width 3 at each hidden layer. However, \cref{lemma:uni_f_V2_bound,th:univariate_containment} imply that, for any depth $L' \geq 2$, the norm required to represent $T^{\circ L}$ must grow exponentially in $L$. In particular:
\begin{align} \label{eq:sawtooth_norm_exponential}
    \| T^{\circ L} \|_{\cV_{L'}} \geq \frac{1}{2} \| D^2 T^{\circ L} \|_{\cV_{2}} \geq \frac{1}{2}  \| D^2 T^{\circ L} \|_{\TV} = (2^L - 1)2^{L}.
\end{align}
Therefore, although depth permits representations of $T^{\circ L}$ with small numbers of neurons, the weights of these neurons (as measured by our function-space norm) must be large. From this perspective, there is little or no benefit from depth.

The reason for this apparent discrepancy is that the highly oscillatory behavior of $T^{\circ L}$ depends crucially on the exact scaling of the base triangle function $T$. For example, if the base triangle $T$ is rescaled to $aT$ for some $0 \leq a \leq 1/2$, then the $L$-fold composition is $(aT)^{\circ L} = a(2a)^{L-1} T$, so the composition in this case does nothing but rescale the original function $T$ by the constant $a(2a)^{L-1} \leq a \leq 1/2$. Indeed, \eqref{eq:sawtooth_norm_exponential} shows that none of the functions $T^{\circ L}$, $L \geq 1$ are in our classes $\cB_{L'}$ for any depth $L' \geq 2$. The scaled function $aT$ for any $a \leq 1/6$ is in our class $\cB_2$, and the $L$-fold composition $(aT)^{\circ L}$ is in $\cB_{L+1}$, but if the base function $T$ is rescaled in this way, the composition $(aT)^{\circ L}$ no longer exhibits any oscillatory behavior: it is merely a scaled-down version of the original function $T$. (See \cref{fig:telgarsky_sawtooth}.) As discussed in \cref{sec:relationship_other_spaces}, our norm inevitably enforces this scaling behavior because it is a true norm on the function output: constant rescaling greater than 1 at each hidden layer accumulates an exponential-in-depth norm cost of the final function. In this sense, our norm separates the effect of the nonlinearity itself from the compounded effects of constant rescaling across layers: functions in the norm-ball $\cB_L$ may depend on the iterative compositional effects of the nonlinear activation $\sigma$, but cannot in any fundamental sense depend on the effect of compounded constant dilation across layers.

\begin{figure}
    \centering
    \includegraphics[width=0.24\textwidth]{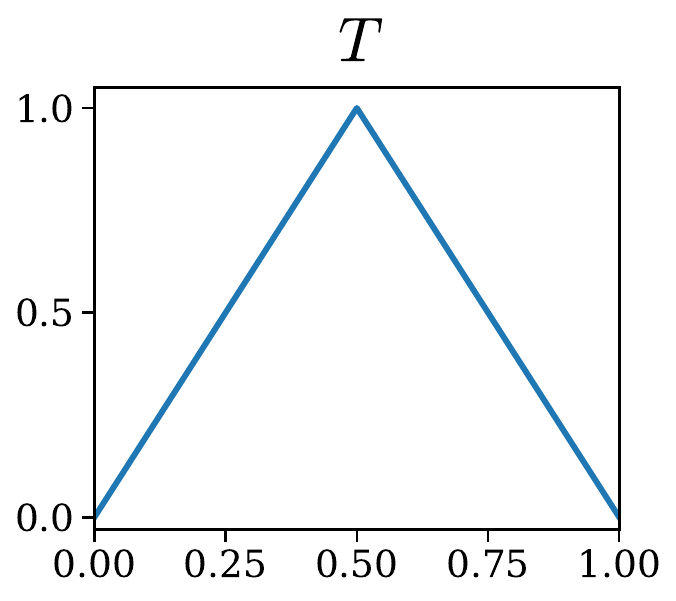}
    \hfill
    \includegraphics[width=0.24\textwidth]{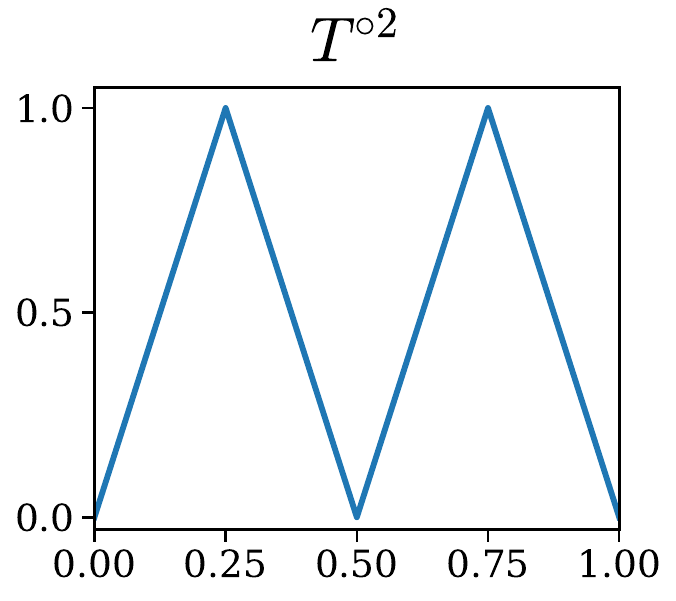}
    \hfill
    \includegraphics[width=0.24\textwidth]{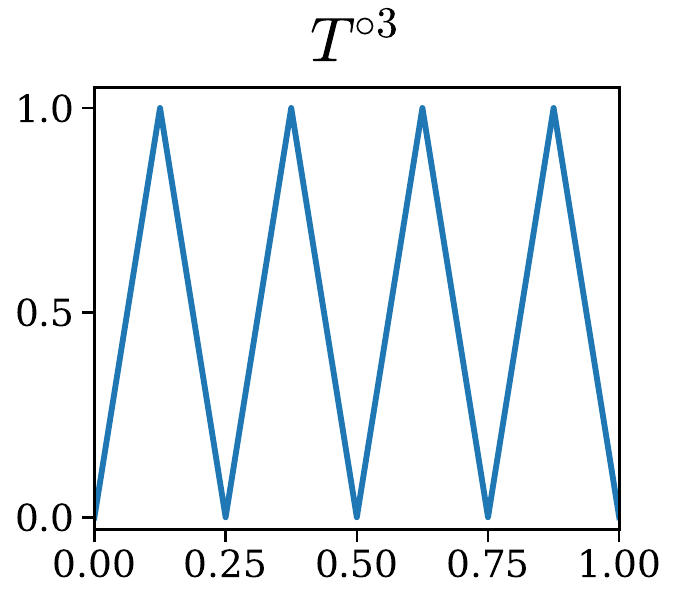}
    \hfill
    \includegraphics[width=0.24\textwidth]{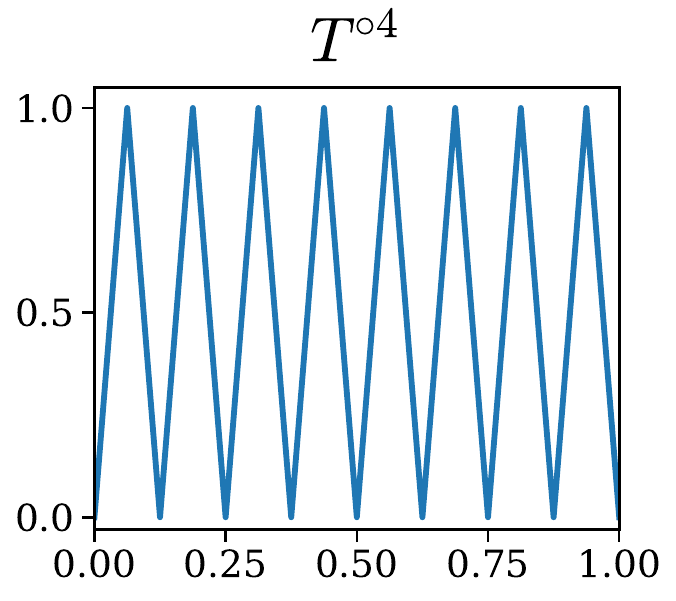}

    \par\medskip  \includegraphics[width=0.24\textwidth]{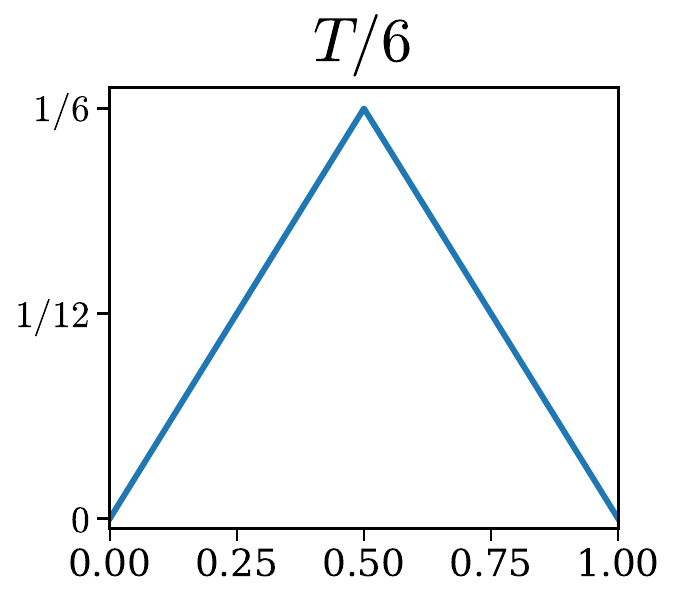} \hfill \includegraphics[width=0.24\textwidth]{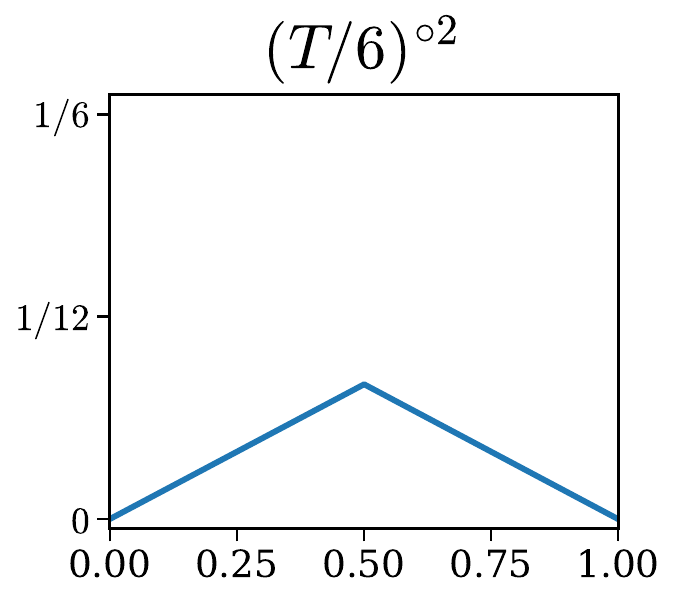} \hfill \includegraphics[width=0.24\textwidth]{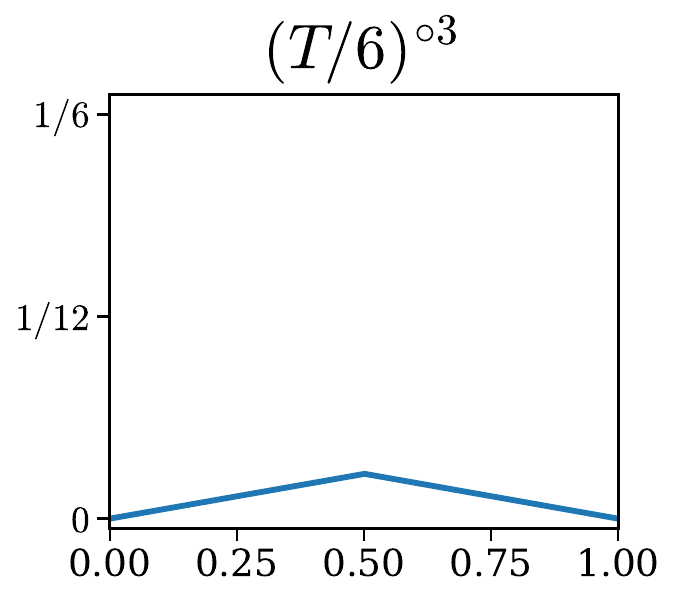} \hfill \includegraphics[width=0.24\textwidth]{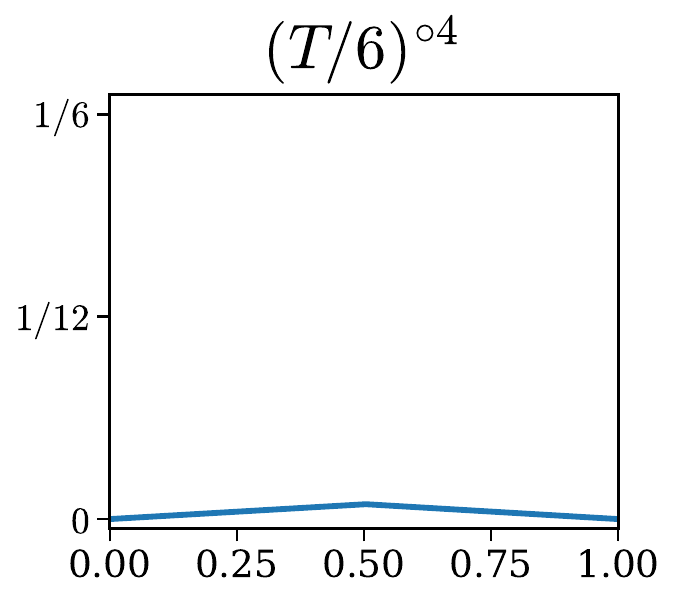}
    \caption{Top row: iterated compositions $T^{\circ L}$ of the Telgarsky sawtooth function \eqref{eq:telgarsky_sawtooth}. The $L$-fold composition $T^{\circ L}$ has $2^{L/2}$ ``teeth'' and $2^L-1$ total kinks. However, this compositional oscillatory behavior is highly dependent on the scaling of $T$. Bottom row: if $T$ is rescaled to $T/6$ (or to $aT$ for any $0 \leq a < 1/2$), iterated composition produces no additional oscillations, but instead progressively shrinks the original function.}
    \label{fig:telgarsky_sawtooth}
\end{figure}

These univariate results also have interesting implications for the more general \textit{multivariate} ($d \geq 1$) case, wherein they imply that no functions in the multivariate classes $\cB_L$ can have excessively high frequencies along any individual direction. This is a consequence of the following simple observation: if $f$ is a function in the multivariate class $\cB_L$, its restriction to any given line must lie in a corresponding univariate class $\cB_L^{\textrm{uni}}$.

\begin{proposition} \label{prop:line_restriction}
    For any activation function and any input dimension ($d \geq 1$), consider the class $\cB_L$ defined on any compact $\Omega \subset \R^d$, $\cW \subset \R^d$, and $\cB \subset \R$, with infinite-width limits taken in $C(\Omega)$. Let $\gamma(t) = \boldsymbol{\xi}_0 + \boldsymbol{\xi} t$ be a line in $\R^d$ which intersects $\Omega$ in a connected line segment. Then for any $f \in \cB_L$, the restriction $(f \circ \gamma)(t) := f(\gamma(t))$ must be in the univariate class $\cB_L^{\mathrm{uni}}$ corresponding to the same activation function, defined with
    \begin{align}
        \Omega^{\mathrm{uni}} &:= [a,b], \  \mathrm{where} \ a := \inf\{ t: \gamma(t) \in \Omega \}, \ b := \sup \{ t: \gamma(t) \in \Omega \}  \\
        \cW^{\mathrm{uni}} &:= \{ \vw^\top \boldsymbol{\xi}: \vw \in \cW \} \\
        \cB^{\mathrm{uni}} &:= \{ \vw^\top \boldsymbol{\xi}_0+ b: \vw \in \cW, b \in \cB \}
    \end{align}
\end{proposition} 
\begin{proof}
    For any functions $f_1, \dots, f_K$ and any coefficients $v_1, \dots, v_K \in \R$, we have
    \begin{align} \label{eq:aconv_restrict}
        \left( \sum_{k=1}^K v_k (\sigma \circ f_k) \right) \circ \gamma = \sum_{k=1}^K v_k (\sigma \circ (f_k \circ \gamma)).
    \end{align}
    Therefore, for any $f \in \widetilde{\cB}_L$ of the form \eqref{eq:deep_nn_sum}, the restriction $f \circ \gamma$ is equivalent to
    \begin{align} \label{eq:deep_nn_sum_restricted}
        (f \circ \gamma)(t) &= \sum_{k_{L-1}=1}^{K_{L-1}} w_{k_{L-1}}^{(L)} \sigma \left(  \dots \sum_{k_1=1}^{K_1} W_{k_2, k_1}^{(2)} \sigma \left( ( \vw_{k_1}^{(1)} )^\top \boldsymbol{\xi} t + ( \vw_{k_1}^{(1)} )^\top \boldsymbol{\xi_0} + b_{k_1}^{(1)}  \right) \dots  \right),
    \end{align}
    which is in the class $\cB_L^{\mathrm{uni}}$ as defined in the statement. Restriction to a line also commutes with uniform limits, so the statement also holds for functions $f$ in the uniform closure $\cB_L$ of $\widetilde{\cB}_L$.
\end{proof}
\begin{remark}
    \eqref{eq:aconv_restrict} is a generic fact which follows trivially from the definition of function composition: it is not specific to any particular functions $f_k$, $\sigma$, or $\gamma$. Linearity of $\gamma$ is only required for \cref{prop:line_restriction} because the base dictionary $\cB_1$ is itself linear, so restrictions $f \circ \gamma$ of functions $f$ in the multivariate linear dictionary $\cB_1$ are themselves in the univariate linear dictionary $\cB_1^{\textrm{uni}}$ for appropriate choices of $\Omega^{\textrm{uni}}$, $\cW^{\textrm{uni}}$, and $\cB^{\textrm{uni}}$. 
\end{remark}

In light of \cref{lemma:uni_f_V2_bound,th:univariate_containment}, \cref{prop:line_restriction} implies that functions in the multivariate classes $\cB_L$ must exhibit controlled variational behavior along lines in all directions.

\begin{corollary}[of \cref{prop:line_restriction,lemma:uni_f_V2_bound,th:univariate_containment}] \label{corr:line_restriction_TV2}
    Consider the multivariate ReLU classes $\cB_L$ defined on the domain $\Omega = \bB_2^d := \{ \vx \in \R^d: \| \vx \|_2 \leq 1 \}$, the Euclidean unit ball in $\R^d$, with $\cW = \bB_2^d$ (or $\bS^{d-1}$, the Euclidean unit sphere) and $\cB = [-1,1]$. Let $\gamma(t) = \vx t, t \in [-1,1]$ be a line passing through the origin, parameterized with unit direction vector $\| \vx \|_2 = 1$. Suppose that the restriction $f \circ \gamma$ of some function $f: \bB_2^d \to \R$ to the line $\gamma$ has
    \begin{align} \label{eq:D2f_gamma_TV2_bound}
        \| D^2 (f \circ \gamma) \|_{\TV} > 2.
    \end{align}
    Then $f \notin \cB_L$, for any $L \geq 1$.
\end{corollary}
\begin{proof}
    If such an $f$ were in some $\cB_L$ for some $L \geq 1$, then by \cref{prop:line_restriction}, its restriction $f \circ \gamma$ to $\gamma$ would be in the univariate ReLU class $\cB_L^{\mathrm{uni}}$ defined on $\Omega^\mathrm{uni} = [-1,1]$ with $\cW^\mathrm{uni} = \cB^\mathrm{uni} = [-1,1]$. By \cref{th:univariate_containment}, this would imply that $f \circ \gamma \in 2 \cB_2^{\mathrm{uni}}$, which in turn implies by \cref{lemma:uni_f_V2_bound} that
    \begin{align}
        \| D^2 (f \circ \gamma) \|_{\TV} \leq \| f \circ \gamma \|_{\cV_2^{\mathrm{uni}}} \leq 2.
    \end{align}
    But this contradicts the assumption that $\| D^2 (f \circ \gamma) \|_{\TV} > 2$.
\end{proof}

\begin{figure}
    \centering
    \includegraphics[width=0.5\textwidth]{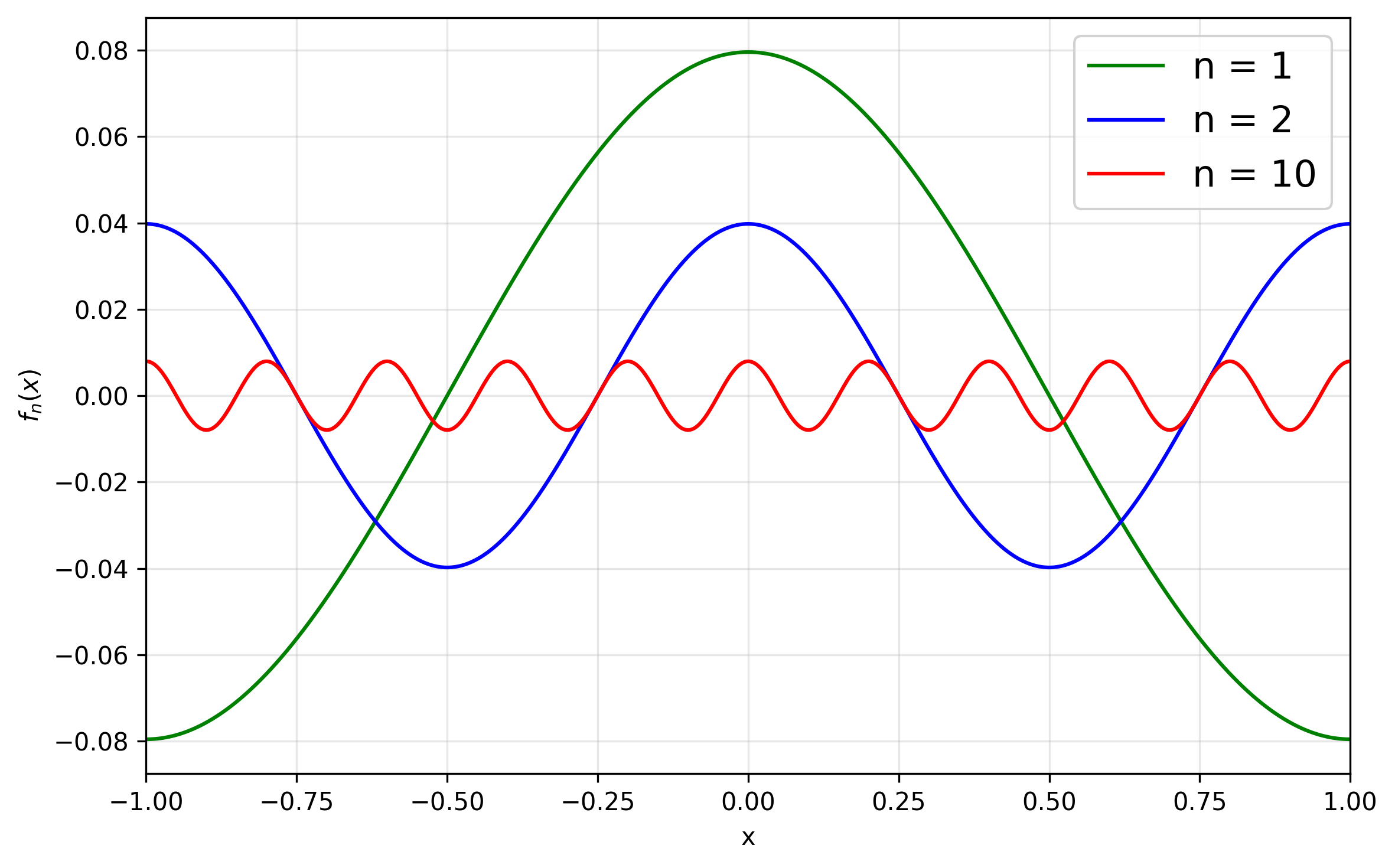} 
    \caption{The functions $f_n(x) = \frac{1}{4 \pi n} \cos(\pi n x)$ for $n = 1, 2, 10$. These functions have $\| D^2 f_n \|_{\TV} = -\int_{-1}^1 |f_n''(x)| \, dx = n$, so by \cref{lemma:uni_f_V2_bound} and \cref{th:univariate_containment}, they are excluded (for $n > 1$) from the univariate ReLU classes $\cB_L^{\mathrm{uni}}$ at all depths $L$.}
    \label{fig:high_frequency_uni_example}
\end{figure}

\begin{figure}
    \centering
    \includegraphics[width=0.32\textwidth]{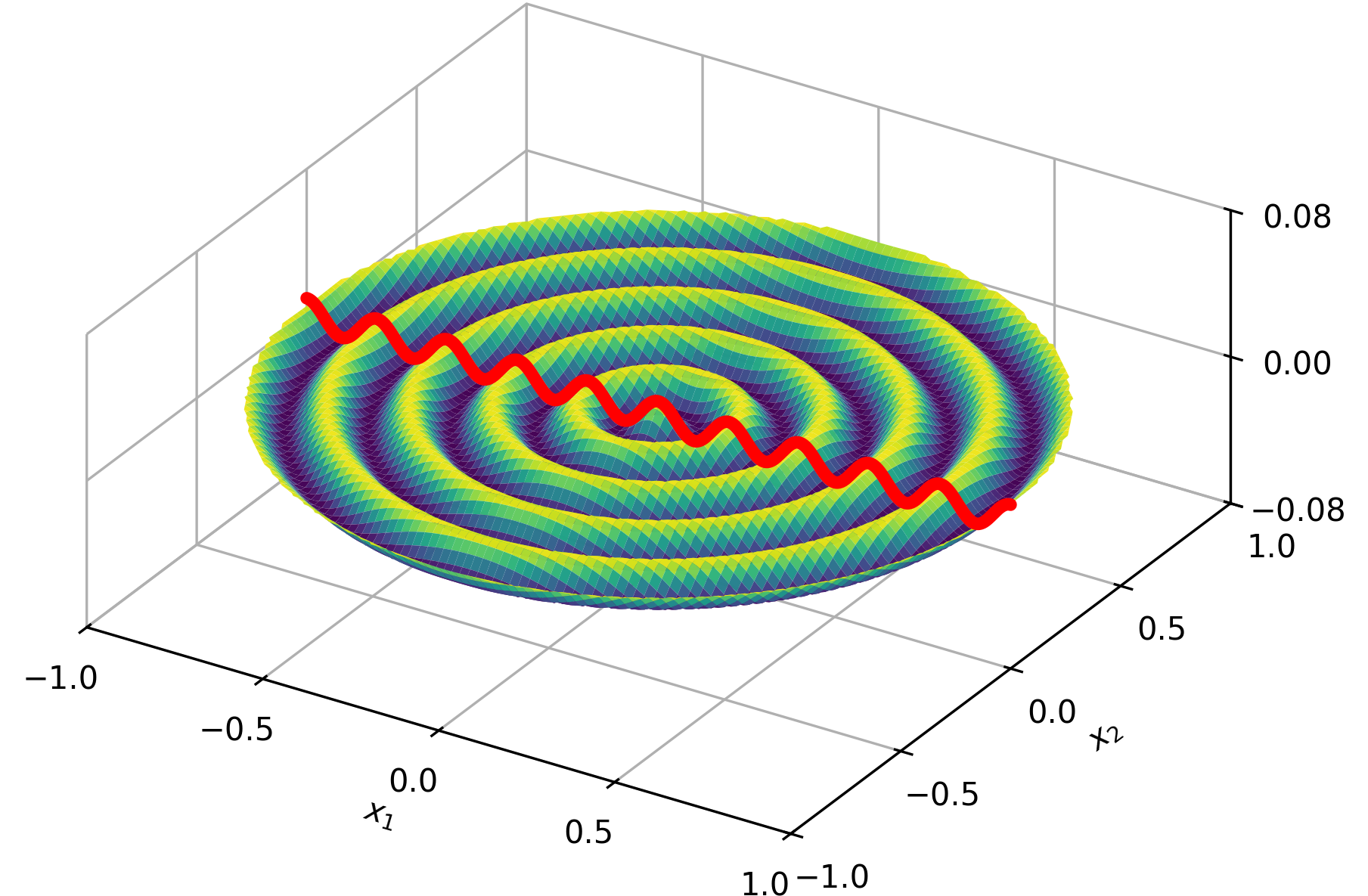} \includegraphics[width=0.32\textwidth]{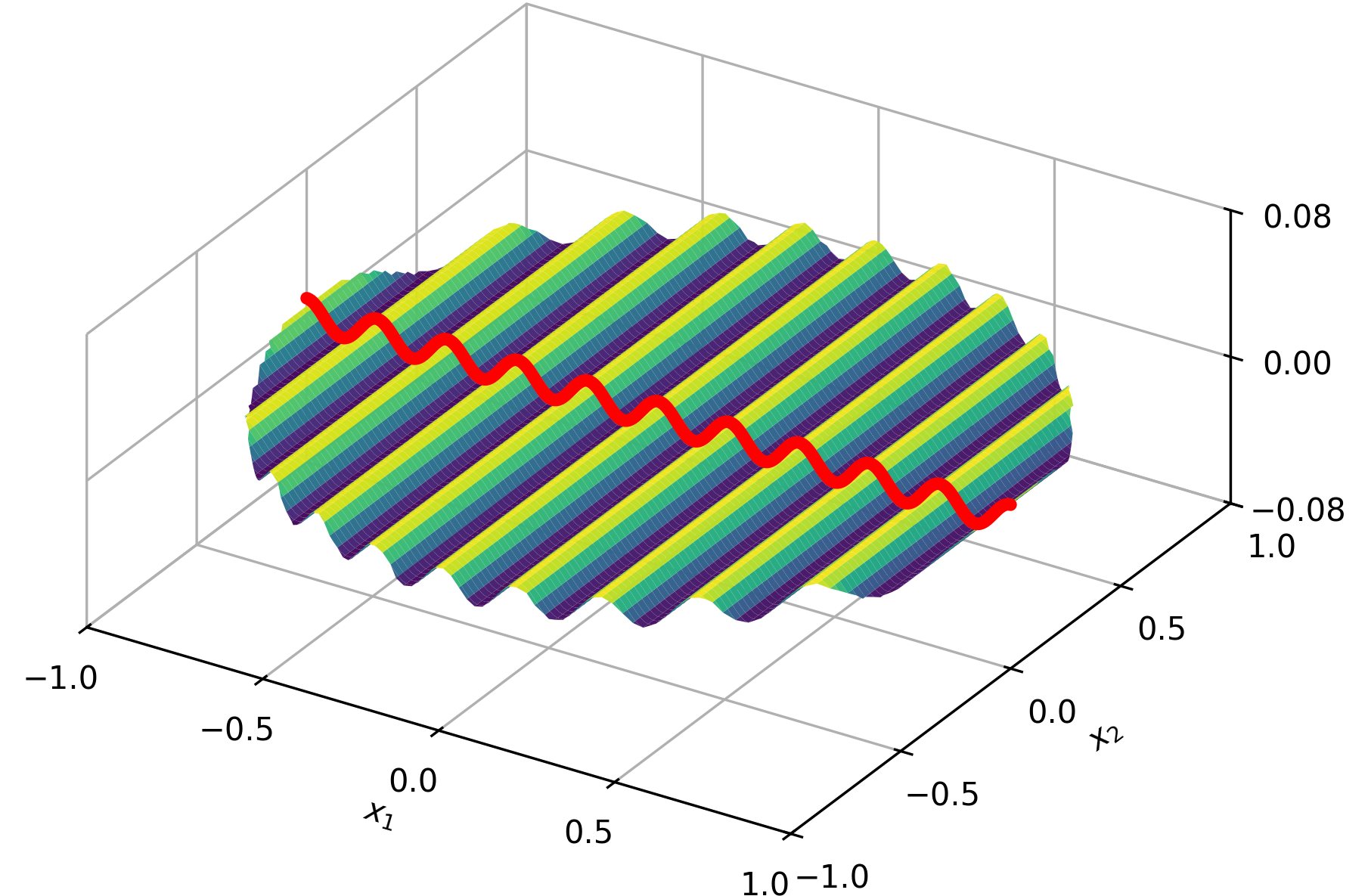} \includegraphics[width=0.32\textwidth]{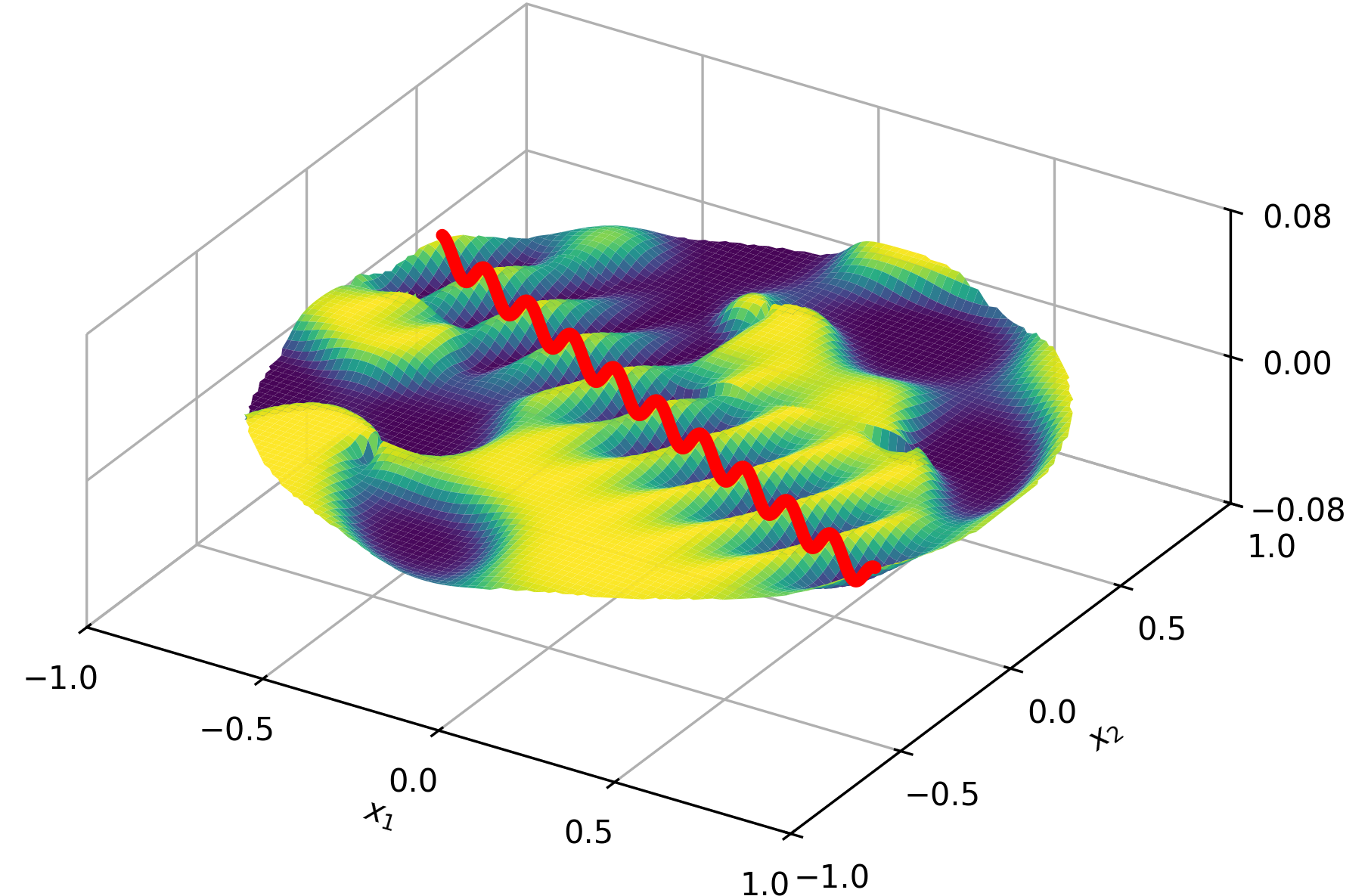}
    \caption{Example functions in $d=2$ whose restriction to some line in the input domain $\bB_2^2$ agrees with the highly oscillatory function $f_n$ ($n = 10$) of \cref{fig:high_frequency_uni_example}. By \cref{corr:line_restriction_TV2}, all such functions are excluded from the corresponding multivariate ReLU classes $\cB_L$.}
    \label{fig:multivar_high_frequency_lines_example}
\end{figure}

\cref{corr:line_restriction_TV2} shows that multivariate functions which exhibit excessively high frequencies along any direction in the input domain are excluded from our deep ReLU classes $\cB_L$, at any depth $L$. (See examples in \cref{fig:high_frequency_uni_example,fig:multivar_high_frequency_lines_example}.) This result is perhaps surprising, because many existing results concerning the benefits of depth specifically focus on high-frequency functions (\cite{telgarsky2016benefits,eldan2016power,daniely2017depth,perekrestenko2018universal, chatziafratisdepth, bresler2020sharp,venturi2022depth}, among others). The general flavor of all of these results is that deep networks can represent or approximate functions which exhibit high-frequencies (such as high Fourier modes or similar oscillatory behavior) more efficiently---i.e., using fewer neurons or smaller widths---than shallow networks. \cite{parkinson2024depth} show that several of these results are also valid if representational cost is measured by the network's sum-of-squared-weights (SOSW) representational cost (see \eqref{eq:sosw}), rather than widths or parameter counts. In particular, if a function requires $\exp(d)$ width to approximate with two layers but can be approximated by a three-layer network of $\poly(d)$ width, \cite{parkinson2024depth} (Corollary 3.3) show that the same function requires $\mathrm{SOSW} = \exp(d)$ to approximate with two layers, but can be approximated with three layers with $\mathrm{SOSW} = \poly(d)$. However, as discussed above with regard to the Telgarsky sawtooth function \eqref{eq:telgarsky_sawtooth}, these results depend crucially on the hidden effect of compounded constant rescaling across layers. The SOSW representation cost ``disguises'' the effect of compounded layerwise rescaling, in the sense that dilation of all parameters by a constant factor incurs only a $\poly(L)$ increase in SOSW, even though it inflates the represented function by a factor of $\exp(L)$ (see discussion in \cref{sec:relationship_other_spaces}). Width and parameter count, of course, do not account for these compounded rescaling effects at all. Therefore, \cref{corr:line_restriction_TV2} highlights the fact that many of these existing results rely on representational cost/complexity measures that do not fully account for the function-space effects of layerwise amplification. When representational cost is measured by a true function space norm such as our $\cV_L$ norm, depth alone does not allow for cheap representation of highly oscillatory functions.

\section{Conclusion, discussion, and open problems}

In this work, we developed a unified function-space framework for studying the relationship between depth and functional complexity in fully connected neural networks. Our construction applies to general continuous activation functions, including many non-homogeneous activations commonly used in practice, through an activation-dependent family of normalized nonlinearities. For homogeneous activations, and in particular for ReLU, the resulting norm coincides with the path norm studied in previous works. We derived novel complexity bounds for the unit balls $\cB_L$ of our function spaces, which grow only mildly with depth in most cases. We also related our norm to several other existing notions of deep-network norms and representational costs; as a consequence, we find that the complexity bounds for our classes $\cB_L$ also apply to the function classes associated with many of these other costs. Additionally, we prove a novel representer theorem for our spaces $\cB_L$, demonstrating that solutions to infinite-dimensional data-fitting problems over our spaces are realized by width-bounded neural networks.

Our function space perspective also allowed us to derive several novel results concerning the structural effects of depth. In particular, we proved that the unit balls $\cB_L$ of the univariate ReLU function spaces satisfy the containment relationship $\cB_2 \subset \cB_L \subset 2 \cB_2$ at all depths $L$. This shows that the effects of depth are extremely limited in this case: depth introduces no new functional shapes or structures beyond those achievable by simple rescaling of the shallow class $\cB_2$. As a corollary of this result, we prove that functions in the \textit{multivariate} deep ReLU classes $\cB_L$ must exhibit low frequencies along every line in the input domain. These conclusions highlight the fact that many existing results on the benefits of depth for functional expressivity---which typically focus on representation of functions with high oscillatory frequency---depend strongly on the chosen notion of representational cost with respect to which representational efficiency is measured. When this efficiency is measured by a true function space norm like ours, compounded rescaling across hidden layers is reflected in the norm cost of the represented function. Therefore, our results indicate that, when this rescaling effect is accounted for by true function-space norm control, depth itself (i.e., the repeated application of the nonlinear activation function) has a much more limited structural effect.

Our analyses in this paper leave many interesting open questions, which we review below.

\subsection{Characterizing the deep balls and spaces}

The most important open problem raised by our results is to characterize the functions contained in the classes $\cB_L$ and spaces $\cV_L$ for depths $L\geq 3$. This problem is also the main obstruction to obtaining sharper metric-entropy lower bounds. Our present lower bounds in \cref{sec:tightness_upper_bounds} are inherited from the shallow class $\cB_2$, but it is not clear how much additional complexity is contributed by functions lying outside the shallow class. Understanding this additional structure appears to require substantially more information about the geometry of the deep balls than we have currently.

In the univariate ReLU setting, our depth-saturation result in \cref{th:univariate_containment} makes an important stride in this direction; however, even in this comparatively simple case, difficult open questions remain concerning characterization of the balls $\cB_L$. \cref{remark:f_alpha} demonstrates example functions $f \in \cB_2$ whose positive parts satisfy $f_+ \in \cB_3 \setminus \cB_2$. Thus $\cB_2\subsetneq \cB_3$, despite the $\cV_2=\cV_3$ of the underlying spaces. It remains unknown, however, whether there is any function in $\cB_4 \setminus \cB_3$, or more generally in $\cB_L \setminus \cB_{L'}$ for some $3 \leq L' < L$. Our present conjecture is that, in this univariate ReLU setting, the unit balls stabilize after depth $L=3$; i.e., that $\cB_3 = \cB_4 = \cB_5 = \dots$. Proving or disproving this conjecture appears nontrivial, despite the comparatively simple one-dimensional structure of the functions involved. The multivariate ReLU case is qualitatively different. The pyramid functions discussed in \cref{sec:tightness_upper_bounds,appendix:pyramid_relu_rep} constitute explicit functions in $\cB_3$ whose $\cV_2$-norm is infinite whenever $d>1$. Therefore, in the multivariate case, not only are the balls $\cB_2$ and $\cB_3$ genuinely distinct, but so are the entire spaces $\cV_2 \subsetneq \cV_3$. Therefore, in this setting, depth introduces genuinely new functions to the associated space, rather than merely changing the geometry of the unit ball within a fixed space. Beyond these special examples, however, very little is known about the structure of $\cV_3\setminus \cV_2$. It would be valuable to identify broader geometric or analytic families of functions that belong to $\cV_3$ but not $\cV_2$, and ultimately to find an intrinsic characterization of the three-layer space.

For depths beyond three, the situation is even less understood. We do not know whether there exist functions in $\cB_L\setminus \cB_{L'}$, or in $\cV_L\setminus \cV_{L'}$, for any $3\leq L'<L$. Several qualitatively different possibilities remain open. The balls and spaces might stabilize after some finite depth; they might form a strict hierarchy at every depth; or they might continue to increase while approaching a nontrivial limiting infinite-depth class or space. Even distinguishing among these possibilities would constitute meaningful progress toward understanding the functional role of depth. It would also likely lead to sharper statistical-complexity bounds, since any genuinely new structural family appearing at a deeper layer could potentially be used to construct depth-dependent entropy lower bounds. The following general criterion gives one possible point of entry into this question.

\begin{proposition}
\label{prop:depth-stabilization}
For any activation function, any input dimension $d \geq 1$, any choice of $\Omega, \cW, \cB$, and any $L\geq 2$, we have
\begin{align}
    \cB_{L+1} \subset \cB_L \quad \textrm{if and only if} \quad \sigma_s \circ f \in \cB_L, \quad \forall f \in \cB_L, s > 0.
\end{align}
In particular, for the ReLU activation:
\begin{align}
    \cB_{L+1} = \cB_L \quad \textrm{if and only if} \quad f_+ \in \cB_L, \quad \forall f \in \cB_L.
\end{align}
\end{proposition}
\begin{proof}
    We necessarily have $\sigma_s \circ f \in \cB_{L+1}$ for every $f \in \cB_L$ and every $s > 0$. Therefore, if $\cB_{L+1} \subset \cB_L$, then $\sigma_s \circ f \in \cB_L$ for every $f \in \cB_L$ and $s > 0$. Conversely, suppose that $\sigma_s \circ f \in \cB_L$ for every $f \in \cB_L$ and $s > 0$. Because $\cB_L$ is convex, it must contain all absolutely convex combinations of such functions $\sigma_s \circ f$, and because it is closed, it must contain all limits of all such absolutely convex combinations. Therefore, $\cB_{L+1}$---which is exactly the set of all such limits of absolutely convex combinations of $\sigma_s \circ f$---must be contained in $\cB_L$. The second statement for the ReLU activation follows from \cref{prop:relu_nested_containment_entropy_lb} and the fact each $\sigma_s$ in this case is equal to the ReLU activation function itself.
\end{proof}

An immediate consequence of this criterion is that, once the ReLU balls stabilize at any depth, they remain stabilized at every subsequent depth. Indeed, if $\cB_{L+1} = \cB_L$, then the recursive definitions of $\cB_{L+2}$ and $\cB_{L+1}$ use the same input ball, and hence $\cB_{L+2}=\cB_{L+1}$. Therefore, the univariate conjecture above reduces to determining whether $f_+ \in \cB_3$ for all $f \in \cB_3$. One possible analytic approach would be to seek a sharper variational characterization of $\cB_3$, analogous to the characterization of $\cV_2$ through second-order variation, that is stable under the positive-part operation. Numerical methods may also be useful for exploring the problem. For example, finite-dimensional discretizations of the atomic norm could be used to search for functions $f\in \cB_3$ for which the estimated $\cV_3$ norm of $f_+$ exceeds one. Such computations would not themselves prove strict containment, but they could reveal candidate counterexamples or suggest additional inequalities needed for a proof of stabilization.

\subsection{Connections with depth separation and exact representation of CPWL functions}
The open question discussed above bears a conceptual resemblance to longstanding difficulties in the classical depth-separation literature. That literature typically seeks functions representable or approximable by polynomial-size networks at one depth but requiring superpolynomial, and in particular cases exponential, size at a smaller depth. Strong constant-depth separations are known primarily between depths two and three (\cite{daniely2017depth,eldan2016power,venturi2022depth,parkinson2024depth}). \cite{vardi2020neural} showed that, under mild regularity assumptions, a separation between a depth-$k$ ReLU network and a deeper constant-depth network for $k\geq 4$ would imply major lower bounds for constant-depth threshold circuits. Their formal complexity-theoretic barrier does not cover separation between depth three and a larger constant depth, although they show that the radial and effectively univariate constructions underlying many existing separation results cannot yield such a separation. Our questions differ in an important respect: rather than bounding finite-network width, we compare norm-bounded, potentially infinite-width function classes and the Banach spaces generated by them. Nevertheless, both settings exhibit the same broad phenomenon that the first nontrivial increase in depth can be demonstrated, while establishing further strict increases is considerably more difficult. It would be interesting to understand whether this similarity reflects a deeper connection or merely the different technical obstacles arising in the two models.

Another potentially useful source of ideas is the literature on exact representation of continuous piecewise-linear functions by ReLU networks. \cite{hertrich2021towards} study the hierarchy of CPWL functions exactly representable by networks of a prescribed depth, without imposing a restriction on network size, using tools from mixed-integer optimization, polyhedral theory, tropical geometry, and combinatorics. This is closely aligned with the qualitative form of our open questions: in both cases, one asks whether adding a layer strictly enlarges a recursively defined class of functions. \cite{hertrich2021towards} conjectured that the depth required to represent all CPWL functions on $\mathbb{R}^d$ is $\lceil\log_2(d+1)\rceil$, and equivalently that appropriate maximum functions require logarithmic depth. More recent work of \cite{bakaev2026better} disproved this precise conjecture and obtained the improved upper bound $\left\lceil \log_3(d-1)\right\rceil+1$. It remains unknown whether a fixed constant depth---possibly even two hidden layers---suffices for all CPWL functions; for example, exact two-hidden-layer representation of the maximum of six inputs remains open. Although these works differ from ours in that they consider exact finite-width representation rather than infinite-width approximation of functions, these techniques may prove valuable in our setting.

Ultimately, a satisfactory characterization of the deep function classes $\cB_L$ and spaces $\cV_L$ would substantially sharpen our understanding of the role of depth in functional expressivity. Such a characterization would clarify whether expressivity saturates after a given depth, or produces an indefinitely growing hierarchy of functions, and how that hierarchy might change between adjacent depths. Answers to these questions would likely also yield tighter bounds on the associated function-space complexities of the classes $\cB_L$ and allow for a more precise description of the inductive biases imposed by depth.

\bibliography{refs}
\bibliographystyle{abbrvnat}

\begin{appendices}
\section{Appendix}

\subsection{Summary of common activation functions} \label{appendix:activation_summary}

\cref{tab:activation_summary} summarizes the relevant properties of a variety of neural network activation functions used in practice. For softplus, logistic sigmoid, and hard sigmoid, we note that the more standard versions used in practice are
\begin{align}
    \textrm{Softplus:} \qquad &\sigma(t) = \log \left( 1 + e^t \right) \\
    \textrm{Logistic sigmoid:} \qquad &\sigma(t) = \frac{1}{1+e^{-t}} \\
    \textrm{Hard sigmoid:} \qquad &\sigma(t) = \max \left\{ 0, \min \left\{ 1, \frac{t+3}{6} \right\} \right\}.
\end{align}
The ``centered'' versions of these which are described in the table are shifted to satisfy $\sigma(0) = 0$. (Indeed, all activations in the table satisfy $\sigma(0) = 0$.) Moreover, all activations listed in \cref{tab:activation_summary} satisfy the following property: for any $R > 0$, the normalized activations $\sigma_s$ converge uniformly on $[-R,R]$ to limiting functions $\sigma_\infty$ (as $s \uparrow \infty$) and $\sigma_0$ (as $s \downarrow 0$). The values of these limiting functions are summarized in the table. Additionally, although $\ReLU^m$ is not globally Lipschitz, it is locally Lipschitz and thus compatible with the complexity bounds in \cref{sec:statistical_complexities}.

\begin{table}[ht]
\centering
\begingroup
\setlength{\tabcolsep}{2pt}
\resizebox{\textwidth}{!}{%
\begin{tabular}{c|c|c|c|c|c}
\hline
\textbf{Activation} 
& \textbf{Homogeneous?}
& \textbf{Formula} $\sigma(t)$
& $\displaystyle \sigma_0(t)$ 
& $\displaystyle \sigma_\infty(t)$
& $\displaystyle \rho_\sigma := \sup_{s > 0} \Lip(\sigma_s)$ \\[1pt]
\hline
ReLU 
& Yes (degree $1$)
& $(t)_+ := \max\{0,t\}$ 
& $(t)_+$ 
& $(t)_+$ 
& $1$ \\[1pt]

$\mathrm{ReLU}^m$ $(m > 1)$
& Yes (degree $m$)
& $(t)_+^m$
& $(t)_+^m$
& $(t)_+^m$
& $\infty$ \\[1pt]

$\mathrm{LeakyReLU}_\alpha$ $(0 < \alpha < 1)$
& Yes (degree $1$)
& $t\mathbbm{1}_{\{t \geq 0\}} + \alpha t\mathbbm{1}_{\{t < 0\}}$
& $t\mathbbm{1}_{\{t \geq 0\}} + \alpha t\mathbbm{1}_{\{t < 0\}}$
& $t\mathbbm{1}_{\{t \geq 0\}} + \alpha t\mathbbm{1}_{\{t < 0\}}$
& $1$ \\[1pt]

Identity 
& Yes (degree $1$)
& $t$ 
& $t$ 
& $t$ 
& $1$ \\[1pt]

Absolute value 
& Yes (degree $1$)
& $|t|$ 
& $|t|$ 
& $|t|$ 
& $1$ \\[1pt]

GELU 
& No
& $t\Phi(t)$
& $\frac{1}{2}t$
& $(t)_+$
& $\approx 1.129$ \\[1pt]

SiLU/Swish 
& No
& $\frac{t}{1+e^{-t}}$
& $\frac{1}{2}t$
& $(t)_+$
& $\approx 1.100$ \\[1pt]

Mish 
& No
& $t\tanh(\log(1+e^t))$
& $\frac{3}{5}t$
& $(t)_+$
& $\approx 1.089$ \\[1pt]

$\mathrm{ELU}_\alpha$ $(0 < \alpha \leq 1)$
& No
& $t\mathbbm{1}_{\{t \geq 0\}} + \alpha(e^t-1)\mathbbm{1}_{\{t < 0\}}$
& $t\mathbbm{1}_{\{t \geq 0\}} + \alpha t\mathbbm{1}_{\{t < 0\}}$
& $(t)_+$
& $1$ \\[1pt]

$\mathrm{SELU}_{\alpha,\lambda}$ $(\lambda \approx 1.67, \alpha \approx 1.05)$
& No
& $\lambda t\mathbbm{1}_{\{t \geq 0\}} + \lambda\alpha(e^t-1)\mathbbm{1}_{\{t < 0\}}$
& $\lambda t\mathbbm{1}_{\{t \geq 0\}} + \lambda\alpha t\mathbbm{1}_{\{t < 0\}}$
& $\lambda (t)_+$
& $\lambda \alpha$ \\[1pt]

$\mathrm{CELU}_\alpha$ $(\alpha > 0)$
& No
& $t\mathbbm{1}_{\{t \geq 0\}} + \alpha(e^{t/\alpha}-1)\mathbbm{1}_{\{t < 0\}}$
& $t$
& $(t)_+$
& $1$ \\[1pt]

Centered softplus 
& No
& $\log(1+e^t) - \log 2$
& $\frac{1}{2}t$
& $(t)_+$
& $1$ \\[1pt]

Bent identity 
& No
& $\frac{\sqrt{t^2+1}-1}{2} + t$
& $t$
& $\frac{3}{2}t\mathbbm{1}_{\{t \geq 0\}} + \frac{1}{2}t\mathbbm{1}_{\{t < 0\}}$
& $\frac{3}{2}$ \\[1pt]

Softsign 
& No
& $\frac{t}{1+|t|}$
& $t$
& $0$ 
& $1$ \\[1pt]

Tanh 
& No
& $\tanh(t)$
& $t$
& $0$ 
& $1$ \\[1pt]

Centered logistic sigmoid 
& No
& $\frac{1}{1+e^{-t}} - \frac{1}{2}$
& $\frac{1}{4}t$
& $0$ 
& $\frac{1}{4}$ \\[1pt]

Arctan 
& No
& $\arctan(t)$
& $t$
& $0$ 
& $1$ \\[1pt]

Hard tanh 
& No
& $\max\{-1,\min\{1,t\}\}$
& $t$
& $0$ 
& $1$ \\[1pt]

Centered hard sigmoid 
& No
& $\max\left\{-\frac{1}{2}, \min\left\{\frac{1}{2}, \frac{t}{6}\right\}\right\}$
& $\frac{1}{6}t$
& $0$ 
& $\frac{1}{6}$ \\[2pt]
\hline
\end{tabular}%
}
\endgroup
\caption{Summary of the properties of various neural network activation functions. All satisfy $\sigma(0) = 0$. The function $\Phi$ in the definition of GELU is the cumulative distribution function (CDF) of the standard normal distribution. The choice of $\lambda \approx 1.67$, $\alpha \approx 1.05$ for SELU reflects the canonical choice for this activation (see equation 14 in \cite{klambauer2017self} for full analytical expressions of these canonical values).}
\label{tab:activation_summary}
\end{table}

\subsection{Proof of \cref{prop:Lp_C_representative}}

\label{appendix:Lp_C_representative_proof}
\begin{proof}
    The bulk of the proof is devoted to showing that, under the given assumptions, the uniformly-closed sets $\cB_L^\infty$, $L \geq 1$ are compact in $C(\Omega)$. Once this conclusion is established, the claim is easy to prove using standard relationships between uniform limits and $L^p$ limits.
    \paragraph{Compactness of the sets $\cB_L^\infty$, $L \geq 1$.} We will first argue that the sets $\cB_{L}^\infty$ for $L \geq 1$ are all \textit{compact}\footnote{A subset $\cS$ of a Banach space $\cF$ is \textit{compact} if and only if it is \textit{closed} (any convergent sequence in $\cS$ has a limit in $\cS$) and \textit{totally bounded} ($\cS$ can be covered by finitely many $\cF$-norm balls of radius $\epsilon$, for any $\epsilon > 0$). An equivalent condition is that $\cS$ is \textit{sequentially compact} (any sequence in $\cS$ has a subsequence which converges to a limit in $\cS$). \label{footnote:compact_def}} in the ambient $C(\Omega)$ topology. Throughout this section of the proof, we will denote $\cB_L^\infty$ as simply $\cB_L$ to reduce notational clutter. 
    
    First consider the $\cW \times \cB \to C(\Omega)$ map $(\vw,b) \mapsto f_{\vw, b}$, where $f_{\vw,b}(\vx) := \vw^\top \vx + b$. Let $(\vw_n, b_n)$ be a sequence in $\cW \times \cB$ which converges to some $(\vw, b) \in \cW \times \cB$. Then:
\begin{align}
    \| f_{\vw_n, b_n} - f_{\vw,b} \|_\infty \leq \| \vw_n - \vw \|_2 \left( \sup_{\vx \in \Omega} \| \vx \|_2 \right) + |b_n - b| \to 0
\end{align}
as $n \to \infty$. This shows that $(\vw,b) \mapsto f_{\vw,b}$ is continuous, so $\cB_1$---as the image of the compact set $\cW \times \cB$ under this continuous map---is compact (\cite{aliprantis2006infinite}, Theorem 2.32).

Next, consider the $\cB_1 \times [0,\infty] \mapsto C(\Omega)$ map $(f,s) \mapsto \sigma_s(f)$. Here, $[0,\infty] := [0,\infty) \cup \{ \infty \}$ is the one-point compactification of $[0,\infty)$. Let $(f_n, s_n)$ be a sequence in $\cB_1 \times [0,\infty]$ which converges to some limit $(f,s) \in \cB_1 \times [0,\infty]$. Because $\cB_1$ is compact, it is bounded, i.e. $C_1 := \sup_{f \in \cB_1} \| f \|_\infty < \infty$. If the limit $s \in (0,\infty)$, then the convergent sequence $s_n$ is necessarily bounded below by some $S_1 > 0$ and above by some $S_2 > 0$. By the definition of the normalized activations, the $(0,\infty) \times \R \to \R$ map $(s,t) \mapsto \sigma_s(t)$ is continuous, hence uniformly continuous on $[S_1,S_2] \times [-C_1,C_1]$. Fix $\epsilon > 0$ and pick $\delta$ such that
\begin{align}
    \left| \sigma_s(t) - \sigma_{s'}(t') \right| \leq \epsilon
\end{align}
for all $(s,t), (s',t') \in [S_1, S_2] \times [-C_1, C_1]$ such that $\max\{ |s-s'|, |t-t'| \} \leq \delta$. Choose $N_0$ large enough that $\| f_n - f \|_\infty \leq \delta$ and $|s_n - s| \leq \delta$ for $n \geq N_0$. Then
\begin{align}
   \| \sigma_{s_n} \circ f_n - \sigma_s \circ f \|_\infty :=  \sup_{\vx \in \Omega} \left| \sigma_{s_n}(f_n(\vx)) - \sigma_s(f(\vx)) \right| \leq \epsilon
\end{align}
for all $n \geq N_0$. This proves that $\sigma_{s_n} \circ f_n \to \sigma_s \circ f$ uniformly in the case $s \in (0,\infty)$. In the case $s = 0$, we have
\begin{align}
   \| \sigma_{s_n} \circ f_n - \sigma_0 \circ f \|_\infty &\leq \sup_{u \in [-C_1, C_1]} \left| \sigma_{s_n}(u) - \sigma_0(u) \right| + \| \sigma_0 \circ f_n - \sigma_0 \circ f \|_\infty \to 0
\end{align}
by the compact-uniform convergence of $\sigma_s$ to $\sigma_0$ summarized in \cref{tab:activation_summary} and the uniform continuity of $\sigma_0$ on $[-C_1,C_1]$. Similarly $\| \sigma_{s_n} \circ f_n - \sigma_\infty \circ f \|_\infty \to 0$ when $s=\infty$. We have shown that the $\cB_1 \times [0,\infty] \mapsto C(\Omega)$ map $(f,s) \mapsto \sigma_s(f)$ is continuous. Therefore, the image $\widehat{\cS}_\sigma := \{ \sigma_s \circ f: s \in [0,\infty], f \in \cB_1 \}$ is compact by Theorem 2.32 in \cite{aliprantis2006infinite}. Therefore, by Theorem 3.20 (c) in \cite{rudin1991functional}, the closed absolutely convex hull $\overline{\aconv}(\widehat{\cS}_\sigma)$ is also compact.

To prove compactness of $\cB_2$, denote $\cS_\sigma := \{ \sigma_s \circ f: s > 0, f \in \cB_1 \}$. It remains only to prove that $\cB_2 := \overline{\aconv}(\cS_\sigma) = \overline{\aconv}(\widehat{\cS}_\sigma)$. The functions in $\widehat{\cS}_\sigma \setminus \cS_\sigma$ are exactly those of the form $\sigma_0 \circ f$ or $\sigma_\infty \circ f$ for some $f \in \cB_1$. Because the values of all $f \in \cB_1$ lie in the compact interval $[-C_1, C_1]$, the compact-uniform convergence summarized in \cref{tab:activation_summary} guarantees that any such function $\sigma_0 \circ f$ or $\sigma_\infty \circ f$ is the uniform limit of a sequence $\sigma_s \circ f$ for some $s \downarrow 0$ or $s \uparrow \infty$. Therefore, $\widehat{\cS}_\sigma$ is exactly the uniform closure of $\cS_\sigma$. The desired equality $\cB_2 = \overline{\aconv}(\widehat{\cS}_\sigma)$ now follows from the more general fact (whose proof is in \cref{appendix:proof_limit_final_layer}) that, for any subset $\cS$ of a normed space, we have $\overline{\aconv}(\overline{\cS}) = \overline{\aconv}(\cS)$. Assuming inductively that $\cB_L$ is compact, the same argument as above (applied to $\cB_L$ instead of $\cB_1$) shows that $\cB_{L+1}$ is compact.

\paragraph{Relationship between $\cB_L^\infty$ and $\cB_L^{L^p(\mu)}$.} The inclusion $\cB_L^\infty \subset \cB_L^{L^p(\mu)}$ follows from
\begin{align} \label{eq:unif_implies_Lp_conv}
    \| f - f_n \|_{L^p(\mu)} \leq \mu(\Omega)^{1/p} \| f_n - f \|_\infty \to 0
\end{align}
whenever $f_n \to f$ uniformly. For the second claim, fix any $f \in \cB_L^{L^p(\mu)}$ which is the $L^p(\mu)$ limit of a sequence $f_n \in \aconv(\cB_{L-1}^\sigma) \subset \cB_L$. By compactness of $\cB_L$, this sequence $f_n$ has a subsequence $f_{n_k}$ which converges uniformly to some $\tilde{f} \in \cB_L$. Therefore, by \eqref{eq:unif_implies_Lp_conv}, it must be that $\| f_{n_k} - \tilde{f} \|_{L^p(\mu)} \to 0$. Because we also have $\| f_{n_k} - f \|_{L^p(\mu)} \to 0$, the functions $f \in \cB_L^{L^p(\mu)}$ and $\tilde{f} \in B_L^\infty$ must agree with each other outside of a set of $\mu$-measure zero.
\end{proof}

\subsection{Proof of \cref{prop:limit_final_layer}} \label{appendix:proof_limit_final_layer}
\begin{proof} Define
\begin{align}
    \widetilde{\cS}_\sigma := \{ \sigma_s \circ f: s > 0, f \in \widetilde{\cB}_2 \}, \qquad \cS_\sigma := \{ \sigma_s \circ f: s > 0, f \in \cB_2 \}.
\end{align}
First we will show that $\overline{\widetilde{\cS}_\sigma} = \overline{\cS_\sigma}$. The inclusion $\overline{\widetilde{\cS}_\sigma} \subset \overline{\cS_\sigma}$ is immediate from $\widetilde{\cB}_2 \subset \cB_2$. For the reverse inclusion, let $g \in \overline{\cS_\sigma}$ be the $\cF$-norm limit of a sequence $\sigma_{s_n} \circ g_n \in \cS_\sigma$. Each $g_n$ is in turn the $\cF$-norm limit of a sequence $h_{n,m} \in \widetilde{\cB}_2$. Consider the doubly-indexed sequence $\sigma_{s_n} \circ h_{n,m} \in \widetilde{\cS}_\sigma$. In the case $\cF = C(\Omega)$, fix $\epsilon > 0$ and choose $N$ large enough that $\| g - \sigma_{s_N} \circ g_N \|_\cF \leq \epsilon/2$. For this $N$, choose $\delta > 0$ such that $|\sigma_{s_N} (u) - \sigma_{s_N} (v)| \leq \epsilon/2$ whenever $u,v \in [-C_2, C_2]$ have $|u-v| \leq \delta$. (Here $C_2 := \sup_{f \in \cB_2} \| f \|_\infty < \infty$ due to compactness of $\cB_2$ in $C(\Omega)$, as established in the proof of \cref{prop:Lp_C_representative}.) For this $N$ and $\delta$, choose $M$ large enough that $\| g_N - h_{N,m} \|_\infty\leq \delta$ for $m \geq M$. Then
\begin{align}
    \| g - \sigma_{s_N} \circ h_{N,m} \|_\infty \leq \| g - \sigma_{s_N} \circ g_N \|_\infty + \| \sigma_{s_N} \circ g_N - \sigma_{s_N} \circ h_{N,m} \|_\infty \leq \epsilon
\end{align}
for all $m \geq M$. By repeating this selection for a sequence of $\epsilon \downarrow 0$, we find that a subsequence of $\sigma_{s_n} \circ h_{n,m}$ converges uniformly to $g$. By \eqref{eq:unif_implies_Lp_conv}, this same subsequence also converges in $L^p(\mu)$ to $g$. This proves the reverse inclusion $\overline{\cS_\sigma} \subset \overline{\widetilde{\cS}_\sigma}$.

Next, we will show the generic identity $\overline{\aconv}(\overline{\cS}) = \overline{\aconv}(\cS)$ for any subset $\cS$ of $\cF$. Combined with the above result, this will show that
\begin{align}
    \overline{\widetilde{\cB}}_3 = \overline{\aconv}(\widetilde{\cS}_\sigma) = \overline{\aconv}(\overline{\widetilde{\cS}_\sigma}) = \overline{\aconv}(\overline{\cS_\sigma}) = \overline{\aconv}(\cS_\sigma) = \cB_3.
\end{align}
To prove the identity, again note that $\overline{\aconv}(\cS) \subset \overline{\aconv}(\overline{\cS})$ is immediate from $\cS \subset \overline{\cS}$. For the reverse inclusion, let $f \in \overline{\aconv}(\overline{\cS})$ be the $\cF$-norm limit of a sequence $f_n = \sum_{k=1}^{K_n} v_{n,k} s_{n,k} \in \aconv(\overline{\cS})$. Then each $s_{n,k}$ is the $\cF$-norm limit of a sequence $s_{n,k,m} \in \cS$. As above, form the doubly-indexed sequence $f_{n,m} = \sum_{k=1}^K v_{n,k} s_{n,k,m}$ and, for some $\epsilon > 0$, choose $N$ such that $\| f - f_N \|_{\cF} \leq \epsilon/2$ and $M$ such that $\| s_{N,k} - s_{N,k,m} \|_\cF \leq \epsilon/2$ for all $k = 1, \dots, K_N$ and all $m \geq M$. Then
\begin{align}
    \| f - f_{N,m} \|_{\cF} \leq \| f - f_N \|_\cF + \| f_N - f_{N,m} \|_\cF \leq \epsilon/2 + \sum_{k=1}^{K_n} |v_k| (\epsilon/2) \leq \epsilon.
\end{align}
As above, repeating this selection for a sequence of $\epsilon \downarrow 0$ shows that a subsequence of $f_{n,m}$ converges in $\cF$ to $f$, which proves the reverse inclusion $\overline{\aconv}(\overline{\cS}) \subset \overline{\aconv}(\cS)$ as desired. 

Assuming inductively that $\widetilde{\cB}_L = \cB_L$, the above argument also shows that $\widetilde{\cB}_{L+1} = \cB_{L+1}$.
\end{proof}

\subsection{Proof of \cref{prop:path_norm_equivalence}} \label{appendix:proof_path_norm_equivalence}
\begin{proof}
    By assumption \eqref{eq:s_k_assumption}, there are $\widetilde{\vw}_{k_1}^{(1)} \in \cW$ and $\widetilde{b}_{k_1}^{(1)} \in \cB$ such that $\vw_{k_1}^{(1)} = s_{k_1} \widetilde{\vw}_{k_1}^{(1)}$ and $b_{k_1}^{(1)} = s_{k_1} \widetilde{b}_{k_1}^{(1)}$. Therefore, in the case $L = 2$, the network is of the form
    \begin{align} \label{eq:f_L_sigma_scaled}
        f(\vx) = \sum_{k_1=1}^{K_1} w_{k_1}^{(2)} \sigma \left( (\vw_{k_1}^{(1)})^\top \vx + b_{k_1}^{(1)} \right) = \sum_{k_1=1}^{K_1} w_{k_1}^{(2)} s_{k_1} \sigma_{s_{k_1}} \left( ( \widetilde{\vw}_{k_1}^{(1)})^\top \vx + \widetilde{b}_{k_1}^{(1)} \right).
    \end{align}
    Therefore $f \in \cV_2$, with
    \begin{align}
        \| f \|_{\cV_2} \leq \Phi(\vtheta) := \sum_{k_1=1}^{K_1} \left| w_{k_1}^{(2)} \right| s_{k_1}.
    \end{align}
    The right-hand side above is further upper bounded by $\Psi(\vtheta) := \| \vw^{(2)} \|_1 \| \vs \|_\infty$. This proves \eqref{eq:path_norm_sum}--\eqref{eq:path_norm_upper_bound_psi} in the base case $L=2$.
    
    Next, assume inductively the theorem statement holds for $L-1$. Then any depth-$L$ network of the form \eqref{eq:deep_nn_sum} which satisfies \eqref{eq:s_k_assumption} is given by
    \begin{align} \label{eq:deep_net_g}
        f(\vx) = \sum_{k_{L-1}=1}^{K_{L-1}} w_{k_{L-1}}^{(L)} \sigma( g_{k_{L-1}} (\vx))
    \end{align}
    where each $g_{k_{L-1}}$ is a depth-$L-1$ network of the form \eqref{eq:deep_nn_sum} satisfying \eqref{eq:s_k_assumption} and the bound \eqref{eq:path_norm_sum}. Denote $h_{k_{L-1}} := g_{k_{L-1}}/ \Phi(\vtheta_{k_{L-1}})$. Then:
    \begin{align} \label{eq:f_L_2_sigma_scaled}
        f(\vx) &= \sum_{k_{L-1}=1}^{K_{L-1}} w_{k_{L-1}}^{(L)} \Phi(\vtheta_{k_{L-1}}) \left( \frac{\sigma \left(  \Phi(\vtheta_{k_{L-1}}) h_{k_{L-1}} (\vx) \right)}{\Phi(\vtheta_{k_{L-1}})} \right) \\
        &= \sum_{k_{L-1}=1}^{K_{L-1}} w_{k_{L-1}}^{(L)} S_{k_{L-1}} \sigma_{S_{k_{L-1}}} (h_{k_{L-1}}(\vx)),
    \end{align}
    where $S_{k_{L-1}} := \Phi(\vtheta_{k_{L-1}})$. By the inductive hypothesis, $ h_{k_{L-1}} \in \cB_{L-1}$ for each $k_{L-1}$, and therefore $f \in \cV_L$ with
    \begin{align}
        \| f \|_{\cV_L} \leq \Phi(\vtheta) :=  \sum_{k_{L-1}=1}^{K_{L-1}} \left| w_{k_{L-1}}^{(L)} \right| \Phi(\vtheta_{k_{L-1}}).
    \end{align}
    Expanding this recursion gives exactly \eqref{eq:path_norm_sum}. The right-hand side $\Phi(\vtheta)$ is in turn upper bounded by
    \begin{align}
        \sum_{k_{L-1}=1}^{K_{L-1}} \left| w_{k_{L-1}}^{(L)} \right| \Phi(\vtheta_{k_{L-1}}) &\leq \| \vw^{(L)} \|_1 \max_{k_{L-1}=1, \dots, K_{L-1}} \Phi(\vtheta_{k_{L-1}}) \\
        &\leq \| \vw^{(L)} \|_1 \max_{k_{L-1}=1, \dots, K_{L-1}} \Psi(\vtheta_{k_{L-1}}) \\
        &= \| \vw^{(L)} \|_1 \| \mW^{(L-1)} \|_{1,\infty} \left( \prod_{\ell=2}^{L-2} \| \mW^{(\ell)} \|_{1,\infty} \right) \| \vs \|_\infty \\
        &= \| \vw^{(L)} \|_1 \left( \prod_{\ell=2}^{L-1} \| \mW^{(\ell)} \|_{1,\infty} \right) \| \vs \|_\infty,
    \end{align}
    which is \eqref{eq:path_norm_upper_bound_psi}.
    
    Finally, if $\sigma$ is homogeneous, the layers of $f$ can be successively renormalized from the inside out. The renormalized parameters $\widetilde{\vtheta}$ can be taken to satisfy \eqref{eq:s_k_assumption} with $\widetilde{s}_{k_1} = 1$ for each $k_1 = 1, \dots, K_1$. The renormalized $\widetilde{\mW}^{(2)}, \dots, \widetilde{\mW}^{(L-1)}$ each have rowwise $\ell^1$ norms equal to 1, and the renormalized output weights $\widetilde{\vw}^{(L)}$ have $\ell^1$ norm equal to $\Phi(\vtheta) = \Phi (\widetilde{\vtheta}) = \Psi(\widetilde{\vtheta})$.  This proves the theorem's final claim. 
\end{proof}

\subsection{The $\ReLU^m$ analogue of \cref{prop:path_norm_equivalence}} \label{appendix:relu_m_path_norm_equivalence}
\begin{proposition} \label{prop:relu_m_path_norm_equivalence}
    Let $\sigma(t) = \ReLU^m(t)$ and $\sigma_s(t) = \sigma(st)/s^m$. Let $f$ be a neural network of the form \eqref{eq:deep_nn_sum}/\eqref{eq:deep_nn_matrix} with activation function $\sigma$. Suppose that for each $k_1 = 1, \dots, K_1$, there exists an $s_{k_1} > 0$ such that \eqref{eq:s_k_assumption} holds. Then
    \begin{align} \label{eq:relu_m_path_norm_sum}
        \| f \|_{\cV_L} \leq \Phi_m(\vtheta) := \sum_{k_{L-1}=1}^{K_{L-1}} \left| w_{k_{L-1}}^{(L)} \right| \left( \sum_{k_{L-2}=1}^{K_{L-2}} \left| W_{k_{L-1},k_{L-2}}^{(L-1)} \right| \left( \dots \left( \sum_{k_1=1}^{K_1}  \left| W_{k_2,k_1}^{(2)} \right| s_{k_1}^m \right)^m \dots \right)^m \right)^m.
    \end{align}
    Moreover,
    \begin{align} \label{eq:relu_m_path_norm_upper_bound}
        \Phi_m(\vtheta) \leq \Psi_m(\vtheta) := \| \vw^{(L)} \|_1 \left( \prod_{\ell=2}^{L-1} \| \mW^{(\ell)} \|_{1,\infty}^{m^{L-\ell}} \right) \| \vs \|_\infty^{m^{L-1}},
    \end{align}
    where $\vs := [s_1, \dots, s_{K_1}]^\top$. There is another network with parameters $\widetilde{\vtheta}$ which represents the same function as $f$ and satisfies $\Phi_m(\widetilde{\vtheta}) = \Psi_m(\widetilde{\vtheta})$.
\end{proposition}
\begin{proof}
    By assumption \eqref{eq:s_k_assumption}, there are $\widetilde{\vw}_{k_1}^{(1)} \in \cW$ and $\widetilde{b}_{k_1}^{(1)} \in \cB$ such that $\vw_{k_1}^{(1)} = s_{k_1} \widetilde{\vw}_{k_1}^{(1)}$ and $b_{k_1}^{(1)} = s_{k_1} \widetilde{b}_{k_1}^{(1)}$. Therefore, in the case $L = 2$, the network is of the form
    \begin{align}
        f(\vx) = \sum_{k_1=1}^{K_1} w_{k_1}^{(2)} \sigma \left( (\vw_{k_1}^{(1)})^\top \vx + b_{k_1}^{(1)} \right) = \sum_{k_1=1}^{K_1} w_{k_1}^{(2)} s_{k_1}^m \sigma_{s_{k_1}} \left( ( \widetilde{\vw}_{k_1}^{(1)})^\top \vx + \widetilde{b}_{k_1}^{(1)} \right).
    \end{align}
    Therefore $f \in \cV_2$, with
    \begin{align}
        \| f \|_{\cV_2} \leq \Phi_m(\vtheta) := \sum_{k_1=1}^{K_1} \left| w_{k_1}^{(2)} \right| s_{k_1}^m.
    \end{align}
    The right-hand side above is further upper bounded by $\Psi_m(\vtheta) := \| \vw^{(2)} \|_1 \| \vs \|_\infty^m$. This proves \eqref{eq:relu_m_path_norm_sum}--\eqref{eq:relu_m_path_norm_upper_bound} in the base case $L=2$.
    
    Next, assume inductively the theorem statement holds for $L-1$. Then any depth-$L$ network of the form \eqref{eq:deep_nn_sum} which satisfies \eqref{eq:s_k_assumption} is given by
    \begin{align}
        f(\vx) = \sum_{k_{L-1}=1}^{K_{L-1}} w_{k_{L-1}}^{(L)} \sigma( g_{k_{L-1}} (\vx))
    \end{align}
    where each $g_{k_{L-1}}$ is a depth-$L-1$ network of the form \eqref{eq:deep_nn_sum} satisfying \eqref{eq:s_k_assumption} and the bound \eqref{eq:relu_m_path_norm_sum}. Denote $h_{k_{L-1}} := g_{k_{L-1}}/ \Phi_m(\vtheta_{k_{L-1}})$. Then:
    \begin{align}
        f(\vx) &= \sum_{k_{L-1}=1}^{K_{L-1}} w_{k_{L-1}}^{(L)} \Phi_m(\vtheta_{k_{L-1}})^m \left( \frac{\sigma \left(  \Phi_m(\vtheta_{k_{L-1}}) h_{k_{L-1}} (\vx) \right)}{\Phi_m(\vtheta_{k_{L-1}})^m} \right) \\
        &= \sum_{k_{L-1}=1}^{K_{L-1}} w_{k_{L-1}}^{(L)} S_{k_{L-1}}^m \sigma_{S_{k_{L-1}}} (h_{k_{L-1}}(\vx)),
    \end{align}
    where $S_{k_{L-1}} := \Phi_m(\vtheta_{k_{L-1}})$. By the inductive hypothesis, $ h_{k_{L-1}} \in \cB_{L-1}$ for each $k_{L-1}$, and therefore $f \in \cV_L$ with
    \begin{align}
        \| f \|_{\cV_L} \leq \Phi_m(\vtheta) :=  \sum_{k_{L-1}=1}^{K_{L-1}} \left| w_{k_{L-1}}^{(L)} \right| \Phi_m(\vtheta_{k_{L-1}})^m.
    \end{align}
    Expanding this recursion gives exactly \eqref{eq:relu_m_path_norm_sum}. The right-hand side $\Phi_m(\vtheta)$ is in turn upper bounded by
    \begin{align}
        \sum_{k_{L-1}=1}^{K_{L-1}} \left| w_{k_{L-1}}^{(L)} \right| \Phi_m(\vtheta_{k_{L-1}})^m &\leq \| \vw^{(L)} \|_1 \max_{k_{L-1}=1, \dots, K_{L-1}} \Phi_m(\vtheta_{k_{L-1}})^m \\
        &\leq \| \vw^{(L)} \|_1 \max_{k_{L-1}=1, \dots, K_{L-1}} \Psi_m(\vtheta_{k_{L-1}})^m \\
        &= \| \vw^{(L)} \|_1 \| \mW^{(L-1)} \|_{1,\infty}^{m} \left( \prod_{\ell=2}^{L-2} \| \mW^{(\ell)} \|_{1,\infty}^{m^{L-\ell}} \right) \| \vs \|_\infty^{m^{L-1}} \\
        &= \| \vw^{(L)} \|_1 \left( \prod_{\ell=2}^{L-1} \| \mW^{(\ell)} \|_{1,\infty}^{m^{L-\ell}} \right) \| \vs \|_\infty^{m^{L-1}},
    \end{align}
    which is \eqref{eq:relu_m_path_norm_upper_bound}.
    
    Finally, since $\ReLU^m$ is homogeneous of degree $m$, the layers of $f$ can be successively renormalized from the inside out exactly as in the proof of \cref{prop:path_norm_equivalence}. The renormalized parameters $\widetilde{\vtheta}$ can be taken to satisfy \eqref{eq:s_k_assumption} with $\widetilde{s}_{k_1} = 1$ for each $k_1 = 1, \dots, K_1$. The renormalized $\widetilde{\mW}^{(2)}, \dots, \widetilde{\mW}^{(L-1)}$ each have rowwise $\ell^1$ norms equal to 1, and the renormalized output weights $\widetilde{\vw}^{(L)}$ have $\ell^1$ norm equal to $\Phi_m(\vtheta) = \Phi_m (\widetilde{\vtheta}) = \Psi_m(\widetilde{\vtheta})$. This proves the claim.
\end{proof}

\subsection{Proof of \cref{lemma:integral_rep}} \label{appendix:proof_integral_rep}

\begin{proof}
    For $L \geq 2$, let $\cU_L$ denote the set of all functions which admit a pointwise integral representation of the form \eqref{eq:pointwise_integral_rep} for some $\mu \in \cM(\cB_{L-1} \times [0,\infty])$ with $\| \mu \|_{\TV} \leq 1$. We will prove that $\cB_L = \cU_L$, which implies the stated formula for $\| \cdot \|_{\cV_L}$. The particular forms of the integral representations for homogeneous activations follow from the fact that $\sigma_s = \sigma$ in all such cases, so the integrands depending on $\sigma_s$ can be replaced with $\sigma$, and the integrals need only be taken over $\cB_{L-1}$ (or $\cW \times \cB$ for $L = 2$) rather than $\cB_{L-1} \times [0,\infty]$. The proof in those cases uses identical functional analytic arguments to the ones below.
    
    \paragraph{Integral representations of the form \eqref{eq:pointwise_integral_rep} are finite.} First, to establish that the integral in any such representation is finite, recall from the proof of \cref{prop:Lp_C_representative} that the set $\cB_{L-1} \times [0,\infty]$ is compact, and the $\cB_{L-1} \times [0,\infty] \to C(\Omega)$ map $(f,s) \mapsto \sigma_s(f)$ is continuous. This means that, whenever $(f_n, s_n) \to (f,s) \in \cB_{L-1} \times [0,\infty]$, we have $\sigma_{s_n}(f_n) \to \sigma_s(f)$ uniformly (thus pointwise). Therefore, for any $\vx \in \Omega$, the $\cB_{L-1} \times [0,\infty] \to \R$ map $(f,s) \mapsto \sigma_s(f(\vx))$ is continuous, and in particular is bounded by some constant $C_\vx < \infty$ on the compact set $\cB_{L-1} \times [0,\infty]$. The integral in the representation \eqref{eq:pointwise_integral_rep} is correspondingly bounded as
    \begin{align}
        \left| \int_{\cB_{L-1} \times [0,\infty]} \sigma_s(g(\vx)) \ d \mu(s,g) \right| \leq \int_{\cB_{L-1} \times [0,\infty]} C_\vx \ d |\mu|(s,g) \leq C_\vx \| \mu \|_{\TV} < \infty
    \end{align}
    for any finite Radon measure $\mu$.

    \paragraph{Inclusion $\cB_L \subset \cU_L$.} 

Here we will use the Riesz-Markov-Kakutani representation theorem (\cite{folland1999real}, Theorem 7.17), which relies on compactness of $\cB_{L-1} \times [0,\infty]$. In our case, this theorem says that the space $C(\cB_{L-1} \times [0,\infty])^*$ of bounded linear functionals on $C(\cB_{L-1} \times [0,\infty])$ is exactly the space of linear functionals of the form
\begin{align} 
    I_\mu(\phi) := \int_{\cB_{L-1} \times [0,\infty]} \phi(s,g) \ d \mu(s,g), \qquad \phi \in C( \cB_{L-1} \times [0,\infty])
\end{align}
for some $\mu \in \cM([0,\infty] \times\cB_{L-1})$, with the dual norm $\| I_\mu \|$ exactly equal to $\| \mu \|_{\TV}$. Additionally, we will use the Banach-Alaoglu theorem (\cite{folland1999real}, Theorem 5.18), which states that for every sequence $I_{\mu_n}$ in the unit ball $\cB^* := \{ I_\mu \in C(\cB_{L-1} \times [0,\infty])^*: \| I_\mu \| \leq 1 \}$ of $C(\cB_{L-1} \times [0,\infty])^*$, there is a subsequence $I_{\mu_{n_k}}$ satisfying
\begin{align} \label{eq:weak_star_conv_def}
    I_{\mu_{n_k}}(\phi) := \int_{\cB_{L-1} \times [0,\infty] } \phi(s,g) \ d \mu_{n_k}(s,g) \to I_{\mu}(\phi) := \int_{\cB_{L-1} \times [0,\infty]} \phi(s,g) \ d \mu(s,g) 
\end{align}
for every $\phi \in C(\cB_{L-1} \times [0,\infty])$. 

To apply these theorems to our setup, first note that the inclusion $\widetilde{\cB}_L \subset \cU_L$ holds because any function $f = \sum_{k=1}^K v_k (\sigma_{s_k} \circ g_k)$ in $\widetilde{\cB}_L$ is represented as an integral of the form \eqref{eq:pointwise_integral_rep} with respect to the atomic measure $\mu = \sum_{k=1}^K v_k \delta_{s_k,g_k}$, which has $\| \mu \|_{\TV} = \sum_{k=1}^K |v_k| \leq 1$. To extend the conclusion to the uniform closure $\cB_L$, let $f \in \cB_L$ be the uniform limit of some sequence $f_n \in \widetilde{\cB}_L$. As argued above, the maps $(g,s) \mapsto \sigma_s(g(\vx))$ are continuous on $\cB_{L-1} \times [0,\infty]$ for all $\vx \in \Omega$. Therefore, by Riesz-Markov-Kakutani and Banach-Alaoglu, there is a subsequence $\mu_{n_k}$ and a measure $\mu \in \cM(\cB_{L-1} \times [0,\infty])$ with $\| \mu \|_{\TV} \leq 1$ which satisfy
\begin{align}
    f_{n_k}(\vx) = \int_{\cB_{L-1} \times [0,\infty]} \sigma_s(g(\vx)) \ d \mu_{n_k}(s,g) \to \int_{\cB_{L-1} \times [0,\infty]} \sigma_s(g(\vx)) \ d \mu(s,g).
\end{align}
for all $\vx \in \Omega$. Moreover, uniform convergence $f_n \to f$ implies that $f_n(\vx) \to f(\vx)$ pointwise for each $\vx \in \Omega$. Since pointwise limits of the subsequence $f_{n_k} (\vx)$ must coincide with the limit of the full convergent sequence $f_n(\vx)$, the aforementioned measure $\mu$ must satisfy
\begin{align}
    f(\vx) = \int_{\cB_{L-1} \times [0,\infty]} \sigma_s (g(\vx)) \ d \mu(s,g)
\end{align}
for all $\vx \in \Omega$. This shows that $f \in \cU_L$.

\paragraph{Inclusion $\cU_L \subset \cB_L$.} The idea here is to uniformly approximate  the integral in \eqref{eq:pointwise_integral_rep} using finite absolutely convex sums, which are themselves integrals of simple functions. To do so, recall again from the proof of \cref{prop:Lp_C_representative} that the $\cB_{L-1} \times [0,\infty] \to C(\Omega)$ map $(f,s) \mapsto \sigma_s \circ f$ is continuous, hence uniformly continuous on the compact set $\cB_{L-1} \times [0,\infty]$. Therefore, fixing some $\epsilon > 0$, there is a $\delta > 0$ such that $\| \sigma_s \circ f - \sigma_t \circ g \|_\infty \leq \epsilon$ for any $(f,s), (g,t) \in \cB_{L-1}  \times [0,\infty]$ with $\max\{ \| f - g \|_\infty, |s-t| \} \leq \delta$. Additionally, because $\cB_{L-1} \times [0,\infty]$ is compact, it is totally bounded (see \footref{footnote:compact_def}). Let $(g_1, s_1), \dots, (g_n, s_n)$ be a $\delta$-cover of $\cB_{L-1} \times [0,\infty]$. By definition, the sets
\begin{align}
    \cE_i := \left\{ (g,s) \in  \cB_{L-1} \times [0,\infty]: \max\{ \| g - g_i \|_\infty, |s-s_i| \} \leq \delta\right\}
\end{align}
satisfy $\cB_{L-1} \times [0,\infty] \subset \bigcup_{i=1}^n \cE_i$. Each $\cE_i$ is Borel, as the preimage of the closed interval $[0, \delta]$ under the continuous $\cB_{L-1} \times [0,\infty] \to \R$ map $(g,s) \mapsto \max\{ \| g - g_i \|_\infty, |s-s_i| \}$. We can turn the sets $\cE_1, \dots, \cE_n$ into a \textit{disjoint} Borel cover $\cA_1, \dots, \cA_n$ of $\cB_{L-1} \times [0,\infty]$ by taking $\cA_1 = \cE_1$ and $\cA_i = \cE_i \setminus \left( \bigcup_{j=1}^{i-1} \cE_j \right)$ for $i=2, \dots, n$.

With this in mind, fix some $f \in \cU_L$ and let $\mu$ be its representing measure. Define the function
\begin{align}
    f_\epsilon(\vx) := \sum_{i=1}^n \mu(\cA_i) \sigma_{s_i}(g_i(\vx)), \qquad \forall \vx \in \Omega
\end{align}
where $s_i$, $g_i$ and $\cA_i$ are as above. Recall that by definition, $\| \mu \|_{\TV} := |\mu|(\cB_{L-1} \times [0,\infty])$, where the total variation measure $|\mu|$ is defined as
\begin{align}
    |\mu|(\cS) := \sup_{\pi} \sum_{\cP_i \in \pi} |\mu(\cP_i)|.
\end{align}
The supremum in the above definition is taken over all partitions $\pi$ of $\cS$ into finitely many disjoint Borel sets. From this definition, it is clear that 
\begin{align}
    \sum_{i=1}^n |\mu(\cA_i)| \leq \| \mu \|_{\TV} \leq 1
\end{align}
and therefore $f \in \widetilde{\cB}_L$. Furthermore, for each $\vx \in \Omega$, we have
\begin{align}
    |f(\vx) - f_\epsilon(\vx)| &= \left| \int_{\cB_{L-1} \times [0,\infty]} \sigma_s(g(  \vx))  \ d \mu(s,g) - \int_{\cB_{L-1} \times [0,\infty]} \sum_{i=1}^n \mathbbm{1}_{\cA_i}(g) \sigma_{s_i}(g_i(\vx)) \ d \mu(s,g)   \right| \\
    &= \left| \sum_{i=1}^n \int_{\cA_i} \sigma_s(g(\vx)) - \sigma_{s_i}(g_i(\vx)) \ d \mu(s,g) \right| \\
    &\leq \sum_{i=1}^n \int_{\cA_i} \left| \sigma_s(g(\vx)) - \sigma_{s_i}(g_i(\vx)) \right| \ d |\mu|(g) \\
    &\leq \sum_{i=1}^n \epsilon |\mu|(\cA_i) \leq \epsilon \| \mu \|_{\TV} \leq \epsilon.
\end{align}
Here we have used the fact that $\max\{ \| g - g_i \|_\infty, |s - s_i| \} \leq \delta$, and therefore $|\sigma_s(g(\vx)) - \sigma_{s_i}(g_i(\vx))| \leq \epsilon$. The above holds for any $\vx \in \Omega$, so it also holds in the supremum over all $\vx \in \Omega$, which shows that
\begin{align}
    \| f - f_\epsilon \|_\infty \leq \epsilon.
\end{align}
Since $\epsilon$ was arbitrary, we can repeat the above process for a sequence $\epsilon_i \downarrow 0$, generating a sequence of functions $f_{\epsilon_i} \in \widetilde{\cB}_L$ which converge uniformly to $f$. This shows that $f \in \cB_L$.
\end{proof}

\subsection{Proof of \cref{lemma:relu_nonrelu_space_equivalence}} \label{appendix:proof_relu_nonrelu_space_equivalence}
\begin{proof}
    Let $C_1 := \sup_{f \in \cB_1} \| f \|_\infty$. The stated assumptions on $\cW$ and $\cB$ imply that, if $A_1$ is any constant satisfying
    \begin{align} \label{eq:A1_bounds}
        A_1 \geq \max \{ 1+ C_1/C_{\cB}, 1+C_1/C_{\cB} \},
    \end{align}
    then the affine functions $f_{\vw,b}(\vx) := \vw^\top \vx + b$ satisfy
    \begin{align} \label{eq:rescaled_linear_containment}
            f_{\vw,b} - u = f_{\vw,b-u} \in A_1 \cB_1
    \end{align}
    for all $\vw \in \cW$, all $b \in \cB$, and all $|u| \leq C_1$. This holds because the above bound on $A_1$ ensures that $\cW \subset A_1 \cW$ and $\{ b-u: b \in \cB, |u| \leq C_1 \} \subset [-C_{\cB} - C_1, C_{\cB} + C_1] \subset A_1 \cB$. The assumptions on $\cW$ and $\cB$ also imply that both $\cB_2^{\ReLU}$ and $\cB_2^{|\cdot|}$ contain the constant functions $f(\vx) = c = (\bm{0}^\top \vx + c)_+ = | \bm{0}^\top \vx + c | $ for all $0 \leq c \leq C_{\cB}$. 
    
    Additionally, the proof utilizes several facts relating to functions of finite total variation and associated integral representations using distributional/weak derivatives. These same properties are utilized in the proof of \cref{lemma:uni_f_V2_bound}, wherein they are presented in more detail: see \cref{appendix:proof_uni_f_V2_bound}.
    
    \paragraph{Integral representation of $\sigma_s$ in terms of its derivatives.} First note that all of the activations $\sigma$ in the theorem statement have second distributional derivative $D^2 \sigma$ which is a finite Radon measures on $\R$. For the piecewise linear activations Leaky ReLU, absolute value, hard tanh, and centered hard sigmoid, this holds because $\sigma'$ is piecewise constant with finitely many jumps. For GELU, SiLU/Swish, Mish, centered softplus, and bent identity, it holds because $\sigma$ is twice continuously differentiable with $\int_\R |\sigma''| < \infty$. The piecewise smooth activations ELU, SELU, and CELU are twice continuously differentiable on the respective smooth pieces, with finitely many jumps between the smooth pieces; as a result, the corresponding values of $V_\R(\sigma')$ are given by the integrals of $|\sigma''|$ on the smooth pieces plus the magnitudes of any jumps of $\sigma'$ between these pieces. Therefore the normalized activations $\sigma_s(t) = \sigma(st)/s$ have an a.e. first derivative $\sigma_s'(t) = \sigma'(st)$, and
    \begin{align} \label{eq:sigma_s_var_bound}
        \| D^2 \sigma_s  \|_{\TV([-C,C])} = \| D^2 \sigma\|_{\TV([-sC, sC])} \leq \| D^2 \sigma \|_{\TV(\bR)} < \infty 
    \end{align}
    for any $C > 0$. As in the proof of \cref{lemma:uni_f_V2_bound}, this implies that $\sigma_s$ admits an integral representation
    \begin{align} \label{eq:sigma_s_int_rep}
        \sigma_s(t) = \sigma_s(-C) + \sigma_s'(-C^+) (t+C) + \int_{-C}^C (t-u)_+ \ d (D^2 \sigma_s)(u)
    \end{align}
    for every $t \in [-C,C]$. Furthermore, by \cref{tab:activation_summary}, all activations satisfy
    \begin{align} \label{eq:sigma_s_deriv_bound}
        |\sigma_s'(-C^+)| \leq \rho_\sigma := \sup_{s > 0} \Lip(\sigma_s) < \infty
    \end{align}
    and
    \begin{align} \label{eq:sigma_s_endpoint_bound}
        |\sigma_s(-C)| \leq \rho_\sigma C + |\sigma_s(0)| = \rho_\sigma C.
    \end{align}

    \paragraph{Upper bound on $\| f \|_{\cV_2^{\ReLU}}$ (base case).} Let $\sigma$ be any of the activations in the theorem statement. Fix an affine function $f_{\vw,b}(\vx) := \vw^\top \vx + b$ in $\cB_1$ and some $s > 0$. By \eqref{eq:sigma_s_int_rep} and the fact that $C_1 := \sup_{f \in \cB_1} \| f \|_\infty < \infty$ (see proof of \cref{prop:Lp_C_representative}), the composition $\sigma_s \circ f_{\vw,b}$ is represented pointwise as
    \begin{align}
        \sigma_s(f_{\vw,b}(\vx)) &= \sigma_s(-C_1) + \sigma_s'(-C_1^+) ( f_{\vw,b}(\vx) +C_1) + \int_{-C_1}^{C_1} (f_{\vw,b}(\vx)-u)_+ \ d (D^2 \sigma_s)(u) \\
        &= \underbrace{\sigma_s(-C_1) + \sigma_s'(-C_1^+) A_1 \left( \frac{f_{\vw,b}(\vx) +C_1}{A_1} \right)_+}_{\sigma_\mathrm{aff}} + A_1 \underbrace{ \int_{-C_1}^{C_1} \left( \frac{f_{\vw,b}(\vx)-u}{A_1} \right)_+ \ d (D^2 \sigma_s)(u)}_{\sigma_\mathrm{int}} 
    \end{align}
    for all $\vx \in \Omega$. The second line above uses nonnegativity of the function $f_{\vw,b} + C_1$ on $\Omega$, which implies that $f_{\vw,b} + C_1 = (f_{\vw,b} + C_1)_+$, as well as homogeneity of the ReLU. By the discussion at the beginning of the proof, the integral term $\sigma_\mathrm{int}$ is exactly a $\cV_2^{\ReLU}$ integral representation of the form \eqref{eq:pointwise_integral_rep_homogeneous_shallow} with respect to the finite Radon measure $D^2 \sigma_s$. Furthermore, because $\cB_2^{\ReLU}$ contains the nonzero constant function $C_{\cB}$, any constant function $c = (c/C_{\cB}) C_{\cB}$ has $\cV_2^{\ReLU}$ norm upper bounded by $|c/C_{\cB}|$. Therefore, the affine term $\sigma_{\mathrm{aff}}$ has
    \begin{align}
        \| \sigma_{\mathrm{aff}} \|_{\cV_2^{\ReLU}} \leq \left| \frac{\sigma_s(-C_1)}{u_0} \right| + |\sigma_s'(-C_1^+) A_1|.
    \end{align}
    Combining the above with \eqref{eq:sigma_s_var_bound}, \eqref{eq:sigma_s_deriv_bound}, and \eqref{eq:sigma_s_endpoint_bound}, we have
    \begin{align} \label{eq:sigma_s_f_w_b_bound}
        \| \sigma_s \circ f_{\vw,b} \|_{\cV_2^{\ReLU}} &\leq \left| \frac{\sigma_s(-C_1)}{u_0} \right| + |\sigma_s'(-C_1^+) A_1| + A_1 \| D^2 \sigma_s  \|_{\TV([-C_1, C_1])} \\
        &\leq \frac{\rho_\sigma C_1}{|u_0|} + A_1 \left( \rho_\sigma + \| D^2 \sigma \|_{\TV(\R)} \right) =: \beta_{2,\sigma}.
    \end{align}
    This implies that
    \begin{align}
        \{ \sigma_s \circ f: s > 0, f \in \cB_1 \} \subset \beta_{2,\sigma} \cB_2^{\ReLU}.
    \end{align}
    By the triangle inequality and the fact that $\cB_2^{\ReLU}$ is itself closed, we have
    \begin{align}
        \cB_2^\sigma \subset \beta_{2,\sigma} \cB_2^{\ReLU}. 
    \end{align}
    This containment relationship says that, for any function $f$, we have $f \in c \beta_{2,\sigma} \cB_2^{\ReLU}$ (or equivalently, $f/\beta_{2,\sigma} \in c \cB_2^{\ReLU}$) whenever $f \in c  \cB_2^\sigma$. By definition, we have $f \in \| f \|_{\cV_2^\sigma} \cB_2^\sigma$, and therefore $f / \beta_{2,\sigma} \in \| f \|_{\cV_2^\sigma} \cB_2^{\ReLU}$. Therefore:
    \begin{align}
        \| f / \beta_{2,\sigma} \|_{\cV_2^{\ReLU}} \leq \| f \|_{\cV_2^\sigma} \implies \| f \|_{\cV_2^{\ReLU}} \leq \beta_{2,\sigma} \| f \|_{\cV_2^\sigma}.
    \end{align}

    \paragraph{Upper bound on $\| f \|_{\cV_L^{\ReLU}}$ for $L > 2$ (inductive step).} Again let $\sigma$ be any of the activations in the theorem statement, and fix some $f \in \cB_L^\sigma$ and some $s > 0$. Let $C_L := \sup_{f \in \cB_L^\sigma} \| f \|_\infty$. \eqref{eq:sigma_s_int_rep} then implies that
    \begin{align} \label{eq:sigma_s_int_rep_f}
        \sigma_s(f(\vx)) &= \sigma_s(-C_L) + \sigma_s'(-C_L^+)(f(\vx) + C_L) + \int_{-C_L}^{C_L} (f(\vx) - u)_+ \ d(D^2 \sigma_s)(u) \\
        &= \underbrace{\sigma_s(-C_L) + \sigma_s'(-C_L^+) C_L}_{\sigma_\mathrm{const}} + \underbrace{\sigma_s'(-C_L^+) f(\vx)}_{\sigma_\mathrm{aff}} + \underbrace{\int_{-C_L}^{C_L} (f(\vx) - u)_+ \, d(D^2 \sigma_s)(u)}_{\sigma_\mathrm{int}}.
    \end{align} 
    As in the base case, \eqref{eq:sigma_s_endpoint_bound} and \eqref{eq:sigma_s_deriv_bound} imply that
    \begin{align} \label{eq:sigma_const_upper_bound_inductive}
        \| \sigma_\mathrm{const} \|_{\cV_{L+1}^{\ReLU} } \leq \| \sigma_\mathrm{const} \|_{\cV_2^{\ReLU}} \leq \left| \frac{\sigma_s(-C_L) + \sigma_s'(-C_L^+) C_L}{C_{\cB}} \right| \leq  \frac{2 \rho_\sigma C_L }{|C_{\cB}|} 
    \end{align}
    for any constant function $u_0 \in \cB_2$. Here we have used the fact that the ReLU balls are nested as $\cB_2^{\ReLU} \subset \cB_3^{\ReLU} \subset \dots$ (see \cref{prop:relu_nested_containment_entropy_lb}), which implies that $\| \cdot \|_{\cV_L^{\ReLU}} \leq \| \cdot \|_{\cV_{L'}^{\ReLU}}$ for any $2 \leq L' \leq L$. By this same inequality and the inductive hypothesis, the $\cV_{L+1}^{\ReLU}$ norm of $\sigma_{\mathrm{aff}}$ can be bounded as
    \begin{align} \label{eq:sigma_aff_upper_bound_inductive}
        \| \sigma_{\mathrm{aff}} \|_{\cV_{L+1}^{\ReLU}} \leq \| \sigma_{\mathrm{aff}} \|_{\cV_L^{\ReLU}} = |\sigma_s'(-C_L^+)| \| f \|_{\cV_L^{\ReLU}} \leq |\sigma_s'(-C_L^+)| B_{L,\sigma} \| f \|_{\cV_L^\sigma} = |\sigma_s'(-C_L^+)|  B_{L,\sigma}.
    \end{align}
    For $\sigma_{\mathrm{int}}$, observe that 
    \begin{align} 
        \| ( f - u )_+ \|_{\cV_{L+1}^{\ReLU}} \leq \| f - u \|_{\cV_L^{\ReLU}} \leq \| f \|_{\cV_L^{\ReLU}} + \| u \|_{\cV_2^{\ReLU}} \leq  B_{L,\sigma} + |u/u_0|.
    \end{align}
    Here, in addition to the previously stated norm inequality, we have used positive homogeneity of the ReLU: for any $f \in \cV_L^{\ReLU}$, its positive part $(f)_+ = \| f \|_{\cV_L^{\ReLU}} (f/\| f \|_{\cV_L^{\ReLU}})_+$ is in $\| f \|_{\cV_L^{\ReLU}} \cB_{L+1}$, and therefore $\| (f)_+ \|_{\cV_{L+1}^{\ReLU}} \leq \| f \|_{\cV_{L}^{\ReLU}}$. Along with \eqref{eq:sigma_s_var_bound}, this implies that
    \begin{align} \label{eq:sigma_int_upper_bound_inductive}
        \| \sigma_{\mathrm{int}} \|_{\cV_{L+1}^{\ReLU}} \leq \int_{-C_L}^{C_L} \| (f - u)_+ \|_{\cV_{L+1}^{\ReLU}} \, d|D^2 \sigma_s|(u) \leq \left( B_{L,\sigma} + \frac{C_L}{|u_0|} \right) \| D^2 \sigma_s \|_{\TV([-C_L, C_L])}. 
    \end{align}
    Combining \eqref{eq:sigma_const_upper_bound_inductive}, \eqref{eq:sigma_aff_upper_bound_inductive}, and \eqref{eq:sigma_int_upper_bound_inductive}, along with \eqref{eq:sigma_s_var_bound} and \eqref{eq:sigma_s_endpoint_bound}, we find that
    \begin{align}
        \| \sigma_s \circ f \|_{\cV_{L+1}^{\ReLU}} &\leq \frac{2 \rho_\sigma C_L}{ |u_0|} + B_{L,\sigma} \left( |\sigma_s'(-C_L^+)| + \| D^2 \sigma_s \|_{\TV([-C_L, C_L])} \right) + \frac{C_L}{|u_0|} \| D^2 \sigma_s \|_{\TV([-C_L, C_L])} \\
        &\leq \frac{2 \rho_\sigma C_L}{ |u_0|} + B_{L,\sigma} \left( |\sigma_s'(-\infty)| + \| D^2 \sigma \|_{\TV(\R)} \right) + \frac{C_L}{|u_0|} \| D^2 \sigma \|_{\TV(\R)}.
    \end{align}
    GELU, SiLU/Swish, Mish, ELU, SELU, CELU, and centered softplus all satisfy $\sigma_s'(-\infty) = 0$ for $s > 0$. Therefore, in these cases, $\| \sigma_s \circ f \|_{\cV_{L+1}^{\ReLU}}$ is upper bounded by
    \begin{align} \label{eq:beta_sigma_inf_0}
        B_{L+1,\sigma} := \frac{2 \rho_\sigma C_L}{ |u_0|} + B_{L,\sigma} \| D^2 \sigma \|_{\TV(\R)}  + \frac{C_L}{|u_0|} \| D^2 \sigma \|_{\TV(\R)}.
    \end{align}
    For Leaky ReLU, absolute value, and bent identity, it is possible to use $|\sigma_s(-\infty)| \leq \rho_\sigma$ at this stage to obtain an $s$-independent bound on $\| \sigma_s \circ f \|_{\cV_{L+1}^{\ReLU}}$. However, as we will discuss subsequently, better bounds $B_{L+1,\sigma}$ on $\| \sigma_s \circ f \|_{\cV_{L+1}^{\ReLU}}$ are available in each of these cases. Given any such $B_{L+1,\sigma}$, we have
    \begin{align}
        \cB_{L+1}^\sigma \subset B_{L+1,\sigma} \cB_{L+1}^{\ReLU}.
    \end{align}
    As in the base case, this implies that
    \begin{align}
        \| f \|_{\cV_{L+1}^{\ReLU}} \leq B_{L+1,\sigma} \| f \|_{\cV_{L+1}^{\sigma}}.
    \end{align}
    for all functions $f$.
    \paragraph{Depth-dependence of the constants $B_{L,\sigma}$.}
    
    GELU, SiLU/Swish, Mish, ELU, CELU, and centered softplus satisfy $|\sigma(t)| \leq |t|$, and therefore $|\sigma_s(t)| \leq |st|/|s| = |t|$. For these activations, $C_L \leq C_1$ for all $L \geq 1$. Therefore, in these cases, \eqref{eq:beta_sigma_inf_0} is upper bounded as 
    \begin{align}
        B_{L+1,\sigma} \leq  B_{L,\sigma} \| D^2 \sigma \|_{\TV(\R)} + C_\sigma
    \end{align}
    where $C_\sigma > 0$ is a constant dependent on the individual activation $\sigma$, but independent of depth $L$. Unwinding this recursion yields
    \begin{align}
        B_{L+1,\sigma} \leq \beta_{2,\sigma} \| D^2 \sigma \|_{\TV(\R)}^{L-1} + C_\sigma \sum_{j=0}^{L-2} \| D^2 \sigma \|_{\TV(\R)}^j.
    \end{align}
    For GELU, SiLU/Swish, and Mish, $\| D^2 \sigma \|_{\TV(\R)}$ is approximately 1.516, 1.400, and 1.403, respectively. In these cases, $B_{L+1,\sigma} = \cO(\| D^2 \sigma \|_{\TV(\R)}^L  )$. For ELU, CELU, and centered softplus, $\| D^2 \sigma \|_{\TV(\R)} = 1$ and therefore $B_{L+1,\sigma} = \cO(L)$.

    The bent identity activation does not satisfy $C_L \leq C_1$, but it does satisfy the generic bound $C_L \leq \rho_\sigma^{L-1}$. (This is obtained by applying \eqref{eq:sigma_s_endpoint_bound} inductively with depth.) Furthermore, bent identity has $\| D^2 \sigma \|_{\TV(\R)} = \rho_\sigma = 3/2$, so \eqref{eq:beta_sigma_inf_0} is upper bounded in this case as 
    \begin{align}
        B_{L+1,\sigma} \leq  \frac{3}{2} B_{L,\sigma}  + C_\sigma \left( \frac{3}{2} \right)^{L-1}
    \end{align}
    for some constant $C_\sigma > 0$. In this case, the unwound recursion is
    \begin{align}
        B_{L+1,\sigma} \leq  \frac{3}{2}^{L-2} \beta_{2,\sigma}  + (L-1) C_\sigma \left( \frac{3}{2} \right)^{L-1} = \cO \left( L \left( \frac{3}{2} \right)^L \right).
    \end{align}

    For the homogeneous Leaky ReLU and absolute value activation functions, we employ separate arguments to obtain better bounds. For Leaky ReLU with $0 < \alpha < 1$, we will use the identities
    \begin{align} \label{eq:leaky_relu_identities}
        \LeakyReLU_\alpha(t) = \alpha t + (1-\alpha)(t)_+ = (t)_+ - \alpha(-t)_+.
    \end{align}
    The first identity in \eqref{eq:leaky_relu_identities} shows that
    \begin{align}
        \LeakyReLU_\alpha \circ f = (f)_+ - \alpha(-f)_+ \in (1+\alpha) \cB_{2}^{\ReLU}.
    \end{align}
    for every $f \in \cB_1$. Here we have used the fact that $\cW$ and $\cB$ (and hence $\cB_1$) are symmetric. Because $\cB_2^{\ReLU}$ is absolutely convex and closed, this implies that
    \begin{align}
        \cB_2^{\LeakyReLU_\alpha} \subset (1+\alpha) \cB_2^{\ReLU}.
    \end{align}
    Now assume inductively that
    \begin{align}
        \cB_L^{\LeakyReLU_\alpha} \subset (1+\alpha) \cB_L^{\ReLU}
    \end{align}
    for all $L \geq 2$. Using the second identity in \eqref{eq:leaky_relu_identities}, we have
    \begin{align}
        \LeakyReLU_\alpha \circ f &= \alpha f + (1-\alpha) (f)_+ \\
        & = (1+\alpha) \LeakyReLU_\alpha \circ \left( \frac{f}{1+\alpha} \right) = (1+\alpha) \left( \alpha \left( \frac{f}{1+\alpha} \right) + (1-\alpha) \left( \frac{f}{1+\alpha}  \right)_+ \right).
    \end{align}
    for every $f \in \cB_L^{\LeakyReLU_\alpha}$. The second line above uses 1-homogeneity of both the Leaky ReLU and ReLU. By the inductive hypothesis and \cref{prop:relu_nested_containment_entropy_lb}, the function $f/(1+\alpha)$ is in $\cB_L^{\ReLU} \subset \cB_{L+1}^{\ReLU}$. Therefore:
    \begin{align}
         \LeakyReLU_\alpha \circ f &= (1+\alpha) \left( \alpha \left( \frac{f}{1+\alpha} \right) + (1-\alpha) \left( \frac{f}{1+\alpha}  \right)_+ \right) \in (1+\alpha) \cB_{L+1}^{\ReLU}. 
    \end{align}
    As above, this implies that
    \begin{align}
        \cB_{L+1}^{\LeakyReLU_\alpha} \subset (1+\alpha) \cB_L^{\ReLU}.
    \end{align}
    Therefore, the constant $\cB_{L, \LeakyReLU_\alpha} = 1+\alpha$ in the Leaky ReLU case is completely independent of depth.

    Finally, for the absolute value activation $|\cdot|$ (which is also homogeneous), we use the identity
    \begin{align}
        |f| = (f)_+ + (-f)_+,
    \end{align}
    which holds for any function $f$. This identity shows that
    \begin{align}
        \cB_2^{|\cdot|} \subset 2 \cB_2^{\ReLU}.
    \end{align}
    Assuming inductively that $\cB_L^{|\cdot|} \subset 2^{L-1} \cB_L^{\ReLU}$ for $L \geq 2$, we have
    \begin{align}
        |f| = (f)_+ + (-f)_+ \in 2^L \cB_{L+1}^{\ReLU}
    \end{align}
    for any $f \in \cB_L^{|\cdot|}$. Therefore
    \begin{align}
        \cB_{L+1}^{|\cdot|} \subset 2^L \cB_{L+1}^{\ReLU},
    \end{align}
    so we may take $B_{L,|\cdot|} = 2^{L-1} = \cO(2^L)$.

    \paragraph{Lower bound on $\| f \|_{\cV_L^{\ReLU}}$.} GELU, SiLU/Swish, Mish, ELU, CELU, and centered softplus all have $\sigma_\infty = \ReLU$, so in these cases $\cB_L^{\ReLU} \subset \cB_L^\sigma$. As shown in the proof of \cref{prop:Lp_C_representative}, the functions $\sigma_\infty \circ f$ and $\sigma_0 \circ f$ are in $\cB_{L+1}^{\sigma}$ for all $f \in \cB_L^\sigma$. Therefore, in all of these cases, we may take $A_{L,\sigma} = 1$. SELU has $\sigma_\infty = \lambda \ReLU$, so $\lambda^{L-1} \cB_L^{\ReLU} \subset \cB_L^{\SELU}$. As above, this implies that $\lambda^{L-1} \| \cdot \|_{\cV_L^{\SELU}} \leq \| \cdot \|_{\cV_L^{\ReLU}}$. For the bent identity, we have 
    \begin{align}
        (t)_+ = \frac{3}{2}t \mathbbm{1}_{t \geq 0} + \frac{1}{2} t \mathbbm{1}_{t < 0} - \frac{1}{2} t = \sigma_\infty - \sigma_0/2.
    \end{align}
    Applying this inductively, we find that
    \begin{align}
        \cB_L^{\ReLU} \subset \left( \frac{3}{2} \right)^{L-1} \cB_L^{\mathrm{BentIdentity}}.
    \end{align}
    Similarly, for Leaky ReLU, writing
    \begin{align}
        (t)_+ = \frac{\LeakyReLU_\alpha(t) + \alpha \LeakyReLU_\alpha(-t)}{1-\alpha^2}
    \end{align}
    shows that
    \begin{align}
        \cB_L^{\ReLU} \subset \left( \frac{1+\alpha}{1 - \alpha^2} \right)^{L-1} \cB_L^{\LeakyReLU_\alpha} = \left( \frac{1}{1 - \alpha} \right)^{L-1} \cB_L^{\LeakyReLU_a}
    \end{align}
    so we can take $A_{L,\LeakyReLU_a} = (1-\alpha)^{L-1}$.
    
    Finally, for absolute value, we will use the identities
    \begin{align} \label{eq:abs_val_identities}
        (t)_+ = \frac{|t| + t}{2}, \ \forall t \in \R, \qquad \textrm{and} \qquad t = \frac{|t+c|-|t-c|}{2}, \ \forall |t| \leq c.
    \end{align}
    Denote
    \begin{align}
        M := \sup_{\vw \in \cW, \vx \in \Omega} | \vw^\top \vx|, \qquad \lambda := \min \left\{ 1, \frac{C_{\cB}}{M} \right\}.
    \end{align}
    Then, for any $f_{\vw,b}(\vx) := \vw^\top \vx + b \in \cB_1$, the linear part $\vw^\top \vx$ can be represented using the second identity in \eqref{eq:abs_val_identities} as
    \begin{align} \label{eq:linear_part_abs_rep}
        \vw^\top \vx = \frac{1}{\lambda} (\lambda \vw^\top \vx) = \frac{1}{2 \lambda} \left( | \lambda \vw^\top \vx + \lambda M| -  | \lambda \vw^\top \vx - \lambda M|\right) \in \frac{1}{\lambda} \cB_2^{|\cdot|}.
    \end{align}
    This holds because  $\lambda \vw^\top \vx \leq \lambda M$, $\lambda M = C_{\cB}$, and $\lambda \vw \in \cW$ (since that $\lambda \leq 1$). Moreover, the constant function $b$ is in $\cB_1$, and thus it is also representable in $\cB_2^{|\cdot|}$ as $b = \sgn(b) |b|$. Therefore:
    \begin{align}
        f_{\vw,b} \in \left( \frac{1}{\lambda} + 1 \right)\cB_2^{|\cdot|}.
    \end{align}
    Using the second identity in \eqref{eq:abs_val_identities}, this implies that
    \begin{align}
        (f_{\vw,b})_+ \in \frac{1}{2} \left(1 + \frac{1}{\lambda} + 1 \right) \cB_2^{|\cdot|} = \left( 1 + \frac{1}{2 \lambda} \right) \cB_2^{|\cdot|}.
    \end{align}
    This holds for all $f_{\vw,b} \in \cB_1$, so
    \begin{align}
        \cB_2^{\ReLU} \subset \left( 1 + \frac{1}{2 \lambda} \right) \cB_2^{|\cdot|}. 
    \end{align}
    Now, assume inductively that $\cB_L^{\ReLU} \subset \left(1 + \frac{1}{2 \lambda} \right) \cB_L^{|\cdot|}$ for some $L > 2$. Because the absolute value is also idempotent, the proof of \cref{prop:relu_nested_containment_entropy_lb} implies that the absolute value balls are also nested as
    \begin{align}
        \cB_2^{|\cdot|} \subset \cB_3^{|\cdot|} \subset \dots.
    \end{align}
    Therefore, the first identity in \eqref{eq:abs_val_identities} implies that
    \begin{align}
        (f)_+ = \frac{|f| + f}{2} \in \cB_{L+1}^{|\cdot|} \implies \cB_{L+1}^{\ReLU} \subset \left(1 + \frac{1}{2 \lambda} \right) \cB_{L+1}^{|\cdot|}.
    \end{align}
    Therefore, we can take the constant
    \begin{align}
        A_{L, |\cdot|} = \left( 1 + \frac{1}{2 \lambda} \right)^{-1}
    \end{align}
    to be independent of depth $L$.
\end{proof}

\subsection{Proof of \cref{th:rep_theorem}} \label{appendix:rep_theorem_proof}
\begin{proof}
    We break the proof into the following steps. 
    \paragraph{Existence of solutions to \eqref{opt:reg_loss}.} By \cref{lemma:integral_rep}, problem \eqref{opt:reg_loss} can be expressed as
    \begin{equation} \label{opt:reg_loss_mu}
        \min_{\mu \in \cM(\cB_{L-1} \times [0,\infty])} \sum_{i=1}^N \cL \left(y_i, \int_{\cB_{L-1} \times [0,\infty]} \sigma_s(g(\vx_i)) \ d \mu(s,g) \right) + \lambda \| \mu \|_{\TV}
    \end{equation}
    Any solution to \eqref{opt:reg_loss} must also solve
    \begin{equation} \label{opt:reg_loss_mu_constrained}
     \min_{\mu \in \cM(\cB_{L-1} \times [0,\infty])}  G(\mu) + \lambda \| \mu \|_{\TV} \ , \ \mbox{subject to } \| \mu \|_{\TV} \leq C_0/\lambda
    \end{equation}
    where $G$ denotes the data-fitting term in the objective functional of \eqref{opt:reg_loss_mu} and $C_0 := G(\mu_0) + \lambda \| \mu_0 \|_{\TV}$ for an arbitrary $\mu_0 \in \cM(\cB_{L-1} \times [0,\infty])$. This holds because any $\mu$ with $\| \mu \|_{\TV } > C_0/\lambda$ has
    \begin{align}
        G(\mu) + \lambda \| \mu \|_{\TV} \geq \lambda \| \mu \|_{\TV} > C_0
    \end{align}
    and thus cannot solve \eqref{opt:reg_loss_mu}.

    As in the proof of \cref{lemma:integral_rep} we have the identification 
    \begin{align}
        \cM(\cB_{L-1} \times [0,\infty]) \cong C(\cB_{L-1} \times [0,\infty])^*
    \end{align}
    via the Riesz-Markov-Kakutani representation theorem (\cite{folland1999real}, Theorem 7.17). Therefore, by the Banach-Alaoglu theorem (\cite{folland1999real}, Theorem 5.18), the feasible set of \eqref{opt:reg_loss_mu_constrained} is weak* compact. Also recall from the proof of \cref{lemma:integral_rep} that the functionals $(g,s) \mapsto \sigma_s(g(\vx))$ are continuous for all $\vx \in \Omega$. By Riesz-Markov-Kakutani and Banach-Alaoglu, this implies that the functionals 
    \begin{align} \label{eq:G_i}
        G_i (\mu) := \int_{\cB_{L-1} \times [0,\infty]} \sigma_s(g(\vx_i)) \ d \mu(s,g), \qquad i = 1, \dots, N
    \end{align}
    are weak* continuous on $\cM(\cB_{L-1} \times [0,\infty])$. Combined with the assumption that $\cL$ is lower semicontinuous in its second argument and the fact that any norm on a dual Banach space is weak* lower semicontinuous, we see that the objective functional $G$ of  \eqref{opt:reg_loss_mu_constrained} is a weak* lower semicontinuous functional, and the constraint set of the optimization is weak* compact. The Weierstrass extreme value theorem for general topological spaces (\cite{aliprantis2006infinite}, Theorem 2.41) thus implies existence of a solution to \eqref{opt:reg_loss_mu_constrained}, hence to \eqref{opt:reg_loss_mu} and \eqref{opt:reg_loss}.

    \paragraph{Existence of solutions to \eqref{opt:interp}.} 
    First, we argue that the interpolation constraint in problem \eqref{opt:interp} is feasible. All activations in \cref{tab:activation_summary} except the linear unit $\sigma(t) = t$ are continuous and nonpolynomial, so by Theorem 5.1 in \cite{pinkus1999approximation}, there is always some shallow network in $\cV_2$ which fits the data. For $\ReLU$ and $\ReLU^m$, \cref{prop:relu_nested_containment_entropy_lb,prop:relu_m_nested_containment_entropy_lb} imply that this shallow interpolating network is also in $\cV_L$. By \cref{lemma:relu_nonrelu_space_equivalence}, this shallow interpolating ReLU network must also be contained in the deep spaces $\cV_L$ associated with the Leaky ReLU, GELU, SiLU/Swish, Mish, ELU, SELU, CELU, centered softplus, absolute value, and bent identity. For all remaining activation functions in \cref{tab:activation_summary}, the limiting normalized activations $\sigma_0$ are linear, so any shallow interpolating network for any of these activations is also included in the corresponding deep space for that activation. Therefore, in all cases, interpolation of the given data with some $\cV_L$ function is possible.

    Having shown that the interpolation constraint in \eqref{opt:interp} is feasible, we may recast \eqref{opt:interp} as
    \begin{equation} \label{opt:interp_mu}
     \min_{\mu \in \cM(\cB_{L-1} \times [0,\infty])}  \| \mu \|_{\TV} \ , \ \mbox{subject to } \| \mu \|_{\TV} \leq C_0, \  G_i(\mu) = y_i, \  i = 1, \dots, N
    \end{equation}
    where $C_0 = \| \mu_0 \|_{\TV}$ for some arbitrary $\mu_0 \in \cM(\cB_{L-1} \times [0,\infty])$ satisfying $G_i(\mu_0) = y_i$ for $i = 1, \dots, N$. The preimage of the closed singleton set $\{ y_i \}$ under the weak* continuous map $G_i$ is weak* closed. The intersection of these weak* closed sets for $i = 1, \dots, N$ is weak* closed. The feasible set of \eqref{opt:interp_mu} is the intersection of this weak* closed set with the weak* compact (by Banach-Alaoglu) set $\{ \mu: \| \mu \|_{\TV} \leq C_0 \}$, and thus is weak* compact. Because norms on dual spaces are weak* lower semicontinuous, the Weierstrass extreme value theorem again implies existence of solutions to \eqref{opt:interp_mu} and thus to \eqref{opt:interp}.
    
    \paragraph{Solutions to \eqref{opt:reg_loss} and \eqref{opt:interp} are finite linear combinations.} Let $\tilde{\mu}$ be a measure which solves either \eqref{opt:reg_loss_mu} or \eqref{opt:interp_mu}. This $\tilde{\mu}$ must also solve
    \begin{equation} \label{opt:interp_mu_y_tilde}
     \min_{\mu \in \cM(\cB_{L-1} \times [0,\infty])}  \| \mu \|_{\TV} \ , \ \mbox{subject to } \| \mu \|_{\TV} \leq \| \tilde{\mu} \|_{\TV}, \   G_i(\mu) = G_i(\tilde{\mu}), \  i = 1, \dots, N.
    \end{equation}
    If this were not true, there would be some other measure achieving the same data-fitting loss/satisfying the interpolation constraint with smaller total variation norm than $\tilde{\mu}$, contradicting optimality of $\tilde{\mu}$. For the same reason, there are no measures which satisfy the equality constraints of \eqref{opt:interp_mu_y_tilde} with total variation norm strictly less than $\| \mu \|_{\TV}$. Therefore, the solution set $\cS^*$ of \eqref{opt:interp_mu_y_tilde} is exactly its feasible set, and any solution to \eqref{opt:interp_mu_y_tilde} also solves the original problem \eqref{opt:reg_loss_mu} or \eqref{opt:interp_mu}. As argued above for \eqref{opt:interp_mu}, $\cS^*$ is weak* compact, and it is also convex by convexity of the total variation norm $\| \cdot \|_{\TV}$. Because dual Banach spaces are locally convex with respect to the weak* topology (\cite{rudin1991functional}, Theorem 3.1 and discussion on p. 68), the Krein-Milman theorem (\cite{aliprantis2006infinite}, Theorem 4.103) implies that $\cS^*$ is the the weak* closed convex hull of its extreme points.

    Let $\mu$ be any extreme point of $\cS^*$. We claim that $\mu = \sum_{i=1}^N c_i \delta_{s_i,g_i}$ for some $s_1, \dots, s_N \in [0,\infty]$ and $g_1, \dots, g_N \in \cB_{L-1}$. By \cref{lemma:integral_rep}, this claim implies that the function 
    \begin{align}
        f(\vx) = \sum_{i=1}^N c_i \sigma_{s_i} (g_i(\vx))
    \end{align}
    solves the original problem \eqref{opt:reg_loss} or \eqref{opt:interp}. 
    
    To prove the claim, it suffices to prove that there are at most $N$ disjoint Borel subsets of $\cB_{L-1} \times [0,\infty]$ with nonzero $\mu$-measure,\footnote{In more detail: if $\cG_1, \dots, \cG_n$ (with $n \leq N$) is a maximal collection of disjoint Borel sets with nonzero $\mu$-measure, then each $\cG_i$ is an \textit{atom}, i.e., $\mu(\cG_i) \neq 0$ and every Borel subset $\cA$ of each $\cG_i$ has either $\mu(\cA) = 0$ or $\mu(\cG_i \setminus \cA) = 0$. (Otherwise, there would be some $\cA \subset \cG_i$ with $\mu(\cA) \neq 0$ and $\mu(\cG_i \setminus \cA) \neq 0$, yielding $n+1$ disjoint Borel sets with nonzero $\mu$-measure.) Because the atoms of $\mu$ coincide with those of $|\mu|$ (\cite{kadets2018course}, p. 184, Exercise 3) and $\cB_{L-1} \times [0,\infty]$ is a compact, each atom $\cG_i$ has $|\mu|(\cG_i \Delta \{ g_i \}) = |\mu|(\cG_i  \setminus \{ g_i \}) + |\mu|(\{g_i\} \setminus \cG_i)  = 0$ for some $g_i \in \cB_{L-1}$ (\cite{kadets2018course}, p. 45, Theorem 2). This fact, along with the identity $\cA = (\cA \cap \cB) \cup (\cA \setminus \cB)$, implies that $|\mu|(\cG_i) = |\mu|(\cG_i \cap \{ g_i \}) = |\mu|(\{ g_i \} ) > 0$. For any Borel set $\cA$, we thus have
    \begin{equation}
        |\mu|(\cA) = \sum_{i=1}^n |\mu|(\cA \cap \cG_i) = \sum_{i=1}^n |\mu|(\cA \cap \{ g_i \}) = \sum_{i=1}^n |\mu|(\{g_i\}) \mathbbm{1}_{g_i \in \cA} = \sum_{i=1}^n |\mu|(\{g_i\}) \delta_{g_i}(\cA)
    \end{equation}
    Finally, because $\mu \ll |\mu|$, any Borel set $\cA$ with $\cA \cap \{g_1, \dots, g_n \} = \emptyset$ has $\mu(\cA) = 0$. We conclude that
    \begin{equation}
        \mu(\cA) = \sum_{i=1}^n \mu(\cA \cap \{ g_i \}) = \sum_{i=1}^n \mu(\{g_i\}) \mathbbm{1}_{g_i \in \cA} = \sum_{i=1}^n \mu(\{g_i\}) \delta_{g_i}(\cA)
    \end{equation}
    as desired.
    } which we proceed to do following the argument of \cite{fisher1975spline}. Assume by contradiction that there are $N+1$ disjoint Borel subsets $(\cG_1, \cS_1), \dots, (\cG_N, \cS_N)$ of $\cB_{L-1} \times [0,\infty]$ with nonzero $\mu$-measure, and let $(\cG, \cS)$ be the union of these. For each $j = 1, \dots, N+1$, let $\mu_j$ be the restriction of $\mu$ to $(\cG_j, \cS_j)$, defined as $\mu_j(\cA, \cB) = \mu(\cA \cap \cS_j, \cB \cap \cG_j)$ for all Borel subsets $(\cA, \cB)$ of $\cB_{L-1} \times [0,\infty]$. Let $\vu_j \in \R^{N}$ be the vector with $i$th coordinate given by
    \begin{equation}
        u_{j,i} = G_i(\mu_j).
    \end{equation}
    The vectors $\vu_1, \dots \vu_{N+1} \in \R^N$ are linearly dependent, meaning that there are constants $a_1, \dots, a_{N+1} \in \R$ (not all zero) with $\sum_{j=1}^{N+1} a_j \vu_j = \vzero$. Let $\nu = \sum_{j=1}^{N+1} a_j \mu_j$. Observe that $\nu$ is \textit{not} the zero measure, because at least one of the $a_j$ is nonzero and, under our assumption that the sets $\cS_j \times \cG_j$ have nonzero $\mu$-measure, none of the $\mu_j$ are the zero measure. For any $i = 1, \dots, N$, we therefore have
    \begin{equation}
        G_i(\nu) =  \sum_{j=1}^{N+1} a_j G_i(\mu_j) =  \sum_{j=1}^{N+1} a_j \mu_{j,i} = 0
    \end{equation}
    and thus
    \begin{equation}
        G_i(\mu + \epsilon \nu) =  G_i(\mu) 
    \end{equation}
    for any $\epsilon \in \R$. Additionally, for any $\epsilon \in \R$ we have
    \begin{align}
        \| \mu + \epsilon \nu \|_{\TV} = | \mu | \left( \cS^c, \cG^c \right) + \sum_{j=1}^{N+1}  | (1 + \epsilon a_j) \mu | (\cS_j,\cG_j) 
    \end{align}
    As long as the magnitude of $\epsilon$ is sufficiently small (in particular, if $\epsilon > -1/a_j$ for $a_j > 0$ and $\epsilon < -1/a_j$ for $a_j < 0$), the above is equal to 
    \begin{align}
        \| \mu + \epsilon \nu \|_{\TV} &= | \mu | \left( \cS^c, \cG^c \right) + \sum_{j=1}^{N+1} (1+\epsilon a_j)  |  \mu | (\cS_j,\cG_j) \\
        &= \| \mu \|_{\TV} + \epsilon \sum_{j=1}^{N+1} a_j |\mu|(\cS_j,\cG_j) \\
        &= \| \mu \|_{\TV} + \epsilon \sum_{j=1}^{N+1} a_j \| \mu_j \|_{\TV}.
    \end{align}
    If $\sum_{j=1}^{N+1} a_j \| \mu_j \|_{\TV} \neq 0$, then any choice of $\epsilon \neq 0$ whose sign is opposite $\sum_{j=1}^{N+1} a_j \| \mu_j \|_{\TV}$ would yield $\| \mu + \epsilon \nu \|_{\TV} < \| \mu \|_{\TV}$, contradicting optimality of $\mu$. Hence it must be that $\sum_{j=1}^{N+1} a_j \| \mu_j \|_{\TV} = 0$, and both $\mu + \epsilon \nu$ and $\mu - \epsilon \nu$ are in $S^*$ for any $\epsilon$ of sufficiently small magnitude. But $\mu$ is a nontrivial convex combination of $\mu + \epsilon \nu$ and $\mu - \epsilon \nu$, contradicting the fact that $\mu$ is an extreme point of $\cS^*$.

    \paragraph{Final form of the solution.} Let $f(\vx) = \sum_{i=1}^N c_i \sigma_{s_i}(g_i(\vx))$ be the solution whose existence we have shown above. For any fixed $i = 1, \dots, N$, consider the optimization
    \begin{equation} \label{opt:interp_h}
     \min_{\tilde{g}_i \in \cV_{L-1}} \| \tilde{g}_i  \|_{\cV_L} \ , \ \mbox{subject to } \tilde{g}_i (\vx_j)=g_i(\vx_j), \ j=1,\dots,N.
    \end{equation}
    Repeating the previous arguments shows that \eqref{opt:interp_h} admits a solution of the form $\tilde{g}_i (\vx) = \sum_{i=1}^N \tilde{c}_i \sigma_{\tilde{s}_i}(h_i(\vx))$ for some $\tilde{s}_1, \dots, \tilde{s}_N \in [0,\infty]$ and $h_1, \dots, h_N \in \cB_{L-2}$. Furthermore, because $g_i$ itself is feasible for problem \eqref{opt:interp_h}, the solution $\tilde{g}_i$ must have $\| \tilde{g}_i \|_{\cV_{L-1}} \leq \| g_i \|_{\cV_{L-1}}$, so that $\tilde{g}_i \in \cB_{L-1}$. This shows that $\tilde{f}(\vx) = \sum_{i=1}^N c_i \sigma_{s_i}(\tilde{g}_i (\vx))$ is also a solution to the original problem \eqref{opt:reg_loss} or \eqref{opt:interp}. Repeating this argument recursively in $L$, we see that each term in the $N$-term linear combination of a solution at layer $L$ is itself given by an $N$-term linear combination, which proves the result.
\end{proof}

\subsection{Proof of \cref{prop:rep_cost_comparison}} \label{appendix:proof_rep_cost_comparison}
\begin{proof}
We address the cases individually.
\paragraph{Bound in terms of $R_{\mathrm{SOSW}} (\vtheta)$.}
    Consider the $\ell^2$-$\ell^2$ operator norm $\| \cdot \|_{2 \to 2}$ (also called the \textit{spectral norm}), which is defined as
    \begin{align}
        \| \mA \|_{2 \to 2} := \sup_{\| \vx \|_2=1} \| \mA \vx \|_2.
    \end{align}
    Observe that the spectral norm is upper bounded by the Frobenius norm, since
    \begin{align}
        \| \mA \vx \|_2^2 = \sum_i \left( \sum_j A_{ij} x_j \right)^2 \leq \sum_i \left( \sum_j A_{ij}^2 \right) \left( \sum_j x_j^2 \right) = \| \mA \|_F^2
    \end{align}
    whenever $\| \vx \|_2 = 1$. Using this fact along with \eqref{eq:path_norm_prod} and submultiplicativity of the operator norm, we have:
    \begin{align}
        \| f \|_{\cV_L} \leq \Phi(\vtheta) \leq \| \vw^{(L)} \|_2 \left( \prod_{\ell=2}^L \| \mW^{(\ell)} \|_{2 \to 2} \right) \| \vs \|_2 \leq \| \vw^{(L)} \|_2 \left( \prod_{\ell=2}^L \| \mW^{(\ell)} \|_{F} \right) \| \vs \|_2. 
    \end{align}
    Therefore, by the arithmetic-geometric mean (AM-GM) inequality:
    \begin{align} \label{eq:phi_int_sosw_bound}
        \Phi(\vtheta)^{2/L} \leq \left( \| \vw^{(L)} \|_2^2 \left( \prod_{\ell=2}^L \| \mW^{(\ell)} \|_{F}^2 \right) \| \vs \|_2^2 \right)^{1/L} \leq \frac{\| \vw^{(L)} \|_2^2 + \sum_{\ell=2}^{L-1} \| \mW^{(\ell)} \|_F^2 + \| \vs \|_2^2 }{L}.
    \end{align}
    Furthermore, under the theorem assumptions, \eqref{eq:s_k_assumption} is  satisfied by taking
    \begin{align} \label{eq:s_k_choice_equiv_proof}
        s_k = \max \left\{ \| \vw_k^{(1)} \|_2, |b_k^{(1)}| \right\}.
    \end{align}
    With this choice, we have
    \begin{align}
        \| \vs \|_2^2 = \sum_{k=1}^{K_1} \max \left\{ \| \vw_k^{(1)} \|_2^2, |b_k^{(1)}|^2 \right\} \leq \| \mW^{(1)} \|_F^2 + \| \vb^{(1)} \|_2^2.
    \end{align}
    Combining this with \eqref{eq:phi_int_sosw_bound} yields the $R_{\mathrm{SOSW}} (\vtheta)$ bound.

    \paragraph{Bound in terms of $R_{\mathrm{Ba}} (\vtheta)$.} For the $R_{\mathrm{Ba}}$ bound, observe that
    \begin{align} \label{eq:ba_upper_bound_init}
        \| f \|_{\cV_L} \leq \Phi(\vtheta) \leq \Psi_1(\vtheta) := \| \vw^{(L)} \|_1 \left( \prod_{\ell=2}^{L-2} \| \mW^{(\ell)} \|_{1,\infty}  \right) \| \vs \|_\infty \leq \| \vw^{(L)} \|_1  \prod_{\ell=2}^{L-2} \| \mW^{(\ell)} \|_{1,\infty} 
    \end{align}
    where $\Psi_1$ is the upper bound from \eqref{eq:path_norm_upper_bound_psi}, and the final inequality follows by choosing each $s_k \leq 1$ (which is possible under the theorem assumptions). Furthermore, any matrix $\mA$ has
    \begin{align}
        \| \mA \|_{1,\infty} := \max_i \sum_j |a_{i,j}| \leq  \max_i \sum_j  \| \mA_{:,j} \|_2 = \| \mA \|_{2,1}.
    \end{align}
    Using this fact along with the AM-GM inequality, the rightmost expression in \eqref{eq:ba_upper_bound_init} is further upper bounded by
    \begin{align}
        \| \vw^{(L)} \|_1  \prod_{\ell=2}^{L-2} \| \mW^{(\ell)} \|_{2,1} \leq \left( \frac{\| \vw^{(L)} \|_1 +  \sum_{\ell=2}^{L-2} \| \mW^{(\ell)} \|_{2,1}}{L-1} \right)^{L-1} \leq \left( \frac{R_{\mathrm{Ba}} (\vtheta)}{L-1} \right)^{L-1}.
    \end{align}
    \paragraph{Bound in terms of $R_{\mathrm{Sh}}$ and $R_{\mathrm{Pa}}$.}
    Here we choose each $s_k$ to satisfy
    \begin{align} \label{eq:choice_sk_sh_pa}
        s_k := \| \mU_{k,:}^{(1)} \|_2 = \sqrt{\| \vw_k^{(1)} \|_2^2 + |b_k^{(1)}|^2}.
    \end{align}
    This choice satisfies
    \begin{align}
        \| \vw_k^{(1)} / s_k \|_2 \leq 1, \qquad |b_k^{(1)}/ s_k | \leq 1
    \end{align}
    and thus satisfies the requirement \eqref{eq:s_k_assumption}. 
    
    Next, using \eqref{eq:path_norm_prod}: 
    \begin{align} \label{eq:sh_pa_bound_int}
        \| f \|_{\cV_L} \leq \Phi(\vtheta) &=  \big| \vw^{(L)} \big|^\top \big| \mW^{(L-1)} \big| \dots \big| \mW^{(2)} \big| |\vs| \\
        &= \left| \mV^{(L-1)} \right| \left| \mU^{(L-1)} \mV^{(L-2)} \right| \dots \left| \mU^{(2)} \mV^{(1)} \right| |\vs| \\
        &\leq \left| \mV^{(L-1)} \right| \left| \mU^{(L-1)} \right| \left| \mV^{(L-2)} \right| \dots \left| \mU^{(2)} \right| \left| \mV^{(1)} \right| |\vs| \\
        &\leq \left( \prod_{\ell=2}^{(L-1)} \left\| \left| \mV^{(\ell)} \right| \left| \mU^{(\ell)} \right|  \right\|_{2 \to 2} \right) \left\| \left| \mV^{(1)} \right| |\vs| \right\|_{2 \to 2}.
    \end{align}
    First inequality above is the triangle inequality, and the second is submultiplicativity of the operator norm. Now observe that
    \begin{align}
        \left\| \vu \vv^\top \right\|_{2 \to 2} = \sup_{\| \vx \|_2 = 1} \| \vu \vv^\top \vx \|_2 = \sup_{\| \vx \|_2 = 1}  \| \vu  \|_2 \left( \vv^\top \vx \right) = \| \vu \|_2 \| \vv \|_2
    \end{align}
    for any $\vu, \vv$. As a result, any $\mA, \mB$ satisfy
    \begin{align}
        \left\| |\mA| |\mB| \right\|_{2 \to 2} = \left\| \sum_k |\mA_{:,k}| |\mB_{k,:}| \right\|_{2 \to 2} \leq  \sum_k \left\| |\mA_{:,k}| |\mB_{k,:}| \right\|_{2 \to 2} = \| \mA_{:,k} \|_2 \| \mB_{k,:} \|_2.
    \end{align}
    Applying this to \eqref{eq:sh_pa_bound_int} and recalling the choice \eqref{eq:choice_sk_sh_pa} of $s_k$, we see that
    \begin{align}
        \| f \|_{\cV_L} \leq \Phi(\vtheta) \leq \left( \frac{R_{\mathrm{Sh}}(\vtheta)}{L-1} \right)^{L-1} \leq \left( \frac{R_{\mathrm{Pa}}(\vtheta)}{L-1} \right)^{L-1}.
    \end{align}
    Here we have again used the AM-GM inequality, along with the fact that $R_{\mathrm{Sh}}(\vtheta) \leq R_{\mathrm{Pa}}(\vtheta)$, which follows from $\| \vv \|_2 \leq \| \vv \|_1$.
\end{proof}

\subsection{Proof of \cref{lemma:rad_bound}} \label{appendix:proof_rad_bound}
\begin{proof}
    Fix some $\vx_1, \dots, \vx_N \in \Omega$ and let
    \begin{align}
        \widehat{\cR}_N(\cB_L) := \bE_{r_1, \dots, r_N \, \overset{\mathrm{iid}}{\sim} \, \mathrm{Rad}}  \left[ \sup_{f \in \cB_L} \frac{1}{N}  \sum_{i=1}^N r_i f(\vx_i) \right]
    \end{align}
    denote the empirical Rademacher complexity of $\cB_L$ for this dataset. We will show that the stated bound applies to $\widehat{\cR}_N(\cB_L)$. Because the $\vx_i$ are arbitrary, this implies that the bound also holds for the worst-case Rademacher complexity $\cR_N(\cB_L)$.

    We will use the strategy of \cite{golowich2018size} to obtain a bound on $\widehat{\cR}_N(\cB_L)$ for which the explicit dependence on depth $L$ is linear, rather than exponential. This argument uses a variant of the Ledoux-Talagrand contraction inequality (\cite{ledoux1991probability}, Equation 4.20 in) which \cite{golowich2018size} apply to the 1-Lipschitz $\sigma = \ReLU$. In our setup, this argument works as-is to bound the Rademacher complexities of our classes whenever the function $\sigma$ is homogeneous and locally Lipschitz: however for non-homogeneous activations, there is no longer a single activation function $\sigma$ but rather an infinite family of normalized activations $\sigma_s, s > 0$. In order to adapt the argument to this case, we cover the family of normalized activations with radius $\delta$, approximate each $\sigma_s$ with its nearest covering element, and then apply the contraction inequality to these finitely many covering elements. This approximation step incurs some relatively modest error when repeated at each layer, which is the source of the $\sqrt{\log N}$ in the numerator along with the $\frac{L-1}{N}$ factor in the final bound (neither of which is present in the bound for homogeneous activations, where this covering/approximation step is not necessary). 
    \paragraph{Covering bound on the class of restricted functions $\sigma_s$.} As claimed in \eqref{eq:covering_bound_sigma_s}, we will first demonstrate that for any $R > 0$, the covering bound
    \begin{align} \label{eq:covering_bound_sigma_s}
        \cN \left( \left\{ \sigma_s \big\rvert_{[-R,R]} : s > 0 \right\}, \delta, \| \cdot \|_\infty \right) \leq \frac{A_R}{\delta} + 1
    \end{align}
    holds with some constant $A_R > 0$. We need only focus on the nonhomogeneous activations in \cref{tab:activation_summary} since, if the activation is homogeneous, the set of functions $\sigma_s$ is simply the singleton set $\sigma$. All non-homogeneous activations $\sigma$ in \cref{tab:activation_summary} are either globally $C^\infty$ smooth (infinitely differentiable) or piecewise $C^\infty$ smooth with finitely many pieces. Therefore, the same holds for
    \begin{align}
        \sigma(s,t) := \frac{\sigma(st)}{s}
    \end{align}
    when viewed solely as a function of the variable $s \in (0,\infty)$, with $t \in \R$ fixed. This implies that $\sigma(s,t)$, as a function of $s$, is locally absolutely continuous with a.e. partial derivative
    \begin{align}
        \frac{\partial}{\partial s} \sigma(s,t) = \frac{st \sigma'(st) - \sigma(st)}{s^2}.
    \end{align}
    Because this a.e. partial derivative obeys the fundamental theorem of calculus (\cite{leoni2017first}, Theorem 3.30), we have
    \begin{align}
        \sigma_{s_1} (t) - \sigma_{s_0}(t) = \int_{s_0}^{s_1} \frac{\partial}{\partial s} \sigma(s,t) \ ds
    \end{align}
    for any $t \in [-R,R]$ and any $0 < s_0 < s_1$. Therefore:
    \begin{align}
        \left\| \sigma_{s_1} \big\rvert_{[-R,R]} - \sigma_{s_0} \big\rvert_{[-R,R]} \right\|_\infty &= \sup_{t \in [-R,R]} |\sigma(s_1,t) - \sigma(s_0,t) | \leq \int_{s_0}^{s_1} \left\| \frac{\partial}{\partial s} \sigma(s,t) \right\|_\infty ds \\
        &\leq  \int_0^\infty \left\| \frac{\partial}{\partial s} \sigma(s,t) \right\|_\infty ds =   \int_0^\infty \sup_{|t| \leq R} \left| \frac{st \sigma'(st) - \sigma(st)}{s^2} \right| ds =: A_R \label{eq:AR_def}
    \end{align}
    We will now show that the value $A_R$ of the integral in \eqref{eq:AR_def} above is finite for all activations $\sigma$ in \cref{tab:activation_summary}. To do so, write
    \begin{align}
        A_R = \underbrace{\int_0^1 \sup_{|t| \leq R} \left| \frac{st \sigma'(st) - \sigma(st)}{s^2} \right| ds}_{A_1} +  \underbrace{\int_1^\infty \sup_{|t| \leq R} \left| \frac{st \sigma'(st) - \sigma(st)}{s^2} \right| ds}_{A_2}. \label{eq:AR_A1_A2}
    \end{align}
    For the first integral $A_1$: if $\sigma$ is globally $C^\infty$, then for any $u \in \R$, there are $\xi_u$ and $\zeta_u$ between 0 and $u$ such that
    \begin{align}
        \sigma(u) = u \sigma'(0) + \frac{1}{2} u^2 \sigma''(\xi_u), \  \textrm{and} \ \sigma'(u) = \sigma'(0) + u \sigma''(\zeta_u).
    \end{align}
    Here we have used Taylor's theorem as well as $\sigma(0) = 0$, which holds for all activations in \cref{tab:activation_summary}. Therefore, for any $0 < s \leq 1$, we have
    \begin{align}
        \sup_{|t| \leq R} |st \sigma'(st) - \sigma(st)| = \sup_{|t| \leq R} \left| s^2 t^2 \sigma''(\zeta_{st}) - \frac{1}{2} s^2 t^2 \sigma''(\xi_{st}) \right| \leq s^2 R^2 \sup_{|u| \leq R} \frac{3}{2} |\sigma''(u)|.
    \end{align}
    The $s^2$ in the numerator of this bound cancels with the $s^2$ in the denominator of the integrand in $A_1$, so the integral $A_1$ converges. If $\sigma$ is only piecewise $C^\infty$, the same argument applies by further breaking up the integral $A_1$ into individual integrals along each $C^\infty$ piece (of which there are finitely many). The second integral $A_2$ converges because each of the activations $\sigma$ in \cref{tab:activation_summary} satisfies
    \begin{align}
        \sup_{u \in \R} |u \sigma'(u) - \sigma(u)| < \infty.
    \end{align}
    This fact can be verified in each case individually. In general, this property holds because each $\sigma$ is either ``sigmoid-like'' (with $\sigma$ bounded and $\sigma'(u) \to 0$ as $u \to \pm \infty$, so that $u \sigma'(u)$ is also bounded) or ``ReLU-like'' (in which case $\sigma(u)$ is asymptotically affine as $u \to \pm \infty$, so that $\sigma(u) \approx u \sigma'(u) + c$).

    Having shown that the value of $A_R$ in \eqref{eq:AR_def} is finite: fix $\delta > 0$ and partition the interval $[0, A_R]$ as $0 = u_0 < \dots < u_m = A_R$, where each interval $[u_i, u_{i+1}]$ has length exactly $\delta$ (so that $m = A_R/\delta$). For each $u_i$, choose a corresponding $s_i$ such that
    \begin{align}
        u_i = \int_0^{s_i} \left\| \frac{\partial}{\partial s} \sigma(s,t) \right\|_\infty ds.
    \end{align}
    Here we can choose $s_0 = 0$ and $s_m = \infty$. Then, for any $s_i < s < s_{i+1}$, we have
    \begin{align}
        \left\| \sigma_{s_{i+1}} \big\rvert_{[-R,R]} - \sigma_{s_i}\big\rvert_{[-R,R]} \right\|_\infty \leq \int_{s_i}^{s_{i+1}} \left\| \frac{\partial}{\partial s} \sigma(s,t) \right\|_\infty = u_{i+1} - u_i = \delta.
    \end{align}
    Therefore, the interior values $s_1, \dots, s_{m-1}$ cover the set
    \begin{align}
        \left\{ \sigma_s \big \rvert_{[-R, R]} : s > 0 \right\}
    \end{align}
    which proves \eqref{eq:covering_bound_sigma_s}.
    
    \paragraph{Empirical Rademacher complexity is invariant to uniform closures.} Next, we establish the generic fact that $\widehat{\cR}_N(\overline{\cS}) = \widehat{\cR}_N(\cS)$ for any $\cS \subset C(\Omega)$. Taken together with \cref{prop:limit_final_layer}, this tells us that the Rademacher complexity of $\cB_L$ is equal to that of $\widetilde{\cB}_L$, so the uniform limits need not be considered when bounding the former. 
    
    Let $r_1, \dots, r_N \in \{ \pm 1 \}$ be any fixed signs. For the given signs and data points $\vx_1, \dots, \vx_N$, let $\phi(f) := \sum_{i=1}^N r_i f(\vx_i)$. It suffices to show that 
    \begin{align} \label{eq:A_S_equality}
        C_{\cS} := \sup_{f \in \cS} \phi(f) =  \sup_{f \in \overline{\cS}} \phi(f) =: C_{\overline{\cS}}
    \end{align}
    First note that if $\cS$ is bounded, i.e. $\sup_{f \in \cS} \| f \|_\infty =: C_{\cS} < \infty$, both $C_\cS$ and $_{\overline{\cS}}$ are finite. Finiteness of $C_{\cS}$ follows from $|C_{\cS}| \leq N C_{\cS}$. Finiteness of $C_{\overline{\cS}}$ follows from the same inequality, along with the fact that any $f \in \overline{\cS}$ has
    \begin{align}
        | \| f \|_\infty - \| f_n \|_\infty| \leq \| f - f_n \|_\infty \to 0
    \end{align}
    and thus $\| f \|_\infty = \lim_{n \to \infty} \| f_n \|_\infty \leq C_{\cS}$ for some $\{ f_n \}_{n=1}^\infty \subset \cS$.

    Next, let $\{ g_n \}_{n=1}^\infty \subset \overline{\cS}$ be a sequence with $\phi(g_n) \to C_{\overline{\cS}}$. Each $g_n$ is the uniform (hence pointwise) limit of a sequence $\{ h_{n,m} \}_{m=1}^\infty \subset \cS$, and therefore $\lim_{m \to \infty} \phi(h_{m,n}) = \phi(g_n)$ for each $n \in \bN$. For any $\epsilon > 0$, fix $n$ (depending on $\epsilon$) large enough that $|\phi(g_n) - C_{\overline{\cS}}| \leq \epsilon/2$, and fix $m$ large enough (depending on $n$ and $\epsilon$) that $|\phi(h_{m,n}) - \phi(g_n)| \leq \epsilon/2$. Choosing $m,n$ in this way for some sequence of $\epsilon \downarrow 0$, we form a subsequence $\{ h_{m,n,k} \}_{m,n,k \in \bN}$ satisfying $\lim_{k \to \infty} \phi(h_{m,n,k}) = C_{\overline{\cS}}$. This shows that $C_{\cS} \geq C_{\overline{\cS}}$. The reverse inequality  $C_{\cS} \leq C_{\overline{\cS}}$ follows from $\cS \subset \overline{\cS}$, so both quantities are equal.

    \paragraph{Reduction to a bound on the base class (homogeneous case).} Note that for any $\lambda > 0$ we have
    \begin{align} 
        N \widehat{\cR}_N(\cB_L) &= \bE_r \sup_{f \in cB_L} \sum_{i=1}^N r_i f(\vx_i)  \\
        &\leq \frac{1}{\lambda} \log \bE_r \sup_{K \in \bN, \sum_{k=1}^K |v_k|  \leq 1, \sigma_{s_k} \circ f_k \in \cB_{L-1}} \exp \lambda \left( \sum_{k=1}^K v_k \sum_{i=1}^N r_i  \sigma_{s_k} (f_k(\vx_i)) \right) \label{eq:rad_bound_init} \\
        &\leq \frac{1}{\lambda} \log \bE_r \sup_{K \in \bN, \sum_{k=1}^K |v_k| \leq 1, \sigma_{s_k} \circ f_k \in \cB_{L-1}} \exp \lambda \left(  \sum_{k=1}^K |v_k| \cdot \max_{k=1, \dots, K} \left| \sum_{i=1}^N r_i  \sigma_{s_k}(f_k(\vx_i)) \right| \right) \\
        &\leq \frac{1}{\lambda} \log \bE_r \sup_{K \in \bN, \sigma_{s_k} \circ f_k \in \cB_{L-1}} \exp \lambda \left( \max_{k=1, \dots, K} \left| \sum_{i=1}^N r_i  \sigma_{s_k} (f_k(\vx_i)) \right| \right) \\
        &= \frac{1}{\lambda} \log \bE_r \sup_{f \in \cB_{L-1}, s > 0} \exp \lambda \left( \left| \sum_{i=1}^N r_i \sigma_s(f(\vx_i)) \right| \right) \\
        &\leq \frac{1}{\lambda} \log \bE_r \sup_{f \in \cB_{L-1}, s > 0} \left( \exp \lambda \left( \sum_{i=1}^N r_i \sigma_s(f(\vx_i))  \right) + \exp \lambda \left( -\sum_{i=1}^N r_i \sigma_s(f(\vx_i))  \right)\right) \\
        &\leq \frac{1}{\lambda} \log 2 \bE_r \sup_{f \in \cB_{L-1}, s > 0} \exp \lambda \left( \sum_{i=1}^N r_i \sigma_s (f(\vx_i))  \right) \label{eq:rad_bound_pre_contraction}
    \end{align}
    The first inequality above is Jensen's; the fourth uses $\exp(\lambda |t|) \leq \exp(\lambda t) + \exp(-\lambda t)$; and the final inequality follows from symmetry in distribution of Rademacher random variables. 
    
    At this stage, if $\sigma$ is homogeneous, we have $\sigma_s = \sigma$ for all $s > 0$. In this case, we can apply equation 4.20 in \cite{ledoux1991probability} to the contraction $\sigma/\rho_{L-1}$ (or  more precisely to its globally Lipschitz extension; see e.g. \cite{mcshane1934extension}) to bound \eqref{eq:rad_bound_pre_contraction} above by
    \begin{align} \label{eq:rad_bound_post_contraction_homogeneous}
        \frac{1}{\lambda} \log 2 \bE_r \sup_{f \in \cB_{L-1}} \exp \rho_{L-1} \lambda \left( \sum_{i=1}^N r_i f(\vx_i)  \right).
    \end{align}
    Repeating the above steps yields the homogeneous ``base class'' bound
    \begin{align} \label{eq:base_rad_bound_homogeneous}
        N \widehat{\cR}_N(\cB_L) &\leq \frac{1}{\lambda} \log 2^{L-1} \bE_s \sup_{f \in \cB_1} \exp \lambda \Pi_L \left( \sum_{i=1}^N s_i f(\vx_i) \right)  \\
        &= \frac{1}{\lambda} \log 2^{L-1} \bE_s \sup_{\widetilde{\vw} \in \cW \times \cB} \exp \lambda \Pi_L \left( \sum_{i=1}^N s_i \widetilde{\vw}^\top \widetilde{\vx}_i \right) \label{eq:rad_base_class_bound_homogeneous}
    \end{align}
    where $\widetilde{\vx}_i := [\vx_i^\top, 1]^\top \in \R^{d+1}$ and $\widetilde{\vw} := [\vw^\top, b]^\top$ for $\vw \in \cW, b \in \cB$.
    
    \paragraph{Reduction to a bound on the base class (non-homogeneous case).} The bound in lines \eqref{eq:rad_bound_init}-\eqref{eq:rad_bound_pre_contraction} is also valid in the non-homogeneous case. However, the term in \eqref{eq:rad_bound_pre_contraction} can no longer be handled directly with an application of the contraction lemma, so here we require an additional approximation step using the covering bound \eqref{eq:covering_bound_sigma_s}. Fix $\delta_{L-1} > 0$ and form a finite cover of cardinality $M_{L-1} := \lceil A_{L-1}/\delta_{L-1}\rceil \leq 1 + A_{L-1}/\delta_{L-1}$ of the function family
    \begin{align}
        \left\{ \sigma_s \big \rvert_{[-C_{L-1}, C_{L-1}]} : s > 0 \right\}.
    \end{align}
    Regarding the constant $A_{L-1}$, note that the depth-dependence of this constant arises due to possible depth dependence of the constant $C_{L-1}$ to which the $\sigma_s$ are restricted. However, all activations in \cref{tab:activation_summary} \textit{except} $\ReLU^m$, SELU, and bent identity satisfy $|\sigma(t)| \leq |t|$. For such activations, we have $C_\ell \leq C_1$, and therefore $A_\ell \leq A_1$ and $\rho_\ell \leq \rho_1$, for all $\ell$.
    
    Let $\sigma_{j(s)}$ denote the closest covering element to any $\sigma_s$. Then
    \begin{align}
        \left| \sum_{i=1}^N r_i \left( \sigma_s(f(\vx_i)) - \sigma_{j(s)} (f(\vx_i)) \right) \right| \leq N \delta_{L-1} \implies \sum_{i=1}^N r_i \sigma_s (f(\vx_i)) \leq \sum_{i=1}^N r_i \sigma_{j(s)} (f(\vx_i)) + N \delta_{L-1}
    \end{align}
    so \eqref{eq:rad_bound_pre_contraction} is upper bounded by
    \begin{align}
        &\frac{1}{\lambda} \log 2 \bE_r \sup_{f \in \cB_{L-1},s} e^{\lambda N \delta_{L-1}} \exp  \lambda \left( \sum_{i=1}^N r_i \sigma_{j(s)} ( f(\vx_i) ) \right) \\ 
        &= N \delta_{L-1} + \frac{1}{\lambda} \log 2 \bE_r \max_{j=1, \dots, M_{L-1}} \sup_{f \in \cB_{L-1}} \exp  \lambda \left( \sum_{i=1}^N r_i \sigma_{j} ( f(\vx_i) ) \right) \\
        &\leq N \delta_{L-1} + \frac{1}{\lambda} \log 2 \sum_{j=1}^{M_{L-1}} \bE_r \sup_{f \in \cB_{L-1}} \exp  \lambda \left( \sum_{i=1}^N r_i \sigma_{j} ( f(\vx_i) ) \right) \\
        &\leq N \delta_{L-1}  + \frac{1}{\lambda} \log 2 M_{L-1} \bE_r \sup_{f \in \cB_{L-1}} \exp \rho_{L-1} \lambda \left( \sum_{i=1}^N r_i  f(\vx_i)  \right) \\
        &\leq N \delta_{L-1}  + \frac{1}{\lambda} \log M_{L-1} +  \frac{1}{\lambda} \log 2  \bE_r \sup_{f \in \cB_{L-1}} \exp \rho_{L-1} \lambda \left( \sum_{i=1}^N r_i  f(\vx_i)  \right). \label{eq:rad_bound_post_contraction_nonhomogeneous}
    \end{align}
    In this case, repeating the previous steps yields the non-homogeneous ``base class'' bound
    \begin{align} \label{eq:base_rad_bound_nonhomogeneous}
        N \widehat{\cR}_N(\cB_L) &\leq N \sum_{\ell=1}^{L-1} \delta_\ell \pi_\ell + \frac{1}{\lambda} \sum_{\ell=1}^{L-1} \log M_{\ell} +  \frac{1}{\lambda} \log 2^{L-1} \bE_s \sup_{f \in \cB_1} \exp \lambda \Pi_L \left( \sum_{i=1}^N s_i f(\vx_i) \right)  \\
        &= N \sum_{\ell=1}^{L-1} \delta_\ell \pi_\ell + \frac{1}{\lambda} \sum_{\ell=1}^{L-1} \log M_{\ell} +  \frac{1}{\lambda} \log 2^{L-1} \bE_s \sup_{\widetilde{\vw} \in \cW \times \cB} \exp \lambda \Pi_L \left( \sum_{i=1}^N s_i \widetilde{\vw}^\top \widetilde{\vx}_i \right) \\
        &\leq N \sum_{\ell=1}^{L-1} \delta_\ell \pi_\ell + \frac{1}{\lambda} \sum_{\ell=1}^{L-1} \log (1 + A_\ell/\delta_\ell) +  \frac{1}{\lambda} \log 2^{L-1} \bE_s \sup_{\widetilde{\vw} \in \cW \times \cB} \exp \lambda \Pi_L \left( \sum_{i=1}^N s_i \widetilde{\vw}^\top \widetilde{\vx}_i \right). \label{eq:rad_base_class_bound_nonhomogeneous}
    \end{align}

    \paragraph{Bound the base class entropy.} Recall the constant $C_{\cW, \cB} := \sup_{\widetilde{\vw} \in \cW \times \cB} \| \widetilde{\vw} \|_1$ from the lemma statement, and denote $C := C_{\cW,\cB}  \Pi_L$ for notational convenience. The expression in \eqref{eq:rad_base_class_bound_homogeneous} (which is also the final expression in \eqref{eq:rad_base_class_bound_nonhomogeneous}) is upper bounded by
    \begin{align}
        &\frac{1}{\lambda} \log 2^{L-1} \bE_s \exp \lambda C \left\| \sum_{i=1}^N s_i \widetilde{\vx}_i \right\|_\infty \leq \frac{1}{\lambda} \log 2^{L-1} \max_j \bE_s \exp \lambda C \left| \sum_{i=1}^N s_i \widetilde{x}_{i,j} \right| \\
        &\leq \frac{1}{\lambda} \log 2^{L-1} \sum_{j=1}^{d+1} \bE_s \exp \lambda C \left| \sum_{i=1}^N s_i \widetilde{x}_{i,j} \right| \\
        &\leq \frac{1}{\lambda} \log 2^{L-1} \sum_{j=1}^{d+1}  \bE_s \left[  \exp \left( \lambda C  \sum_{i=1}^N s_i \widetilde{x}_{i,j} \right) +   \exp \left( -\lambda C  \sum_{i=1}^N s_i \widetilde{x}_{i,j} \right) \right]   \\
        &\leq \frac{1}{\lambda} \log 2^L \sum_{j=1}^{d+1}  \bE_s  \exp \left( \lambda C  \sum_{i=1}^N s_i \widetilde{x}_{i,j} \right) = \frac{1}{\lambda} \log 2^L \sum_{j=1}^{d+1} \prod_{i=1}^N  \bE_s  \exp \left( \lambda C   s_i \widetilde{x}_{i,j} \right) \\
        &= \frac{1}{\lambda} \log 2^L \sum_{j=1}^{d+1} \prod_{i=1}^N  \frac{\exp(\lambda C \widetilde{x}_{i,j}) + \exp(-\lambda C \widetilde{x}_{i,j})}{2} \leq \frac{1}{\lambda} \log 2^L \sum_{j=1}^{d+1} \prod_{i=1}^N  \exp \left( \lambda^2 C^2 \widetilde{x}_{i,j}^2 \right) \\
        &= \frac{1}{\lambda} \log 2^L \sum_{j=1}^{d+1}  \exp \left( \lambda^2 C^2 \sum_{i=1}^N \widetilde{x}_{i,j}^2 \right) \leq \frac{1}{\lambda} \log 2^L (d+1) \max_j  \exp \left( \lambda^2 C^2 \sum_{i=1}^N \widetilde{x}_{i,j}^2 \right) \\
        &= \frac{L \log 2 + \log(d + 1)}{\lambda} + C^2 \lambda \max_j \sum_{i=1}^N \widetilde{x}_{i,j}^2 
    \end{align}
    where we have used $\frac{\exp(t) + \exp(-t)}{2} \leq \exp(z^2/2) \leq \exp(z^2)$. In the homogeneous case, choosing
    \begin{align}
        \lambda = \frac{1}{C} \sqrt{\frac{L \log 2 + \log(d+1)}{\max_j \sum_{i=1}^N \widetilde{x}_{i,j}^2}}
    \end{align}
    yields
    \begin{align}
        N \widehat{\cR}_N(\cB_L) &\leq 2 C  \sqrt{L \log 2 + \log(d+1)} \sqrt{\max_j \sum_{i=1}^N \widetilde{x}_{i,j}^2} \\
        &\leq 2 C \sqrt{L \log 2 + \log(d+1)} \sqrt{N C_\Omega} 
    \end{align}
    where $C_\Omega := \max \left\{ 1, \sup_{\vx \in \Omega} \| \vx \|_\infty^2 \right\}$. This demonstrates the homogeneous bound \eqref{eq:rad_bound_homogeneous}. In the non-homogeneous case, instead choose
    \begin{align}
        \lambda = \frac{1}{C} \sqrt{\frac{L \log 2 + \log(d+1) + \sum_{\ell=1}^{L-1} \log (1+A_\ell/\delta_\ell)}{\max_j \sum_{i=1}^N \widetilde{x}_{i,j}^2}}.
    \end{align}
    Then:
    \begin{align}
        N \widehat{\cR}_N(\cB_L) &\leq N \sum_{\ell=1}^{L-1} \delta_\ell \pi_\ell  +   \frac{ \sum_{\ell=1}^{L-1} \log (1+A_\ell/\delta_\ell) +  L \log 2 + \log(d + 1) }{\lambda} + C^2 \lambda \max_j \sum_{i=1}^N \widetilde{x}_{i,j}^2 \\ &= N \sum_{\ell=1}^{L-1} \delta_\ell \pi_\ell  +   2 C  \sqrt{N C_\Omega}  \sqrt{\sum_{\ell=1}^{L-1} \log(1+A_\ell/\delta_\ell) +  L \log 2 + \log(d+1)}
    \end{align}
    which proves \eqref{eq:rad_bound_nonhomogeneous}.
\end{proof}

\subsection{Proof of \cref{th:entropy_bound}} \label{appendix:proof_entropy_bound}
\begin{proof}
    The general proof strategy is as follows. We first convert the Rademacher complexity bounds from \cref{lemma:rad_bound} into \textit{empirical} metric entropy bounds: i.e., bounds on the $L^p(\mu_N)$ metric entropies for empirical measures $\mu_N := \frac{1}{N} \sum_{i=1}^N \delta_{\vx_i}$. This is accomplished using known inequalities relating Rademacher complexity, fat-shattering dimension, and empirical covering numbers (\cite{srebro2010smoothness,mendelson2001geometric,alon1997scale}). In the $1 \leq p \leq 2$ case, a variant of Sudakov's minoration inequality---which bounds empirical $L^2$ metric entropy directly in terms of Rademacher complexity---yields an improved bound which removes a polylog factor. Finally, we use uniform boundedness and Lipschitzness of the classes $\cB_L$ (or, more precisely, of the $C(\Omega)$-closed classes $\cB_L^\infty$---the $L^p(\mu)$-closed classes $\cB_L^{L^p(\mu)}$ may not be uniformly bounded or Lipschitz) to translate these empirical $L^p(\mu_N)$ entropy bounds into general $L^p(\mu)$ entropy bounds, for arbitrary finite measures $\mu$. By \cref{prop:Lp_C_representative}, any $L^p(\mu)$ cover of the $C(\Omega)$-closed class $\cB_L^\infty$ yields an $L^p(\mu)$ cover of the $L^p(\mu)$-closed class $\cB_L^{L^p(\mu)}$ and vice versa, so these $L^p(\mu)$ entropy bounds for the classes $\cB_L^\infty$ also apply to the classes $\cB_L^{L^p(\mu)}$.

    \paragraph{Rademacher complexity bounds the fat-shattering dimension.} Recall that a set of points $\vx_1, \dots, \vx_N$ is said to be \textit{fat-shattered} at scale $\epsilon > 0$ (or simply \textit{$\epsilon$-shattered}) by a function class $\cH$ if there exist real numbers $r_1, \dots, r_N$ which satisfy the following property: for any binary labeling $y_1, \dots, y_N \in \{ \pm 1 \}$ of the points $\vx_1, \dots, \vx_N$, there is some $h \in \cH$ such that
    \begin{align}
        \textrm{$h(\vx_i) \geq r_i + \epsilon$ if $y_i = 1$}, \qquad \textrm{$h(\vx_i) \leq r_i - \epsilon$ if $y_i = -1$}.
    \end{align}
    The \textit{fat-shattering dimension} of $\cH$, denoted $\fat_\epsilon(\cH)$, is the maximum number of points which can be $\epsilon$-shattered by $\cH$. By Lemma A.2 in \cite{srebro2010smoothness}, the fat-shattering dimension of a symmetric\footnote{\cite{srebro2010smoothness} use a slightly different definition of Rademacher complexity which takes the absolute value of the Rademacher average. For a symmetric function class (such as our classes $\cB_L$), their definition and our definition in \eqref{eq:worst_case_empirical_rad_def} are equivalent.} function class $\cH$ can be bounded as
    \begin{align} \label{eq:fat_rad_bound}
        \fat_\epsilon(\cH) \leq 4N \cR_N(\cH)^2 \epsilon^{-2}
    \end{align}
    for any $N \in \bN$ and any $\epsilon > \cR_N(\cH)$. 
    
    With this in mind, fix some such fat-shattering scale $\epsilon$. Using the Rademacher bounds \eqref{eq:rad_bound_homogeneous} and \eqref{eq:rad_bound_nonhomogeneous_delta}, we will choose $N(\epsilon)$ in a way that satisfies the assumption $\epsilon > \cR_{N(\epsilon)} (\cB_L)$ of \eqref{eq:fat_rad_bound}. If $\sigma$ is homogeneous, for any given $\epsilon > 0$, we can simply select
    \begin{align}
        N > 4 C_{\cW, \cB, \Omega}^2 \Pi_L^2 \left(L \log 2 + \log(d+1) \right) \epsilon^{-2}
    \end{align}.
    The homogeneous Rademacher bound \eqref{eq:rad_bound_homogeneous} then implies that
    \begin{align} \label{eq:fat_rad_bound_homogeneous}
        \fat_{\epsilon}(\cB_L) \leq 4N \left( \frac{4 C_{\cW, \cB, \Omega}^2 \Pi_L^2 \left(L \log 2 + \log(d+1) \right)}{N} \right) \epsilon^{-2} = 16 C_{\cW, \cB, \Omega}^2 \Pi_L^2 \left(L \log 2 + \log(d+1) \right) \epsilon^{-2}.
    \end{align}

    In the non-homogeneous case, first choose $\delta_\ell(\epsilon)$ so that the second term in the Rademacher bound \eqref{eq:rad_bound_nonhomogeneous_delta} satisfies
    \begin{align} \label{eq:delta_ell_choice}
        \sum_{\ell=1}^{L-1} \delta_\ell(\epsilon) \pi_\ell \leq \frac{\epsilon}{4} \implies \delta_\ell(\epsilon) := \frac{\epsilon}{4(L-1) \max\{ 1, \pi_\ell\}}.
    \end{align}
    For notational convenience, let
    \begin{align} \label{eq:tau_L_epsilon}
        \tau_L(\epsilon) &:= L \log 2 + \log(d+1) + \sum_{\ell=1}^{L-1} \log (1+ A_\ell/\delta_\ell(\epsilon)) \\
        &= L \log 2 + \log(d+1) + \sum_{\ell=1}^{L-1} \log \left(1+ \frac{4(L-1) A_\ell \max\{ 1, \pi_\ell \}}{\epsilon} \right).
    \end{align}
    We now choose some $N(\epsilon)$ large enough to satisfy $\cR_{N(\epsilon)} (\cB_L) \leq \epsilon/2 < \epsilon$:
    \begin{align}
        \cR_{N(\epsilon)}(\cB_L) \leq 2 C_{\cW, \cB,\Omega} \Pi_L \sqrt{\frac{\tau_L(\epsilon)}{N(\epsilon)}} + \frac{\epsilon}{4} \leq \frac{\epsilon}{2} \implies N(\epsilon) \geq 64 C_{\cW, \cB, \Omega}^2 \Pi_L^2 \tau_L(\epsilon) \epsilon^{-2}.
    \end{align}
    Choosing $N(\epsilon)$ to be equal to this threshold value, we apply \eqref{eq:fat_rad_bound} to get
    \begin{align} \label{eq:fat_rad_bound_nonhomogeneous}
        \fat_\epsilon(\cB_L) \leq 4 N(\epsilon) \cR_{N(\epsilon)}(\cB_L)^2 \epsilon^{-2}  \leq 256 C_{\cW, \cB, \Omega}^2 \Pi_L^2 \tau_L(\epsilon) \epsilon^{-2}. 
    \end{align}

    \paragraph{Fat-shattering dimension bounds empirical $L^\infty$ metric entropy.} We now invoke Lemma 3.5 in \cite{alon1997scale} (see also Theorem 1 in \cite{mendelson2001geometric}) to bound the empirical $L^\infty(\mu_N)$ metric entropy of the normalized class
    \begin{align} \label{eq:B_L_normalized}
        \cB_L' := \left\{ \frac{f + C_L}{2 C_L}: f \in \cB_L \right\}
    \end{align}
    in terms of its fat-shattering dimension, where $\mu_N := \frac{1}{N} \sum_{i=1}^N \delta_{\vx_i}$ is the empirical measure associated with some $\vx_1, \dots, \vx_N \in \Omega$. (Working with the normalized class is necessary here to employ the bound of \cite{alon1997scale}, which is only stated for $[0,1]$-valued functions.) The translated and rescaled class has
    \begin{align}
        \fat_\epsilon(\cB_L') = \fat_{2 C_L \epsilon}(\cB_L).
    \end{align}
    Using the homogeneous fat-shattering bound \eqref{eq:fat_rad_bound_homogeneous}, this gives us
    \begin{align} \label{eq:L_inf_empirical_bound_homogeneous}
        \log \cN(\cB_L', \epsilon, L^\infty(\mu_N)) &\leq \log 2 + \fat_{\epsilon/4} (\cB_L') \log \left( \frac{eN}{\epsilon \fat_{\epsilon/4} (\cB_L')} \right) \log \left( \frac{4N}{\epsilon^{2}} \right) \\
        &\leq \log 2 + \fat_{\epsilon/4} (\cB_L') \log \left( \frac{eN}{\epsilon} \right) \log \left( \frac{4N}{\epsilon^{2}} \right) \\
        &\leq \log 2 + \fat_{\epsilon/4} (\cB_L') \log^2 \left( \frac{4N}{\epsilon^2} \right) \\
        &= \log 2 + \fat_{C_L \epsilon/2} (\cB_L) \log^2 \left( \frac{4N}{\epsilon^2} \right) \\
        &\leq \log 2 +  256 C_{\cW, \cB, \Omega}^2 \Pi_L^2 \left(L \log 2 + \log(d+1) \right) (2 C_L \epsilon)^{-2} \log^2 \left( \frac{4N}{\epsilon^2} \right).
    \end{align}
    In the non-homogeneous case, the bound \eqref{eq:fat_rad_bound_nonhomogeneous} instead yields
    \begin{align} \label{eq:L_inf_empirical_bound_nonhomogeneous}
        \log \cN(\cB_L', \epsilon, L^\infty(\mu_N)) &\leq \log 2 + \fat_{\epsilon/4} (\cB_L') \log \left( \frac{eN}{\epsilon \fat_{\epsilon/4} (\cB_L')} \right) \log \left( \frac{4N}{\epsilon^{2}} \right) \\
        &\leq \log 2 + \fat_{C_L \epsilon/2} (\cB_L) \log^2 \left( \frac{4N}{\epsilon^2} \right) \\
        &\leq \log 2 +  4096 C_{\cW, \cB, \Omega}^2 \Pi_L^2 \tau_L(C_L \epsilon/2) (2 C_L \epsilon)^{-2} \log^2 \left( \frac{4N}{\epsilon^2} \right).
    \end{align}
    We will later transfer these empirical bounds into general $L^\infty(\mu)$ metric entropy bounds for $\cB_L'$, where $\mu$ is any finite measure on $\Omega$, and then translate the resulting bounds back to $\cB_L$. This will then imply that that the same bounds hold in $L^p(\mu)$ (up to explicit constant factors) for all $p < \infty$.

    \paragraph{Improved empirical $L^2$ entropy bounds via Sudakov's minoration (homogeneous case).} We now demonstrate improved bounds (with a polylog factor removed) on the empirical $L^2$ metric entropy of $\cB_L$. After subsequently converting these improved bounds into metric entropy bounds for arbitrary finite measures, they will also apply to $1 \leq p \leq 2$ up to explicit constant factors. We consider the homogeneous case first. 
    
    For an arbitrary empirical measure $\delta_N := \frac{1}{N} \sum_{i=1}^N \delta_{\vx_i}$ on $\Omega$ and arbitrary $\epsilon > 0$, let $f_1, \dots, f_m$ be a maximal $\epsilon$-packing of $\cB_L$ with respect to $L^2(\mu_N)$. Define the vectors $\vt_1, \dots, \vt_m \in \R^N$ by $t_{j,i} := f_j(\vx_i)/N$. Then for each $j,j'$ we have
\begin{align}
    \| \vt_j - \vt_{j'} \|_2 = \frac{1}{N} \sqrt{\sum_{i=1}^N \left( f_j(\vx_i) - f_{j'}(\vx_i) \right)^2} = \frac{1}{\sqrt{N}} \| f_j - f_{j'} \|_{L^2(\mu_N)} \geq \frac{\epsilon}{\sqrt{N}}
\end{align}
Furthermore, each $t_j$ satisfies $\| t_j \|_\infty \leq \max_{j,i} |f_j(\vx_i)|/N \leq C_L/N$. Applying Theorem 6.4.1 (Sudakov's minoration for Rademacher processes) in \cite{talagrand2021upper} with $a = \epsilon/\sqrt{N}$ and $b = C_L/N$, we have
\begin{align}
    \cR_N(\cB_L) &= \bE \sup_{f \in \widetilde{\cB}_L} \frac{1}{N} \sum_{i=1}^N s_i f(\vx_i) \geq \bE \max_{j=1, \dots, m} \frac{1}{N} \sum_{i=1}^N s_i f_j(\vx_i) \\
    &\geq c \min \left\{ \frac{\epsilon}{\sqrt{N}} \sqrt{\log m}, \frac{\epsilon^2}{C_L} \right\} = \frac{c}{\sqrt{N}} \min \left\{ \epsilon \sqrt{\log m}, \frac{\epsilon^2 \sqrt{N}}{C_L} \right\}
\end{align}
for some universal constant $c > 0$. The above display and \cref{lemma:rad_bound} imply that
\begin{align} \label{eq:min_bound}
     \min \left\{ \epsilon \sqrt{\log m}, \frac{\epsilon^2 \sqrt{N}}{C_L} \right\} \leq \frac{2 C_{\cW, \cB, \Omega} \Pi_L \sqrt{L \log 2 + \log(d+1)}}{c}.
\end{align}
Now choose $\epsilon$ so that
\begin{align}
    \frac{\epsilon^2 \sqrt{N}}{C_L} &\geq \frac{2 C_{\cW, \cB, \Omega} \Pi_L \sqrt{L \log 2 + \log(d+1)}}{c} \\
    \implies \epsilon &\geq \sqrt{\frac{2 C_L C_{\cW, \cB, \Omega} \Pi_L \sqrt{L \log 2 + \log(d+1)}}{c \sqrt{N}}}. \label{eq:eps_cond_sudakov_homogeneous}
\end{align}
Whenever $\epsilon$ satisfies this condition, the minimum in \eqref{eq:min_bound} is given by the $
\epsilon \sqrt{\log m}$ term (otherwise the bound \eqref{eq:min_bound} would not hold), and therefore
\begin{align}
    \min \left\{ \epsilon \sqrt{\log m}, \frac{\epsilon^2 \sqrt{N}}{C_L} \right\} = \epsilon \sqrt{\log m} &\leq \frac{2 C_{\cW, \cB, \Omega} \Pi_L \sqrt{L \log 2 + \log(d+1)}}{c} \\
    \implies \log m &\leq \frac{4 C_{\cW, \cB, \Omega}^2 \Pi_L^2 \left( L \log 2 + \log(d+1) \right)}{c^2 \epsilon^2}. 
\end{align}
Because the packing number $m$ is an upper bound on the covering number, we thus have
\begin{align} \label{eq:empirical_L2_bound_homogeneous}
    \log \cN(\cB_L, \epsilon, L^2(\mu_N)) \leq \frac{4 C_{\cW, \cB, \Omega}^2 \Pi_L^2 \left( L \log 2 + \log(d+1) \right)}{c^2 \epsilon^2}
\end{align}
for any $\epsilon$ which satisfies \eqref{eq:eps_cond_sudakov_homogeneous}.

\paragraph{Improved empirical $L^2$ entropy bounds via Sudakov's minoration (nonhomogeneous case).}
As in the homogeneous case, we use the Sudakov minoration inequality in \cite{talagrand2021upper} (Theorem 6.4.1). In this case, instead of applying Sudakov directly to the nonhomogeneous Rademacher bound \eqref{eq:rad_bound_nonhomogeneous} as we did in the homogeneous case, we will apply it to the discretized classes $\cB_L^\delta$, which directly replace the family $\{ \sigma_s: s > 0 \}$ at each hidden layer by a $\delta_\ell$ uniform norm-cover of size $A_\ell/\delta_\ell+1$ (see proof of \cref{lemma:rad_bound} in \cref{appendix:proof_rad_bound}). That proof directly bounds the Rademacher complexity of $\cB_L$ by approximating it with $\cB_L^\delta$ at each layer, resulting in the extra additive term $\sum_{\ell=1}^{L-1} \delta_\ell \pi_\ell$ in the final bound. The same argument applied to the classes $\cB_L^\delta$ themselves accordingly yields
\begin{align}
    \cR_N(\cB_L^\delta) \leq 2 C_{\cW, \cB, \Omega} \Pi_L \sqrt{\frac{L \log 2 + \log(d+1) + \sum_{\ell=1}^{L-1} \log(1+A_\ell/\delta_\ell)}{N}}
\end{align}
which is identical to \eqref{eq:rad_bound_nonhomogeneous}, but without the additive $\sum_{\ell=1}^{L-1} \delta_\ell \pi_\ell$ factor. Moreover, each $f \in \cB_L$ is uniformly within $\sum_{\ell=1}^{L-1} \delta_\ell \pi_\ell$ of some $\tilde{f} \in \cB_L^\delta$. Since the uniform norm upper bounds the $L^p(\mu_N)$ norm on $\Omega$ for any $p < \infty$, we thus have
\begin{align}
    \cN \left( \cB_L, \epsilon + \sum_{\ell=1}^{L-1} \delta_\ell \pi_\ell, L^2(\mu_N) \right) \leq \cN \left( \cB_L^\delta, \epsilon, L^2(\mu_N) \right).
\end{align}
Fixing $\epsilon > 0$ and choosing $\delta_\ell(\epsilon)$ as in \eqref{eq:delta_ell_choice}, we have
\begin{align}
    \cR_N(\cB_L^\delta) \leq 2 C_{\cW, \cB, \Omega} \Pi_L \sqrt{\frac{\tau_L(\epsilon)}{N}}
\end{align}
where $\tau_L(\epsilon)$ denotes the expression in \eqref{eq:tau_L_epsilon}. With any such choice of $\epsilon$, we also have $\sum_{\ell=1}^{L-1} \delta_\ell(\epsilon) \pi_\ell \leq \epsilon/4$, and therefore
\begin{align} \label{eq:B_L_B_L_delta_entropy_bound}
    \cN \left( \cB_L, 5 \epsilon/4, L^2(\mu_N) \right) \leq \cN \left( \cB_L^{\delta(\epsilon)}, \epsilon, L^2(\mu_N) \right).
\end{align}
Now let $m(\epsilon)$ denote the packing number of $\cB_L^{\delta(\epsilon)}$. Using the same Sudakov minoration argument as in the homogeneous case, we get
\begin{align} \label{eq:min_bound_nonhomogeneous}
     \min \left\{ \epsilon \sqrt{\log m}, \frac{\epsilon^2 \sqrt{N}}{C_L} \right\} \leq \frac{2 C_{\cW, \cB, \Omega} \Pi_L \sqrt{\tau_L(\epsilon)}}{c}
\end{align}
and therefore
\begin{align} 
    \log \cN(\cB_L^{\delta(\epsilon)}, \epsilon, L^2(\mu_N)) \leq \frac{4 C_{\cW, \cB, \Omega}^2 \Pi_L^2 \tau_L(\epsilon)}{c^2 \epsilon^2}.
\end{align}
whenever
\begin{align} \label{eq:eps_cond_sudakov_nonhomogeneous}
     \epsilon \geq \sqrt{ \frac{2 C_L C_{\cW, \cB, \Omega} \Pi_L \sqrt{\tau_L(\epsilon)}}{c \sqrt{N}}}.
\end{align}
By \eqref{eq:B_L_B_L_delta_entropy_bound}, this implies that
\begin{align} \label{eq:empirical_L2_bound_nonhomogeneous}
    \log \cN(\cB_L, \epsilon, L^2(\mu_N)) \leq \frac{25  C_{\cW, \cB, \Omega}^2 \Pi_L^2 \tau_L(4 \epsilon/5)}{4 c^2 \epsilon^2}.
\end{align}
for any $\epsilon$ satisfying \eqref{eq:eps_cond_sudakov_nonhomogeneous}.

    \paragraph{Convert empirical to general metric entropy bound ($p \leq \infty$).} We now show that the empirical $L^p$ metric entropy bounds derived above imply bounds on the $L^p$ entropies with respect to general finite measures. Let $\mu$ be any finite measure on $\Omega$ and let $\cC_{\Omega,\eta}$ be a minimal Euclidean $\eta$-cover of $\Omega$. For this cover $\cC_{\Omega,\eta}$, let
    \begin{align}
        \mu_{\Omega, \eta} := \frac{1}{|\cC_{\Omega,\eta}|} \sum_{\vx \in \cC_{\Omega,\eta}} \delta_{\vx}
    \end{align}
    be the empirical measure whose atoms are located at the cover elements. For any $\vx \in \Omega$, let $\vomega_{\vx}$ denote the closest element in $\cC_{\Omega,\eta}$ to $\vx$. Additionally, note that all functions in the class $\cB_L$ are Lipschitz with constant $P_L := C_{\cW} \Pi_L$, where $\cW := \sup_{\vx \in \cW} \| \vw \|_2$. This holds because
    \begin{align}
        |\sigma(f(\vx_1)) - \sigma(f(\vx_2))| \leq \rho_1 |\vw^\top \vx_1 + b - \vw^\top \vx_2 - b | \leq \rho_1 C_{\cW} \| \vx_1 - \vx_2 \|_2 
    \end{align}
    for any $f \in \cB_1$ and any $\vx_1, \vx_2 \in \Omega$. Because Lipschitz constants are preserved by absolutely convex combinations and uniform limits, the above demonstrates that the functions in $\cB_2$ are $P_2 := C_{\cW} \rho_1$-Lipschitz. Assuming inductively that this holds for some $L > 2$, we also have
    \begin{align}
        |\sigma(f(\vx_1)) - \sigma(f(\vx_2))| \leq \rho_L |f(\vx_1) - f(\vx_2)| \leq \rho_L  C_{\cW} \Pi_L \| \vx_1 - \vx_2 \|_2 = C_{\cW} \Pi_{L+1} \| \vx_1 - \vx_2 \|_2
    \end{align}
    so the functions in $\cB_{L+1}$ are $P_{L+1} := C_{\cW} \Pi_{L+1}$-Lipschitz. Therefore, the functions in $\cB_L'$ are $P_L' := P_L/(2 C_L)$-Lipschitz.
    
    Using the above facts, we have
    \begin{align}
        \| f - g \|_{L^\infty(\mu)} = \sup_{\vx \in \widetilde{\Omega}} |f(\vx) - g(\vx)| &\leq \sup_{\vx \in \widetilde{\Omega}} |f(\vx) - f(\vomega_{\vx})| + |f(\vomega_{\vx}) - g(\vomega_{\vx})| + |g(\vomega_{\vx}) - g(\vx)| \\
        &\leq \sup_{\vx \in \widetilde{\Omega}} |f(\vomega_{\vx}) - g(\vomega_{\vx}) | + 2 P_L' \| \vx - \vomega_{\vx} \|_2 \\
        &\leq 2 P_L' \eta + \max_{\vomega \in \cC_{\Omega,\eta}} |f(\vomega) - g(\vomega)| \\
        &= 2 P_L' \eta+ \| f - g \|_{L^\infty(\mu_{\Omega, \eta})}
    \end{align}
    for any $f, g \in \cB_L'$, where $\widetilde{\Omega} \subset \Omega$ is some full $\mu$-measure set. This holds for any $\eta > 0$, so choosing $\eta_0 := \epsilon/(4 P_L')$, we see that any $\epsilon/2$ cover of $\cB_L'$ with respect to $L^\infty(\mu_{\Omega,\eta_0})$ yields an $\epsilon$-cover of $\cB_L'$ with respect to $L^\infty(\mu)$. In other words,
    \begin{align} \label{eq:infty_general_empirical_bound}
        \log \cN (\cB_L', \epsilon, L^\infty(\mu)) \leq \log \cN \left( \cB_L', \epsilon/2, L^\infty(\mu_{\Omega,\eta_0}) \right).
    \end{align}

    Let $C_{\Omega} := \sup_{\vx \in \Omega} \| \vx \|_2$. By the standard finite-dimensional volumetric bound (\cite{vershynin2020high}, Corollary 4.2.11), the size of the cover $\cC_{\Omega,\eta_0}$---i.e., the number of atoms in the empirical measure $\mu_{\Omega,\eta_0}$---can be bounded as
    \begin{align} \label{eq:C_omega_covering_bound}
        |\cC_{\Omega,\eta_0}| = \cN(\Omega, \eta_0, \| \cdot \|_2)  \leq \left( \frac{8 P_L' C_{\Omega}}{\epsilon}  + 1\right)^d.
    \end{align}

    Combining \eqref{eq:infty_general_empirical_bound} and \eqref{eq:C_omega_covering_bound} with the homogeneous empirical $L^\infty$ bound \eqref{eq:L_inf_empirical_bound_homogeneous}, we obtain
    \begin{align}
        \log \cN(\cB_L', \epsilon, L^\infty(\mu)) &\leq \log \cN(\cB_L', \epsilon/2, L^\infty(\mu_{\Omega, \eta_0})) \\
        &\leq \log 2 +  256 C_{\cW, \cB, \Omega}^2 \Pi_L^2 \left(L \log 2 + \log(d+1) \right) (C_L \epsilon)^{-2} \log^2 \left( \frac{16}{\epsilon^2} \left( \frac{8 P_L' C_{\Omega}}{\epsilon}  + 1\right)^d \right) \\
        &\leq \log 2 +  256 C_{\cW, \cB, \Omega}^2 \Pi_L^2 \left(L \log 2 + \log(d+1) \right) (C_L \epsilon)^{-2} \log^2 \left( \frac{16}{\epsilon^2} \left( \frac{8 P_L' C_{\Omega} + 1}{\epsilon} \right)^d \right).
    \end{align}
     Applied to the nonhomogeneous bound \eqref{eq:L_inf_empirical_bound_nonhomogeneous}, we instead get
    \begin{align}
        \log \cN(\cB_L', \epsilon, L^\infty(\mu)) &\leq \log \cN(\cB_L', \epsilon/2, L^\infty(\mu_{\Omega, \eta_0})) \\
        &\leq \log 2 +  4096 C_{\cW, \cB, \Omega}^2 \Pi_L^2 \tau_L(C_L \epsilon/4) (C_L \epsilon)^{-2} \log^2 \left( \frac{16}{\epsilon^2} \left( \frac{8 P_L' C_{\Omega}}{\epsilon}  + 1\right)^d \right).
    \end{align}

    Finally, translating and rescaling the class gives
    \begin{align}
        \log \cN(\cB_L, \epsilon, L^\infty(\mu)) = \log \cN\left(\cB_L', \frac{\epsilon}{2 C_L}, L^\infty(\mu)\right).
    \end{align}
    If $\epsilon > 2 C_L$, then the covering number on the left is one, so it suffices to consider $0 < \epsilon \leq 2 C_L$. Substituting $\epsilon/(2 C_L)$ into the preceding homogeneous bound and using $P_L' = P_L/(2 C_L)$, we obtain
    \begin{align}
        \log \cN(\cB_L, \epsilon, L^\infty(\mu)) &\leq \log 2 +  1024 C_{\cW, \cB, \Omega}^2 \Pi_L^2 \left(L \log 2 + \log(d+1) \right) \epsilon^{-2} \log^2 \left( \frac{64 C_L^2}{\epsilon^2} \left( \frac{8 P_L C_{\Omega}}{\epsilon}  + 1\right)^d \right) \\
        &\leq \log 2 +  1024 C_{\cW, \cB, \Omega}^2 \Pi_L^2 \left(L \log 2 + \log(d+1) \right) \epsilon^{-2} \log^2 \left( \frac{64 C_L^2}{\epsilon^2} \left( \frac{8 P_L C_{\Omega} + 1}{\epsilon} \right)^d \right). \label{eq:L_inf_general_bound_homogeneous}
    \end{align}
     Applied to the preceding nonhomogeneous bound, we instead get
    \begin{align}
        \log \cN(\cB_L, \epsilon, L^\infty(\mu)) &\leq \log 2 +  16384 C_{\cW, \cB, \Omega}^2 \Pi_L^2 \tau_L(\epsilon/8) \epsilon^{-2} \log^2 \left( \frac{64 C_L^2}{\epsilon^2} \left( \frac{8 P_L C_{\Omega}}{\epsilon}  + 1\right)^d \right)  \label{eq:L_inf_general_bound_nonhomogeneous}
    \end{align}
    where $\tau_L(\epsilon)$ is as in \eqref{eq:tau_L_epsilon}. 
    
    \paragraph{Simplify $p \leq \infty$ upper bounds into the form \eqref{eq:entropy_bounds_th}.} To convert the bounds \eqref{eq:L_inf_general_bound_homogeneous} and \eqref{eq:L_inf_general_bound_nonhomogeneous} into the more interpretable form \eqref{eq:entropy_bounds_th} in the theorem statement (case $2 < p \leq \infty$), we first show that $\tau_L(\epsilon/8) \lesssim T_L(\epsilon)$ for all $0 < \epsilon \leq 1/2$, where $\tau_L$ is as in \eqref{eq:tau_L_epsilon} and $T_L$ is as in \eqref{eq:T_eps}. We bound the $L \log 2 + \log(d+1)$ term as
    \begin{align} \label{eq:L_log_d_homogeneous_lesssim_bound}
        L \log 2 + \log(d+1) \leq L \log 2 + \log(2d) = (L+1) \log 2 + \log d \leq 2L  + 2 \log d = 2(L + \log d).
    \end{align}
    In the nonhomogeneous case, the sum term in $\tau_L(\epsilon/8)$ is upper bounded by
    \begin{align}
        \sum_{\ell=1}^{L-1} \log \left( 1 + \frac{32 L A_\ell \max \{1, \pi_\ell \}}{\epsilon} \right) \leq \sum_{\ell=1}^{L-1} 32 \log \left( 1 + \frac{L A_\ell \max \{1, \pi_\ell \}}{\epsilon} \right)
    \end{align}
    using Bernoulli's inequality $1 + rx \leq (1+x)^r$, which holds for any real $r \geq 1$, $x \geq -1$. This shows that $\tau_L(\epsilon/8) \lesssim T_L(\epsilon)$ for all $0 < \epsilon \leq 1/2$, as desired.
    
    Next, to handle the $\log^2$ terms in \eqref{eq:L_inf_general_bound_homogeneous} and \eqref{eq:L_inf_general_bound_nonhomogeneous}, first note that
    \begin{align}
        C_L \leq C_1 \Pi_L,
    \end{align}
    which follows inductively from $\sigma_s(0)=0$ and the fact that every $\sigma_s$ is $\rho_\ell$-Lipschitz on $[-C_\ell,C_\ell]$. Consequently, writing $\overline{C}_1 := \max\{1,C_1\}$, we have
    \begin{align} 
        &\log \left( \frac{64 C_L^2}{\epsilon^2} \left( \frac{8 P_L C_{\Omega} + 1}{\epsilon} \right)^d \right) \leq \log \left( \frac{64 \overline{C}_1^2 (\Pi_L + 1)^2}{\epsilon^2} \left( \frac{8 P_L C_{\Omega} + 1}{\epsilon} \right)^d \right) \\
        &\qquad \qquad \qquad \qquad = \log(64 \overline{C}_1^2) + (d+2) \log(\epsilon^{-1}) + 2 \log(\Pi_L+1) + d \log( 8  C_{\cW} C_\Omega \Pi_L + 1).  \label{eq:log_homogeneous_L_inf_expanded}
    \end{align}
    If $8 C_{\cW} C_\Omega \leq 1$, then the final term in \eqref{eq:log_homogeneous_L_inf_expanded} is directly upper bounded by $d \log (\Pi_L + 1)$. On the other hand, if $8 C_{\cW} C_\Omega > 1$, then Bernoulli's inequality implies that the last term is upper bounded by $8 C_{\cW} C_\Omega d \log(\Pi_L + 1)$. For the second term in \eqref{eq:log_homogeneous_L_inf_expanded}, we have $(d+2) \log(\epsilon^{-1}) \leq 3d \log(\epsilon^{-1}) \leq 3d \log((\Pi_L + 1) \epsilon^{-1})$. The third term satisfies $2 \log(\Pi_L+1) \leq 2d \log((\Pi_L+1)\epsilon^{-1})$. For the first term in \eqref{eq:log_homogeneous_L_inf_expanded}: whenever $0 < \epsilon \leq 1/2$, we have $\log 2 \leq \log (\epsilon^{-1}) \leq d \log(\epsilon^{-1}) \leq d \log((\Pi_L+1) \epsilon^{-1})$, and therefore
    \begin{align}
        \log (64 \overline{C_1}^2) \leq \frac{\log (64 \overline{C_1}^2)}{\log 2} d \log((\Pi_L+1) \epsilon^{-1}).
    \end{align}
    
    As a result:
    \begin{align}  \label{eq:log_sq_homogeneous_lesssim_bound}
        \log^2 \left( \frac{64 C_L^2}{\epsilon^2} \left( \frac{8 P_L C_{\Omega} + 1}{\epsilon} \right)^d \right)  \lesssim d^2 \log^2 \left( (\Pi_L + 1) \epsilon^{-1} \right)
    \end{align}
    where $\lesssim$ hides only multiplicative constants which are independent of $\epsilon$, $d$, $L$, $C_L$, and $\Pi_L$. Everything above shows that the main terms in \eqref{eq:L_inf_general_bound_homogeneous} and \eqref{eq:L_inf_general_bound_nonhomogeneous} are bounded (up to universal multiplicative constants) by
    \begin{align}
        \Pi_L^2 T_L(\epsilon) \epsilon^{-2} d^2 \log^2 \left((\Pi_L+1) \epsilon^{-1} \right) \leq (\Pi_L+1)^2 T_L(\epsilon) \epsilon^{-2} d^2 \log^2 \left((\Pi_L+1) \epsilon^{-1} \right) 
    \end{align}
    for $0 < \epsilon \leq 1/2$. The additive $\log 2$ obeys the same $\lesssim$ bound, since
    \begin{align}
        \log^2 2 &\leq \log^2 (\epsilon^{-1}) \leq (\Pi_L+1)^2 T_L(\epsilon) \epsilon^{-2} d^2 \log^2 \left((\Pi_L+1) \epsilon^{-1} \right) \label{eq:Pi_Pi_plus_one_ineq} \\
        &\implies \log 2 \leq (\log 2)^{-1} (\Pi_L+1)^2 T_L(\epsilon) \epsilon^{-2} d^2 \log^2 \left((\Pi_L+1) \epsilon^{-1} \right)
    \end{align}
    whenever $0 < \epsilon \leq 1/2$. (Note that the second inequality in \eqref{eq:Pi_Pi_plus_one_ineq} would not necessarily hold without having replaced the outermost $\Pi_L^2$, which could be small, with $(\Pi_L+1)^2$, which is necessarily greater than one.) This shows that both \eqref{eq:L_inf_general_bound_homogeneous} and \eqref{eq:L_inf_general_bound_nonhomogeneous} are $\lesssim (\Pi_L+1)^2 T_L(\epsilon) \epsilon^{-2} d^2 \log^2 \left((\Pi_L + 1) \epsilon^{-1} \right)$ in the case $p = \infty$. Any $L^\infty(\mu)$ cover of radius $\epsilon$ yields an $L^p(\mu)$ cover of radius $\mu(\Omega)^{1/p} \epsilon$, so the same bound holds for $p < \infty$ by rescaling $\epsilon \mapsto \mu(\Omega)^{1/p} \epsilon$. (The extra constants that this rescaling creates inside the log terms can be handled as above using Bernoulli's inequality.)

    \paragraph{Convert empirical to general metric entropy bound ($p \leq 2$).} The final step is to convert the improved empirical $L^2$ metric entropy bounds \eqref{eq:empirical_L2_bound_homogeneous} and \eqref{eq:empirical_L2_bound_nonhomogeneous} into a general metric entropy bound for $1 \leq p \leq 2$. To do so, let $\mu$ be any probability measure on $\Omega$ and define
    \begin{align} \label{eq:empirical_measure_approx_2}
        \Phi_N(\cB_L) := \sup_{f,g \in \cB_L} \left| \| f - g \|_{L^2(\mu)}^2 - \| f - g \|_{L^2(\mu_N)}^2 \right|.
    \end{align}
    We claim that there is a sequence of points $\{ \vx_i \}_{i=1}^\infty \subset \Omega$ such that $\Phi_N(\cB_L) \to 0$ as $N \to \infty$. This demonstrates that $\mu$ can be approximated uniformly over $\cB_L$ by empirical measures, which is what will allow us to transfer the empirical $L^2$ entropy bound to a general one.
    
    To prove this claim, we first argue that the class $\{ (f-g)^2: f,g \in \cB_L \}$ is also uniformly bounded and Lipschitz. Uniform boundedness is easy to see, since
    \begin{align}
        \| (f-g)^2 \|_\infty = \sup_{\vx \in \Omega} (|f(\vx)| + |g(\vx)|)^2 \leq \sup_{\vx \in \Omega} 4 C_L^2 = 4 C_L^2
    \end{align}
    for each $f,g \in \cB_L$. For the Lipschitz claim, note that
    \begin{align}
        \Lip((f-g)^2) &\leq \Lip(f^2) + 2\Lip(fg) + \Lip(g^2).
    \end{align}
    To bound the $\Lip(fg)$ term on the right, write
    \begin{align}
        |f(\vx_1) g(\vx_1) - f(\vx_2) g(\vx_2)| &\leq |f(\vx_1) g(\vx_1) - f(\vx_1) g(\vx_2)| +| f(\vx_1) g(\vx_2) -f(\vx_2) g(\vx_2)| \\
        &\leq 2 C_L P_L \| \vx_1 - \vx_2 \|_2
    \end{align}
    which shows that $\Lip(fg) \leq 2 C_L P_L$. The same argument applies to the $\Lip(f^2)$ and $\Lip(g^2)$ terms, implying that
    \begin{align}
        \Lip((f-g)^2) \leq 8 C_L P_L.
    \end{align}
 
    Having shown that the class $\{ (f-g)^2: f,g \in \cB_L \}$ is uniformly bounded and Lipschitz, we see that it is uniformly equicontinuous. Furthermore, by Varadarajan's theorem (\cite{dudley2002real}, Theorem 11.4.1), there exists an i.i.d. sequence of $\Omega$-valued random variables $\{ X_i \}_{i=1}^\infty \sim \mu$ which almost surely satisfies
    \begin{align} \label{eq:Varadajaran}
        \frac{1}{N} \sum_{i=1}^N h(X_i) \to \int h \ d \mu, \quad \forall h \in C(\Omega).
    \end{align}
    Letting $\{ \vx_i \}_{i=1}^\infty$ be some realization of the $\{ X_i \}_{i=1}^\infty$ for which this convergence occurs, \eqref{eq:Varadajaran} says that the empirical measures $\mu_N = \frac{1}{N} \sum_{i=1}^N \delta_{\vx_i}$ converge weakly to $\mu$. This fact, along with uniform equicontinuity of $\cB_L$, implies the desired statement $\Phi_N(\cB_L) \to 0$ by Corollary 11.3.4 in \cite{dudley2002real}.

    Next, using the fact that $\Phi_N(\cB_L) \to 0$, we proceed to transfer the empirical $L^2$ metric entropy bound to a general one. We first consider the homogeneous case. Fix $\delta > 0$ and $\eta > 0$. By the convergence $\Phi_N(\cB_L) \to 0$, there is some $N_{\delta,\eta}$ large enough that both
    \begin{align} \label{eq:delta_cond_sudakov_homogeneous}
        \delta \geq \sqrt{\frac{2 C_L C_{\cW, \cB, \Omega} \Pi_L \sqrt{L \log 2 + \log(d+1)}}{c \sqrt{N}}} \qquad \textrm{(which satisfies the condition \eqref{eq:eps_cond_sudakov_homogeneous})}
    \end{align}
    and $\Phi_N(\cB_L) \leq \eta$ for any $N \geq N_{\delta, \eta}$. Fix some $N \geq N_{\delta,\eta}$ and let $\cC_{\delta,\eta,N}$ be a minimal $\delta$-cover of $\cB_L$ in $L^2(\mu_N)$. For any $f \in \cB_L$, there is some $g \in \cC_{\delta,\eta,N}$ such that
    \begin{align}
        \| f - g \|_{L^2(\mu)}^2 \leq \| f - g \|_{L^2(\mu_N)}^2 + \Phi_N(\cB_L) \leq \delta^2 + \eta.
    \end{align}
    This shows that $\cC_{\delta,\eta,N}$ is an $L^2(\mu)$ cover of $\cB_L$ of radius $\epsilon := \sqrt{\delta^2 + \eta}$. Therefore, by \eqref{eq:empirical_L2_bound_homogeneous}:
    \begin{align}
        \log \cN(\cB_L, \epsilon, L^2(\mu)) \leq \log \cN(\cB_L, \delta, L^2(\mu_N)) \leq \frac{4 C_{\cW, \cB, \Omega}^2 \Pi_L^2 \left( L \log 2 + \log(d+1) \right)}{c^2 \delta^2}.
    \end{align}
    Repeating this process for the same fixed $\delta$ and a sequence $\eta \downarrow 0$ yields
    \begin{align} \label{eq:L2_general_entropy_bound_homogeneous}
        \log \cN(\cB_L, \epsilon, L^2(\mu)) \leq \frac{4 C_{\cW, \cB, \Omega}^2 \Pi_L^2 \left( L \log 2 + \log(d+1) \right)}{c^2 \epsilon^2}
    \end{align}
    for any probability measure $\mu$ on $\Omega$.  To extend this bound to $1 \leq p \leq 2$ and any finite measure $\mu$ on $\Omega$, simply use the fact that
    \begin{align}
        \| f \|_{L^p(\mu)} \leq  \mu(\Omega)^{1/p-1/2} \| f \|_{L^2(\mu)} = \mu(\Omega)^{1/p}  \| f \|_{L^2(\tilde{\mu})}, \qquad 1 \leq p \leq 2
    \end{align}
    where $\tilde{\mu} := \mu/\mu(\Omega)$ is the normalization of $\mu$. This shows that every $L^2(\tilde{\mu})$ cover of radius $\mu(\Omega)^{-1/p} \epsilon$ yields an $L^p(\mu)$ cover of radius $\epsilon$, so \eqref{eq:L2_general_entropy_bound_homogeneous} also applies in these cases with $\epsilon$ replaced by $\mu(\Omega)^{-1/p} \epsilon$. This proves the $1 \leq p \leq 2$ case of the homogeneous bound in \eqref{eq:entropy_bounds_th}. The argument in the nonhomogeneous case is exactly the same, except we instead choose $N_{\delta,\eta}$ large enough that $\delta$ satisfies the lower bound \eqref{eq:eps_cond_sudakov_nonhomogeneous} rather than \eqref{eq:eps_cond_sudakov_homogeneous}. The resulting nonhomogeneous bound is
    \begin{align} \label{eq:L2_general_entropy_bound_nonhomogeneous}
        \log \cN(\cB_L, \epsilon, L^2(\mu)) \leq \frac{25  C_{\cW, \cB, \Omega}^2 \Pi_L^2 \tau_L(4 \epsilon/5)}{4 c^2 \epsilon^2}.
    \end{align}
    for any probability measure $\mu$ on $\Omega$. As shown for the $p = \infty$ bound above, we have $\tau_L(4 \epsilon/5) \lesssim T_L(\epsilon)$. Here, too, the same $\lesssim$ bound holds general finite measures $\mu$ and any $1 \leq p < 2$ by rescaling $\epsilon \mapsto \mu(\Omega)^{1/p} \epsilon$.
\end{proof}

\subsection{Example of a function class with small Rademacher complexity and large metric entropy} \label{appendix:rad_entropy_example}
\paragraph{Classes of steep disjoint bump functions have small Rademacher complexity.} Here we discuss an example of a function class with small Rademacher complexity, but large metric entropy as measured in the ambient function space norm. Let $\{ \sigma_j \}_{j=1}^\infty \subset C[0,1]$ be a sequence of disjoint bump functions, each of height $\| \sigma_j \|_\infty = 1$, such that $\sigma_j$ is supported on the interval
\begin{align} \label{eq:psi_j_dyadic_support}
    [1-2^{-j-1}, 1-2^{-j}].
\end{align}
Let $\{ a_j \}_{j=1}^\infty \subset \R$ be any sequence which decreases monotonically to 0, and consider the class $\cH := \{ a_j \sigma_j \}_{j=1}^\infty$. Visually, the class $\cH$ is a sequence of disjoint bump functions $\sigma_j$ on $[0,1]$, each of increasingly narrow width, with corresponding heights $a_j$ (see \cref{fig:class_H}).

\begin{figure}
    \centering

    \begin{subfigure}[t]{0.48\textwidth}
        \centering
        \includegraphics[width=\textwidth]{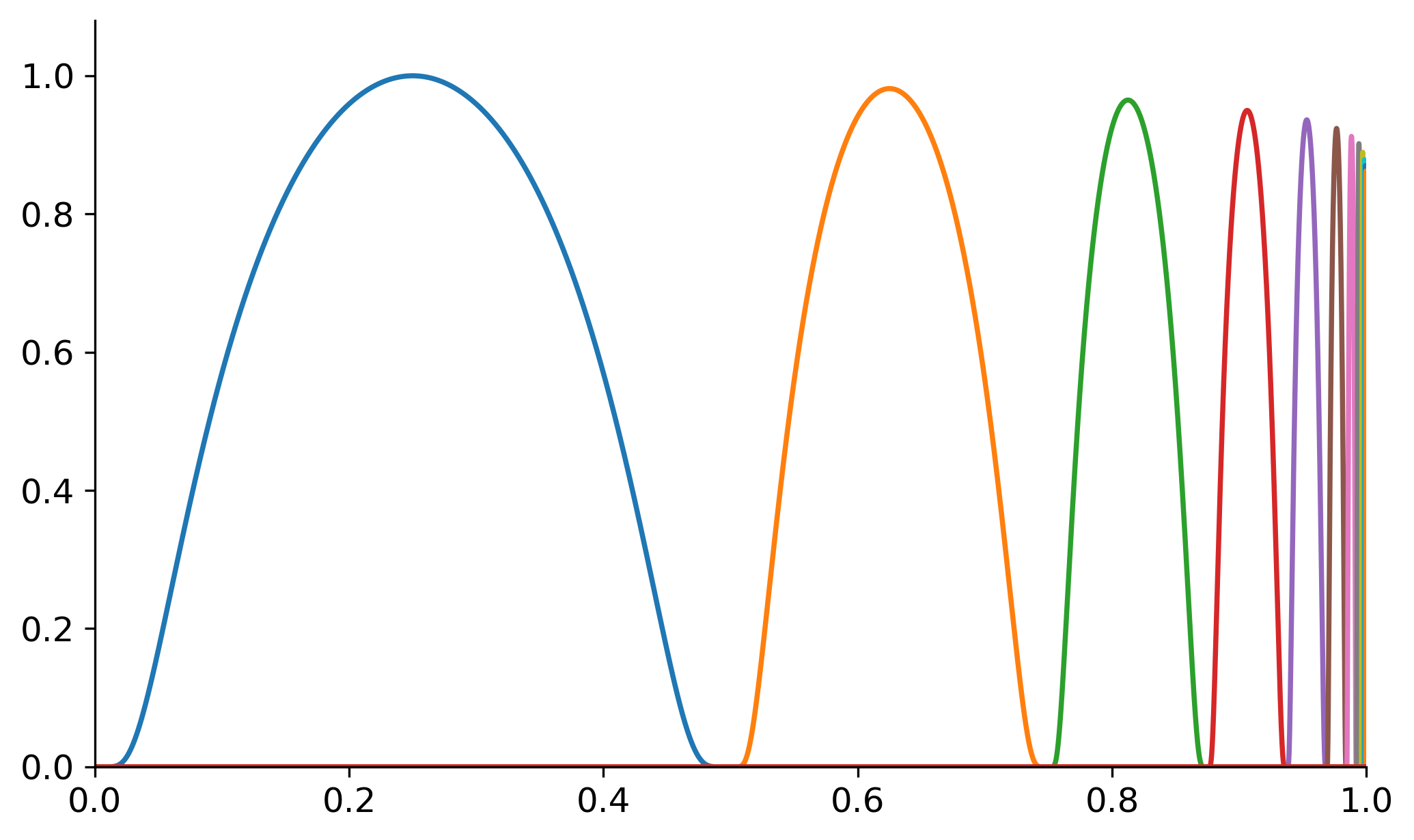}
        \caption{Disjoint bump function class $\cH$.}
        \label{fig:first-image}
    \end{subfigure}
    \hfill
    \begin{subfigure}[t]{0.48\textwidth}
        \centering
        \includegraphics[width=\textwidth]{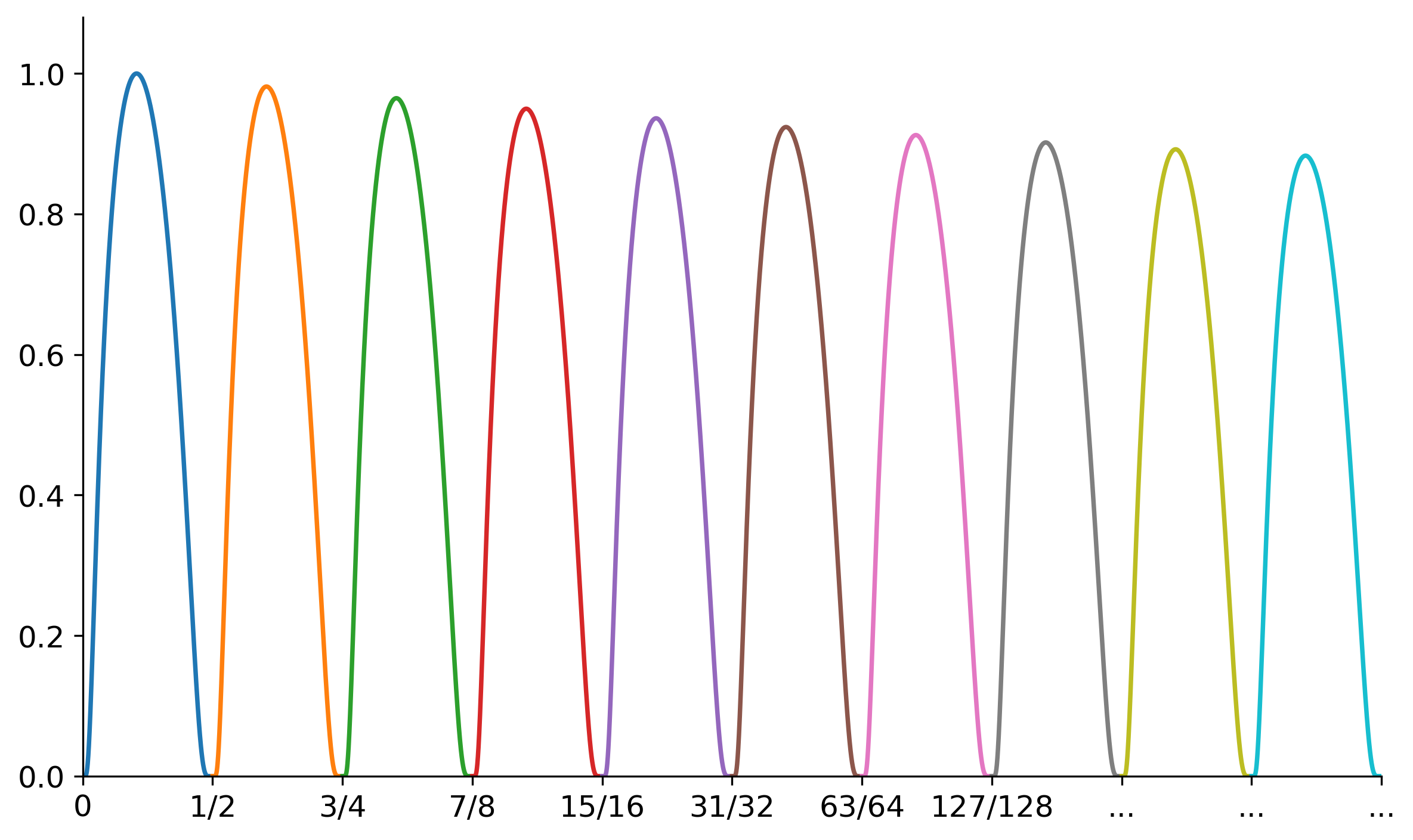}
        \caption{Disjoint bump function class $\cH$ (with rescaled $x$-axis).}
        \label{fig:second-image}
    \end{subfigure}

    \caption{A class $\cH$ of disjoint bump functions on $[0,1]$. The $j^\textrm{th}$ bump is supported on the dyadic interval \eqref{eq:psi_j_dyadic_support}, and has height $a_j \asymp 1/(\log \log j)$.}
    \label{fig:class_H}
\end{figure}

The worst-case empirical Rademacher complexity of $\cH$ can be bounded as
\begin{align}
    \cR_N(\cH) &= \sup_{x_1, \dots, x_N} \frac{1}{N} \bE_{s_1, \dots, s_N} \left[ \sup_j  \sum_{i=1}^N s_i a_j \sigma_j (x_i) \right] \leq \sup_{x_1, \dots, x_N} \frac{1}{N} \bE_{s_1, \dots, s_N} \left[ \sup_j \left|  \sum_{i=1}^N s_i a_j \sigma_j (x_i) \right| \right] \\
    &\leq \sup_{x_1, \dots, x_N} \frac{1}{N} \bE_{s_1, \dots, s_N} \left[ \left( \sum_{j=1}^\infty  \left|  \sum_{i=1}^N s_i a_j \sigma_j (x_i) \right|^2 \right)^{1/2} \right] \\
    &\leq \sup_{x_1, \dots, x_N} \frac{1}{N} \left(  \bE_{s_1, \dots, s_N} \left[ \sum_{j=1}^\infty  \left|  \sum_{i=1}^N s_i a_j \sigma_j (x_i) \right|^2 \right] \right)^{1/2} = \sup_{x_1, \dots, x_N} \frac{1}{N} \left(  \sum_{j=1}^\infty  \sum_{i=1}^N |a_j \sigma_j(x_i)|^2  \right)^{1/2} \\
    &= \sup_{x_1, \dots, x_N} \frac{1}{N} \left(       \sum_{i=1}^N  \left| a_{j(i)} \sigma_{j(i)} (x_i) \right|^2 \right)^{1/2} \leq \frac{a_1}{N} \left( \sum_{i=1}^N 1 \right)^{1/2} = a_1 N^{-1/2}.
\end{align}
The second inequality above bounds the $\ell^\infty$ norm by the $\ell^2$ norm; the third is Jensen's inequality applied to the concave function $t \mapsto t^{1/2}$; and the third equality uses the important fact that, because the bumps $\sigma_j$ are disjoint, each data point $x_i$ has at most one bump $\sigma_{j(i)}$ active on it. Intuitively, the fact that at most one $\sigma_j$ can be active on any data point $x_i$ limits the ability of $\cH$ to fit random signs, which makes its Rademacher complexity small. Also note that this bound does not depend at all on how fast the sequence $\{ a_j \}_{j=1}^\infty$ goes to 0, but only on the first and largest element $a_1$.

\paragraph{These bump classes can have arbitrarily large $L^\infty$ metric entropy.} Despite the fact that $\cH$ has small Rademacher complexity, its metric entropy with respect to the uniform norm $\| \cdot \|_\infty$\footnote{Because the functions in $\cH$ are continuous, the uniform norm $\| \cdot \|_\infty$ is equivalent in this case to the $L^\infty(dx)$ norm with respect to the Lebesgue measure $dx$ on $[0,1]$.} can be extremely large. To see this, note that the norm distance between any two elements $a_j \sigma_j$ and $a_k \sigma_k$ of $\cH$ is
\begin{align} \label{eq:a_psi_norm_max}
    \| a_j \sigma_j - a_k \sigma_k \|_{\infty} = \max \{ a_j, a_k \}.
\end{align}
Suppose that we choose $a_j \asymp 1/(\log \log j)$ and fix some $\epsilon > 0$. Solving $a_J \asymp 1/(\log \log J) = \epsilon$ for $J$  yields $J \asymp \exp \exp (1/\epsilon)$. This implies that the first $J \asymp \exp \exp (1/\epsilon)$ elements of the sequence $\{ a_j \}_{j=1}^\infty$ are strictly greater than $\epsilon$, which in turn implies by \eqref{eq:a_psi_norm_max} that the first $J \asymp \exp \exp (1/\epsilon)$ elements of $\cH$ are more than $\epsilon$ apart from each other in norm. Therefore, any $\epsilon$-cover of $\cH$ must have a separate element dedicated to covering each of the first $J \asymp \exp \exp (1/\epsilon)$ elements (no two can be $\epsilon$-covered by the same covering element). This shows that
\begin{align}
    \log \cN(\cH, \epsilon, \| \cdot \|_\infty) \gtrsim \exp (1/\epsilon).
\end{align}
Note that this final lower bound was specific to the choice of $a_j \asymp 1/(\log \log j)$. A different choice of $a_j$ which decays even more slowly than this would yield a lower bound that grows even faster as $\epsilon$ shrinks. Therefore, we see that the ambient Banach space entropy (measured in uniform norm) of this class $\cH$ may be extremely large, even though its Rademacher complexity is small.

Intuitively, this occurs because of the highly localized nature of the ``spiky'' bump functions. In the uniform norm topology, each bump function is a completely distinct object, and if the heights of the bumps (given by the sequence $a_j$) decay slowly, the norm-separation between any two such functions can be quite large. Therefore, a large amount of global geometric information may be required to describe the entire class $\cH$ with respect to this norm. On the other hand, the class $\cH$ is limited in its ability to correlate with random signs, because at most one bump function of controlled height can be active on any data point---as a result, the Rademacher complexity of the function class is small, despite its large geometric complexity in the uniform norm.

\subsection{Proofs of \cref{sec:tightness_upper_bounds}} \label{appendix:proofs_lower_bound_props}
\subsubsection{Proof of \cref{prop:B2_entropy_lower_bounds}} \label{appendix:B2_entropy_lower_bounds}
    \begin{proof}
        Under the theorem assumptions, we have $\| f \|_{L^2(\mu)} \gtrsim \| f \|_{L^2(d \vx)}$ for every $f$. This is easy to see for simple functions $f = \sum_j a_j \mathbbm{1}_{\cA_j}$, which satisfy
        \begin{align}
            \| f \|_{L^2(\mu)}^2 = \sum_j a_j^2 \mu(\cA_j) \gtrsim \sum_j a_j^2 \, d \vx(\cA_j) = \| f \|_{L^2(d \vx)}^2. 
        \end{align}
        The same inequality holds for general $f$ by taking a sequence of simple functions which increase monotonically to $|f|^2$ and applying the monotone convergence theorem. In the case $2 \leq p \leq \infty$, we thus have $\| f \|_{L^p(\mu)} \gtrsim \| f \|_{L^2(\mu)} \gtrsim \| f \|_{L^2(d \vx)}$, which proves the stated bound in this case. In the case $1 \leq p < 2$, we instead have
        \begin{align}
            \| f \|_{L^2  ( \mu)}^2 = \int_\Omega |f|^p |f|^{2-p} \ d \mu \leq C_2^{2p-p} \| f \|_{L^p( \mu)}^p 
        \end{align}
        where $C_2 := \sup_{f \in \cB_2} \| f \|_{\infty}$. Therefore, in this case we have $\| f \|_{L^p(\mu)} \gtrsim \| f \|_{L^2(\mu)}^{2/p} \gtrsim \| f \|_{L^2(d \vx)}^{2/p}$. This shows that that any $\epsilon$-cover in $L^p(\mu)$ yields a $C \epsilon^{p/2}$-cover in $L^2( d \vx)$, which proves the stated $1 \leq p < 2$ bound.
    \end{proof}
    \subsubsection{Proof of \cref{prop:relu_nested_containment_entropy_lb}} \label{appendix:proof_relu_nested_containment_entropy_lb}
    \begin{proof}
        Denote $\cB_L^{\ReLU}$ here as simply $\cB_L$. First observe that $\cB_1^+ := \{ (f)_+: f \in \cB_1 \} \subset \cB_2$.  \textit{Idempotence} of the ReLU---meaning that $((\cdot)_+)_+ = (\cdot)_+$---therefore implies that $\cB_1^+ \subset \cB_2^+ := \{ (f)_+: f \in \cB_2 \}$. Because set containment is preserved by absolutely convex hulls and set closures, this in turn implies that $\cB_2 \subset \cB_3$. Assuming inductively that $\cB_L \subset \cB_{L+1}$, we have $\cB_L^+ \subset \cB_{L+1}^+$ and therefore $\cB_{L+1} \subset \cB_{L+2}$, which proves the general result.
    \end{proof}
    \subsubsection{Proof of \cref{prop:relu_m_nested_containment_entropy_lb}}
\label{appendix:proof_relu_m_nested_containment_entropy_lb}
\begin{proof}
    We proceed in the following steps.
    \paragraph{Existence of a constant $C > 0$ such that $\cB_1 \subset C \cB_2^{\ReLU,m}$.} By the assumption \eqref{eq:bias_set_assumption} and the fact that $\cW := \bS^{d-1}$ is symmetric, we can choose $\beta_0, \beta_1, \dots, \beta_m$ such that 
    \begin{align}
        \sup_{\vw \in \cW, \vx \in \Omega} = - \inf_{\vw \in \cW, \vx \in \Omega} < \beta_0 < \beta_1 < \dots < \beta_m < b_2. 
    \end{align}
    As long as the $\beta_j$ are chosen in this way, we will have $(\vw^\top \vx + \beta_j)_+^m = (\vw^\top \vx + \beta_j)^m$ for all $\vw \in \cW$, $\vx \in \Omega$, and $j = 1, \dots, m$. Given any such choice of $\beta_j$, we will show that there is some $c > 0$ (dependent on $\cB$ and $\beta_0, \beta_1, \dots, \beta_m$) such that any $f(\vx) = \vw^\top \vx + b$ in $\cB_1$ admits a representation of the form
    \begin{align} \label{eq:f_relu_m_rep}
        f(\vx) = \sum_{j=0}^m c_j (\vw^\top \vx + \beta_j)_+^m = \sum_{j=0}^m c_j (\vw^\top \vx + \beta_j)^m
    \end{align}
    with coefficients $|c_1| + \dots + |c_m| \leq c$. The representation \eqref{eq:f_relu_m_rep} is equivalent to
    \begin{align} \label{eq:f_binom_rep}
        f(\vx) = \sum_{j=0}^m c_j \sum_{i=0}^m {m \choose i} (\vw^\top \vx)^i \beta_j^{m-i} = \sum_{i=0}^m \underbrace{{m \choose i} \left( \sum_{j=0}^m c_j \beta_j^{m-i} \right)}_{\alpha_i} (\vw^\top \vx)^i
    \end{align}
    by the binomial theorem. If we can find a choice of $c_1, \dots, c_m$ such that
    \begin{align}
        \alpha_0 = b, \qquad \alpha_1 = 1, \qquad \alpha_2 = \dots = \alpha_m = 0
    \end{align}
    then the representation \eqref{eq:f_binom_rep} will be equivalent to $f(\vx) = \vw^\top \vx + b$ as desired. Equivalently, we want to find the solutions $c_j$ linear system
    \begin{align}
        \sum_{j=0}^m c_j \beta_j^{m-i} = a_i, \qquad i = 1, \dots, m
    \end{align}
    where $a_0 = b/{m \choose 0} = b$, $a_1 = 1/{m \choose 1} = 1/m$, and $a_2 = \dots = a_m = 0$. This linear system can be expressed in terms of the Vandermonde matrix $\mB$ as
    \begin{align} \label{eq:vandermonde_matrix_equation}
        \underbrace{\begin{bmatrix}
            1 & 1 & \dots & 1 \\
            \beta_0 & \beta_1 & \dots & \beta_m \\
            \beta_0^2 & \beta_1^2 & \dots & \beta_m^2 \\
            \vdots & \vdots & & \vdots \\
            \beta_0^m & \beta_1^m & \dots & \beta_m^m
        \end{bmatrix}}_{\mB} \underbrace{\begin{bmatrix}
            c_0 \\ c_1 \\ c_2 \\ \vdots \\ c_m
        \end{bmatrix}}_{\vc} = \underbrace{\begin{bmatrix}
            a_m \\ a_{m-1} \\ a_{m-2} \\ \vdots \\ a_0
        \end{bmatrix}}_{\va}
    \end{align}
    Because the $\beta_j$ are distinct, the matrix $\mB$ is invertible. To see this: assuming by contradiction that the rows $\boldsymbol{\beta^{(0)}}, \dots, \boldsymbol{\beta^{(m)}}$ of $\mB$ are linearly dependent, there are coefficients $\gamma_0, \dots, \gamma_m$---not all zero---satisfying $\sum_{j=1}^m \gamma_j \boldsymbol{\beta^{(j)}} = \boldsymbol{0}$, or equivalently
    \begin{align}
        \sum_{j=0}^m \lambda_j \beta_i^j = 0, \qquad i = 0, \dots, m.
    \end{align}
    In other words, $\beta_0, \beta_1, \dots, \beta_m$ must all be roots of the nontrivial polynomial
    \begin{align}
        p(x) = \sum_{j=0}^m \lambda_j x^j.
    \end{align}
    But $p$ has degree at most $m$, so by the fundamental theorem of algebra, it has no more than $m$ distinct roots. This contradicts the fact that $\beta_0, \beta_1, \dots, \beta_m$ are distinct.

    Invertibility of $\mB$ implies that a solution $\vc$ to \eqref{eq:vandermonde_matrix_equation} exists and satisfies
    \begin{align}
        \| \vc \|_1 = \| \mB^{-1} \va \|_1 \leq \sup_{b \in [b_1, b_2]} \left\| \mB^{-1} \begin{bmatrix}
            0 \\ \vdots \\ 0 \\ 1/m \\ b
        \end{bmatrix} \right\|_1 = \max_{b \in \{ b_1, b_2 \}} \left\| \mB^{-1} \begin{bmatrix}
            0 \\ \vdots \\ 0 \\ 1/m \\ b
        \end{bmatrix} \right\|_1  =: C < \infty.
    \end{align}
    The $\sup = \max$ above holds because the $\ell^1$ norm term is continuous and convex as a function of $b$, so it attains a maximum over the compact interval $b \in [b_1, b_2]$ at one of the endpoints $b_1$ or $b_2$. Because this constant $C$ is independent of the original $\vw$ and $b$, this shows the desired result $\cB_1 \subset C \cB_2^{\ReLU,m}$. Also note that, depending on the values of $b_1$, $b_2$, and $\beta_0, \dots, \beta_m$, this $C$ may be either greater or less than one. For instance, if $b_1$ 
    and $\beta_1, \dots, \beta_m$ are rescaled as $sb_1$ and $s \beta_1, \dots, s \beta_m$ for some $s > 0$, the solutions $c_j^{(s)}$ to the new rescaled system for $b = b_1$ are given by
    \begin{align}
        \sum_{j=0}^m c_j^{(s)} (s \beta_j)^{m} = s b_1, \quad  \sum_{j=0}^m c_j^{(s)} (s \beta_j)^{m-1} = 1/m, \quad \sum_{j=0}^m c_j^{(s)} (s \beta_j)^{m-i} = 0, \  i=2, \dots, m.
    \end{align}
    All three conditions are satisfied by taking $c_j^{(s)} := c_j s^{1-m}$. If $s$ is large, the coefficients $c_j^{(s)}$ (and hence their $\ell^1$ norm) will be small, and vice versa. The same holds for $b = b_2$. Therefore, it is always possible to force $C > 1$ by making the bias set $\cB$ large enough. In particular,  given some value of $C$ corresponding to an initial choice of $\cB = [b_1, b_2]$ and $\beta_0, \dots, \beta_m$, there is always some $s > 0$ for which the rescaled constant $C^{(s)}$---corresponding to the from rescaling $\cB$ as $s \cB$ and $\beta_0, \dots, \beta_m$ as $s \beta_0, \dots, s \beta_m$---satisfies $C^{(s)} < 1$.
    
    \paragraph{Proof of \eqref{eq:relu_m_nested_containment}.} Denote $\cB_L^{\ReLU,m}$ here as simply $\cB_L$. Having shown above that $\cB_1 \subset C \cB_2$ for some $C > 0$, $m$-homogeneity of $\ReLU^m$ yields
    \begin{align}
        \{ (f)_+^m: f \in \cB_1 \} \subset \{ (Cf)_+^m: f \in \cB_2 \} =  C^m \{ (f)_+^m: f \in \cB_2 \}.
    \end{align}
    Taking absolutely convex hulls and closures, we find that
    \begin{align} \label{eq:B2_cont_B3_ReLU_m}
        \cB_2 \subset C^m \cB_3.
    \end{align}
    Again applying $m$-homogeneity of the $\ReLU^m$ to \eqref{eq:B2_cont_B3_ReLU_m}, we see that
    \begin{align}
        \{ (f)_+: f \in \cB_2 \} \subset \{ (C^m f)_+: f \in \cB_3 \} = C^{m^2} \{ (f)_+: f \in \cB_3 \}
    \end{align}
    and therefore
    \begin{align}
        \cB_3 \subset C^{m^2} \cB_4.
    \end{align}
    Repeating this argument shows that
    \begin{align}
        \cB_L \subset C^{m^{L-1}} \cB_{L+1}
    \end{align}
    and therefore
    \begin{align}
        \cB_2 \subset C^m \cB_3 \subset C^{m^2+m} \cB_4 \subset \dots \subset C^{\sum_{\ell=1}^{L-2} m^\ell} \cB_L
    \end{align}
    for each $L \geq 1$, as desired.
\end{proof}

\subsection{Alternative construction of ResNet classes with penalized residual/skip connections} \label{appendix:resnet_classes}
For convenience, denote
\begin{align}
\cB_L^{\sigma} := \{ \sigma_s \circ f: f \in \cB_L, s > 0 \}.
\end{align}
Suppose that instead of defining $\cB_L := \overline{\aconv}(\cB_{L-1}^\sigma)$ as in \eqref{eq:BL}, we instead define
    \begin{align} \label{eq:B_L_resnet_def}
        \cB_L := \overline{\aconv}( \cB_{L-1} \cup \cB_{L-1}^\sigma).
    \end{align}
    This $\cB_L$ describes the set of functions represented by depth-$L$ ResNet architectures, with norm-penalized residual/skip connections between every hidden layer, along with the functions represented by the infinite-width limits of such architectures. It is easy to see that, with this definition, $\cB_1 \subset \cB_2 \subset \dots$ for any choice $\sigma$, $\cW$, $\cB$, and $\Omega$. Here we will illustrate how the proof of \cref{lemma:rad_bound} (see \cref{appendix:proof_rad_bound}) can be modified for this alternative definition of $\cB_L$, yielding upper bounds that are nearly identical to those in \cref{lemma:rad_bound} and \cref{th:entropy_bound}. Applying the same steps as in the proof of \cref{lemma:rad_bound}, we bound the Rademacher complexities of these ResNet classes $\cB_L$ as
   \begin{align} 
        N \widehat{\cR}_N(\cB_L) &= \bE_r \sup_{f \in \cB_L} \sum_{i=1}^N r_i f(\vx_i)  \\
        &\leq \frac{1}{\lambda} \log \bE_r \sup_{K \in \bN, \sum_{k=1}^K |v_k|  \leq 1, f_k \in \cB_{L-1} \cup \cB_{L-1}^\sigma} \exp \lambda \left( \sum_{k=1}^K v_k \sum_{i=1}^N r_i  f_k(\vx_i) \right) \label{eq:rad_bound_init_resnet} \\
        &\leq \frac{1}{\lambda} \log \bE_r \sup_{K \in \bN, \sum_{k=1}^K |v_k| \leq 1, f_k \in \cB_{L-1} \cup \cB_{L-1}^\sigma} \exp \lambda \left(  \sum_{k=1}^K |v_k| \cdot \max_{k=1, \dots, K} \left| \sum_{i=1}^N r_i  f_k(\vx_i) \right| \right) \\
        &\leq \frac{1}{\lambda} \log \bE_r \sup_{K \in \bN, f_k \in \cB_{L-1} \cup \cB_{L-1}^\sigma} \exp \lambda \left( \max_{k=1, \dots, K} \left| \sum_{i=1}^N r_i  f_k(\vx_i) \right| \right) \\
        &= \frac{1}{\lambda} \log \bE_r \sup_{f \in \cB_{L-1} \cup \cB_{L-1}^\sigma} \exp \lambda \left( \left| \sum_{i=1}^N r_i f(\vx_i) \right| \right) \\
        &\leq \frac{1}{\lambda} \log \bE_r \sup_{f \in \cB_{L-1} \cup \cB_{L-1}^\sigma} \left( \exp \lambda \left( \sum_{i=1}^N r_i f(\vx_i)  \right) + \exp \lambda \left( -\sum_{i=1}^N r_i f(\vx_i) \right)\right) \\
        &\leq \frac{1}{\lambda} \log 2 \bE_r \sup_{f \in \cB_{L-1} \cup \cB_{L-1}^\sigma} \exp \lambda \left( \sum_{i=1}^N r_i f(\vx_i) \right). \label{eq:rad_bound_pre_contraction_resnet}
    \end{align}
    For each fixed realization of the Rademacher variables, the supremum in \eqref{eq:rad_bound_pre_contraction_resnet} is the maximum of the supremum over $\cB_{L-1}$ and the supremum over $\cB_{L-1}^\sigma$. However, after taking expectation, this maximum need not be realized by the same branch for all realizations of the Rademacher variables. Therefore, we instead use $\max\{a,b\} \leq a+b$.

    If $\sigma$ is homogeneous, then the contraction step used in the proof of \cref{lemma:rad_bound} gives
    \begin{align}
        &\frac{1}{\lambda} \log 2 \bE_r \sup_{f \in \cB_{L-1} \cup \cB_{L-1}^\sigma} \exp \lambda \left( \sum_{i=1}^N r_i f(\vx_i) \right) \\
        &\leq \frac{1}{\lambda} \log 2 \bE_r \left[ \sup_{f \in \cB_{L-1}} \exp \lambda \left( \sum_{i=1}^N r_i f(\vx_i) \right) + \sup_{f \in \cB_{L-1}^\sigma} \exp \lambda \left( \sum_{i=1}^N r_i f(\vx_i) \right) \right] \\
        &\leq \frac{1}{\lambda} \log 4 \bE_r \sup_{f \in \cB_{L-1}} \exp \max\{1,\rho_{L-1}\} \lambda \left( \sum_{i=1}^N r_i f(\vx_i) \right). \label{eq:resnet_rad_bound_post_contraction_homogeneous_final}
    \end{align}
    If $\sigma$ is non-homogeneous, the approximation and contraction steps used in the proof of \cref{lemma:rad_bound} similarly give
    \begin{align}
        &\frac{1}{\lambda} \log 2 \bE_r \sup_{f \in \cB_{L-1} \cup \cB_{L-1}^\sigma} \exp \lambda \left( \sum_{i=1}^N r_i f(\vx_i) \right) \\
        &\leq \frac{1}{\lambda} \log 2 \bE_r \left[ \sup_{f \in \cB_{L-1}} \exp \lambda \left( \sum_{i=1}^N r_i f(\vx_i) \right) + \sup_{f \in \cB_{L-1},s>0} \exp \lambda \left( \sum_{i=1}^N r_i \sigma_s(f(\vx_i)) \right) \right] \\
        &\leq N \delta_{L-1} + \frac{1}{\lambda} \log M_{L-1} + \frac{1}{\lambda} \log 4 \bE_r \sup_{f \in \cB_{L-1}} \exp \max\{1,\rho_{L-1}\} \lambda \left( \sum_{i=1}^N r_i f(\vx_i) \right). \label{eq:resnet_rad_bound_post_contraction_nonhomogeneous_final}
    \end{align}
    From this point, repeating the exact steps of the original proof of \cref{lemma:rad_bound} shows that the ResNet classes obey the same bounds of \cref{lemma:rad_bound}, except that $L \log 2$ is replaced by $(2L-1) \log 2$,
    \begin{align}
        \Pi_L^{\mathrm{res}} := \prod_{\ell=1}^{L-1} \max\{1, \rho_\ell \}, \qquad \pi_\ell^{\mathrm{res}} := \prod_{j=\ell+1}^{L-1} \max\{1, \rho_j \}
    \end{align}
    replace $\Pi_L$ and $\pi_\ell$, respectively. In particular, in the non-homogeneous case,
    \begin{align}
        \cR_N(\cB_L) \leq 2 C_{\cW, \cB,\Omega} \Pi_L^{\mathrm{res}} \sqrt{\frac{(2L-1) \log 2 + \log(d+1) + \sum_{\ell=1}^{L-1} \log (1+A_\ell/\delta_\ell)}{N}} + \sum_{\ell=1}^{L-1} \delta_{\ell} \pi_\ell^{\mathrm{res}},
    \end{align}
    and, if $\sigma$ is homogeneous of any degree,
    \begin{align}
        \cR_N(\cB_L) \leq 2 C_{\cW, \cB, \Omega} \Pi_L^{\mathrm{res}} \sqrt{\frac{(2L-1) \log 2 + \log(d+1)}{N}}.
    \end{align}
    The entropy bounds in \cref{th:entropy_bound} can now be easily adapted by observing that the functions in the ResNet class $\cB_L$ are $C_{\cW} \Pi_{L-1}^{\mathrm{res}}$-Lipschitz, rather than the original $C_{\cW} \Pi_{L-1}$-Lipschitz. With this modification, the proof of \cref{th:entropy_bound} goes through otherwise unchanged, and the resulting bounds in \cref{th:entropy_bound} also apply to the ResNet classes $\cB_L$, with $\Pi_L$ and $\pi_\ell$ replaced by $\Pi_L^{\mathrm{res}}$ and $\pi_\ell^{\mathrm{res}}$, respectively, and $T_L(\epsilon)$ modified by replacing $L$ with $2L-1$ in the leading depth term.

\subsection{Representation of pyramid functions in $\cB_3$} \label{appendix:pyramid_relu_rep}

Consider the particular case of the ReLU activation function where $\cW = \bS^{d-1}$ and $\cB = [b_1, b_2]$ is an interval satisfying assumption \eqref{eq:bias_set_assumption}. In this setup, the $\ell^1$ pyramid functions of the form
\begin{align}
    f_{\alpha, \beta, \gamma} (\vx) := (\alpha - \beta \| \vx - \gamma \|_1)_+
\end{align}
are contained in $\cB_3$ for adequate choices of $\alpha, \beta, \gamma$. To see this, choose some $\vu \in \bS^{d-1}$ such that $-b_2 \leq \vu^\top \vx \leq b_2$ (this is always possible due to assumption \eqref{eq:bias_set_assumption}). With any such choice of $\vu$, we have
\begin{align}
	\frac{\alpha}{2b_2} (\vu^\top \vx + b_2)_+ + \frac{\alpha}{2 b_2} ((-\vu)^\top \vx + b_2)_+ = \frac{\alpha}{2b_2} (\vu^\top \vx + b_2) + \frac{\alpha}{2 b_2} ((-\vu)^\top \vx + b_2) = \alpha.
\end{align}
Therefore, the inner function $g_{\alpha, \beta, \gamma} (\vx) := \alpha - \beta \| \vx - \gamma \|_1$ can be represented as
\begin{align} \label{eq:inner_pyramid_relu_rep}
    g_{\alpha, \beta, \gamma} (\vx) = \frac{\alpha}{2b_2} (\vu^\top \vx + b_2)_+ + \frac{\alpha}{2 b_2} ((-\vu)^\top \vx + b_2)_+ - \beta \left(\sum_{i=1}^d (\ve_i^\top \vx - c)_+ + \sum_{i=1}^d (-\ve_i^\top \vx + c)_+ \right)
\end{align}
using the identity $|x_i - c| = (x_i - c)_+ + (-x_i + c)_+$. As long as $c \in [b_1, b_2]$ and
\begin{align}
|\alpha/b_2| + 2 \beta d \leq 1,
\end{align}
the function $g_{\alpha, \beta, \gamma}$ is in $\cB_2$, and therefore $f_{\alpha, \beta, \gamma} = (g_{\alpha, \beta, \gamma})_+$ is in $\cB_3$. 

\subsection{Proof of \cref{lemma:uni_f_V2_bound}}
\label{appendix:proof_uni_f_V2_bound}
\paragraph{Distributions and distributional/weak derivatives.} Before proving \cref{lemma:uni_f_V2_bound}, let us briefly recap the relevant terminology. For an open set $U \subset \R^d$, $C_c^\infty(U)$ denotes the space of infinitely smooth, compactly supported \textit{test functions} $\phi$ on $U$. A \textit{distribution} (or \textit{generalized function}) is a continuous linear functional on $C_c^\infty(U)$. Evaluation of a distribution $F$ on a test function $\phi$ is canonically denoted as $\langle F, \phi \rangle$ rather than $F(\phi)$. Any locally integrable (i.e. integrable on compact subsets of $U$) function $f$ defines a distribution via the functional $\phi \mapsto \int_U f \phi$; similarly, any Radon measure $\mu$ on $U$ defines a distribution via the functional $\phi \mapsto \int_U \phi \ d \mu$. For distributions which are induced in this way by integration against either a function $f$ or a Radon measure $\mu$, we drop the conceptual distinction between the distribution itself and the function/measure which induces it; the meaning is clear from context.

The main utility of distribution theory is that it provides a rigorous way to ``differentiate'' objects which may not be differentiable as functions in the classical sense. This is accomplished using the notion of \textit{distributional derivatives}. The distributional derivative $Df$ of a distribution $f$ is defined by its action on smooth test functions $\phi \in C_c^\infty(U)$ as
\begin{align} \label{eq:distributional_derivative_def}
    \langle Df,\phi\rangle := - \langle f,\phi'\rangle .
\end{align}
To see where this definition comes from, consider the special case where both $f$ and its distributional derivative $Df$ (which we will denote $g$) are actual functions. In this case, \eqref{eq:distributional_derivative_def} is equivalent to
\begin{align}
    \int_U g \,\phi = -\int_U f\,\phi'
\end{align}
for all $\phi \in C_c^\infty(U)$. This is exactly the familiar integration by parts formula, where the boundary terms have vanished due to compact support of $\phi$. In this case, wherein both $f$ and its distributional derivative $Df = g$ are themselves functions, $Df = g$ is also called the \textit{weak derivative} of $f$.

\begin{proof}
    With the relevant terminology in place, we proceed to prove \cref{lemma:uni_f_V2_bound}.
    \paragraph{Existence of $\BV$ weak derivative with one-sided limits.} Define the function $g(x) := D^2 f((-1,x))= \int_{-1}^x  d(D^2 f)(t)$. Here $D^2 f$ is the second distributional derivative of $f$, viewed as a finite Radon measure, and $D^2 f((-1,x))$ is the measure of the open interval $(-1,x)$. We will first argue that there is some constant $c \in \R$ such that the function $f' := g+c$ is a weak derivative of $f$. To see this, first note that 
    \begin{align}
        \langle Dg, \phi \rangle &= -\int_{-1}^1 g(x) \phi'(x) \ dx = -\int_{-1}^1 \int_{-1}^x  d(D^2 f)(t) \  \phi'(x) \ dx = - \int_{-1}^1 \int_{t}^1  \phi'(x)  \ dx \ d(D^2 f)(t) \\
        &= - \int_{-1}^1 \big(\phi(1) - \phi(t)\big) \ d (D^2 f)(t) = \int_{-1}^1 \phi(t) \ d(D^2 f)(t) = \langle D^2 f, \phi \rangle
    \end{align}
    for any $\phi \in C_c^\infty(-1,1)$. Here we have used Fubini's theorem to change the order of integration and the fact that $\phi(1)=0$ because of its compact support. This shows that $D(Df - g) = 0$ in the sense of distributions, implying that $Df-g = c$ is a constant distribution and thus $Df = g+c =: f'$ is a weak derivative of $f$.

    Next, we observe that this $f' \in \BV(-1,1)$. This follows from $Df' = D^2 f \in \cM(-1,1)$ by assumption, along with the fact that $f'$ is bounded---since $|f'(x)| \leq |c| + |D^2 f|((-1,x)) \leq |c| + \| D^2 f \|_{\TV}$ for every $x \in (-1,1)$---and thus in $L^1(-1,1)$. Because $f' \in \BV(-1,1)$, it agrees a.e. with a function $h$ whose \textit{pointwise variation}
    \begin{align} \label{eq:pointwise_variation}
        V_{(-1,1)} (h) := \sup \left\{ \sum_{i=1}^{n-1} |h(x_{i+1}) - h(x_i)|: -1 < x_1 < \dots < x_n < 1 \right\}
    \end{align}
    satisfies $V_{(-1,1)} (h) = \| D^2 f \|_{\TV}$ (\cite{ambrosio2000functions}, Theorem 3.27). Any function satisfying this property is called a \textit{good representative} of $f'$. Indeed, by \cite{ambrosio2000functions} (Theorem 3.28 (a)) it is clear that $f'$ is its own left-continuous good representative, with $V_{(-1,1)} (f') = \| D^2 f \|_{\TV}$. Existence of the one-sided limits $f'(-1^{+})$ and $f'(1^{-})$, as well as one-sided limits $f'(x+)$ and $f'(x^{-})$ for every $x \in (-1,1)$, then follows from \cite{leoni2017first} (Corollary 2.23).
    
    \paragraph{Upper bound on $\| f \|_{\cV_2}$.} Let $f' = g+c \in \BV(-1,1)$ be the weak derivative of $f$ on $(-1,1)$ as described above. Before proceeding, we note that the constant $c = f'(-1^{+})$. This follows from \eqref{eq:limit_identities} (see proof of \cref{lemma:I_f_plus_bound} below), which says that $f'(x) = c + (D^2 f)((-1,x)) = c + f'(x^{-}) - f'(-1^{+})$ for any $x \in (-1,1)$; taking $x \downarrow -1$ yields $c = f'(-1^{+})$.

    We will now demonstrate a convenient representation of $f$ which allows us to prove the stated upper bound on $\| f \|_{\cV_2}$. To do so, note that, because $f$ has a weak derivative in $L^1(-1,1)$, $f$ agrees a.e. with an absolutely continuous function $\tilde{f}$ on $(-1,1)$ (\cite{leoni2017first}, Theorem 7.13). Continuity of $f$ implies that $f = \tilde{f}$ globally: if there were some $x \in (-1,1)$ with $f(x) \neq \tilde{f}(x)$, continuity would force $f - \tilde{f} \neq 0$ on some interval around $x$, contradicting $f - \tilde{f} = 0$ a.e.  Therefore $f$ itself is absolutely continuous, so it has an a.e. classical derivative which obeys the fundamental theorem of calculus (\cite{leoni2017first}, Theorem 3.30). Because this classical derivative must agree a.e. with the weak derivative, the integrals of both are identical, and therefore
    \begin{align} 
        f(x) &= f(-1) + \int_{-1}^x g(t) + f'(-1^{+}) \ dt = f(-1) + f'(-1^{+})(x+1) + \int_{-1}^x \int_{-1}^t  d(D^2 f)(u) \ dt \\
        &= f(-1) + f'(-1^{+})(x+1) + \int_{-1}^x \int_u^x \ dt \ d(D^2 f)(u)  \\
        &= f(-1) + f'(-1^{+})(x+1) + \int_{-1}^x x-u \ d(D^2 f)(u) \\
        &= \underbrace{f(-1) + f'(-1^{+})(x+1)}_{=: f_{\mathrm{aff}}(x)} + \underbrace{\int_{-1}^1 (x-u)_+ \ d(D^2 f)(u)}_{=: f_{\mathrm{int}}(x)} \label{eq:f_int_rep_uni}
    \end{align}
    for all $x \in (-1,1)$. Here we have used Fubini's theorem to change the order of integration.

    From the representation \eqref{eq:f_int_rep_uni}, we see that
    \begin{align}
        \| f \|_{\cV_2} \leq \| f_{\mathrm{aff}} \|_{\cV_2} + \| f_{\mathrm{int}} \|_{\cV_2}.
    \end{align}
    Notice that the term $f_{\mathrm{int}}$ is exactly an integral representation of the form \eqref{eq:pointwise_integral_rep_homogeneous_shallow} with respect to the finite Radon measure 
    \begin{align}
    \nu(\cA, \cB) := \delta_1(\cA) D^2 f(- \cS)
    \end{align}
    on $[-1,1]^2$. Therefore, by \cref{lemma:integral_rep}, we have
    \begin{align}
        \| f_{\mathrm{int}} \|_{\cV_2} \leq \| \mu \|_{\TV} \leq \| D^2 f \|_{\TV}.
    \end{align}
    To show the upper bound in the lemma statement, it remains to bound $\| f_\mathrm{aff} \|_{\cV_2}$.  We do so by showing that any affine function $\phi_{u,v}(x) = ux + v$ on $x \in [-1,1]$ has 
    \begin{align}
        \| \phi_{u,v} \|_{\cV_2} = \frac{|u+v| + |-u+v|}{2} = \max \{ |u|, |v| \}.
    \end{align}
    To see this, note that $\phi_{u,v}$ admits the equivalent representation
    \begin{align} \label{eq:affine_optimal_rep}
        \phi(x) &= \frac{1}{2}(u+v)(x+1)_+ + \frac{1}{2}(-u+v)(-x+1)_+ = \frac{1}{2}(u+v)(x+1) + \frac{1}{2}(-u+v)(-x+1) = ux + v
    \end{align}
    for $x \in [-1,1]$. Indeed, any $\cV_2$ integral representation of $\phi_{u,v}$ of the form
    \begin{align}
        \phi_{u,v}(x) = \int_{[-1,1]^2} (wx + b)_+ \ d \mu(w,b) 
    \end{align}
    must satisfy
    \begin{align}
        |u| = |\phi_{u,v}'(x)| &= \left| \int_{[-1,1]^2} w \mathbbm{1}_{wx+b \geq 0} \ d \mu(w,b) \right| \leq \int_{[-1,1]^2} d |\mu|(w,b) = \| \mu \|_{\TV} 
    \end{align}
    using the distributional derivative formula for ReLU atoms (which must coincide a.e. with the classical derivative $\phi_{u,v}'$). Additionally: 
    \begin{align}
        |v| = |\phi_{u,v}(0)| \leq \int_{[-1,1]^2} |b| \ d |\mu|(w,b) \leq \| \mu \|_{\TV}
    \end{align}
    which shows that the representation \eqref{eq:affine_optimal_rep} is optimal. Applying this with $u = f'(-1^+)$ and $v = f(-1)+f'(-1^+)$ yields the stated upper bound.
    \paragraph{Lower bound on $\| f \|_{\cV_2}$.}  By \cref{lemma:integral_rep}, any $f \in \cV_2$ admits a pointwise integral representation of the form
\begin{align} \label{eq:int_rep_lower_bound}
    f(x) = \int_{[-1,1]^2} (wx+b)_+ \ d \mu(w,b)
\end{align}
for some finite $\mu \in \cM([-1,1]^2)$. It is well-known (and can be shown using a standard approximation argument with test functions) that, for any fixed $w,b$, the first and second distributional derivatives of the integrand $g_{w,b}(x) = (wx+b)_+$ on $(-1,1)$ are
\begin{align}
    D g_{w,b} = w \mathbbm{1}_{wx+b \geq 0},\qquad D^2 g_{w,b} = \begin{cases}
        |w| \delta_{-b/w}, &\textrm{$w \neq 0$ and $-b/w \in (-1,1)$} \\
        0 &\textrm{otherwise}.
    \end{cases} 
\end{align}
Using this fact and the integral representation \eqref{eq:int_rep_lower_bound}, we can compute the first distributional derivative of $f$ as
\begin{align}
    \langle Df, \phi \rangle &= - \langle f, \phi' \rangle = - \int_{-1}^1 \int_{[-1,1]^2} (wx+b)_+ \phi'(x)  \ d \mu(w,b) \ dx \\
    &=  - \int_{[-1,1]^2} \int_{-1}^1 (wx+b)_+ \phi'(x) \ dx \ d \mu(w,b) = \int_{[-1,1]^2} \int_{-1}^1 w \mathbbm{1}_{wx+b \geq 0} \  \phi(x) \ dx \ d \mu(w,b) \\
    &=  \int_{-1}^1 \int_{[-1,1]^2} w \mathbbm{1}_{wx+b \geq 0} \ d \mu(w,b) \  \phi(x) \ dx 
\end{align}
where we have used the definition of the distributional derivative as well as Fubini to change the order of integration. This shows that the function
\begin{align} \label{eq:dist_deriv_f}
    f'(x) := \int_{[-1,1]^2} w \mathbbm{1}_{wx+b \geq 0} \ d \mu(w,b)
\end{align}
is a distributional derivative of $f$. Similarly, we can compute $D^2 f$ as
\begin{align}
    \langle D^2 f, \phi \rangle &= - \langle f' , \phi' \rangle = - \int_{-1}^1 \int_{[-1,1]^2} w \mathbbm{1}_{wx+b \geq 0} \ \phi'(x) \ d \mu(w,b) \ dx \\
    &= - \int_{[-1,1]^2} \int_{-1}^1  w \mathbbm{1}_{wx+b \geq 0} \ \phi'(x) \ dx  \ d \mu(w,b) = \int_{[-1,1]^2} \langle D^2 g_{w,b}, \phi \rangle \ d \mu(w,b) \\
    &= \int_{\cE} |w| \phi(-b/w) \ d \mu(w,b).
\end{align} 
where $\cE := \{ (w,b) \in [-1,1]^2: w \neq 0, -b/w \in (-1,1) \}$. 

Now define the map $T: \cE \to (-1,1)$ by $T(w,b) = -b/w$. Let $T_{\#} (|w| \mu \rvert_{\cE}) $ denote the pushforward of the measure $|w| \mu \rvert_{\cE}$ under $T$; that is, the measure on $(-1,1)$ defined by $T_{\#} (|w| \mu \rvert_{\cE}) (\cA) = |w| \mu(T^{-1}(\cA \cap \cE))$ for Borel $\cA \subset(-1,1)$. By the change-of-variable formula for pushforward measures (\cite{bogachev2007measure}, Theorem 3.6.1 and discussion on p. 191) we have
\begin{align}
     \int_{-1}^1 \phi(t) \ dT_{\#} (|w| \mu \rvert_{\cE})(t) = \int_{-1}^1 \phi(T(w,b)) \ d (|w| \mu \rvert_{\cE})(w,b) = \int_{\cE} |w| \phi(-b/w) \ d \mu(w,b) = \langle D^2 f, \phi \rangle
\end{align}
which shows that $D^2 f$ is exactly the measure $T_{\#} (|w| \mu \rvert_{\cE})$. In general, if $T$ is a continuous map between two topological spaces $\cU$ and $\cV$, the pushforward $T_{\#} \nu$ of any measure $\nu$ on $\cA$ has no greater total variation than $\nu$ itself. This is because the preimages of Borel sets under continuous maps are Borel, and therefore
\begin{align}
     \| T_{\#} \nu \|_{\TV} &:= \sup \left\{ \sum_{i=1}^N |\nu(T^{-1}(\cV_i))|: N \in \bN, \  \cV= \bigcup_{i=1}^N \cV_i, \ \textrm{$\cV_1, \dots, \cV_N$ are Borel} \right\} \\
     &\leq \sup \left\{ \sum_{i=1}^N |\nu(\cU_i)|: N \in \bN, \  \cB_{L-1} = \bigcup_{i=1}^N \cU_i, \ \textrm{$\cU_1, \dots, \cU_N$ are Borel} \right\} \\
     &= \| \nu \|_{\TV}.
 \end{align}
As a result of this general fact, we have $\| D^2 f \|_{\TV} = \| T_{\#} (|w| \mu \rvert_{\cE}) \|_{\TV} \leq \| |w| \mu \rvert_{\cE} \|_{\TV} \leq \| \mu \|_{\TV}$. This holds for any measure $\mu$ which represents $f$ as an integral of the form \eqref{eq:int_rep_lower_bound}, so taking the infimum across all such measures shows that $\| D^2 f \|_{\TV} \leq \| f \|_{\cV_2}$.
\end{proof}

\subsection{Proof of \cref{lemma:I_f_plus_bound}} \label{appendix:proof_I_f_plus_bound}
\begin{proof}
    We first give an overview of the general proof strategy. The idea is to carefully examine the local behavior of $f$ on regions where $f$ and $f_+$ differ. Roughly speaking, these are the intervals for which $f$ is negative on the interior of the interval, and transitions from positive to negative at the beginning of the interval and/or negative to positive at the end of the interval. We first argue that there are only countably many such intervals, which is important for our subsequent limiting arguments to hold. Then, by examining cases, we show that zeroing out $f$ on any individual ``negative interval'' does not increase the value of the functional $I$. We also argue that bounds on $I$ are preserved by uniform limits: if $f_n \to f$ uniformly, then $I(f) \leq \liminf_{n \to \infty} f_n$. Finally, we show that by flattening these negative intervals one-at-a-time, we generate a sequence which uniformly tends to $f_+$, yielding the desired bound on $I(f_+)$. 
    
    \paragraph{Negative set of $f$ has countably many connected components with nonempty interior.} Define $\cN := \{ x \in [-1,1]: f(x) \leq 0 \}$. Recall that a \textit{connected component} of a set $\cE \subset \R$ is a subset $\cC \subset \cE$ such that
    \begin{enumerate}
        \item $\cC$ is \textit{connected}. In $\bR$, this is equivalent to saying that $\cC$ is either a singleton or an interval with nonempty interior (closed, open, or half-open).
        \item $\cC$ is \textit{maximal} in the sense that, if $\cC \subset \cD \subset \cE$ and $\cD$ is connected, then $\cC = \cD$.
    \end{enumerate}
    $\cN$ is closed by continuity of $f$. If some connected component of $\cN$ were an open or half-open interval with nonempty interior, its closure would still be connected and in $\cN$; maximality therefore requires that all connected components of $\cN$ are either singletons or closed intervals with nonempty interior. Maximality also forces that all connected components are pairwise disjoint. Because each of these disjoint connected components with nonempty interior contains a distinct rational number, the set of all connected components of $\cN$ with nonempty interior---which we will denote $\{ [a_j, b_j ] \}_{j=1}^\infty$---is countable. Except in the trivial case where $f \leq 0$ everywhere, all intervals $[a,b]$ in this countable collection must satisfy either
    \begin{enumerate}
        \item $a > -1$, $b < 1$, $f(a) = f(b) = 0$, or
        \item $a = -1$, $b < 1$, $f(b) = 0$, or
        \item $a > -1$, $b = 1$, $f(a) = 0$.
    \end{enumerate}
    \paragraph{Compute derivatives of $f$ after ``flattening'' on a negative component.} Let $[a,b]$ be a connected component of $\cN$ with nonempty interior satisfying $a > -1$, $b < 1$, and $f(a) = f(b) = 0$. We can compute
    \begin{align} \label{eq:Df}
        \langle D(f \mathbbm{1}_{(-1,a) \cup (b,1)}), \phi \rangle = - \int_{-1}^a f \phi' - \int_b^1 f \phi' = \int_{-1}^a f' \phi + \int_b^1 f' \phi = \langle f' \mathbbm{1}_{(-1,a) \cup (b,1)}, \phi \rangle
    \end{align}
    for all $\phi \in C_c^\infty(-1,1)$. Here we have used absolute continuity of $f$ to justify integration by parts, as in the proof of \cref{lemma:uni_f_V2_bound}, as well as $f(a) = f(b) = 0$. 

    Next, we compute $D^2 f(\mathbbm{1}_{(-1,a) \cup (b,1)})$ as
    \begin{align} 
        \langle D^2(f \mathbbm{1}_{(-1,a) \cup (b,1)}), \phi \rangle  &= - \langle D(f \mathbbm{1}_{(-1,a) \cup (b,1)}), \phi' \rangle \\
        &= -\int_{-1}^a f' \phi' - \int_b^1 f' \phi' \\
        &= - f'(a^{-}) \phi(a) + \int_{-1}^a \phi \ d(Df') + f'(b+) \phi(b) + \int_b^1 \phi \ d (Df') \\
        &= \Big\langle D^2 f \rvert_{(-1,a) \cup (b,1)} + f'(b+) \delta_b - f'(a^{-}) \delta_a, \phi \Big\rangle
    \end{align}
    where $\mu \rvert_{\cA}(\cS) := \mu(\cA \cap \cS)$ denotes the restriction of a measure $\mu$ to a set $\cA$. Here we have used the $\BV$ integration by parts formula in \cite{rybka2017bv} (Equation 2.6 in the proof of Proposition 2.2) for convenience. This result can also be derived using a standard dominated convergence argument to approximate the indicator $\mathbbm{1}_{(-1,a) \cup (b,1)}$ with a sequence of $C_c^\infty$ test functions. The relevant derivatives in the cases $a = -1, b < 1, f(b) = 0$ and $a > -1, b = 1, f(a) = 0$ can be computed similarly:
    \begin{align} \label{eq:D2_f_a0_b1}
        D(f \mathbbm{1}_{(b,1)}) &= f' \mathbbm{1}_{(b,1)}, \quad D^2(f \mathbbm{1}_{(b,1)}) = D^2 f \rvert_{(b,1)} + f'(b+) \delta_b \\
        D(f \mathbbm{1}_{(-1,a)}) &= f' \mathbbm{1}_{(-1,a)}, \quad D^2(f \mathbbm{1}_{(-1,a)}) = D^2 f \rvert_{(-1,a)} - f'(a^{-}) \delta_a.
    \end{align}
    In all cases, the second derivatives are finite Radon measures due to everywhere existence of one-sided limits of $f'$. 
    
    Using the same $\BV$ integration by parts rule, we also have
    \begin{align} \label{eq:limit_identities}
        D^2 f((x,y)) &= \int_x^y d(D^2 f) = f'(y^{-}) - f'(x+), \quad -1 < x < y < 1 \\
        D^2 f((-1,x)) &= \int_{-1}^x d(D^2 f) = f'(x^{-}) - f'(-1^{+}), \quad -1 < x < 1 \\
        D^2 f((x,1)) &= \int_x^1 d(D^2 f) = f'(1^{-}) - f'(x+), \quad -1 < x < 1.
    \end{align}
    We will use these identities in the next step of the proof.

    \paragraph{Flattening on a single negative component does not increase $I$.} Let $[a,b]$ be a connected negative component with nonempty interior. First consider the case of an \textit{interior} connected component, for which $a > -1, b < 1, f(a) = f(b) = 0$, which is illustrated in \cref{fig:f_zeroed_internal}. In this case, the values of $I(f)$ and $I(f \mathbbm{1}_{(-1,a) \cup (b,1)})$ differ only by the total variations of the respective second derivatives. We can decompose $D^2 f$ as
    \begin{align}
        D^2 f = D^2 f \rvert_{(-1,1) \setminus \{a,b\}} + D^2 f \rvert_{\{a\}} + D^2 f \rvert_{\{b\}}.
    \end{align}
    Fix $\epsilon > 0$ and denote $a_1 = a- \epsilon, a_2 = a + \epsilon$. \eqref{eq:limit_identities} says that $D^2 f((a_1, a_2)) = f'(a_2^{-}) - f'(a_1+)$. Taking $\epsilon \downarrow 0$, continuity of measures from above implies that $D^2 f (\{ a \}) = f'(a+) - f'(a^{-})$, and the same holds for $b$. Therefore:
    \begin{align}
        D^2 f = D^2 f \rvert_{(-1,1) \setminus \{a,b \}} + \big(f'(a+) - f'(a^{-}) \big) \delta_a +  \big(f'(b+) - f'(b^{-}) \big) \delta_b. 
    \end{align}
    If $f'(a^{-}) > 0$, there would be some $\epsilon > 0$ such that $f'(x) \geq 0$ for all $x \in (a-\epsilon, a)$. By absolute continuity of $f$ and $f(a) = 0$, this would imply that $f(a) - f(x) = - f(x) = \int_x^a f'(t) \dt \geq 0$ for any $x \in (a-\epsilon, a)$, contradicting maximality of the connected component $[a,b] \subset \cN$. So it must be that $f'(a^{-}) \leq 0$, and similarly that $f'(b+) \geq 0$. We thus have
    \begin{align}
        \| D^2 f \|_{\TV} &= |D^2 f|((-1,a) \cup (b,1)) + |f'(a+) - f'(a^{-})| + |D^2 f|((a,b)) + |f'(b+) - f'(b^{-})| \\
        &\geq |D^2 f|((-1,a) \cup (b,1)) + |f'(a+) - f'(a^{-})| + |f'(b^{-}) - f'(a+) | + |f'(b+) - f'(b^{-})| \\
        &\geq |D^2 f|((-1,a) \cup (b,1)) + |f'(b+) - f'(a^{-})| = |D^2 f|((-1,a) \cup (b,1)) + |f'(b+)| + |f'(a^{-})| \\
        &= \| D^2 f \mathbbm{1}_{(-1,a) \cup (b,1)} \|_{\TV}.
    \end{align}
    The first inequality above uses $|f'(b^{-}) - f'(a+)| = |D^2 f((a,b))| \leq |D^2 f|((a,b))$ from \eqref{eq:limit_identities} and the penultimate equality uses $f'(a^{-}) \leq 0 \leq f'(b+)$. We conclude that $I(f \mathbbm{1}_{(-1,a) \cup (b,1)}) \leq I(f)$ in this case.

   Next, consider the case of a left-endpoint connected component with $a = -1, b < 1, f(b) = 0$ (see \cref{fig:f_zeroed_left}). In this case, we have
    \begin{align}
        I(f) &= \max\{ |f(-1)+f'(-1^{+})|, |f'(-1^{+})| \} + |D^2 f|((-1,1)) + I_R(f) \\
        &\geq |f'(-1^{+})| + |D^2 f|((-1,b)) + |f'(b+) - f'(b^{-})| + |D^2 f|((b,1)) + I_R(f) \\
        &\geq |f'(-1^{+})| + |f'(b^{-}) - f'(-1^{+})| + |f'(b+) - f'(b^{-})| + |D^2 f|((b,1)) + I_R(f) \\
        &\geq |f'(b+)| + |D^2 f|((b,1)) + I_R(f) \\
        &= I(f \mathbbm{1}_{(b,1)})
    \end{align}
    where $I_R(f) := \max\{ |f(1)-f'(1^{-})|, |f'(1^{-})| \} $. Here we have again used the identities in \eqref{eq:limit_identities}. Similarly, in the case of a right-endpoint connected component $a > -1, b = 1, f(a) = 0$ (see \cref{fig:f_zeroed_right}), we have
    \begin{align}
        I(f) &=   I_L(f) + |D^2f|((-1,1)) + \max\{ |f(1)-f'(1^{-})|, |f'(1^{-})| \} \\
        &\geq I_L(f) + |D^2 f|((-1,a)) + |f'(a+) - f'(a^{-})| + |D^2 f|((a,1)) + |f'(1^{-})| \\
        &\geq I_L(f) + |D^2 f|((-1,a)) + |f'(a+) - f'(a^{-})| + |f'(1^{-}) - f'(a+)| + |f'(1^{-})| \\
        &\geq I_L(f) + |D^2 f|((-1,a)) + |f'(a^{-})| \\
        &= I(f \mathbbm{1}_{(-1,a)})
    \end{align}
    where $I_L(f) := \max\{ |f(-1)+f'(-1^{+})|, |f'(-1^{+})| \}$. Therefore, in all possible cases, zeroing out $f$ on an individual connected component $[a,b]$ of $\cN := \{ x \in [-1,1]: f(x) \leq 0 \}$ does not increase the value of $I$.

    \begin{figure}
        \centering
    
        \begin{subfigure}[t]{0.24\textwidth}
            \centering
            \includegraphics[width=\linewidth]{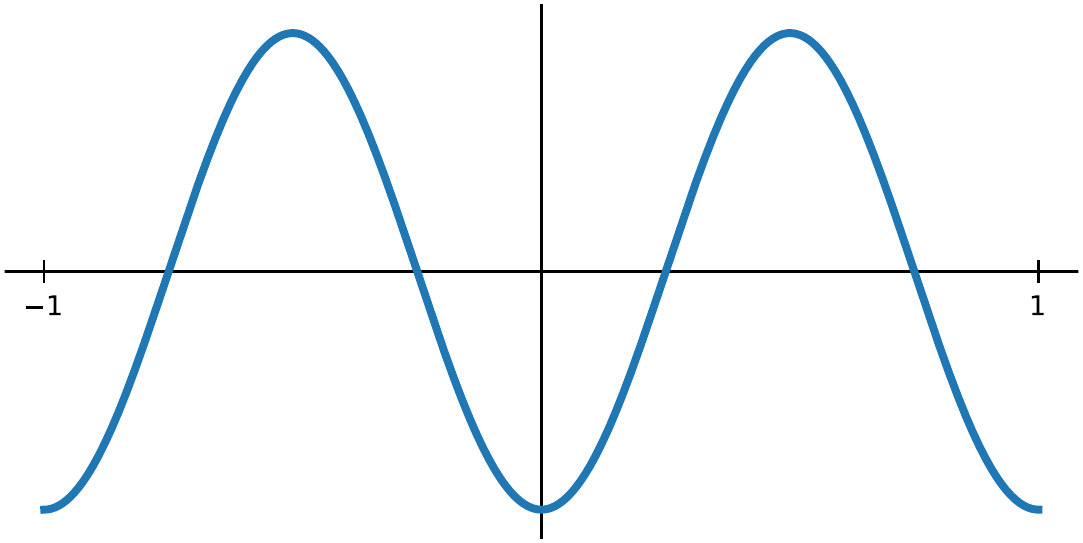}
            \caption{Original function.}
            \label{fig:f_orig}
        \end{subfigure}
        \hfill
        \begin{subfigure}[t]{0.24\textwidth}
            \centering
            \includegraphics[width=\linewidth]{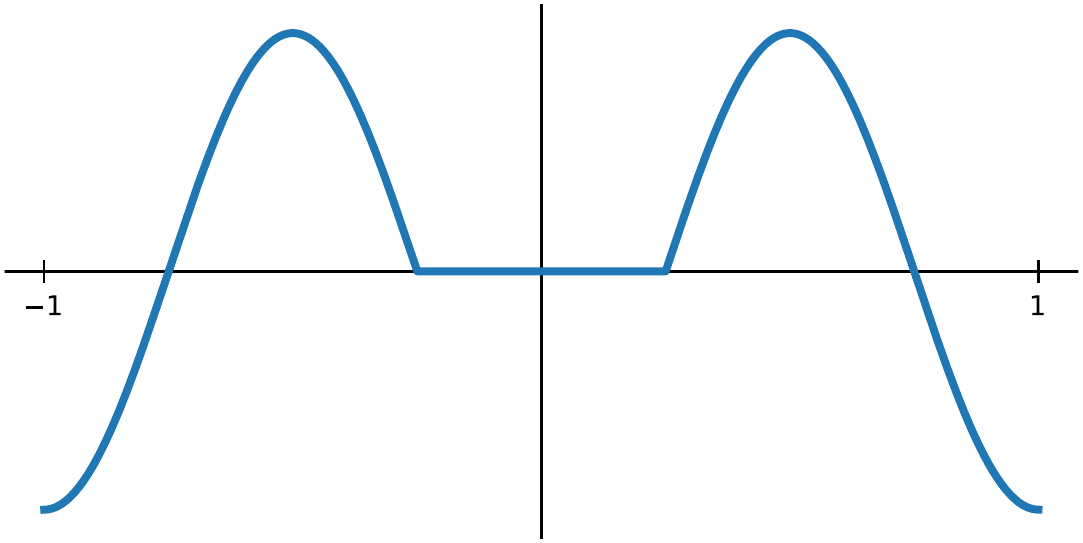}
            \caption{Central negative interval flattened.}
            \label{fig:f_zeroed_internal}
        \end{subfigure}
        \hfill
        \begin{subfigure}[t]{0.24\textwidth}
            \centering
            \includegraphics[width=\linewidth]{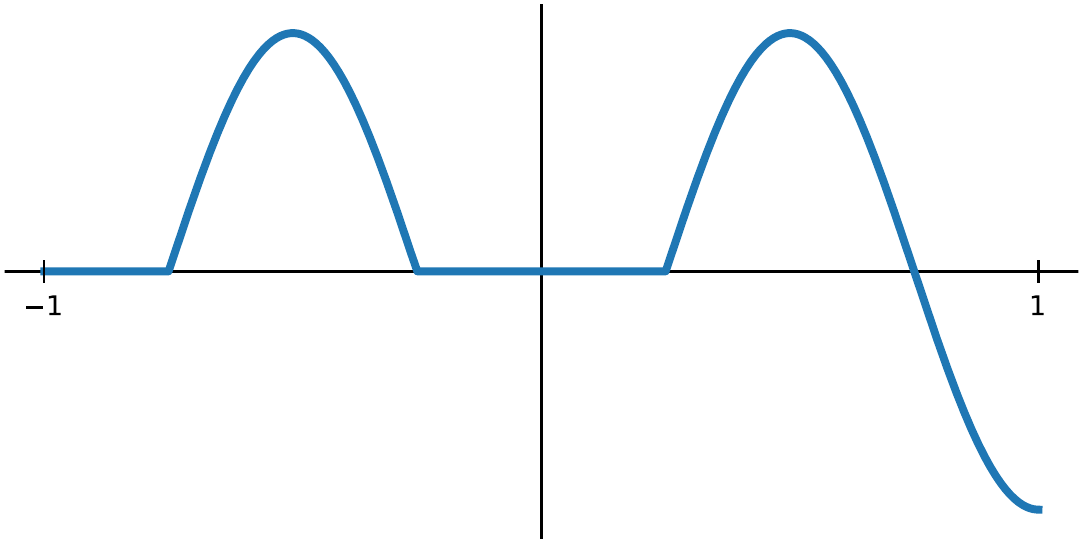}
            \caption{Central and left negative intervals flattened.}
            \label{fig:f_zeroed_left}
        \end{subfigure}
        \hfill
        \begin{subfigure}[t]{0.24\textwidth}
            \centering
            \includegraphics[width=\linewidth]{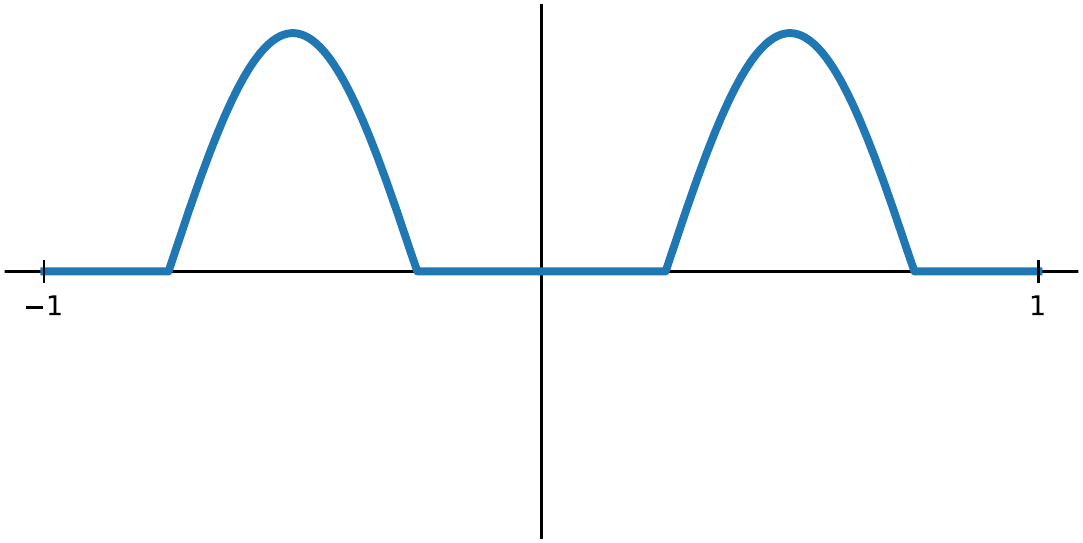}
            \caption{All negative intervals flattened.}
            \label{fig:f_zeroed_right}
        \end{subfigure}
    
        \caption{Modifications of an example function obtained by successively flattening its negative intervals. Flattening the interior negative interval (\cref{fig:f_zeroed_internal}) never increases the total variation of the second derivative. Flattening the leftmost (\cref{fig:f_zeroed_left}) and rightmost (\cref{fig:f_zeroed_right}) increases the variation of the second derivative by no more than the endpoint derivative values $f'(-1)$ and $f'(1)$, respectively.}
        \label{fig:sinusoid-zeroed-comparison}
    \end{figure}
    
    \paragraph{Lower semicontinuity of $I$ under uniform limits.} Suppose $f_n \in C[-1,1]$ satisfies $D^2 f_n \in \cM(-1,1)$ for every $n$, that $f_n \to f$ uniformly on $[-1,1]$, and that $\sup_n I(f_n) < \infty$. Let $f_n'$ be a sequence of right-continuous $\BV(-1,1)$ representatives of the weak derivatives of the functions $f_n$, and extend each $f_n'$ to a function $\tilde{f}_n'$ on $[-1,1]$ by defining $\tilde{f}_n'(-1) = f_n'(-1^{+})$ and $\tilde{f}_n'(1) = f_n'(1^{-})$. We claim that the pointwise variations (see \eqref{eq:pointwise_variation}) satisfy $V_{(-1,1)}(f_n') = V_{[-1,1]}(\tilde{f}_n')$ for each $n$. Clearly $V_{(-1,1)}(f_n') \leq V_{[-1,1]}(\tilde{f}_n')$ because any partition $-1 < x_1 < \dots < x_n < 1$ of $(-1,1)$ yields a partition $-1 = x_0 < x_1 < \dots < x_n < x_{n+1} = 1$ of $[-1,1]$. For the reverse inequality, let $-1 = x_0 < x_1 < \dots < x_n < x_{n+1}=1$ be any partition of $[-1,1]$, and fix some $\epsilon > 0$ and some $a \in (x_0, x_1)$ and some $b \in (x_n, x_{n+1})$ such that $|f_n'(a) - f_n'(-1^{+})| \leq \epsilon$ and $|f_n'(1^{-}) - f_n'(b)| \leq \epsilon$. Then
    \begin{align}
        \sum_{i=0}^n |\tilde{f}_n'(x_{i+1}) - \tilde{f}_n'(x_i)| &\leq 2 \epsilon + |f_n'(x_1) - f_n'(a)| + \sum_{i=1}^{n-1} |f_n'(x_{i+1}) - f_n'(x_i)| + |f_n'(b) - f_n'(x_n)|\\
        &\leq 2 \epsilon + V_{(-1,1)} (f_n').
    \end{align}
    Taking $\epsilon \downarrow 0$ and then taking the sup over all such partitions of $[-1,1]$ yields $V_{[-1,1]} (\tilde{f}_n') \leq V_{(-1,1)} (f_n')$. 
    
    Choose a subsequence $\{ f_{n_k} \}_{k=1}^\infty$ such that $I(f_{n_k}) \to \liminf_{n \to \infty} I(f_n)$. With this fact in hand, the sequence $\tilde{f}_{n_k}' \in \BV[-1,1]$ satisfies
    \begin{align}
        |\tilde{f}_{n_k}'(-1)| + V_{[-1,1]}(\tilde{f}_{n_k}') \leq I(f_{n_k}) \leq \sup_n I(f_n)
    \end{align} 
    for all $k$. By Helly's selection theorem (\cite{leoni2017first}, Theorem 2.35), a further subsequence of $\tilde{f}_{n_k}'$ converges pointwise to some $\tilde{f}' \in \BV[-1,1]$. Not relabeling, we thus have $\tilde{f}_{n_k}'(x) \to \tilde{f}'(x)$ for every $x \in [-1,1]$. Because each $\tilde{f}_{n_k}'$ is a weak derivative of $f_{n_k}$ on $(-1,1)$, we have
    \begin{align}
        \int \tilde{f}_{n_k}' \phi = -\int f_{n_k} \phi'
    \end{align}
    for all $\phi \in C_c^\infty(-1,1)$. Uniform convergence $f_{n_k} \to f$ yields
    \begin{align}
        - \int f_{n_k} \phi' \to - \int f \phi',
    \end{align}
    and $\sup_{x \in [-1,1]} |\tilde{f}_{n_k}'(x)| \leq |\tilde{f}_{n_k}'(-1)| + V_{[-1,1]} (\tilde{f}_{n_k}') \leq \sup_n I(f_n)$, along with dominated convergence, yields
    \begin{align}
        \int \tilde{f}_{n_k}' \phi \to \int \tilde{f}' \phi.
    \end{align}
    This shows that the restriction of $\tilde{f}'$ to $(-1,1)$, which we will denote $f'$, is a weak derivative of $f$ on $(-1,1)$. Since $\tilde{f}' \in \BV[-1,1]$, it follows that $f' \in \BV(-1,1)$ and thus $D^2 f = Df' \in \cM(-1,1)$. 

    At this stage, it is still possible that the endpoint limits $f'(-1^+)$ and $f'(1^-)$ may differ from the endpoint values $\tilde{f}'(-1)$ and $\tilde{f}'(1)$. However, any such difference is necessarily accounted for by the variation $V_{[1,1]} (\tilde{f})$. In particular:
    \begin{align} \label{eq:V_endpoint_internal_var_bound}
        V_{[-1,1]} (\tilde{f}') \geq V_{(-1,1)} (f') + |f'(-1^+) - \tilde{f}'(-1)| + |f'(1^-) - \tilde{f}'(1)|.
    \end{align}
    Moreover, observe that
    \begin{align} \label{eq:left_jump_diff_bound}
        \big| I_L(f) - \max\{ | f(-1) - \tilde{f}'(-1) |, |\tilde{f}'(-1)| \} \big| \leq |f'(-1^+) - \tilde{f}'(-1)|,
    \end{align}
    where $I_L(f) := \max\{|f(-1)-f'(-1^{+})|, |f'(-1^{+})| \}$. This can be seen to hold by the reverse triangle inequality if both maxima are realized by the first or second items. If one maximum is realized by the first term and the other by the second term, it must be the case that $f'(-1^+)$ and $\tilde{f}'(-1)$ have opposite signs, in which case both $|\tilde{f}'(-1)$ and $|f'(-1^+)|$ are no greater than $|\tilde{f}'(-1) - f'(-1^+)|$. By the same token:
    \begin{align} \label{eq:right_jump_diff_bound}
        \big| I_R(f) - \max\{ | f(1) - \tilde{f}'(1) |, |\tilde{f}'(1)| \} \big| \leq |f'(1^-) - \tilde{f}'(1)|
    \end{align}
    where $I_R(f) := \max\{|f(1)-f'(1^{-})|, |f'(1^{-})| \}$. Therefore, we have
    \begin{align}
        I(f) &= I_L(f) + I_R(f) + V_{(-1,1)} (f') \\
        &\leq \max\{ |f(1) - \tilde{f}'(1)|, |\tilde{f}'(1)| \} + \max\{ |f(-1) - \tilde{f}'(-1)|, |\tilde{f}'(-1)| \} + V_{[-1,1]} (\tilde{f}') \label{eq:endpoint_diffs} \\
        &\leq \liminf_{k \to \infty} \left( \max\{ |f_{n_k}(1) - \tilde{f}_{n_k}'(1)|, |\tilde{f}_{n_k}'(1)| \} + \max\{ |f_{n_k}(-1) - \tilde{f}_{n_k}'(-1)|, |\tilde{f}_{n_k}'(-1)| \} + V_{[-1,1]} (\tilde{f}_{n_k}') \right) \label{eq:var_lsc} \\
        &= \liminf_{k \to \infty} I(f_{n_k}) \label{eq:f_n_k_tilde_endpoint}.
    \end{align}
    Here, \eqref{eq:endpoint_diffs} uses \eqref{eq:V_endpoint_internal_var_bound}, \eqref{eq:left_jump_diff_bound}, and \eqref{eq:right_jump_diff_bound}; \eqref{eq:var_lsc} uses lower semicontinuity of pointwise variation (\cite{leoni2017first}, Proposition 2.38) and pointwise convergence $\tilde{f}_{n_k}' \to \tilde{f}'$ and $f_{n_k} \to f$; and \eqref{eq:f_n_k_tilde_endpoint} uses the definition of $\tilde{f}_{n_k}$.
    \paragraph{Uniformly approximate $f_+$ by flattening on finitely many connected components.} Recall that the collection $\{ [a_j, b_j] \}_{j=1}^\infty$ of connected components of $\cN := \{x \in [-1,1]: f(x) \leq 0 \}$ with nonempty interior is countable. Define the functions $f_J: [-1,1] \to \R$ by
    \begin{align}
        f_J(x) = \begin{cases}
            0, &x \in \bigcup_{j=1}^J [a_j, b_j] \\
            f(x), &\textrm{otherwise}
        \end{cases}
    \end{align}
    We have already shown that $I(f_{J+1}) \leq I(f_J)$ for all $J \in \bN$. Next we argue that $f_J \to f_+$ uniformly. If this were not the case, there would be some $\epsilon > 0$ such that
    \begin{align}
        \| f_J - f_+ \|_\infty = \sup_{x \in [-1,1]} |f_J(x) - f_+(x)| = \sup_{x \in \cN \setminus \cup_{j=1}^J [a_j, b_j]} - f(x) > \epsilon
    \end{align}
    for infinitely many $J \in \bN$. For each of these infinitely many $J$, select some $x_J \in \cN \setminus \bigcup_{j=1}^J [a_j, b_j]$ such that $f(x_J) < - \epsilon < 0$. Because the $x_J$ are contained in the compact set $[-1,1]$, they have a subsequence $x_{J_n}$ which converges to some $x^* \in [-1,1]$. By continuity of $f$, the $f(x_{J_n})$ converge to $f(x^*)$, so $f(x^*) \leq -\epsilon < 0$, and thus there is some interval $I^* \ni x^*$ such that $f \leq 0$ on $I^*$. Therefore $I^* \subset [a_j, b_j]$ for some $j$. This $I^*$ must also contain infinitely many of the $x_{J_n}$ since they converge to $x^*$. But this means that $x_J \notin \cN \setminus \bigcup_{j=1}^J [a_j, b_j]$ for all sufficiently large $J$, which is a contradiction.
    
    Applying the lower semicontinuity statement from the previous paragraph to the sequence $f_J \to f_+$, we conclude that $D^2 f_+ \in \cM(-1,1)$ and
    \begin{align}
        I(f_+) \leq \liminf_{J \to \infty} I(f_J) \leq I(f),
    \end{align}
    as desired.
\end{proof}

\subsection{Proof of \cref{th:univariate_containment}} \label{appendix:proof_univariate_containment}
\begin{proof}
    Recall that
    \begin{align}
        \cB_1 := \{ x \mapsto wx + b: x \in [-1,1], w \in [-1,1], b \in [-1,1] \}
    \end{align}
    is our base class of affine functions with slope magnitude at most one on $[-1,1]$. Because any function $f(x) = wx + b$ in $\cB_1$ is affine, it has $D^2 f = 0$, so the only nonzero terms in $I(f)$ are the endpoint function values $|f(-1)+f'(-1^{+})| = |-w+b+w|=|b|$ and $|f(1)-f'(1^{-})| = |w+b-w| = |b|$ and the endpoint slopes $|f'(-1^{+})| = |f'(1^{-})| = |w| \leq 1$. Therefore:
    \begin{align}
        \sup_{f \in \cB_1}  I(f) = \sup_{w \in [-1,1], b \in [-1,1]} \max\{|w|,|b| \} + \max\{|w|, |b|\} = 1+1=2.
    \end{align}
    We now argue that this same bound is inherited by all the classes $\cB_L, L \geq 1$. \cref{lemma:I_f_plus_bound} shows that $I$ is nonincreasing under the positive part operation. It is also easy to see that this bound on $I$ is preserved under finite absolutely convex combinations: if $g = \sum_{k=1}^K v_k f_k$ with $\sum_{k=1}^K |v_k| \leq 1$ and $I(f_k) \leq 2$, we have $g' = \sum_{k=1}^K v_k f_k'$ and $D^2 g = \sum_{k=1}^K v_k D^2 f_k$, so $I(g) \leq \sum_{k=1}^K |v_k| I(f_k) \leq 2$. Preservation of this bound under the uniform closure of finite absolutely convex hulls then follows from the ``Lower semicontinuity of $I$ under uniform limits" step in the proof of \cref{lemma:I_f_plus_bound}. The general conclusion then follows from \cref{lemma:uni_f_V2_bound}, since $\| f \|_{\cV_2} \leq I(f) \leq 2$ for all $f \in \cB_L$, $L \geq 1$.
\end{proof}
\end{appendices}

\end{document}